%% file: main.tex
\input{config/main_preamble}

\input{config/metadata.tex}
\input{config/main_macros}

\title{Model Merging: Foundations and Algorithms}

\begin{document}

\frontmatter %

\pagestyle{plain} %

\input{pages/titlepage.tex}

\input{pages/ack.tex}

\input{0_abstract/content.tex}

\tableofcontents %
\listoffigures %
\listoftables %

\input{pages/glossaries.tex}

\input{pages/notations.tex}

\mainmatter %
\pagestyle{thesis} %

\setpartimage{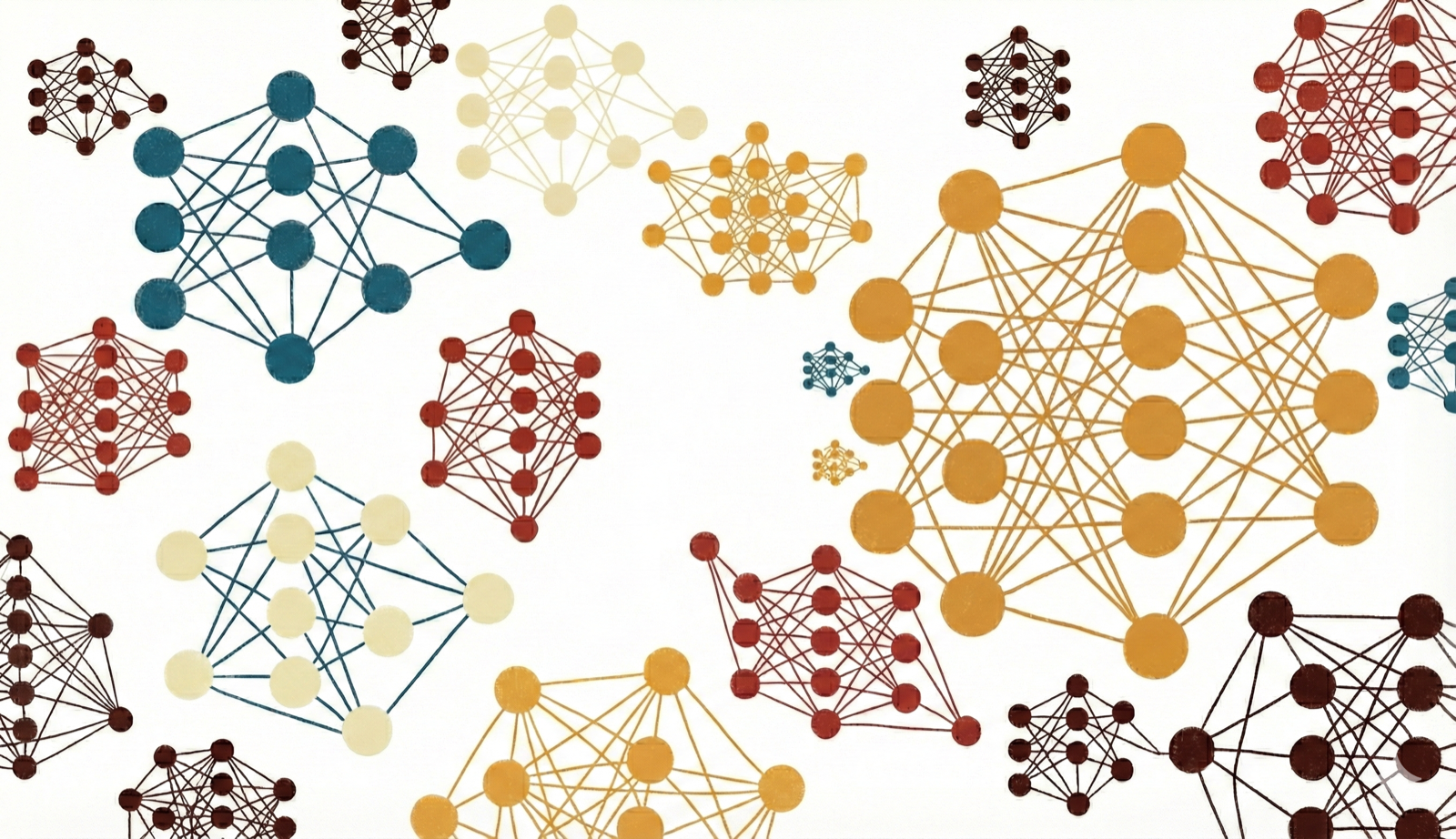}
\setpartquote{The whole is other than the sum of the parts.}{Kurt Koffka}
\part{Introduction}
\chapter{Introduction}
\subimport{1_intro/}{content.tex}

\setpartimage{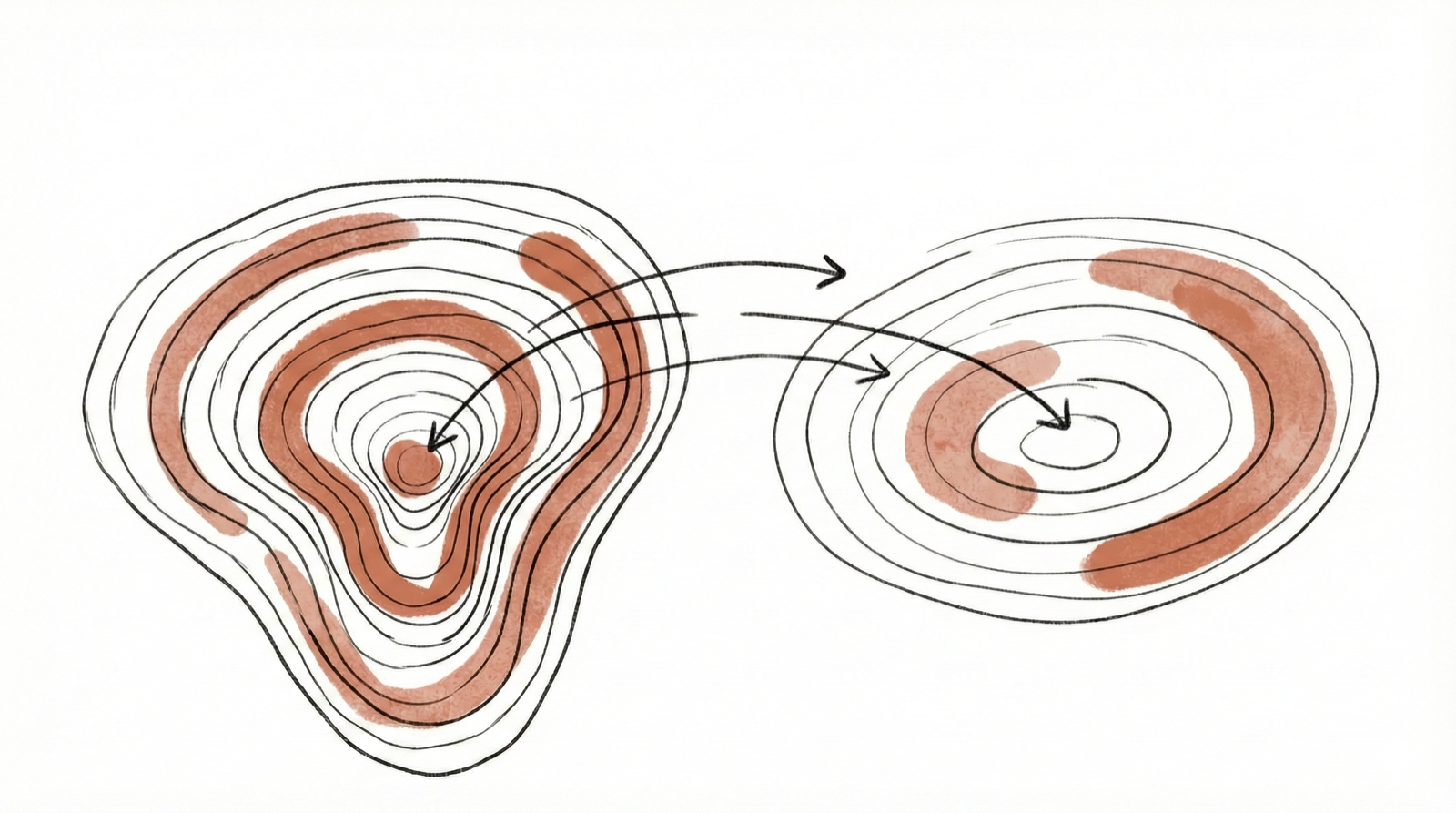}
\setpartquote{Symmetry, as wide or as narrow as you may define its meaning, is one idea by which man through the ages has tried to comprehend and create order.}{Hermann Weyl}
\part{Single-Task Model Merging}\label{part:setting-one}
\subimport{3_single_task/}{0_content.tex}

\setpartimage{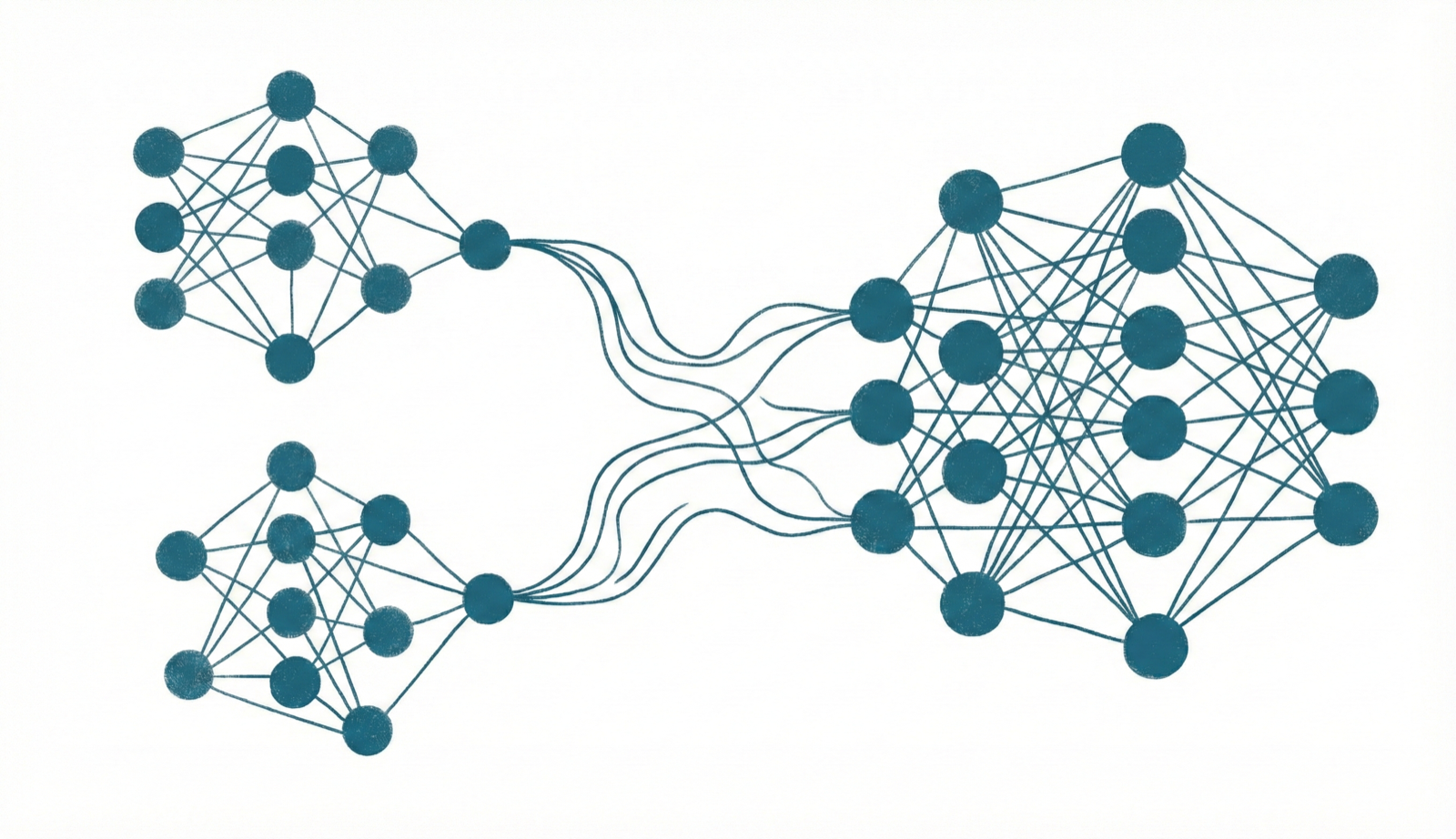}
\setpartquote{Perfection is achieved, not when there is nothing more to add, but when there is nothing left to take away.}{Antoine de Saint-Exupéry}
\part{Multi-Task Model Merging}\label{part:setting-two}
\subimport{4_multi_task/}{0_content.tex}


\setpartimage{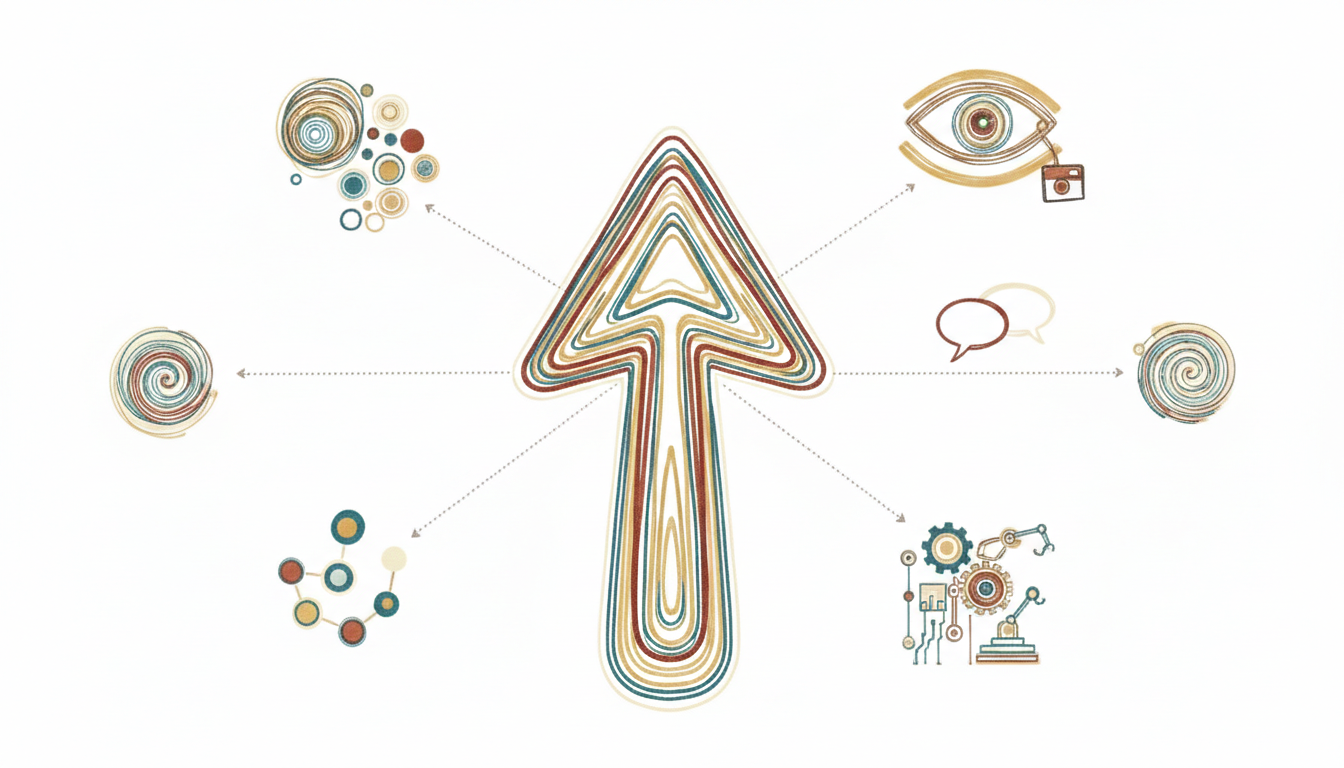}
\setpartquote{We can only see a short distance ahead, but we can see plenty there that needs to be done.}{Alan Turing}
\part{Applications, Outlook, and Conclusions} \label{part:future-directions}
\subimport{6_future_directions/}{content.tex}

\appendix
\setpartimage{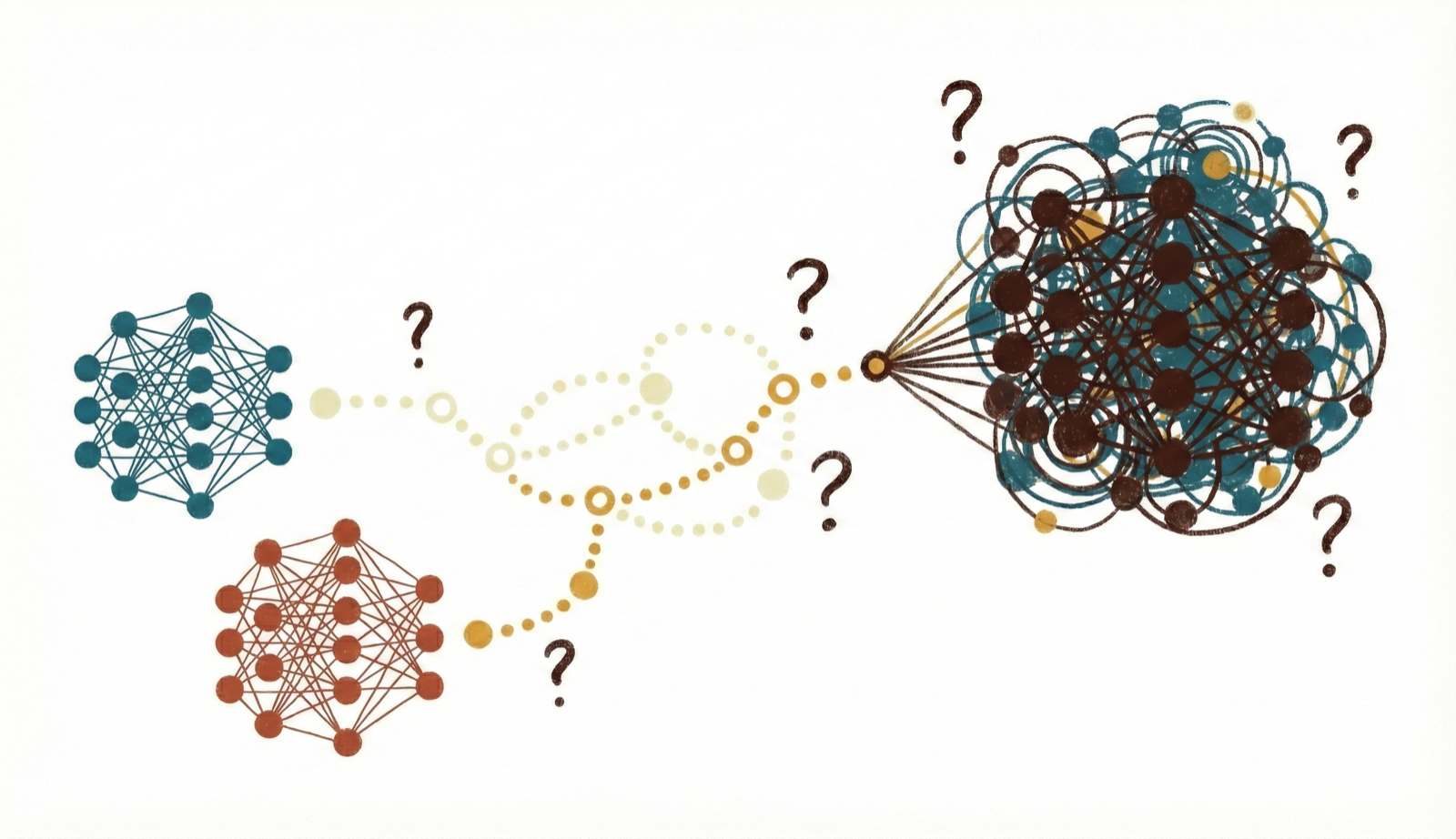}
\setpartquote{The devil is in the details, but so is salvation.}{Hyman G. Rickover}
\part{Appendix}
\subimport{99_appendix/}{content.tex}

\bibliographystyle{plainnat}
\bibliography{references/merging, references/misc, references/ours, references/datasets, references/general}

\end{document}

%% file: config/main_preamble.tex
\documentclass[
11pt, %
english, %
doublespacing, %
liststotoc, %
headsepline, %
]{config/MastersDoctoralThesis} %

\usepackage[utf8]{inputenc}   
\usepackage[T1]{fontenc}      
\usepackage{times}            

\usepackage{amsthm}           
\usepackage{aliascnt}         
\usepackage{cleveref}         
\usepackage{graphicx}         
\usepackage{subcaption}       
\usepackage{wrapfig}          
\usepackage{float}            
\usepackage{booktabs}         
\usepackage{multirow}         
\usepackage{multicol}         
\usepackage{tabularx}         
\usepackage{colortbl}         
\usepackage{arydshln}         
\usepackage{import}
\usepackage{latexsym}
\usepackage{comment}

\usepackage{microtype}        
\usepackage{nicefrac}         
\usepackage{url}              
\usepackage[dvipsnames]{xcolor}       
\usepackage[inline]{enumitem}         
\usepackage[numbers,sort&compress]{natbib}

\usepackage{algorithm}        
\usepackage{algorithmic}      

\usepackage[textsize=tiny]{todonotes} 
\usepackage{tocloft}          
\usepackage{etoc}             
\usepackage[toc,page,header]{appendix} 
\usepackage{etoolbox}         
\usepackage{scalerel}         
\usepackage{placeins}         
\usepackage{dsfont}           
\usepackage{leftindex}        
\usepackage{overpic}          
\usepackage{cancel}           
\usepackage{linguex}          
\usepackage[most]{tcolorbox}  
\usepackage{mdframed}         
\usepackage{framed}

\usepackage{ragged2e}

\newcommand{\extrachapter}[2][\@empty]{%
  \chapter*{#2}%
  \addcontentsline{toc}{chapter}{#2}%
}

\usepackage{tikz}             
\usetikzlibrary{arrows.meta,bending,positioning,3d} 
\usepackage{tikz-3dplot}      

\theoremstyle{plain}
\newtheorem{theorem}{Theorem}[section]
\newtheorem*{theorem*}{Theorem}

\newaliascnt{proposition}{theorem}
\newtheorem{proposition}[proposition]{Proposition}
\aliascntresetthe{proposition}
\newtheorem*{proposition*}{Proposition}

\newaliascnt{lemma}{theorem}
\newtheorem{lemma}[lemma]{Lemma}
\aliascntresetthe{lemma}
\newtheorem*{lemma*}{Lemma}

\newaliascnt{corollary}{theorem}
\newtheorem{corollary}[corollary]{Corollary}
\aliascntresetthe{corollary}
\newtheorem*{corollary*}{Corollary}

\newaliascnt{conjecture}{theorem}
\newtheorem{conjecture}[conjecture]{Conjecture}
\aliascntresetthe{conjecture}

\theoremstyle{definition}
\newaliascnt{definition}{theorem}
\newtheorem{definition}[definition]{Definition}
\aliascntresetthe{definition}

\newaliascnt{assumption}{theorem}

\aliascntresetthe{assumption}

\theoremstyle{remark}
\newaliascnt{remark}{theorem}
\newtheorem{remark}[remark]{Remark}
\aliascntresetthe{remark}

\newaliascnt{fact}{theorem}

\aliascntresetthe{fact}

\crefname{theorem}{Theorem}{Theorems}
\Crefname{theorem}{Theorem}{Theorems}
\crefname{proposition}{Proposition}{Propositions}
\Crefname{proposition}{Proposition}{Propositions}
\crefname{lemma}{Lemma}{Lemmas}
\Crefname{lemma}{Lemma}{Lemmas}
\crefname{corollary}{Corollary}{Corollaries}
\Crefname{corollary}{Corollary}{Corollaries}
\crefname{conjecture}{Conjecture}{Conjectures}
\Crefname{conjecture}{Conjecture}{Conjectures}
\crefname{definition}{Definition}{Definitions}
\Crefname{definition}{Definition}{Definitions}
\crefname{assumption}{Assumption}{Assumptions}
\Crefname{assumption}{Assumption}{Assumptions}
\crefname{remark}{Remark}{Remarks}
\Crefname{remark}{Remark}{Remarks}
\crefname{fact}{Fact}{Facts}
\Crefname{fact}{Fact}{Facts}

\definecolor{iccvblue}{rgb}{0.21,0.49,0.74}
\definecolor{cvprblue}{rgb}{0.21,0.49,0.74}
\definecolor{mygreen}{RGB}{129,178,154}
\definecolor{myblue}{RGB}{51,92,103}
\definecolor{myyellow}{RGB}{224,159,62}
\definecolor{myred}{RGB}{158,42,43}
\definecolor{mydarkred}{RGB}{84,11,14}
\definecolor{mywhite}{RGB}{255,243,176}
\definecolor{forestGreen}{RGB}{34,139,34}
\definecolor{sapienzared}{RGB}{130,36,51}
\definecolor{sapienzagold}{HTML}{B08D3A}
\definecolor{charchoal}{HTML}{4A4A4A}

\usepackage[symbols,abbreviations,nomain,nogroupskip]{glossaries-extra}

\makeglossaries
\setabbreviationstyle[withoutdesc]{short-long}
\setabbreviationstyle[withdesc]{short-long-desc}
\setabbreviationstyle[nolong]{short-nolong}

\input{glossaries/symbols}
\input{glossaries/abbreviations}

\def\currentpartimage{}
\def\currentpartquotetext{}
\def\currentpartquotesource{}

\newcommand{\setpartimage}[1]{\def\currentpartimage{#1}}
\newcommand{\setpartquote}[2]{%
  \def\currentpartquotetext{#1}%
  \def\currentpartquotesource{#2}%
}

\makeatletter
\def\@part[#1]#2{%
    \ifnum \c@secnumdepth >-2\relax
      \refstepcounter{part}%
      \addcontentsline{toc}{part}{\thepart\hspace{1em}#1}%
    \else
      \addcontentsline{toc}{part}{#1}%
    \fi
    \markboth{}{}%
    {\centering
     \interlinepenalty \@M
     \normalfont
     \ifnum \c@secnumdepth >-2\relax
       \huge\bfseries \partname\nobreakspace\thepart
       \par
       \vskip 20\p@
     \fi
     \Huge \bfseries #2\par}%
    \vspace{2em}%
    \begin{center}%
      \ifx\currentpartimage\empty
        \begin{tikzpicture}
          \fill[gray!12] (0,0) rectangle (10,6);
          \draw[gray!40, thick] (0,0) rectangle (10,6);
          \node[gray!60] at (5,3) {\large\itshape [image placeholder]};
        \end{tikzpicture}%
      \else
        \includegraphics[width=0.75\textwidth,height=0.35\textheight,keepaspectratio]{\currentpartimage}%
      \fi
    \end{center}%
    \vspace{1.5em}%
    \ifx\currentpartquotetext\empty\else
      {\centering\large\itshape ``\currentpartquotetext''\par}%
      \vspace{0.6em}%
      {\centering\normalsize\upshape --- \currentpartquotesource\par}%
    \fi
    \@endpart
}
\makeatother

%% file: glossaries/symbols.tex
\glsxtrnewsymbol[sort=Rd, description={A Euclidean space with $d$ dimensions.}]{Rd}{\ensuremath{\mathbb{R}^d}}

%% file: glossaries/abbreviations.tex
\newabbreviation[category=withoutdesc, sort=llm]{llm}{LLM}{Large Language Model}

\newabbreviation[category=withoutdesc, sort=mlp]{mlp}{MLP}{Multi-Layer Perceptron}

\newabbreviation[category=withoutdesc, sort=cnn]{cnn}{CNN}{Convolutional Neural Network}

\newabbreviation[category=withoutdesc, sort=vit]{vit}{ViT}{Vision Transformer}

\newabbreviation[category=withoutdesc, sort=ae]{ae}{AE}{AutoEncoder}

\newabbreviation[category=withoutdesc, sort=vae]{vae}{VAE}{Variational AutoEncoder}

%% file: config/metadata.tex
\thesistitle{Model Merging: Foundations and Algorithms} 
\supervisor{Prof. Emanuele \textsc{Rodolà}} 
\examiner{} 
\degree{Doctor of Philosophy} 
\author{Donato \textsc{Crisostomi}} 
\addresses{} 

\subject{Computer Science} 
\keywords{} 
\university{\href{https://www.uniroma1.it/}{Sapienza University of Rome}} 
\department{\href{https://phd.uniroma1.it/web/COMPUTER-SCIENCE_nD3507_EN.aspx}{Department of Computer Science}} 
\group{\href{https://gladia.di.uniroma1.it/}{GLADIA research lab}} 
\faculty{\href{https://web.uniroma1.it/i3s/en}{Faculty of Information Engineering, Informatics, and Statistics}} 

\AtBeginDocument{
    \hypersetup{pdftitle=\ttitle} 
    \hypersetup{pdfauthor=\authorname} 
    \hypersetup{pdfkeywords=\keywordnames} 
}

%% file: config/main_macros.tex
\DeclareMathOperator*{\argmax}{arg\,max}
\usepackage{bm}
\newcommand{\ones}{\bm{1}}

\newcommand{\approach}{\mergethree}
\newcommand{\dataset}[1]{\textsc{#1}}
\newcommand{\model}[1]{\textsc{#1}}
\newcommand{\method}[1]{\texttt{#1}}
\newcommand{\baseline}[1]{\texttt{#1}}

\newcommand{\negxmark}{\color{mydarkred}{×}}
\newcommand{\poscheckmark}{\color{mygreen}{\checkmark}}
\newcommand{\improv}[1]{\textcolor{mygreen}{(+#1)}}
\newcommand{\worse}[1]{\textcolor{mydarkred}{(-#1)}}

\newcommand{\vitsmall}{\model{ViT-B-32}}
\newcommand{\vitmedium}{\model{ViT-B-16}}
\newcommand{\vitlarge}{\model{ViT-L-14}}

\newcommand{\ntasks}{T}

\newcommand{\Comment}[1]{\hfill \textcolor{gray}{\# #1}}
\newcommand{\vertical}[1]{\rotatebox[origin=c]{90}{#1}}

\makeatletter
\def\adl@drawiv#1#2#3{%
\hskip.5\tabcolsep
\xleaders#3{#2.5\@tempdimb #1{1}#2.5\@tempdimb}%
#2\z@ plus1fil minus1fil\relax
\hskip.5\tabcolsep}
\newcommand{\cdashlinelr}[1]{%
\noalign{\vskip\aboverulesep
\global\let\@dashdrawstore\adl@draw
\global\let\adl@draw\adl@drawiv}
\cdashline{#1}
\noalign{\global\let\adl@draw\@dashdrawstore
\vskip\belowrulesep}}
\makeatother

\newcommand{\mass}{\method{MASS}}

\newcommand{\methodimagetag}[1]{%
  \tikz[baseline=(char.base)]{%
    \node[
      anchor=base, 
      fill=white,
      rounded corners=1pt,
      text=white,
      font=\bfseries,
      minimum width=1em,
      minimum height=1em,
      align=center,
      inner sep=0pt
    ](char){\includegraphics[height=0.9em]{#1}};%
  }%
}

\newcommand{\MT}{\mathrm{MT}}

\newtcolorbox{chapteroverview}{%
  enhanced,
  breakable,
  colback=sapienzared!6,
  colframe=sapienzared!55!black,
  boxrule=0pt,
  leftrule=4pt,
  arc=2pt,
  outer arc=2pt,
  left=10pt, right=10pt, top=10pt, bottom=10pt,
  before skip=1.5\baselineskip,
  after skip=1.5\baselineskip,
}

\tcbuselibrary{breakable}

\newcommand{\mergethree}{\method{MERGE$^3$}~}

\newcommand{\mpirt}{%
  \ifmmode
    \text{\textsc{mp-IRT}}%
  \else
    \textsc{mp-IRT}%
  \fi
}

\newcommand{\pirt}{%
  \ifmmode
    \text{\textsc{p-IRT}}%
  \else
    \textsc{p-IRT}%
  \fi
}

\newcommand{\gpirt}{%
  \ifmmode
    \text{\textsc{gp-IRT}}%
  \else
    \textsc{gp-IRT}%
  \fi
}

\newcommand{\gmpirt}{%
  \ifmmode
    \text{\textsc{gmp-IRT}}%
  \else
    \textsc{gmp-IRT}%
  \fi
}

\usepackage[normalem]{ulem}

\newtcolorbox{promptbox}[1]{%
  enhanced,
  width=\linewidth,
  colback=blue!3,
  colframe=blue!50!black,
  colbacktitle=blue!10,
  coltitle=black,
  boxrule=0.6pt,
  arc=3pt,
  outer arc=3pt,
  left=3pt,right=3pt,top=3pt,bottom=3pt, 
  fontupper=\ttfamily\scriptsize,        
  title=\scriptsize\bfseries #1,         
  before skip=2pt,
  after skip=2pt
}

\newtcolorbox{answerbox}[1]{%
  enhanced,
  width=\linewidth,
  colback=gray!3,
  colframe=gray!50!black,
  boxrule=0.4pt,
  arc=3pt,
  outer arc=3pt,
  left=3pt,right=3pt,top=3pt,bottom=3pt,
  fontupper=\ttfamily\scriptsize,
  title=\scriptsize\bfseries #1
}

\hypersetup{
    colorlinks,
    linkcolor=charchoal, 
    citecolor=charchoal, 
    urlcolor=sapienzared,
}

\newmdenv[
  backgroundcolor=yellow!20,
  linecolor=yellow!50!black,
  skipabove=1em,
  skipbelow=1em
]{draftnotes}

\tcbset{
  colback=white,
  colframe=black!15,
  boxrule=0.5pt,
  arc=2pt,
  left=6pt,right=6pt,top=4pt,bottom=4pt,
  fontupper=\footnotesize\ttfamily
}

%% file: pages/titlepage.tex
\begin{titlepage}
	\begin{center}

		\vspace*{.06\textheight}
		{\scshape\LARGE \univname\par}\vspace{1.5cm} %
		\textsc{\Large Doctoral Thesis}\\[0.5cm] %

		\HRule \\[0.4cm] %
		{\LARGE \bfseries \ttitle\par}\vspace{0.4cm} %
		\HRule \\[1.5cm] %

		\begin{minipage}[t]{0.4\textwidth}
			\begin{flushleft} \large
				\emph{Author:}\\
				\href{https://crisostomi.com/}{\authorname} %
			\end{flushleft}
		\end{minipage}
		\begin{minipage}[t]{0.4\textwidth}
			\begin{flushright} \large
				\emph{Supervisor:} \\
				\href{https://gladia.di.uniroma1.it/authors/rodola/}{\supname} \\
				\href{https://www.cl.cam.ac.uk/~pl219/}{Prof. Pietro \textsc{Liò}} \\
			\end{flushright}
		\end{minipage}\\[3cm]

		\vfill

		\large \textit{A thesis submitted in fulfillment of the requirements\\ for the degree of \degreename}\\[0.3cm] %
		\textit{in the}\\[0.4cm]
		\groupname\\\deptname\\\facname\\[2cm] %

		\vfill

		\includegraphics[width=0.25\textwidth]{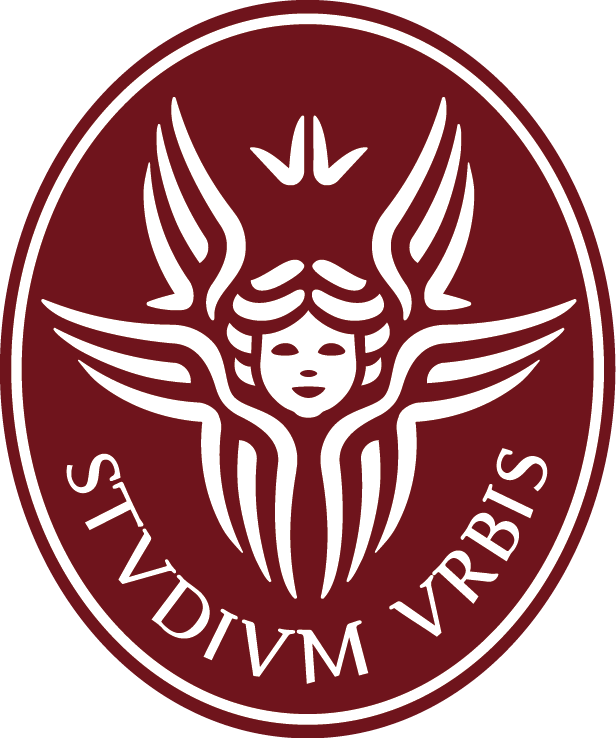}\\[4cm] %

		\vfill
	\end{center}
\end{titlepage}

%% file: pages/ack.tex
\begin{acknowledgements}
    \addchaptertocentry{\acknowledgementname}
    First and foremost, I thank my wife, Jovita. Throughout a journey often marked by challenges and uncertainty, you have been my constant source of strength. Your patience, love, and support have meant more than I can express.

    I am deeply grateful to my advisor, Prof. Emanuele Rodolà, and to the GLADIA Lab as a whole. Thank you for your trust, your vision, and for fostering an environment in which I could grow from a student into a researcher. The intellectual freedom, guidance, and inspiration I found under your mentorship have shaped me in lasting ways.

    I also sincerely thank my ELLIS co-advisor, Prof. Pietro Liò, and everyone at the Computer Lab. Thank you for welcoming me so warmly during my visits and for the many memories I will carry with me long after this chapter closes.

    A research journey is often defined by those who take the time to guide it. In this respect, I am especially grateful to Prof. Iacopo Masi and Prof. Simone Scardapane. Your mentorship, advice, and example were invaluable when I was first finding my way in academia.

    To my fellow GLADIA lab members, Daniele Baieri, Irene Cannistraci, Irene Tallini, Andrea Santilli, Luca Moschella, Marco Pegoraro, Filippo Maggioli, Emilian Postolache, Antonio Ricciardi, Antonio Norelli, Silvio Severino, Michele Mancusi, Michele Miranda, Valentino Maiorca, Emanuele Rossi, Marco Fumero, Luca Cosmo, Riccardo Marin, Simone Melzi, Giorgio Strano, Francesco Pappone, Simone Facchiano, and Giorgio Mariani: thank you. Your camaraderie, conversations, and shared moments made the lab a lively and rewarding place to work.

    Finally, I thank the colleagues with whom I had the privilege of working on model merging research: Antonio Gargiulo, Tommaso Mencattini, Giorgos Nikolaou, Alessandro Zirilli, Luca Zhou, Daniele Solombrino, Davide Marincione, Adrian Robert Minut, and Maria Sofia Bucarelli. I also thank Prof. Florian Bernard and Prof. Fabrizio Silvestri for the insightful exchanges and collaborations we shared across several of these projects. Working alongside all of you, exchanging ideas, and helping push this field forward has been one of the most rewarding parts of my PhD.
\end{acknowledgements}

%% file: 0_abstract/content.tex
\begin{abstract}
	\addchaptertocentry{\abstractname} %
	Modern deep learning typically treats models as separate artifacts: trained independently, specialized for particular purposes, and replaced when improved versions appear. This thesis studies an alternative paradigm, \emph{model merging}: combining independently trained neural networks into a single model directly in weight space, without access to additional training data and with little or no optimization.

	The thesis is organized around two regimes. In the \emph{single-task setting}, where models share a common objective but differ in initialization, we introduce \(C^2M^3\), a cycle-consistent merging algorithm grounded in Frank-Wolfe optimization. \(C^2M^3\) aligns collections of networks into a shared parameter space that serves as a reference-free aggregation point, making weight averaging meaningful without privileging any one model as the anchor.

	In the \emph{multi-task setting}, where models are fine-tuned for distinct downstream tasks, we first develop a theoretical account of \emph{task vectors}, the parameter differences between a fine-tuned model and its pretrained initialization. We show that task vectors admit a gradient-based interpretation under standard assumptions, clarifying both the success and the limits of task arithmetic. This gradient view has a direct consequence: gradients are known to exhibit low-rank structure, and task vectors inherit this property. We formalize and exploit this low-rank structure through \emph{Task Singular Vectors} (TSV), a decomposition that supports both model compression and interference reduction in TSV-Merge. We then present \emph{MASS}, an input-adaptive routing mechanism that uses TSV geometry to direct inference through task-relevant subspaces. Finally, we introduce MERGE\(^3\), an evolutionary merging framework that incorporates Item Response Theory to reduce evaluation costs by up to \(50\times\) while preserving solution quality.

	Taken together, these contributions place model merging on firmer theoretical and algorithmic foundations, advancing a paradigm in which learned capabilities can be composed, reused, and extended across models.
\end{abstract}

%% file: pages/glossaries.tex
\let\oldtexttt\texttt %
\renewcommand{\texttt}[1]{#1} %

\printglossary[type=\glsxtrabbrvtype,title={Glossary}]

\let\texttt\oldtexttt

\printglossary[type=symbols,title={List of Symbols}]

%% file: pages/notations.tex
\extrachapter{Notation}

\subsection*{General mathematics}

\bgroup
\def\arraystretch{1.5}
\begin{tabular}{p{1in}p{4in}}
$\displaystyle \mathbb{R}^d$ & Euclidean space of dimension $d$\\
$\displaystyle \|x\|_2$ & Euclidean norm of vector $x$\\
$\displaystyle \|W\|_F$ & Frobenius norm of matrix $W$\\
$\displaystyle \|W\|_2$ & Spectral norm of matrix $W$\\
$\displaystyle \langle A, B \rangle$ & Frobenius inner product between matrices $A$ and $B$\\
$\displaystyle I$ & Identity matrix\\
$\displaystyle \mathbb{E}[\cdot]$ & Expectation\\
$\displaystyle \operatorname{std}(\cdot)$ & Standard deviation\\
$\displaystyle \mathrm{vec}(A)$ & Vectorization of matrix $A$ (stacking columns)\\
$\displaystyle \mathds{1}_{[\cdot]}$ & Indicator function\\
$\displaystyle \mathrm{span}(V)$ & Subspace spanned by the columns of $V$\\
\end{tabular}
\egroup
\vspace{0.5em}

\subsection*{Data and tasks}
\bgroup
\def\arraystretch{1.5}
\begin{tabular}{p{1in}p{4in}}
$\displaystyle \mathcal{X}$ & Input space\\
$\displaystyle \mathcal{Y}$ & Output space\\
$\displaystyle (x_i, y_i)$ & Training example (input, label) pair\\
$\displaystyle \mathcal{D}$ & Dataset\\
$\displaystyle T$ & Number of tasks, or set of tasks\\
$\displaystyle t$ & Task index\\
$\displaystyle n_t$ & Number of training samples for task $t$\\
\end{tabular}
\egroup
\vspace{0.5em}

\subsection*{Neural network}
\bgroup
\def\arraystretch{1.5}
\begin{tabular}{p{1in}p{4in}}
$\displaystyle \mathcal{W}$ & Weight space\\
$\displaystyle f(\cdot\,;\theta)$ & Neural network parameterized by $\theta$\\
$\displaystyle \theta$ & Model parameters (weights)\\
$\displaystyle \theta_{\text{pre}}$ & Pretrained (base) model parameters\\
$\displaystyle \theta_t$ & Parameters fine-tuned on task $t$\\
$\displaystyle L$ & Number of layers\\
$\displaystyle \ell$ & Layer index\\
$\displaystyle d_\ell$ & Width of layer $\ell$\\
$\displaystyle W^{(\ell)}$ & Weight matrix at layer $\ell$\\
$\displaystyle b_\ell$ & Bias vector at layer $\ell$\\
$\displaystyle \phi$ & Activation function\\
$\displaystyle \mathbf{z}_\ell$ & Intermediate embedding (hidden representation) at layer $\ell$\\
\end{tabular}
\egroup
\vspace{0.5em}

\subsection*{Loss and optimization}
\bgroup
\def\arraystretch{1.5}
\begin{tabular}{p{1in}p{4in}}
$\displaystyle \ell(x,y,\theta)$ & Per-sample loss function\\
$\displaystyle \mathcal{L}(\theta)$ & Loss function\\
$\displaystyle \overline{L}_t(\theta)$ & Empirical loss for task $t$: $\frac{1}{n_t}\sum_{i=1}^{n_t}\ell(x_i,y_i,\theta)$\\
$\displaystyle \nabla L_t(\theta)$ & Gradient of loss for task $t$\\
$\displaystyle \nabla^2 \overline{L}_t$ & Hessian of empirical loss for task $t$\\
$\displaystyle \eta$ & Learning rate\\
\end{tabular}
\egroup
\vspace{0.5em}

\subsection*{Task vectors and model merging}
\bgroup
\def\arraystretch{1.5}
\begin{tabular}{p{1in}p{4in}}
$\displaystyle \tau_t$ & Task vector for task $t$: $\tau_t \coloneqq \theta_t - \theta_{\text{pre}}$\\
$\displaystyle \alpha$ & Scaling coefficient for task vectors\\
$\displaystyle \lambda$ & Interpolation parameter between two models\\
$\displaystyle \theta_{\text{MT}}$ & Merged multi-task model parameters\\
$\displaystyle \Delta_i$ & Per-layer task matrix for task $i$: $\theta_i^{(\ell)} - \theta_{\text{pre}}^{(\ell)}$\\
$\displaystyle \hat{\Delta}_i$ & Low-rank approximation of task matrix $\Delta_i$\\
\end{tabular}
\egroup
\vspace{0.5em}

\subsection*{Permutation alignment}
\bgroup
\def\arraystretch{1.5}
\begin{tabular}{p{1in}p{4in}}
$\displaystyle P_\ell$ & Permutation matrix at layer $\ell$\\
$\displaystyle P^{pq}$ & Permutation mapping from model $q$ to model $p$\\
$\displaystyle P^p$ & Object-to-universe permutation for model $p$\\
$\displaystyle \mathbb{P}$ & Set of all permutation matrices\\
$\displaystyle \operatorname{LAP}(\cdot)$ & Linear Assignment Problem solver\\
\end{tabular}
\egroup
\vspace{0.5em}

\subsection*{SVD and low-rank structure}
\bgroup
\def\arraystretch{1.5}
\begin{tabular}{p{1in}p{4in}}
$\displaystyle U_i, \Sigma_i, V_i$ & SVD of task matrix: $\Delta_i = U_i \Sigma_i V_i^\top$\\
$\displaystyle \sigma_j^{i}$ & $j$-th singular value of task $i$\\
$\displaystyle u_j^{i},\; v_j^{i}$ & $j$-th left and right singular vectors of task $i$\\
$\displaystyle k$ & Number of retained singular components (rank)\\
$\displaystyle \operatorname{STI}$ & Singular Task Interference measure\\
\end{tabular}
\egroup
\vspace{0.5em}

\subsection*{Routing and adaptive merging}
\bgroup
\def\arraystretch{1.5}
\begin{tabular}{p{1in}p{4in}}
$\displaystyle g_i(\mathbf{x})$ & Per-task gating function for input $\mathbf{x}$\\
$\displaystyle \Omega$ & Subset of relevant tasks selected by the router\\
$\displaystyle r_i$ & Euclidean residual: $\|\mathbf{z}_\ell - \mathrm{Proj}_{V_i}(\mathbf{z}_\ell)\|_2$\\
$\displaystyle h_i$ & Classification head for task $i$\\
$\displaystyle C_i$ & Number of classes for task $i$\\
$\displaystyle \mathrm{Proj}_{V}(\mathbf{x})$ & Orthogonal projection of $\mathbf{x}$ onto subspace spanned by columns of $V$\\
\end{tabular}
\egroup
\vspace{0.5em}

\subsection*{Item Response Theory}
\bgroup
\def\arraystretch{1.5}
\begin{tabular}{p{1in}p{4in}}
$\displaystyle \gamma_m$ & Latent ability vector for model $m$\\
$\displaystyle a_i$ & Discrimination (ability dimensions needed to answer example $i$)\\
$\displaystyle \beta_i$ & Difficulty parameter for example $i$\\
$\displaystyle Y_{im}$ & Binary correctness of model $m$ on example $i$\\
$\displaystyle \xi$ & Interpolation coefficients for latent abilities\\
$\displaystyle F(\theta; D)$ & Fitness/performance of parameters $\theta$ on dataset $D$\\
\end{tabular}
\egroup

%% file: 1_intro/content.tex
\section{Motivation}

The broad applicability of model merging has driven its rapid adoption across diverse domains, spurred by a variety of compelling motivations. This section explores these fundamental drivers, progressing from high-level conceptual inspirations to immediate, practical use cases.

\subsection{Instant knowledge transfer}
In one of his most popular talks\footnote{Prof.\ Geoffrey Hinton: \emph{Two Paths To Intelligence}. \url{https://www.youtube.com/watch?v=rGgGOccMEiY}}, Prof.\ G.\ Hinton likened human learning to knowledge distillation in neural networks: as a student listens to a teacher, they adapt their internal cognitive circuitry to eventually reproduce the teacher's knowledge. Regardless of the accuracy of this biological analogy, it highlights a shared limitation of both human learning and traditional model distillation: the process is inherently slow, effortful, and computationally expensive. 

In contrast, distributed training enables the instantaneous transfer of information (in the form of a gradient) across different model instances exposed to varying data distributions. This mechanism is precisely what permits modern Large Language Models (LLMs) to effectively train on internet-scale datasets. Ultimately, neural models possess a distinct advantage over human cognition: \emph{they allow instant knowledge transfer}. Because neural networks can share the same underlying architecture, it is trivial to transfer learned information between instances, a feat that is biologically impossible for human brains due to their intrinsic structural differences. Hinton identifies this capability as a critical advantage that could eventually propel silicon-based intelligence past its biological counterpart. From this perspective, gradient aggregation in distributed training is itself a rudimentary form of model merging: each worker's update encodes knowledge learned from its local data shard, and aggregating them yields a model that synthesizes knowledge from all shards simultaneously. Model merging is thus not an exotic post-hoc technique, but \emph{a primitive built into the very fabric of how modern neural networks are trained}, with gradient averaging being just one of its many possible implementations.

\subsection{Towards a model algebra}
``Knowledge transfer'' remains a somewhat ambiguous term, prompting deeper questions: What specific knowledge is actually being transferred, and how should ``knowledge'' be formally defined in this context? While a rigorous definition is deferred to later chapters, we can intuitively explore this concept. If two distinct models possess different task-solving abilities (for instance, one excels at MNIST classification while the other is specialized for CIFAR10), the ideal merged model would seamlessly perform classification across the union of both datasets (MNIST $\cup$ CIFAR10). Intuitively, this process resembles an arithmetic addition or an averaging of the two models. 

A natural follow-up is whether a ``subtraction'' operation between models makes conceptual sense. In specific scenarios, the answer is yes. For example, if a model is fine-tuned on toxic data, subtracting it\footnote{Actually, subtracting its task vector; we will give the necessary formalism in \cref{part:setting-two}.} from the base model actively reduces toxicity in the resulting network~\citep{task-vectors}. More generally, model difference can be leveraged for targeted machine unlearning. As will be demonstrated in \cref{part:setting-two}, these operations can even facilitate complex analogies. Taken together, these capabilities point toward the development of a cohesive ``model algebra,'' where models can be dynamically added and subtracted on the fly to synthesize entirely new functional networks.

\subsection{Weights as a data modality}
By enabling meaningful algebraic operations between model weights, model merging shares profound connections with the emerging field of weight-space learning~\citep{schurholt2024neural}. This paradigm shift treats model parameters not merely as static configurations, but as first-class citizens that act as raw data within novel learning pipelines. As evidence of this synergy, weight-space techniques are already being utilized to facilitate model merging~\citep{navon2023equivariant}. Furthermore, theoretical insights pioneered in the model merging literature, such as the study of neuron permutation symmetries, have since become foundational pillars of weight-space methodologies~\citep{limgraph, kalogeropoulos2024scale, kofinasgraph}.

\subsection{The need for reuse}
Stepping away from theoretical conceptualization, model merging is heavily driven by immediate, practical necessities. The sheer volume of models hosted on platforms like Hugging Face\footnote{\href{https://huggingface.co/models}{https://huggingface.co/models}} has grown exponentially since their inception~(\cref{fig:hf_models_growth}). Even accounting for identical re-uploads and inactive repositories, this proliferation is staggering. In the face of such abundance, we must ask: is it still necessary to train new models from scratch? 

While fine-tuning remains essential to inject current knowledge and recent facts into models, many practical applications do not require constant data updates. Instead, these use cases can be more efficiently addressed by reusing and repurposing existing models through merging, circumventing the need for computationally expensive tuning. This shift towards reuse has also begun to reshape how model artifacts are versioned and shared: frameworks such as \texttt{Git-Theta}~\citep{kandpal2023gittheta} extend distributed version control to neural network parameters, exposing familiar abstractions (commits, diffs, merges) over checkpoints rather than source files. Under this lens, model merging is no longer a post-hoc operation but a first-class primitive in the open-source model lifecycle, on par with branching and integration in software development.

\begin{figure}[htbp]
    \centering
    \includegraphics[width=\textwidth]{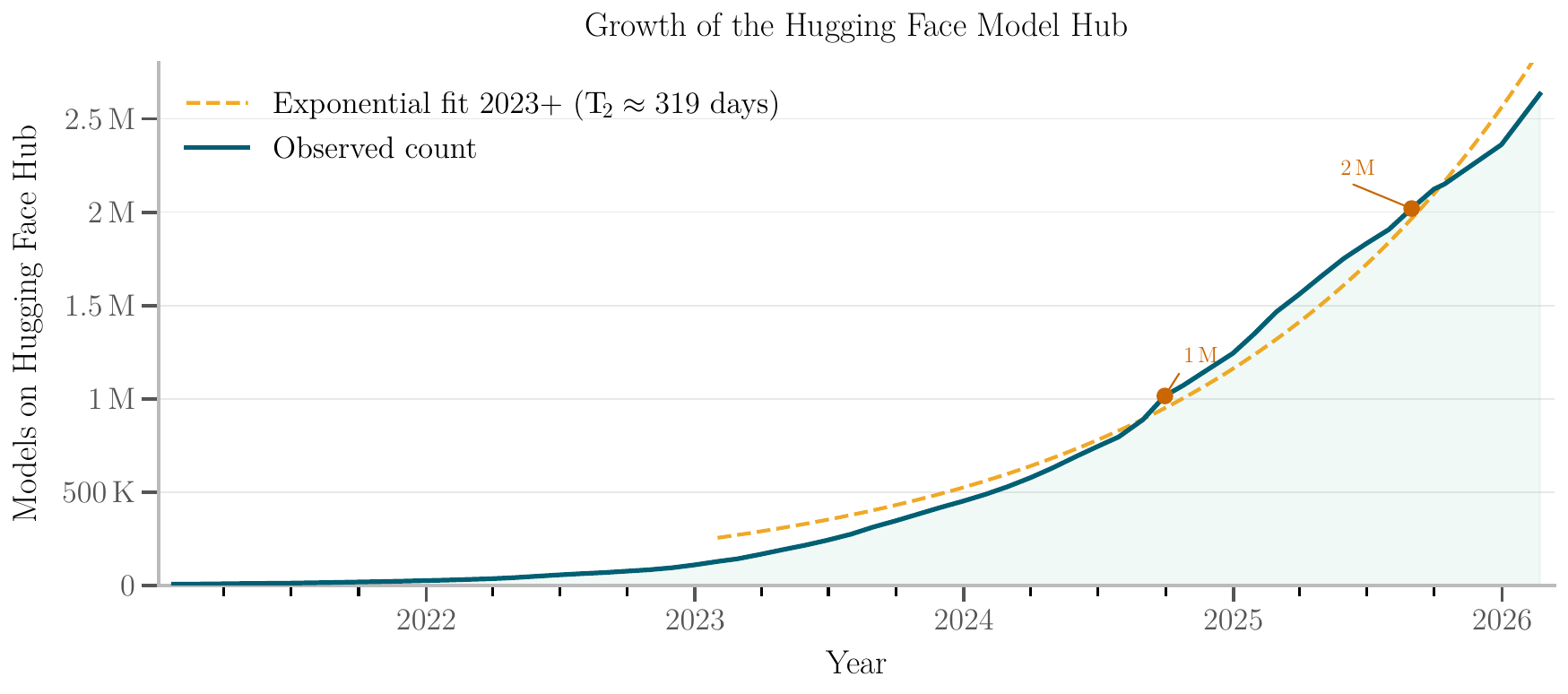}
    \caption{Growth of models hosted on Hugging Face over time.}
    \label{fig:hf_models_growth}
\end{figure}

\subsection{A democratic alternative to fine-tuning}
In the early days of deep learning, standard practice dictated training models entirely from scratch. With the advent of massive pre-trained networks, the community shifted toward fine-tuning these robust backbones for specialized downstream tasks~\citep{devlin2019bert, tan2018survey, yosinski2014transferable, hu2022lora, radford2021learningtransferablevisualmodels}. This paradigm significantly lowered the entry barrier, as fine-tuning a large model is substantially cheaper and more effective than training a comparable architecture from scratch. 

Model merging lowers this accessibility barrier further still. End users can synthesize custom, state-of-the-art models entirely on consumer-grade hardware, sometimes without even requiring a GPU. Furthermore, merging drastically reduces the required technical expertise. Unlike fine-tuning, which demands custom code and delicate hyperparameter tuning, model merging can be as straightforward as selecting simple interpolation coefficients. This simplicity has sparked a proliferation of community-driven models; at the time of writing, roughly 30\% of the models on the Hugging Face Open LLM leaderboard are the direct result of merging. As the resource divide between major AI labs and smaller research teams continues to widen, model merging offers a practical path toward reducing this disparity.

\subsection{Model merging as a pre-training regularizer}
While model merging is frequently framed as a lightweight alternative to traditional training, these techniques also share rich, underexplored connections with general optimization and find notable applications in the pre-training of foundation models. For instance, established optimization strategies like Polyak averaging~\citep{polyak-avg} and Exponential Moving Average (EMA)~\citep{tarvainen2017mean} can be reinterpreted as specific instances of model merging, where consecutive checkpoints are simply aggregated according to a predefined weighting schedule. Viewed through this lens, it is unsurprising that model merging has shown promise in the pre-training phases of large language models~\citep{sanyalearly, limodel, tian2025wsm}. We argue that the deeper theoretical connection between model merging and optimization remains largely untapped, a topic that will be further explored in \cref{part:future-directions}.

\section{Landscape of merging approaches} \label{sec:landscape}

The field of model merging has developed rapidly, producing a diverse set of techniques that address different needs; these can be broadly organized into five categories. This thesis revolves around the first and the third family, presented in \cref{part:setting-one} and \cref{part:setting-two} respectively.

\subsection{Merging of models independently trained on the same task}\label{subsec:intro-single-task}
Mostly originating from the study of linear mode connectivity~\citep{linear-mode-connectivity, convexityandlmc, Entezari2021-me}, this setting studies how to merge different instances of the same architecture trained on the same task starting from different initializations. From a theoretical standpoint, this sheds light on interesting properties of the loss landscape: effective merging through a simple averaging signals that the region of the loss landscape containing the different modes might be more convex than usually believed for deep neural networks. From a practical perspective, this makes it possible to perform efficient ensembling without maintaining separate models: given two models, one can simply interpolate between them to obtain additional well-performing models~\citep{benton_loss_2021,Garipov2018-pz}. \Cref{part:setting-one} shows that the study of neuron permutation symmetries is key for this family of approaches, as they stand as a primary obstacle to linear connectivity: different sets of weights become incompatible even when the models compute the same function.

\subsection{Merging of models independently fine-tuned on the same task starting from the same base model} 
The difficulty of connecting models trained from different seeds, even when sharing the same task, seems to leave no chance of merging models trained on different tasks. However, in the modern pre-train and fine-tune regime, models are far more often fine-tuned than trained from scratch. Sharing the same initialization (the pre-trained base model), fine-tuned models have been observed to be much more linearly connected, partially due to neuron configurations being fixed by the pre-training process~\citep{neyshabur2020being, model-soups}. This opens up new merging possibilities: when fine-tuning a foundation model onto some downstream task, practitioners often run multiple configurations, characterized by different hyperparameters or data splits, and eventually pick the one that maximizes some criterion on a validation split. Several works, originating from model soups \citep{model-soups}, argue that this selection procedure is suboptimal, and that a stronger model can be obtained by merging (often simply averaging) the different resulting models~\citep{model-soups, jang_model_2024,ilharco2022patching}. Unlike the first family, this approach is as simple as it is practical, and is often used in the industry to produce the best-performing models.

\subsection{Merging of models independently fine-tuned on different tasks starting from the same base model}
The ease with which foundation models can be cheaply and quickly specialized to downstream tasks has spurred an abundance of fine-tuned models on public model repositories such as Hugging Face. This abundance has paved the way for a new use case: collecting multiple models, fine-tuned on different downstream tasks, and aggregating them into a single, all-encompassing one~\citep{task-vectors,zhou2025atmimprovingmodelmerging,TSV,Iso-C,daheim2024model,yu2024language,ties}. While the previous category of approaches found the most adoption in industry, this third family is by far the most adopted in the open-source community: practitioners can produce new models starting from publicly available fine-tunings, often without having access to their tuning data. Most interestingly, this is the first category in which compositionality appears: a model that can code and one that can speak Japanese can now be merged into one that can code in Japanese~\citep{sakana,mencattini2025merge}. \Cref{part:setting-two} treats this family of approaches in depth.

\subsection{Merging to mitigate catastrophic forgetting}
A perhaps less systematically studied but commonly employed practice is to use merging to recover general skills of a base model that were lost in favour of specialized ones during fine-tuning. In effect, this amounts to mitigating catastrophic forgetting post hoc, induced by aggressive fine-tuning. For this scenario, model merging offers a practical solution by aggregating the base and the fine-tuned models, recovering general skills such as instruction following while maintaining strong performance on the downstream fine-tuning task(s). Examples include recovering instruction-following capabilities for a large language model fine-tuned on bioacoustic tasks~\citep{marincionemodel} or robotics tasks~\citep{yadav2025robustfinetuningvisionlanguageactionrobot}.

\subsection{Merging during pre-training}
Finally, merging has direct implications for optimization and the pre-training of foundation models \citep{sanyalearly, limodel, tian2025wsm}. Classical techniques such as Polyak averaging and EMA can already be reinterpreted as merging consecutive checkpoints along the training trajectory. Examples include using snapshot merging to replace learning rate decay~\citep{tian2025wsm}, or to improve convergence in the early stages of pre-training~\citep{sanyalearly, limodel}.

\section{Thesis outline and contributions}

This thesis is organized into four parts, progressing from single-task to multi-task merging and concluding with applications, future directions, and a synthesis of the contributions. The five core contributions are:
\begin{itemize}
  \item \textbf{C\textsuperscript{2}M\textsuperscript{3}}~\citep{cycle-consistent} (\cref{ch:cycle-consistent}): cycle-consistent multi-model merging via Frank-Wolfe optimization over a shared universe space, with accuracy gains up to 20\% when merging five models simultaneously.
  \item \textbf{Task vector theory}~\citep{zhoutask} (\cref{ch:gradient}): formal gradient equivalence showing that task vectors are scaled negative loss gradients under single-epoch full-batch GD, with a curvature-controlled $O(\eta^2)$ deviation for multi-epoch fine-tuning.
  \item \textbf{TSV}~\citep{TSV} (\cref{ch:tsv}): low-rank compression reducing task vector storage by a factor of $T$ (number of tasks) while retaining 99\% accuracy, combined with Procrustes-based interference reduction yielding up to $+15$ percentage points on 20-task merging benchmarks.
  \item \textbf{MASS}~\citep{crisostomimass} (\cref{ch:mass}): data-free input-adaptive routing via projection-residual distances, MAP-optimal under an isotropic Gaussian noise model, achieving state of the art on 8 out of 9 benchmarks.
  \item \textbf{MERGE\textsuperscript{3}}~\citep{mencattini2025merge} (\cref{ch:merge3}): efficient evolutionary merging with a $50\times$ cost reduction through IRT-based performance estimation, enabling cross-lingual transfer with $+10$--$20\%$ gains over all endpoint models.
\end{itemize}

We now provide an overview of each part.

\paragraph{\Cref{part:setting-one}: Single-Task Model Merging.}
\Cref{ch:background} introduces mode connectivity and the neuron matching problem, showing how permutation symmetries of neurons create a vast number of functionally equivalent networks that hinder the aggregation of independently trained models. Existing alignment methods, such as \texttt{Git Re-Basin}~\citep{git-rebasin}, tackle this problem in a pairwise fashion, solving for the optimal permutations one layer at a time. However, when merging more than two models, a setting that naturally arises in distributed learning and federated optimization, these pairwise approaches offer no guarantees on the consistency of the permutations across the set. In \cref{ch:cycle-consistent}, which presents the first contribution of this thesis, we address this limitation by introducing a cycle-consistent alignment algorithm that matches any number of models simultaneously through a shared universal space, and we leverage it to perform principled multi-model merging.

\paragraph{\Cref{part:setting-two}: Multi-Task Model Merging.}
Having addressed the problem of merging models trained on the same task but from different initializations, where the central challenge was aligning neurons across networks, we turn to a fundamentally different and arguably more impactful setting: merging models that have been fine-tuned on different tasks from a shared pre-trained backbone. In this multi-task regime, the dominant paradigm is \emph{task arithmetic}: constructing a multi-task model by summing the weight differences between each fine-tuned model and the common base. Despite its widespread adoption and surprising effectiveness, a rigorous understanding of \emph{why} task arithmetic works has remained elusive. \Cref{ch:gradient} lays the theoretical groundwork for the rest of this part by establishing a formal connection between task vectors and the gradients of the task losses, reframing task arithmetic as approximate multi-task learning.

Building on this theoretical foundation, \cref{ch:tsv} departs from the flattened view of networks as monolithic parameter vectors and instead examines task vectors at the \emph{layer} level, preserving their natural matrix form. Through singular value decomposition, we uncover that per-layer task matrices are inherently low-rank, and we exploit this structure both for compression and for a more fine-grained understanding of task interference, ultimately yielding a merging method that significantly outperforms prior approaches.

A natural follow-up question is whether this low-rank structure can be leveraged not just at merge time, but also at \emph{inference} time, to adapt the merged model to each individual input. All methods discussed so far produce a single, static merged model that is applied identically regardless of the input, a limitation that leaves performance on the table. In \cref{ch:mass}, we explore a different paradigm: rather than committing to a fixed merge, we route each input to the most relevant task subspaces, effectively making the merging process \emph{input-dependent}.

Finally, a complementary line of work, \emph{evolutionary merging}, sidesteps the design of hand-crafted merging rules altogether by using evolutionary search to optimize merging coefficients directly. While this approach has produced models of unprecedented quality, its computational demands (thousands of fitness evaluations, each involving full inference over large datasets) place it out of reach for practitioners working on consumer hardware. \Cref{ch:merge3} addresses this accessibility gap, introducing an efficient evolutionary merging framework that reduces the computational cost by 50$\times$ and brings evolutionary merging to single-GPU setups. In doing so, we also extend the scope of this thesis from vision models to Large Language Models, demonstrating cross-lingual knowledge transfer through merging.

\paragraph{\Cref{part:future-directions}: Applications, Outlook, and Conclusions.}
The final part of the thesis broadens the perspective from individual methods to their practical impact, open problems, and the future of the field.
\Cref{ch:applications} examines how the methods developed in \cref{part:setting-one,part:setting-two} transfer to concrete deployment scenarios: the cycle-consistent alignment of \cref{ch:cycle-consistent} naturally extends to federated learning, where it compensates for the drift induced by prolonged local training; the low-rank structure of task vectors yields practical compression of multi-task expert libraries; and the evolutionary framework of \cref{ch:merge3} enables cross-lingual knowledge transfer on consumer hardware.
\Cref{ch:future-directions} identifies the most promising avenues for advancing the field, including more expressive alignment frameworks based on quadratic assignment and functional maps, heterogeneous merging across architectures and tokenizers, the emerging connections between structured merging and spectral optimizers such as \method{Muon}, scaling the layer-wise SVD paradigm to large language models, and the development of a predictive theory of mergeability.
Finally, \cref{ch:conclusions} synthesizes the theoretical and algorithmic insights developed across the thesis, distilling three overarching themes: that weight space is legible and interpretable, that model merging admits formal structure and guarantees, and that exploiting this structure yields methods that are simultaneously more accurate, more efficient, and more accessible.

%% file: 3_single_task/0_content.tex

\chapter{Foundations of Single-Task Merging} \label{ch:background}
\subimport{0_background/}{content.tex}

\chapter{Cycle-Consistent Alignment of Multiple Models}\label{ch:cycle-consistent}
\subimport{1_cycle_consistent/}{content.tex}


%% file: 3_single_task/0_background/content.tex
The simplest incarnation of model merging, introduced in \cref{subsec:intro-single-task}, considers multiple instances of the same architecture trained independently on the same task from different random initializations. The goal is to combine them into a single model, ideally by a simple weight average, that matches or exceeds the performance of the individual networks. This is appealing both theoretically, as it implies that the loss landscape is more benign than traditionally believed, and practically, as it enables efficient ensembling without retraining. A major obstacle, however, is that two networks that have learned the same function may appear far apart in weight space: the permutation symmetries of neurons create a combinatorial number of functionally equivalent parameterizations, and independently trained models will in general converge to different elements of the same equivalence class. Naively averaging such misaligned models destroys the learned representations.

This part of the thesis addresses exactly this challenge. The present chapter reviews the necessary foundations: we first trace how the community's understanding of the loss landscape evolved from isolated minima to connected low-loss manifolds; then formalize the permutation symmetries that obscure this connectivity; and finally present the alignment algorithms that resolve these symmetries, enabling the merging methods developed in \cref{ch:cycle-consistent}.

\begin{figure}[htbp]
    \centering
    \includegraphics[width=\textwidth,trim={0 12cm 0 5cm},clip]{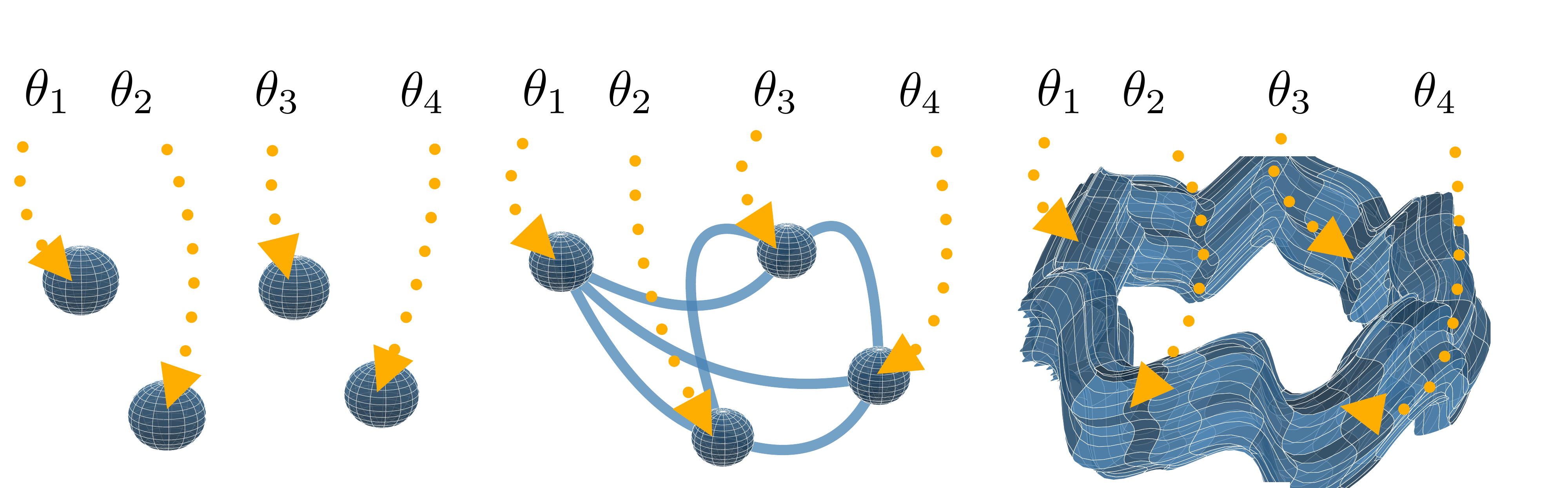}
    \caption[Progressive understanding of neural network loss surfaces]{Progressive understanding of neural network loss surfaces. \textbf{Left}: the traditional view, where low-loss regions are thought to be isolated~\citep{optimization-problems}. \textbf{Center}: low-loss curves can connect apparently distinct modes~\citep{Garipov2018-pz,Draxler2018-vr}. \textbf{Right}: SGD solutions may populate a broader connected low-loss region, suggesting a higher-dimensional low-loss volume rather than isolated paths alone~\citep{benton_loss_2021}. Figure from \citep{benton_loss_2021}.}
    \label{fig:conn-simplices}
\end{figure}

\section{Mode connectivity}
Mode connectivity studies the geometry of the loss landscape, focusing on the regions corresponding to local minima. It was long believed that the highly non-convex loss landscape of deep networks makes optimization difficult, with many poor local minima and saddle points. \citet{optimization-problems} were among the first to challenge this view, evaluating the loss along linear paths $\theta = (1-\lambda)\theta_A + \lambda\theta_B$ from initialization to convergence and finding the path to be approximately convex. This suggests that optimization is more well-behaved than previously thought.

\citet{Daniel_Freeman2016-th} further studied the topology of the loss landscape through the lens of its sublevel sets $\Omega_\lambda = \{\theta : \mathcal{L}(\theta) \leq \lambda\}$. When sublevel sets are connected, the landscape cannot fragment into isolated bad valleys within the regime under study, since one can move within a low-loss region without encountering an unavoidable barrier. They show that half-rectified single-layer networks have \emph{asymptotically connected} sublevel sets, and empirically find near-convex behavior in several architectures. This suggests a loss landscape with no spurious valleys in the regimes they study, though not a general theorem for arbitrary deep networks.

Moving from topology to the relationship between different modes, \citet{Garipov2018-pz} investigate whether distinct solutions can be connected through a high-accuracy path. They propose a training procedure to find such paths and show they can be well approximated as a piecewise-linear path composed of two segments, a finding they exploit to construct efficient ensembles. Concurrently and independently, \citet{Draxler2018-vr} find \emph{Minimum Energy Paths} (MEPs) between modes, showing they are essentially flat in both training and test loss. Taken together, these works suggest that many apparent minima are not isolated in practice, but can instead be joined by continuous paths that remain in regions of low loss (\cref{fig:conn-simplices}).

\subsection{Linear mode connectivity}\label{subsec:lmc}

\citet{linear-mode-connectivity} ask a finer question: are the modes found by two independent training runs of the same network, from the same initialization but with different SGD noise (random data order and augmentation), linearly connected? If so, merging models by weight interpolation within that region incurs no loss.

\paragraph{Instability analysis.}
To study how SGD noise affects the optimization outcome, \citet{linear-mode-connectivity} propose \emph{instability analysis}. A network is initialized at $\theta_{\text{init}}$ and two copies are trained with different noise samples. The loss barrier along the linear interpolation between the two resulting models (defined below) quantifies how stable the outcome is to SGD noise. The analysis can also start from an intermediate checkpoint $\theta_k$, revealing \emph{when} the trajectory commits to a linearly connected minimum; empirically, this happens early in training.

The key metric used throughout this analysis is the \emph{loss barrier}:

\begin{definition}[Loss barrier]\label{def:loss-barrier}
    Given two points $\theta_A, \theta_B$ and a loss function $\mathcal{L}$ such that $\mathcal{L}(\theta_A) \approx \mathcal{L}(\theta_B)$, the \emph{loss barrier} is defined as
    \[
        \max_{\lambda \in [0,1]} \mathcal{L}\!\left((1-\lambda)\,\theta_A + \lambda\,\theta_B\right) - \tfrac{1}{2}\!\left(\mathcal{L}(\theta_A) + \mathcal{L}(\theta_B)\right).
    \]
\end{definition}

\begin{figure}
    \centering
    \includegraphics[width=.5\linewidth]{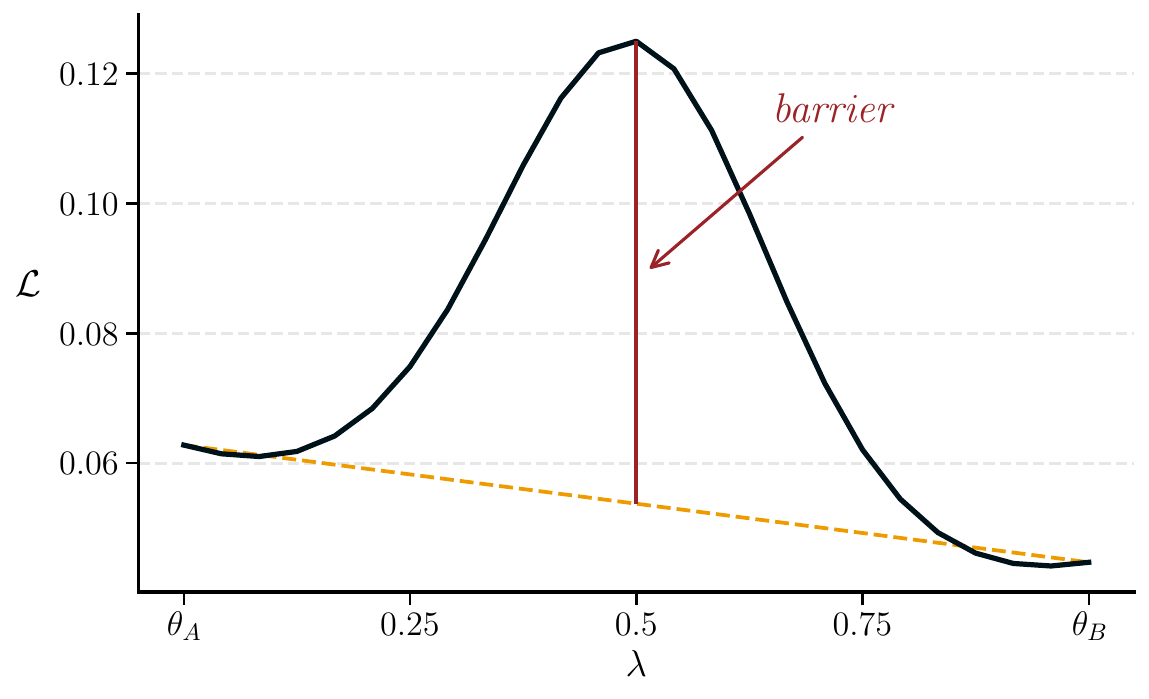}
    \caption{Loss barrier between two modes $\theta_A$ and $\theta_B$.}
    \label{fig:loss-barrier}
\end{figure}

A barrier of zero means that linear interpolation never rises above the average endpoint loss, whereas larger values indicate that the line segment passes through a higher-loss region. An intuitive illustration is given in \cref{fig:loss-barrier}.

\citet{linear-mode-connectivity} find that standard vision models become stable to SGD noise early in training, after which the outcome of optimization belongs to a fixed, linearly connected region. This property has important implications for pruning: the \emph{Lottery Ticket Hypothesis} (LTH)~\citep{franklelottery} posits that dense networks contain sparse subnetworks (``winning tickets'') that, when trained in isolation from the original initialization, match the full network's test accuracy. 

\citet{linear-mode-connectivity} establish a tight connection between LTH and mode connectivity, showing that solutions found from the same initialization that are linearly connected share the same winning ticket, and that the sparsest subnetworks identified by Iterative Magnitude Pruning can only be trained to full accuracy when the full network is robust to SGD noise, the very condition that determines whether two training runs land in the same linearly connected region.

Going further, \citet{benton_loss_2021} provide empirical evidence that many independently trained SGD solutions can be connected through multi-dimensional low-loss simplicial complexes, and propose a technique to discover such structures for ensembling. More generally, these low-loss paths and manifolds translate directly into efficient ensembling strategies: rather than training many independent models, one can sample diverse checkpoints along the low-loss geometry and average their predictions. \citet{Garipov2018-pz} sample along the piecewise-linear curve connecting two modes, obtaining ensemble-quality predictions from a single extended training run, while \citet{benton_loss_2021} scale this to the full simplicial complex, achieving further improvements in accuracy, calibration, and robustness over standard independently trained ensembles.

\section{Weight-space symmetries}
From the previous section, linear connectivity appears easiest to observe when models share not only the same architecture and task, but also the same initialization and an initial portion of the training trajectory. This raises a natural question: if these conditions are not met, what makes the corresponding modes appear isolated in weight space? A key part of the answer lies in \emph{weight-space symmetries}. Even when two networks implement essentially the same function, they need not occupy nearby points in parameter space, because hidden units can be permuted without changing the network's input-output map. As a result, independently trained models may converge to parameterizations that are functionally equivalent yet separated by large Euclidean distance, making naive interpolation fail. To make this precise, consider a Multi-Layer Perceptron (MLP) where the computation at layer $\ell$ is $\mathbf{z}_{\ell+1} = \phi\!\left(W^{(\ell)} \mathbf{z}_{\ell} + b_{\ell}\right)$, with $\phi$ an element-wise activation function. Setting $b_\ell = 0$ for clarity and applying a permutation matrix $P \in \mathbb{P}$ to the rows (neurons) of $W^{(\ell)}$ gives $\mathbf{z}_{\ell+1}' = \phi(P W^{(\ell)} \mathbf{z}_\ell)$. Since $\phi$ is element-wise, it commutes with $P$, so the effect can be fully cancelled by permuting the columns of the next layer by $P^\top$:
\begin{equation*}
    \mathbf{z}_{\ell+2}' = W^{(\ell+1)} P^\top \mathbf{z}_{\ell+1}' = W^{(\ell+1)} \underbrace{P^\top P}_{I} W^{(\ell)} \mathbf{z}_\ell = \mathbf{z}_{\ell+2}.
\end{equation*}
Hence any two models differing only by a neuron permutation are functionally equivalent. For a network with $L$ hidden layers $W_1, \dots, W_L$ with widths $d_1, \dots, d_L$, there are $\prod_{\ell=1}^L d_\ell!$ such symmetries in total. For a parameter set $\Theta$, the full equivalence class of functionally equivalent reparameterizations is denoted $\pi(\Theta)$. This combinatorial explosion of symmetries is a major factor in the apparent isolation of modes: two models that learned the same function may appear far apart in parameter space simply because they converged to different elements of the same equivalence class.

\section{Neuron matching}\label{sec:neuron-matching}

If permutation symmetries are indeed a primary source of the apparent isolation between modes, then connectivity may be hidden rather than absent: two solutions can fail to be linearly connected in their raw parameterizations, yet become connected after an appropriate reordering of their neurons. Under this view, the key operation is to find a symmetry transformation that maps one model into the representation used by the other. Doing so effectively maps one parameterization into the basin of the other, after which a simple linear interpolation can remain in low loss (\cref{fig:git-rebasin}).

This intuition is formalized by the following conjecture of \citet{Entezari2021-me}, which posits that, once permutation symmetries are accounted for, most modes lie in a common convex basin:

\begin{conjecture}[\citealt{Entezari2021-me}]\label{conj:perm-invariance-connectivity}
    Most modes belong to a set $S$ whose elements can be permuted such that there is no loss barrier on the linear interpolation between any two permuted elements of $S$.
\end{conjecture}

Building on \cref{conj:perm-invariance-connectivity}, \citet{git-rebasin} propose three algorithms to find a permutation $\pi$ that aligns a model $B$ with a reference model $A$. Once such a permutation is found, model $B$ can be re-expressed in the basin of model $A$, making weight interpolation possible.

\begin{figure}[htbp]
    \centering
    \includegraphics[width=0.5\textwidth]{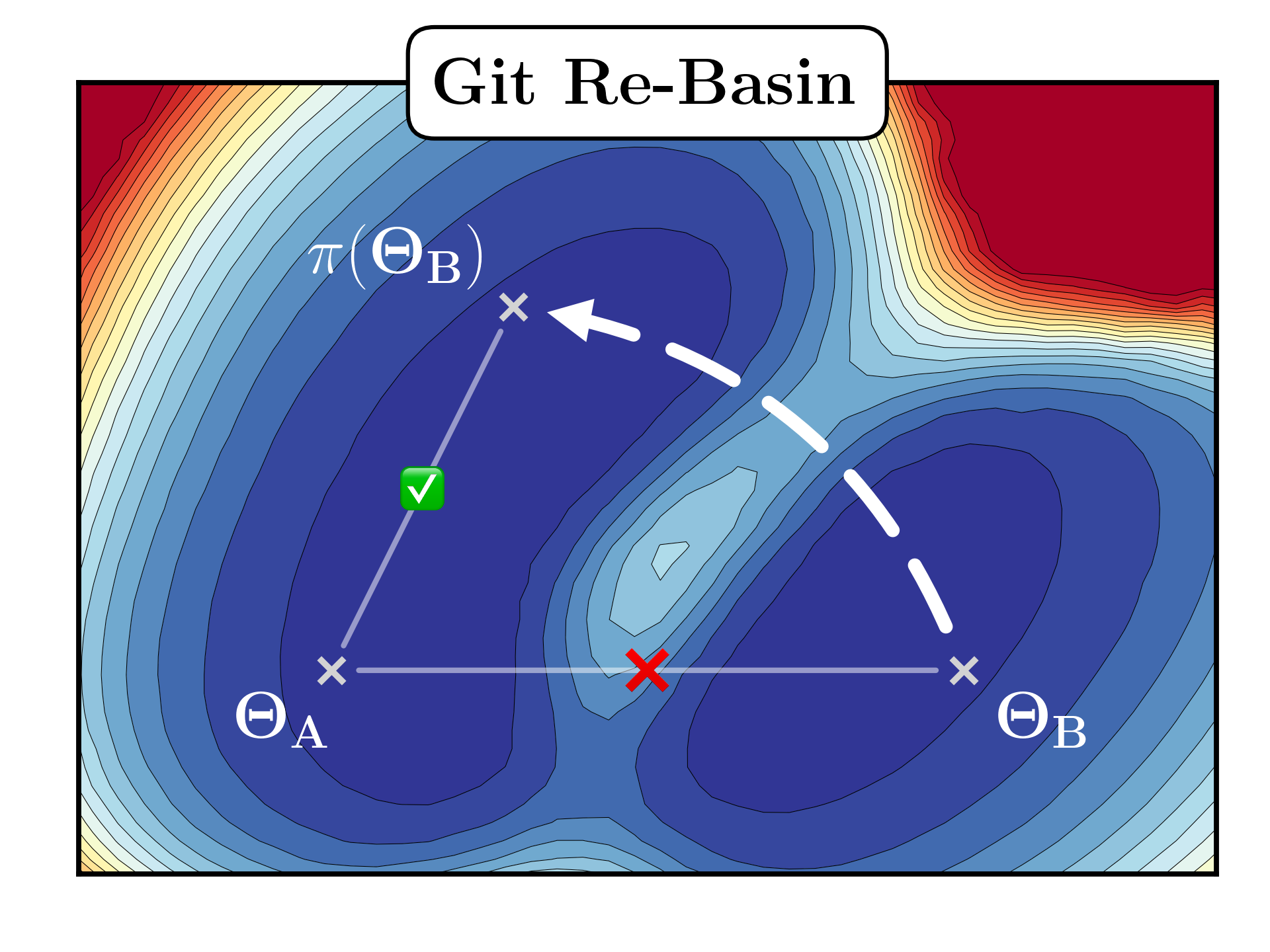}
    \caption[Schematic of \texttt{Git Re-Basin} neuron alignment]{Schematic of \texttt{Git Re-Basin}: a permutation is found that maps model $B$ into the basin of model $A$, enabling low-barrier linear interpolation~\citep{git-rebasin}. Figure from \citep{git-rebasin}.}%
    \label{fig:git-rebasin}
\end{figure}

\paragraph{Activation matching.}
A natural way to align units across two models is to compare their activations and match those that respond similarly over the data. For layer $\ell$, let $\boldsymbol{Z}^{(A)}, \boldsymbol{Z}^{(B)} \in \mathbb{R}^{d \times n}$ denote the $d$-dimensional activations across all $n$ training points. The optimal permutation solves
\begin{equation}
    \boldsymbol{P}_\ell = \underset{\boldsymbol{P} \in \mathbb{P}}{\arg\min} \sum_{i=1}^n \left\| \boldsymbol{Z}_{:,i}^{(A)} - \boldsymbol{P}\,\boldsymbol{Z}_{:,i}^{(B)} \right\|^2 = \underset{\boldsymbol{P} \in \mathbb{P}}{\arg\max} \left\langle \boldsymbol{P},\, \boldsymbol{Z}^{(A)}\!\left(\boldsymbol{Z}^{(B)}\right)^\top \right\rangle_F,
\end{equation}
a Linear Assignment Problem (LAP) solvable in polynomial time (e.g., via the Hungarian algorithm). The permuted weights are set as $\boldsymbol{W}_\ell' = \boldsymbol{P}_\ell \boldsymbol{W}_\ell^{(B)} \boldsymbol{P}_{\ell-1}^\top$.

\paragraph{Weight matching.}
Alternatively, units can be matched directly from the model weights, without requiring any data. The objective is to find permutations $\{P_\ell\}$ maximizing the weight overlap across all layers:
\begin{align}
    \arg\max_{\{P_\ell \in \mathbb{P}\}} \sum_{\ell=1}^L \langle W_A^{(\ell)},\, P_\ell\, W_B^{(\ell)}\, P_{\ell-1}^\top \rangle, \label{eq:weight-matching-obj}
\end{align}
with $P_0 := I$. Since \cref{eq:weight-matching-obj} is NP-hard globally, \citet{git-rebasin} solve it one layer at a time, relaxing each bilinear problem to a LAP that can be solved efficiently.

\paragraph{Straight-through estimator.}
A third approach learns the permutation by directly minimizing the loss of the interpolated model:
\begin{equation}\label{eq:ste}
    \min_{\tilde{\Theta}_B} \mathcal{L}\!\left(\tfrac{1}{2}\bigl(\Theta_A + \operatorname{proj}(\tilde{\Theta}_B)\bigr)\right), \quad \operatorname{proj}(\Theta) \triangleq \underset{\pi}{\arg\max}\; \operatorname{vec}(\Theta) \cdot \operatorname{vec}\!\left(\pi(\Theta_B)\right),
\end{equation}
where $\tilde{\Theta}_B \approx \pi(\Theta_B)$ and $\operatorname{proj}(\cdot)$ finds the nearest realizable permutation. 
Because the projection onto valid permutations is discrete and therefore non-differentiable, optimization uses a straight-through approximation~\citep{bengio2013estimating}: the projected parameters are used to evaluate the loss, while gradients are propagated through the underlying continuous variables.
Among the three methods, the straight-through estimator performs best, though it requires an additional training phase to learn the permutation.

\paragraph{Renormalizing the activations.}
Even with high-quality alignment, loss barriers can remain elevated due to mismatches in activation statistics between the endpoint networks.
\texttt{REPAIR}~\citep{repair} observes that interpolation can induce discrepancies in the mean and standard deviation of intermediate activations, and demonstrates that restoring these statistics greatly improves the quality of the merged network.
Concretely, given endpoint activations $X_1$ and $X_2$, the merged model's activations $X_\alpha$ are renormalized to have:
\begin{align}
    \mathbb{E}\left[X_\alpha\right]         &= (1-\alpha) \cdot \mathbb{E}\left[X_1\right] + \alpha \cdot \mathbb{E}\left[X_2\right], \label{eq:repair-mean}\\
    \operatorname{std}\left(X_\alpha\right) &= (1-\alpha) \cdot \operatorname{std}\left(X_1\right) + \alpha \cdot \operatorname{std}\left(X_2\right). \label{eq:repair}
\end{align}

Alignment through neuron permutations directly enables weight-space model merging: once two models are brought into the same loss basin, their weights can be meaningfully interpolated, often yielding a single network with comparable performance. This geometric perspective---first align, then merge---forms the basis of the single-task methods developed in this part of the thesis. At the same time, most permutation-based approaches are inherently pairwise, aligning one model to another chosen reference. The next chapter asks how this picture changes when the number of models grows beyond two: can one align an entire collection of networks in a globally consistent way, without depending on an arbitrary anchor? In \cref{ch:cycle-consistent}, we answer this question by studying \(N\)-way matching and merging in the single-task setting. After that, in \cref{part:setting-two}, we move to a different but equally fundamental regime, where models are fine-tuned for distinct downstream tasks from a shared pre-trained backbone and the central challenge becomes task interference rather than geometric alignment.

%% file: 3_single_task/1_cycle_consistent/content.tex

\begin{chapteroverview}
In this chapter, we extend the pairwise neuron-matching problem to more than two models. We present a data-free matching and merging technique that enforces cycle consistency of the permutations when merging $n \geq 3$ models, preventing error accumulation in circular compositions. We demonstrate the benefits of this constraint across varying architectures and datasets.
\end{chapteroverview}

\section{From pairwise to multi-model alignment}\label{sec:cc-intro}
\input{1_Introduction/content.tex}
\section{Designing the cycle-consistent objective}\label{sec:cc-approach}
\input{3_Approach/content.tex}

\section{Empirical evaluation}\label{sec:exps}
\input{4_Experiments/content}


\section{Summary and outlook}\label{sec:cc-conclusions}
\input{6_Conclusions/content}

%% file: 3_single_task/1_cycle_consistent/1_Introduction/content.tex

As established in \cref{ch:background}, neuron permutation symmetries are the primary obstacle to weight-space merging: two models may have learned the same function yet appear arbitrarily far apart in parameter space simply because their neurons are ordered differently. The key step toward low-barrier interpolation is therefore to solve the alignment problem: finding the permutations that bring independently trained models into a shared basin. Existing solutions such as \texttt{Git Re-Basin}~\cite{git-rebasin} tackle this pairwise and layer-by-layer, which suffices when only two models are involved. However, real-world scenarios (distributed training, federated optimization, or simply combining a pool of independently trained checkpoints) naturally involve $n > 2$ models, and here a subtler issue emerges.
\input{1_Introduction/teaser.tex}
When permutations are computed independently for each pair, the composition along any cycle generally does \emph{not} reduce to the identity (see \cref{fig:cycle-cons-teaser}). As shown in \cref{tab:error-accumulation} and \cref{fig:cycle-comparison}, mapping a model $A$ to $C$ through $B$ and then back to $A$ yields a different model altogether, one that lands in a completely different basin. This lack of \emph{cycle consistency} also manifests in the $n=2$ case: the permutations optimized to align model $A$ to model $B$ are not guaranteed to be the inverse of those mapping $B$ to $A$, making the alignment pipeline brittle and dependent on an arbitrary choice of mapping direction.

This chapter addresses the problem by introducing an alignment algorithm that works for the general case with $n \geq 2$ models while \emph{guaranteeing} cycle consistency. The key idea is to factorize each pairwise permutation $P^{AB}$ as $P^{A} {{(P^B)}^\top}$, where ${{(P^B)}^\top}$ maps $B$ to a common space denoted as the \emph{universe}, and $P^{A}$ maps from the universe back to $A$. This formulation ensures cycle consistency by design, as any cyclic composition of such permutations equals the identity.

The numerical implementation is based on the Frank-Wolfe algorithm~\cite{frank-wolfe} and optimizes for the permutations of \emph{all} layers simultaneously at each step, naturally accounting for inter-layer dependencies. This is in contrast with \citet{git-rebasin}, who seek the optimal permutations for each layer separately and thus cannot ensure coherence across the entire network. The universe space also serves as a natural aggregation point for merging: once all models are mapped into a shared basin, their weights can be meaningfully averaged.
\begin{figure}
    \begin{subfigure}{0.48\textwidth}
        \centering
            \includegraphics[width=\textwidth]{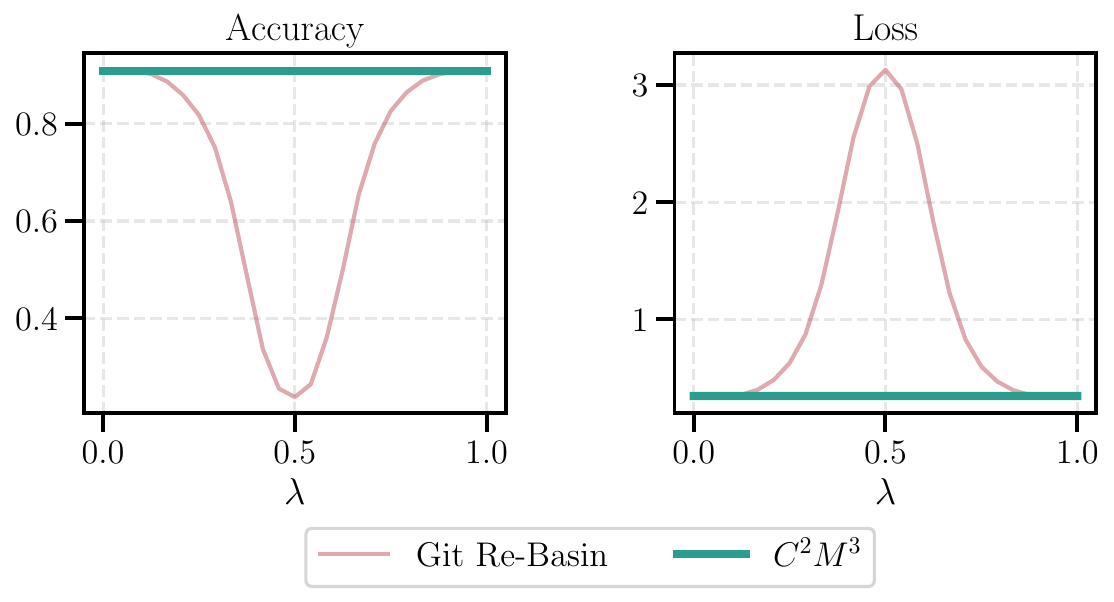}
        \caption{Loss and accuracy curves for a model $A$ and the model mapped back after a cyclic permutation. Models cyclically permuted with \method{Git Re-Basin} end up in a different basin than the one they started from.}
        \label{fig:cycle-comparison}
    \end{subfigure}
    \hfill 
    \begin{subfigure}{0.48\textwidth}
            \begin{center}
                \resizebox{\textwidth}{!}{%
                \begin{tabular}{ccc}
                    \toprule
                    Permutation                                  & Git Re-Basin & $C^2M^3$ \\
                    \midrule
                    $d\left(A,  P_{A\rightarrow B \rightarrow C \rightarrow A}(A)\right)$ & 41.07         & 0.0      \\
                    $d\left(B, P_{B\rightarrow C \rightarrow A \rightarrow B}(B)\right)$  & 41.18         & 0.0      \\
                    $d\left(C, P_{C\rightarrow A \rightarrow B \rightarrow C}(C)\right)$  & 41.19         & 0.0      \\
                    \bottomrule
                \end{tabular}
                }
            \end{center}
            \vspace{1.2cm}
            \caption{Accumulated error obtained when cyclically permuting models $A$, $B$ and $C$ as in \cref{fig:cycle-cons-teaser}. $P_{A\rightarrow B \rightarrow C \rightarrow A}$ refers to the composition $P_{AC} \circ P_{CB} \circ P_{BA}$ and $d(\cdot)$ is the $\ell_2$ loss.
            }
            \label{tab:error-accumulation}
    \end{subfigure}
    \caption[Cycle consistency comparison of permutation methods]{Existing methods accumulate error when cyclically mapping a model through a series of permutations, while $C^2M^3$ correctly maps the model back to the starting point.}
\end{figure}

%% file: 3_single_task/1_cycle_consistent/1_Introduction/teaser.tex
\begin{figure}
    \centering

    \begin{tikzpicture}[scale=0.9]

        \begin{scope}[local bounding box=first]

            \def\r{2}

            \def\angleA{90}
            \def\angleB{-30}
            \def\angleC{210}

            \coordinate (A) at (\angleA:\r);
            \coordinate (B) at (\angleB:\r);
            \coordinate (C) at (\angleC:\r);

            \draw[-latex] (A) to[bend left] node[midway, right] {$P^{BA}$} (B);
            \draw[-latex] (B) to[bend left] node[midway, below] {$P^{CB}$} (C);
            \draw[-latex] (C) to[bend left] node[midway, left] {$P^{AC}$} (A);

            \filldraw (A) circle (1pt) node[above] {A};
            \filldraw (B) circle (1pt) node[below left] {B};
            \filldraw (C) circle (1pt) node[below right] {C};

        \end{scope}

        \begin{scope}[shift={(4.5,0)}, local bounding box=second]

            \coordinate (U) at (0,0);

            \def\r{2}

            \def\angleA{90}
            \def\angleB{-30}
            \def\angleC{210}

            \coordinate (A) at (\angleA:\r);
            \coordinate (B) at (\angleB:\r);
            \coordinate (C) at (\angleC:\r);

            \draw[-latex] (A) to[bend left] node[midway, right] {$(P^{A})^\top$} (U);
            \draw[-latex] (B) to[bend left] node[midway, below] {$(P^{B})^\top$} (U);
            \draw[-latex] (C) to[bend left] node[midway, left] {$(P^{C})^\top$} (U);

            \draw[-latex] (U) to[bend left] node[midway, left] {$P^{A}$} (A);
            \draw[-latex] (U) to[bend left] node[midway, above] {$P^{B}$} (B);
            \draw[-latex] (U) to[bend left] node[midway, right] {$P^{C}$} (C);

            \filldraw (A) circle (1pt) node[above] {A};
            \filldraw (B) circle (1pt) node[below left] {B};
            \filldraw (C) circle (1pt) node[below right] {C};

            \filldraw (U) circle (1pt) node[above right] {U};

        \end{scope}

    \end{tikzpicture}

    \caption[Cycle-consistent multi-model merging via universe permutations]{Cycle-Consistent Multi-Model Merging over three models $A, B, C$. \textbf{Left:} existing methods seek pairwise permutations that map between models; note that $P^{AC} \circ P^{CB}\circ P^{BA} \neq I$ in general, unless this is explicitly enforced. \textbf{Right:} our method computes permutations $P^A$, $P^B$, $P^C$ from each model to a {\em universe} $U$, such that a pairwise permutation $P^{BA}$ mapping $A$ to $B$ can be obtained as $P^{BA} = P^{B} (P^{A})^\top$. This way, cycle-consistency is enforced by design and $P^{AC} \circ P^{CB}\circ P^{BA} = I$.
    }
    \label{fig:cycle-cons-teaser}
\end{figure}
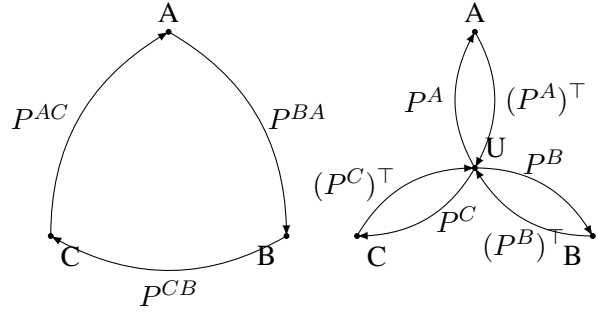

%% file: 3_single_task/1_cycle_consistent/3_Approach/content.tex
We now propose a novel algorithm to tackle the weight matching problem, first introducing its formulation in the pairwise case and then generalizing it to match and merge a larger number $n$ of models in a cycle-consistent fashion.

\paragraph{Pairwise matching}\label{subsec:pairwise-matching}
As we have seen, the NP-hardness of \cref{eq:weight-matching-obj} demands a relaxation of the problem to be tackled. Differently from \citet{git-rebasin}, we opt to maintain the objective global with respect to the layers and instead, iteratively optimize its linear approximation via the Frank-Wolfe algorithm~\cite{frank-wolfe}. This procedure requires the computation of the gradient of \cref{eq:weight-matching-obj} with respect to each permutation ${P}_i$, thus we have to account for two contributions for each $\nabla_{P_i}$, \emph{i.e.}, its gradient from permuting the rows of ${W}_i$ and the one from permuting the columns of ${W}_{i+1}$:
\begin{align}
    \nabla_{P_i}f & = \underbrace{W^A_i P_{i-1} {(W_i^B)}^\top}_{\text{from permuting rows}} + \underbrace{{(W^A_{i+1})}^\top P_{i+1}  W_{i+1}^B}_{\text{from permuting columns}}.
\end{align}
The Frank-Wolfe algorithm then uses the gradient to iteratively update the solution by linearly interpolating between the current solution and the projected gradient. We refer to \citet{lacoste2016convergence} for theoretical guarantees of convergence. The full algorithm is reported in \cref{app:pairwise-frankwolfe}.

\paragraph{Generalization to \texorpdfstring{$n$}{n} models}\label{subsec:generalized-matching}

To generalize to $n$ models, we jointly consider all pairwise problems
\begin{align}
    \arg\max_{P_i^{pq} \in \mathbb{P}} ~ \sum_{p=1}^n \sum_{\substack{q = 1 \\ q \neq p}}^n \sum_{i=1}^L \langle W_i^{p}, P_i^{pq} W_i^{q} {(P_{i-1}^{pq})}^\top \rangle, \label{eq:gen-frank-wolfe-unfactorized}
\end{align}
where the superscript $pq$ indicates that the permutations maps model $q$ to model $p$, with $P_0^{pq} := I$. 
To \emph{ensure cycle consistency by construction}, we replace the quadratic polynomial by a fourth-order polynomial.
Dropping the layer subscript for clarity, we replace the pairwise matchings $P^{pq}$ in the objective of \cref{eq:gen-frank-wolfe-unfactorized} by factorizing the permutations into \emph{object-to-universe matchings} $P^{pq} = P^p \circ {(P^q)}^\top$ 
so that each model $q$ can be mapped back and forth to a common universe $u$ with a permutation and its transpose, allowing us to map model $q$ to model $p$ by composition of ${(P^q)}^\top$ ($q\to u$) and $P^p$ ($u\to p$). 
This way, the objective of \cref{eq:gen-frank-wolfe-unfactorized} becomes
\begin{align}
\sum_{p \neq q}^n\sum_{i=1}^L \langle W_i^{p}, P_i^p {(P_i^q)}^\top W_i^{q} (P_{i-1}^p {(P_{i-1}^q)}^\top)^\top \rangle = \sum_{p \neq q}^n\sum_{i=1}^L \langle (P_i^p)^\top W_i^{p} P_{i-1}^p,  (P_i^q)^\top W_i^{q} P_{i-1}^q  \rangle. \label{eq:gen-frank-wolfe-obj}
\end{align}
As stated by \cref{thm:cycle-consistency}, the permutations we obtain using \cref{eq:gen-frank-wolfe-obj} are cycle consistent. We refer the reader to \citet{Bernard2021SparseQO} for the proof and a complete discussion of the subject. 

\begin{theorem}[Restated from \citealt{Bernard2021SparseQO}]\label{thm:cycle-consistency}
    Given a set of $n$ models $p_1,\dots,p_n$ and object-to-universe permutations $P_i^{p_j}$ computed via \cref{eq:gen-frank-wolfe-obj}, the pairwise correspondences defined by $P_i^{p_l p_j}={P_i^{p_l}}\circ \left(P_i^{p_j}\right)^\top$ are cycle-consistent, \textit{i.e.},
    \begin{equation*}
    P_i^{p_1 p_j}\circ\cdots\circ P_i^{p_3 p_2} \circ P_i^{p_2 p_1} = I
    \end{equation*}
    for all layer indices $i$, $2\leq j\leq n$.
\end{theorem}
Similarly to the pairwise case, the approach requires computing the gradients for the linearization. This time, however, each $\nabla_{P_{i}^A}f$ has four different contributions: one from permuting the rows of its corresponding layer, one from anti-permuting the columns of the subsequent layer, and two other contributions that arise from the symmetric case where $A$ becomes $B$. In detail,
\begin{equation}
    \nabla_{P^A_{\ell}} = \nabla_{P^A_{\ell}}^{\text{rows}}  + \nabla_{P^A_{\ell}}^{\text{cols}} + \nabla_{P^A_{\ell}}^{\text{rows}, \leftrightarrows} + \nabla_{P^A_{\ell}}^{\text{cols}, \leftrightarrows}
\end{equation}
where
\begin{align*}
    \nabla_{P^A_{\ell}}^{\text{rows}} & = W_\ell^A P_{\ell-1}^A(P_{\ell-1}^B)^\top (W_\ell^B)^\top P_{\ell}^B   &\nabla_{P^A_{\ell}}^{\text{cols}}                    = (W_{\ell+1}^A)^\top P^A_{\ell+1} \ (P^B_{\ell+1})^\top \ W_{\ell+1}^B \ P^B_{\ell} \\
    \nabla_{P^A_{\ell}}^{\text{rows}, \leftrightarrows} & = W_\ell^B P_{\ell-1}^B(P_{\ell-1}^A)^\top (W_\ell^A)^\top P_{\ell}^A &\nabla_{P^A_{\ell}}^{\text{cols}, \leftrightarrows}  = (W_{\ell+1}^B)^\top P^B_{\ell+1} \ (P^A_{\ell+1})^\top \ W_{\ell+1}^A \ P^A_{\ell}
\end{align*}
See~\cref{alg:frank-wolfe-generalized} for a complete description of the procedure.
\input{3_Approach/generalized_frank_wolfe}

\begin{figure}
    \begin{subfigure}[b]{0.48\textwidth}
        \centering
        \includegraphics[width=\textwidth]{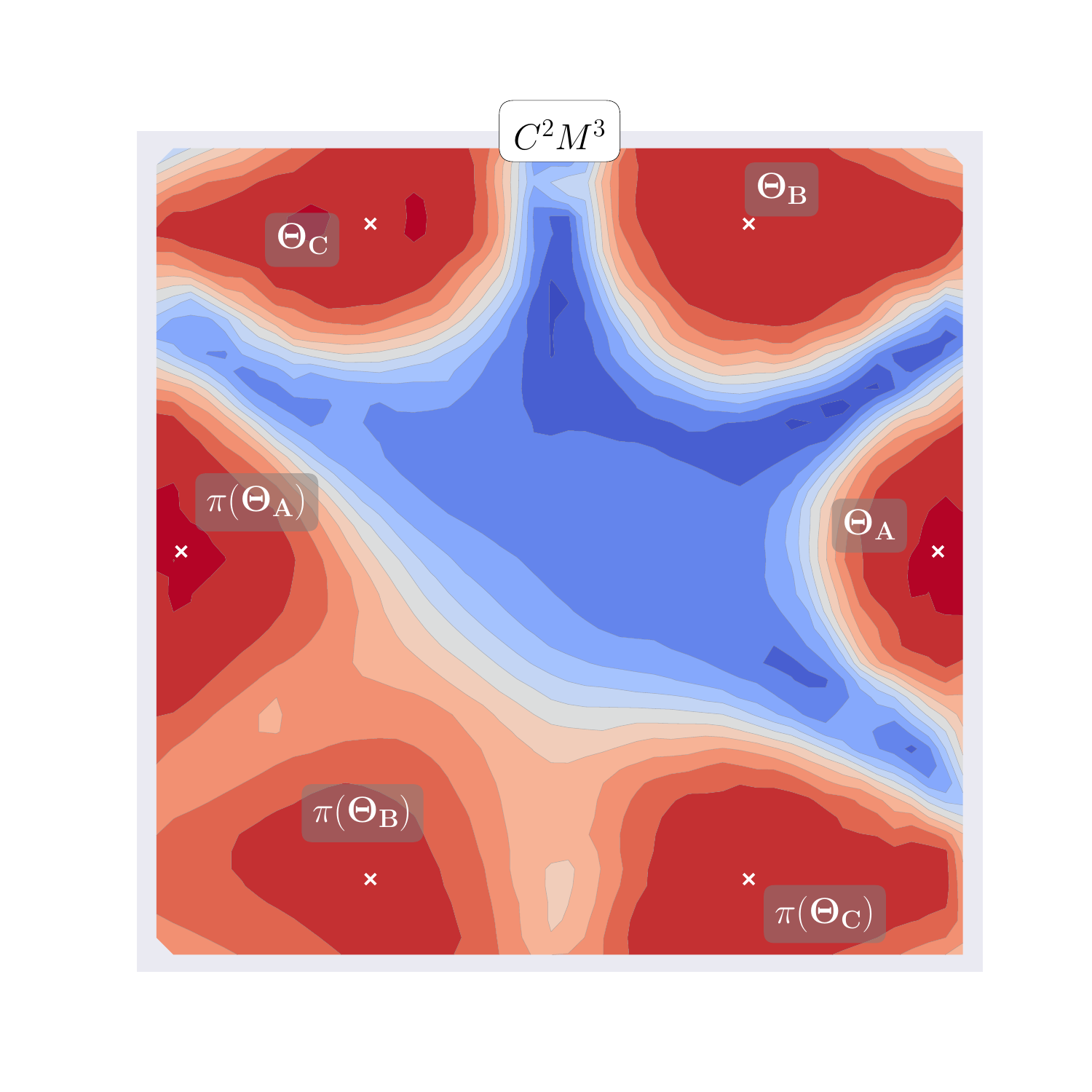}
        \label{fig:resnet_cifar_loss_contour}
        \vspace{-1cm}
        \caption{\model{ResNet20} over \dataset{CIFAR100}.}
    \end{subfigure}
    \hfill
    \begin{subfigure}[b]{0.48\textwidth}
        \centering
        \includegraphics[width=\textwidth]{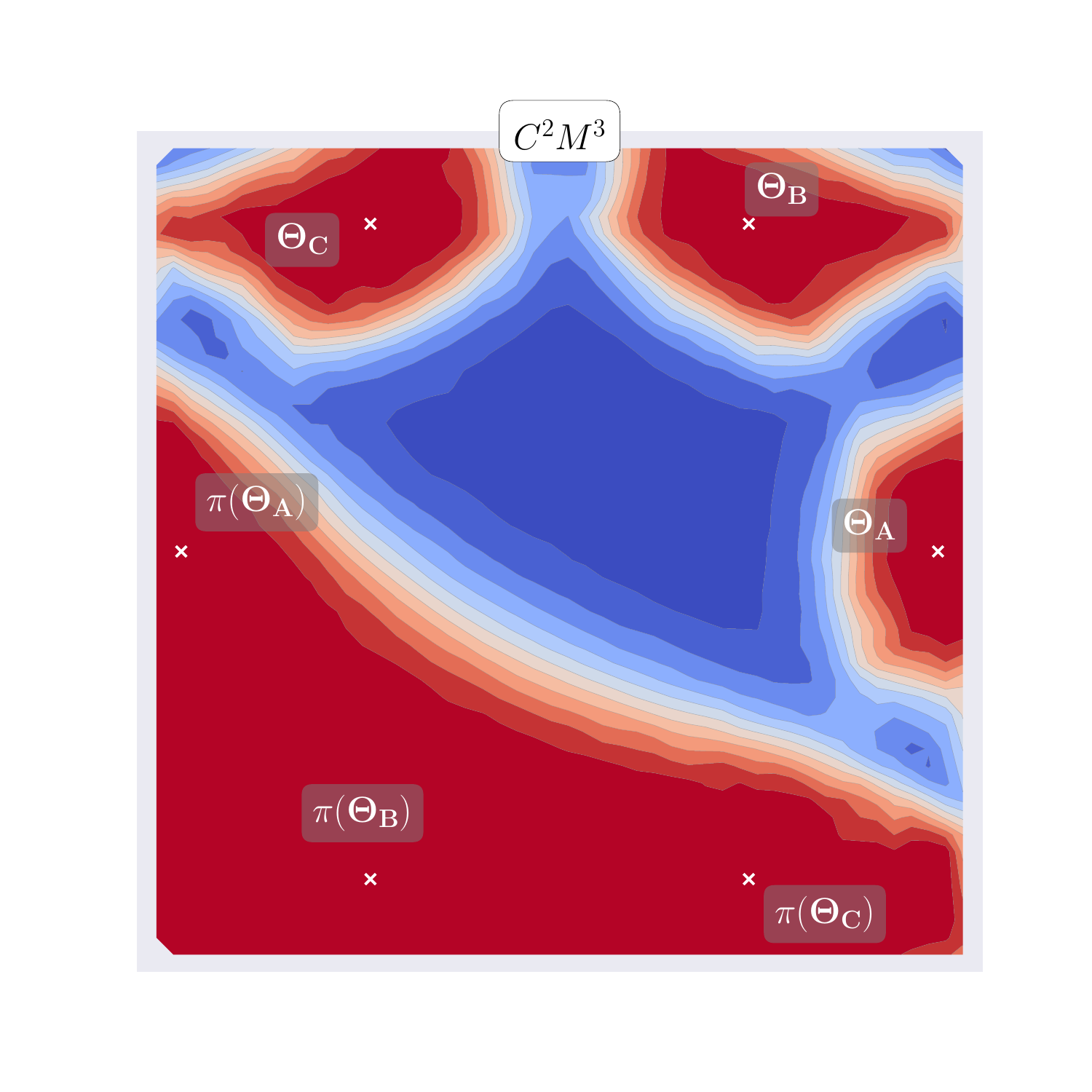}
        \label{fig:MLP_cifar_loss_contour}
        \vspace{-1cm}
        \caption{\model{MLP} over \dataset{MNIST}.}
    \end{subfigure}
    \caption[Loss landscape projection when matching three modes]{2D projection of the loss landscape when matching three modes $\theta_A, \theta_B, \theta_C$; the models $\pi(\theta_A), \pi(\theta_B), \pi(\theta_C)$ are their resulting images in the universe, and lie in the same basin. Red zones indicate low-loss regions (typically basins), while blue zones indicate high-loss ones. }
    \label{fig:cifar_loss_contour}
\end{figure}
\paragraph{Merging in the universe space}\label{subsec:merging-universe}
Looking at the loss landscape resulting from interpolating models in \cref{fig:cifar_loss_contour}, we see that the loss curves are much lower when the models are interpolated in the universe space. In fact, the originally disconnected modes end up in the same basin when mapped onto the universe, making it suitable to average the models. Therefore, our merging method aggregates the models by taking the mean of the weights in the universe space, as detailed in \cref{alg:ccmmm}. 
\begin{algorithm}
    \caption{$C^2M^3$: Cycle-Consistent Multi Model Merging}\label{alg:ccmmm}
    \begin{algorithmic}[1]
        \REQUIRE $n$ models $A_1, \dots, A_n$ with $L$ layers
        \ENSURE merged model $M$
        \STATE $\{P_1, \dots, P_n\} \gets $ Frank-Wolfe($A_1, \dots, A_n$)
        \FOR{$i=1$ to $n$}
            \STATE $M_{i}^{\text{uni}} \gets \text{map\_to\_universe}(A_i, P_i)$
        \ENDFOR
        \STATE $M^{\text{uni}} \gets \frac{1}{n}\sum_{i=1}^n M_{i}^{\text{uni}}$
        \STATE \textbf{return} $M^{\text{uni}}$
    \end{algorithmic}
\end{algorithm}

%% file: 3_single_task/1_cycle_consistent/3_Approach/generalized_frank_wolfe.tex
\begin{algorithm}[H]
\caption{Frank-Wolfe for $n$-Model Matching }
\label{alg:frank-wolfe-generalized}
\begin{algorithmic}[1]
    \REQUIRE Weights of $n$ models $M_{i=1}^n$ \\tolerance $\epsilon > 0$
    \ENSURE Approximate solution to \cref{eq:gen-frank-wolfe-obj}
    \STATE $\mathbf{P}^k \gets $ identity matrices
    \REPEAT{}
        \FOR{$(p, q) \in [1, \dots, n] \times [1, \dots, n]$} 
            \FOR{$i=1$ to $L$}  
                \STATE $P_i^{p,k}, P_{i-1}^{p, k} \gets$ permutations over rows and columns of $W_i^p$ respectively
                \STATE $P_i^{q,k}, P_{i-1}^{q, k} \gets$ permutation over rows and columns of $W_i^q$ respectively
                \STATE $\nabla_{P_i^{p, k}} f \mathrel{+}= (W_{i+1}^p)^\top P^p_{i+1} \ (P^q_{i+1})^\top \ W_{i+1}^q \ P^q_{i}$
                \STATE $\nabla_{P_{i-1}^{p, k}} f \mathrel{+}= (W_{i+1}^p)^\top P^p_{i+1} \ (P^q_{i+1})^\top \ W_{i+1}^q \ P^q_{i}$
            \ENDFOR
        \ENDFOR
        \FOR{$P^k_i \in \mathbf{P}^k$}  
            \STATE $\Pi_i \gets \operatorname{LAP}(\nabla_{P_i^k} f)$
        \ENDFOR
        \STATE $\alpha \gets \text{line search}(f, \mathbf{P}^k, \mathbf{\Pi})$
        \FOR{$P^k_i \in \mathbf{P}^k$}
            \STATE $P_i^{k+1} = (1-\alpha) P_i^k + \alpha  \ \Pi_i$
        \ENDFOR
    \UNTIL $\|f(A, B, \mathbf{P}^{k+1}) - f(A, B, \mathbf{P}^{k})\| < \epsilon$
    \STATE \textbf{return} $\mathbf{P}^k$

\end{algorithmic}
\end{algorithm}

%% file: 3_single_task/1_cycle_consistent/4_Experiments/content.tex
This section evaluates the cycle-consistent framework on both the matching and the subsequent merging operation. Approaches suffixed with a $\dagger$ indicate the application of \texttt{REPAIR}.

\paragraph{Matching and merging two models}\label{subsec:exp-pairwise-matching-brief}
\begin{wrapfigure}[14]{r}{0.4\textwidth}
    \vspace{-0.5cm}
    \centering
    \includegraphics[width=0.4\textwidth]{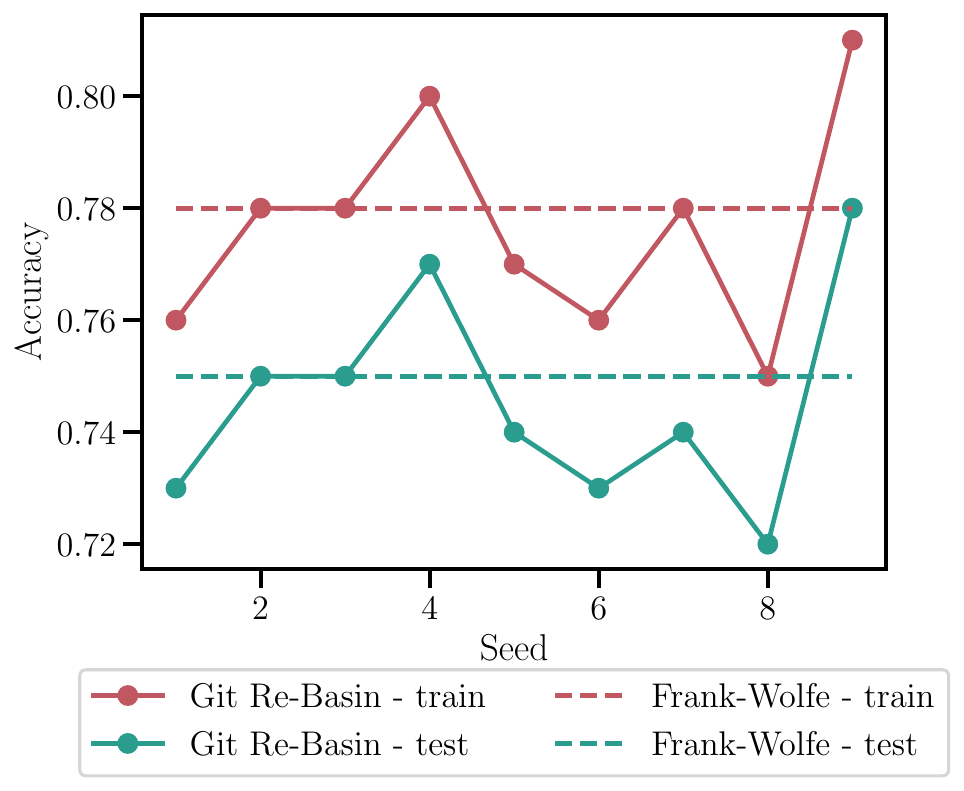}
    \caption[Interpolation accuracy variance across optimization seeds]{Accuracy of the interpolated model using \texttt{Git Re-Basin}~\cite{git-rebasin} over different optimization seeds.}
    \label{fig:git-re-basin-variance}
\end{wrapfigure}
As described in \cref{subsec:pairwise-matching}, our formalization can readily be used to match $n=2$ models. In this case, the energy is given by \cref{eq:weight-matching-obj} and the permutations are not factorized. We compare the performance of our approach against the \texttt{Git Re-Basin} algorithm~\cite{git-rebasin} and the \texttt{naive} baseline that aggregates the models by taking an unweighted mean on the original model weights without applying any permutation. In this setting, our method performs on par with the state of the art. Unlike the latter, however, we do not depend on the random choice of layers, as the optimization is performed over all layers simultaneously. As presented in \cref{fig:git-re-basin-variance}, this results in \texttt{Git Re-Basin} exhibiting variations of up to $10\%$ in accuracy depending on the optimization seed, while our method shows zero variance. We refer the reader to \cref{subsec:exp-pairwise-matching} for a thorough evaluation of $C^2M^3$ over a set of different datasets and architectures. The approach thus matches two models with the same accuracy as the state of the art, while being deterministic and independent of the random choice of layers.

\paragraph{Matching and merging \texorpdfstring{$n$}{n} models}\label{subsec:merging-n-models}%
We now evaluate $C^2M^3$ in matching and merging $n$ models. The matching is given by the factorized permutations obtained by \cref{alg:frank-wolfe-generalized}. 
We compare against two baselines: the simple approach of naively averaging the weights without any matching, and the \texttt{MergeMany} approach proposed by~\citet{git-rebasin}. The latter is reported in \cref{cycle-cons-appendix-a} for convenience.
As reported in \cref{tab:merge-many}, $C^2M^3$ obtains far superior results in terms of accuracy and loss in all considered settings, with accuracy gains as high as $+20\%$. Moreover, our approach natively yields cycle-consistent permutations: \cref{tab:error-accumulation} shows that \texttt{Git Re-Basin}~\cite{git-rebasin} accumulates significant error when computing the distance between the source model and the model obtained by applying a cyclic series of permutations, while our approach is able to perfectly recover the source model. This is further confirmed in \cref{fig:cycle-comparison}, where we show the loss and accuracy curves when interpolating between a model $A$ and the model mapped back after a cyclic permutation. Models cyclically permuted with Git Re-Basin end up in a different basin than the one they started from, while our cycle-consistent approach ensures that the target model is exactly the same as the source. The approach thus matches and merges $n$ models with a significant improvement in performance over the state of the art, while ensuring cycle-consistent permutations.

\begin{wrapfigure}[12]{r}{0.35\textwidth}
    \vspace{-0.5cm}
    \centering
    \includegraphics[width=0.35\textwidth]{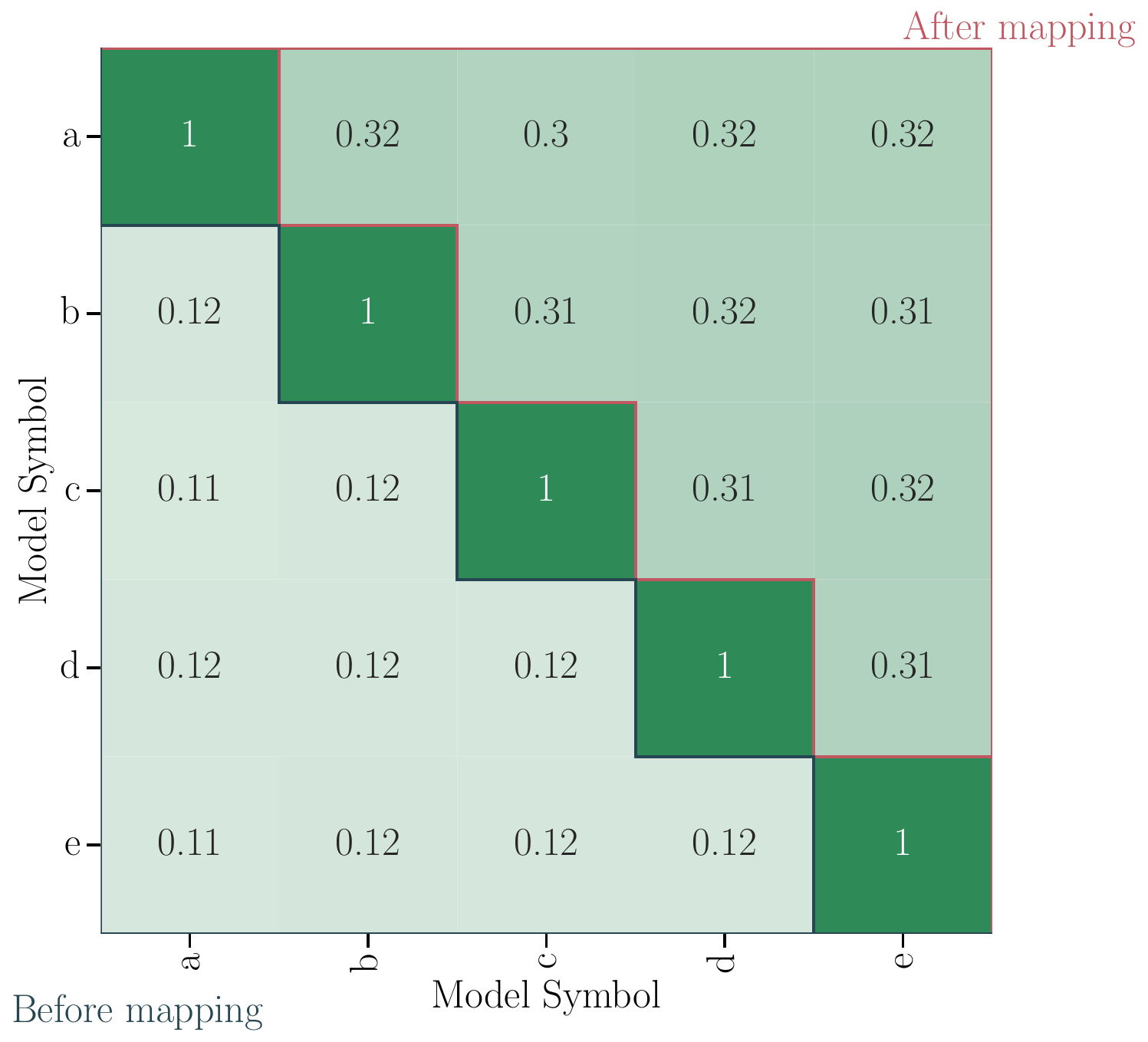}
    \caption[Weight cosine similarity before and after universe mapping]{Cosine similarity of the weights of 5 \texttt{ResNet20} trained on \texttt{CIFAR10} with $2\times$ width.}
    \label{fig:weights-cos-similarities}
\end{wrapfigure}
\paragraph{Model similarity before and after mapping} 
As we can see in \cref{fig:weights-cos-similarities}, the cosine similarity of the weights of the models is $3\times$ higher after mapping the latter to the universe. This suggests that the initial quasi-orthogonality of models is at least partially due to neuron permutation symmetries. We also report in \cref{subsec:exp-weight-similarity} the similarity of the representations between pairs of models. Interestingly, the latter does not change before and after mapping to the universe, but only if we consider a similarity measure that is invariant to orthogonal transformations such as CKA~\cite{Kornblith2019-vr}. When using a measure that does not enjoy this property, such as the Euclidean distance, the representations become much more similar in the universe space. The models are thus approximately $3\times$ more similar in the universe space, and the mapping affects the representations as an orthogonal transformation.

\paragraph{Effect of activation renormalization} 
\begin{wrapfigure}[10]{l}{0.35\textwidth}
    \vspace{-0.45cm}
    \centering
    \includegraphics[width=.35\textwidth, trim=4cm 1.5cm 3cm 3cm, clip]{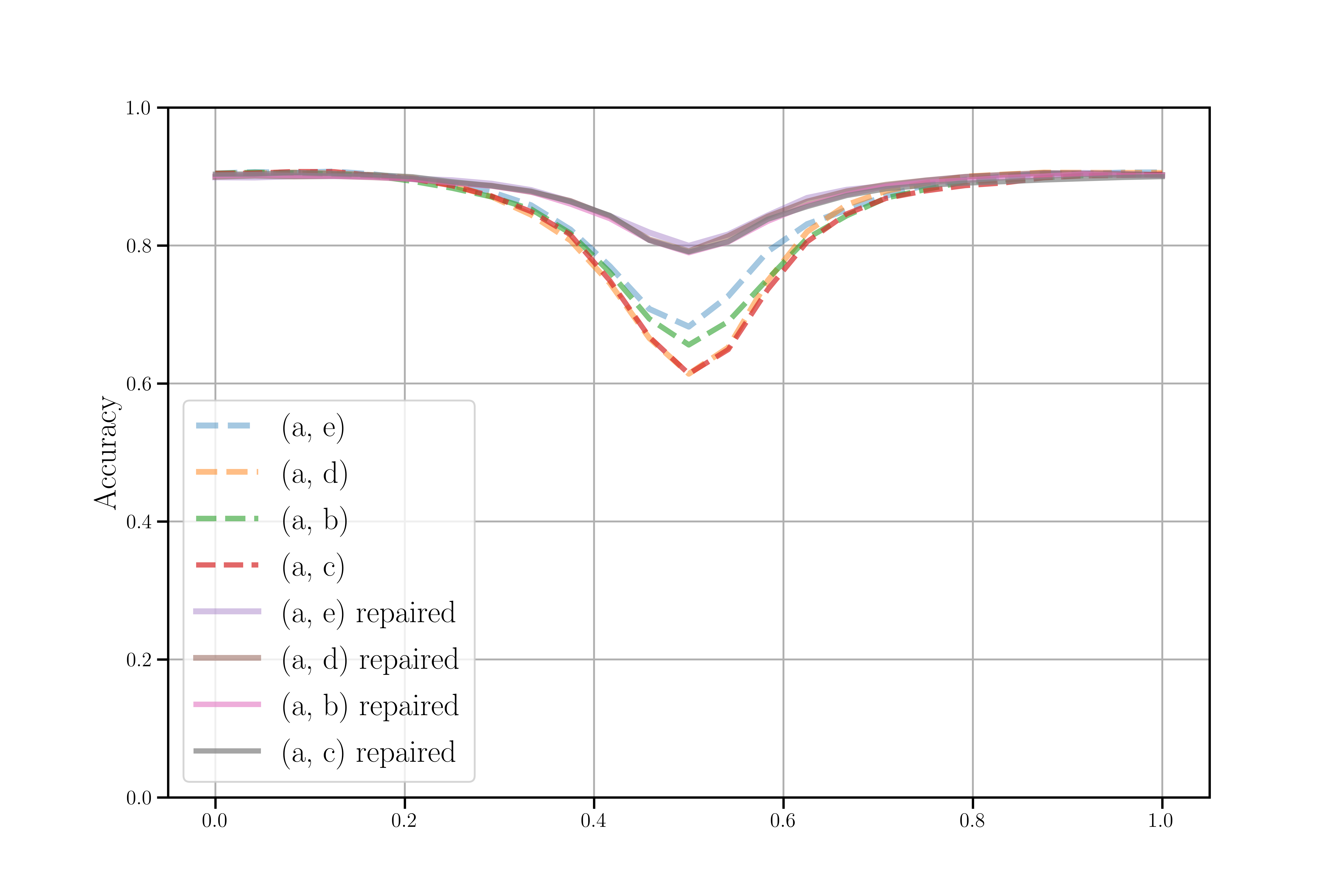}
    \caption{Interpolation curves of VGG models in the universe.}\label{fig:repair_vgg_universe}
\end{wrapfigure}%
Our empirical evidence also points out the benefits of the \texttt{REPAIR} operation~\cite{repair} that is performed after the merging.
In fact, the detrimental effect of model averaging on the activation statistics~\cite{repair} still applies when taking the mean of $n$ models instead of two. Our results clearly show the benefit of \texttt{REPAIR}, making it a key ingredient of our overall framework. Requiring meaningful interpolation endpoints to be effective, \texttt{REPAIR} has lower benefit when employed on the \texttt{MergeMany} algorithm of \citet{git-rebasin}. In fact, iteratively taking means of different random model subsets and aligning the left-out models to the mean is a more complex process than interpolating between some endpoint models. By taking the mean of models in the universe space, we are instead effectively interpolating between endpoint models that can be used for the computation of the statistics in \cref{eq:repair}. \Cref{fig:repair_vgg_universe} shows the benefit of using the repair operation on 5 \texttt{VGG} models trained on \texttt{CIFAR10} mapped to the universe space. Specifically, we fix one model ``a'' and we linearly interpolate in the universe space with respect to the other models, measuring accuracy before and after applying \texttt{REPAIR}. Beyond boosting performance, we observe that the latter reduces the variance over interpolation paths, resulting in the interpolation curves of all the models overlapping. Using the models in the universe as meaningful endpoints to gather activation statistics, the approach can fully leverage activation renormalization techniques such as \texttt{REPAIR}.

\input{4_Experiments/merge_many.tex}

\paragraph{Increasing \texorpdfstring{$n$}{n}}
In this experiment, we show how the merged model behaves when increasing the number of aggregated models. As we can see in \cref{fig:acc_vs_n_MLP}, increasing the number of MLPs up to $20$ causes the performance to slightly deteriorate in a relative sense, but remains stable in an absolute sense, as it does not fall below $98\%$.
More surprisingly, \cref{fig:acc_vs_n_resnet4x} shows that for a \texttt{ResNet20} architecture with $4\times$ width the loss and accuracy are not monotonic, but rather they seem to slightly fluctuate. This may hint at the merging process being more influenced by the composition of the model set than by its cardinality. Intuitively, a model that is difficult to match with the others will induce a harder optimization problem, possibly resulting in a worse merged model. We dive deeper into the effect of the composition of the set of models in \cref{app:subsets}. The approach is effective in merging a larger number of models, suggesting promise in federated settings.

\begin{figure}
    \begin{subfigure}{0.48\textwidth}
        \centering
        \includegraphics[width=\textwidth]{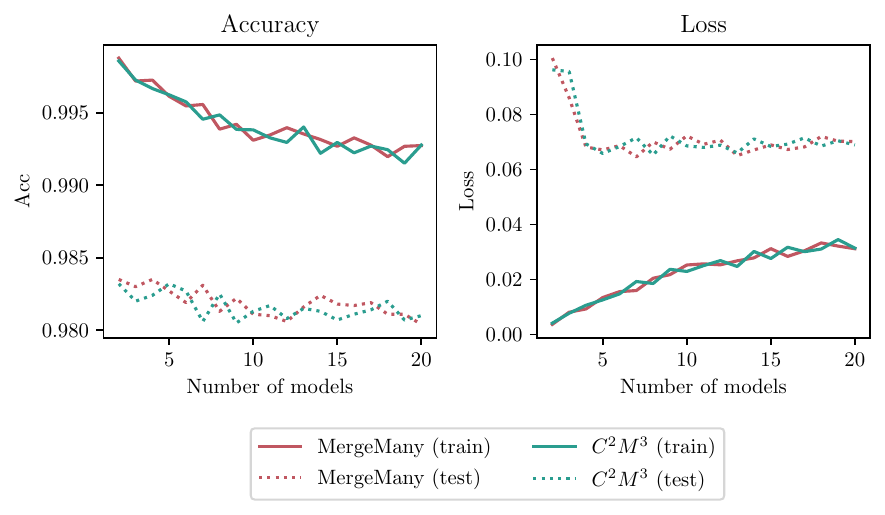}
        \caption{MLPs trained over \texttt{MNIST}.}\label{fig:acc_vs_n_MLP}
    \end{subfigure}
    \hfill 
    \begin{subfigure}{0.48\textwidth}
        \centering
        \includegraphics[width=\textwidth]{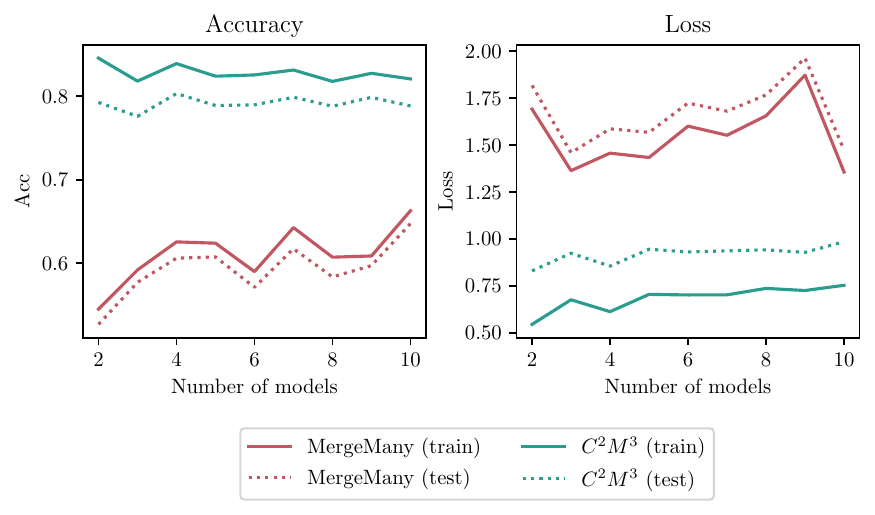}
        \caption{\texttt{ResNet20} models trained over \texttt{CIFAR10}.}\label{fig:acc_vs_n_resnet4x}
    \end{subfigure}
    \caption[Accuracy and loss vs.\ number of merged models]{Accuracy and loss when increasing the number $n$ of models to match and merge.}
    \label{fig:acc_vs_n_models} 
\end{figure}

\begin{wrapfigure}[12]{r}{0.5\textwidth}
    \vspace{-0.4cm}
    \centering
    \includegraphics[width=0.5\textwidth]{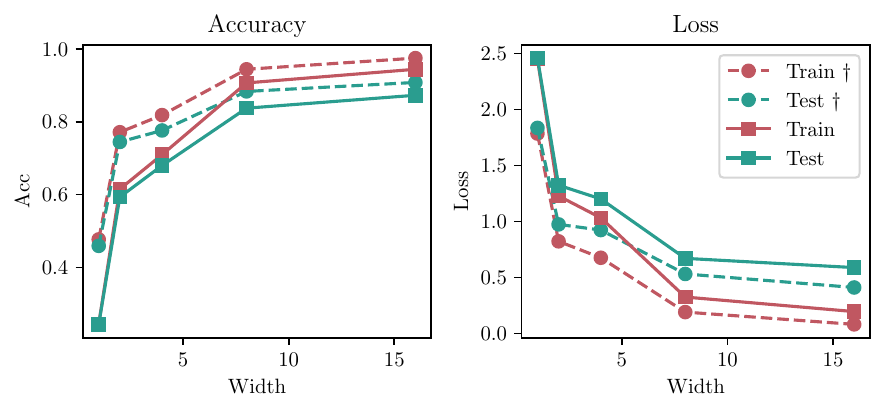}
    \caption[Effect of architecture width on model merging]{Accuracy and loss when merging $3$ \texttt{ResNet20}s trained over \texttt{CIFAR10} with different widths. $\dagger$ indicates models after applying \texttt{REPAIR}.}
    \label{fig:resnet-widths}
\end{wrapfigure}
\paragraph{Varying widths}
We now measure how architectural width affects model merging, taking into consideration \texttt{ResNet20} architectures with width $W \in \{1, 2, 4, 8, 16\}$. As we can see in \cref{fig:resnet-widths}, \emph{width greatly increases the performance of the merged model}, reaching the zero-loss barrier first observed in~\cite{git-rebasin} when $W=16$. This is in line with the observations relating linear mode connectivity and network widths~\cite{Entezari2021-me, git-rebasin}, and confirms the intuition that the merging is only effective when modes \emph{can} be linearly connected. 

\paragraph{Alternative: fixing one model as universe} 
\begin{wrapfigure}[15]{l}{0.45\textwidth}
    \centering
    \includegraphics[width=0.45\textwidth]{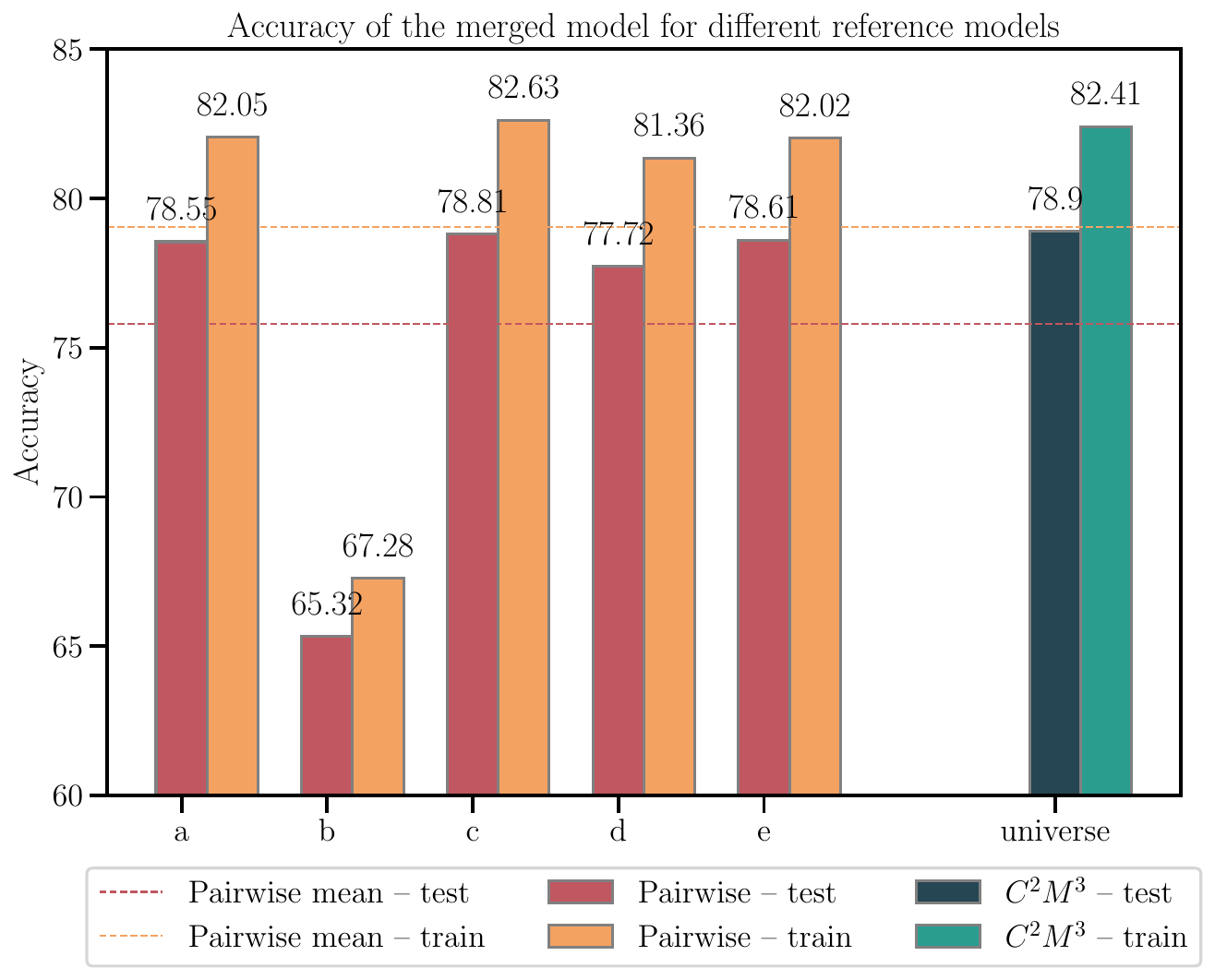}
    \caption[Merging accuracy: fixed reference model vs.\ universe space]{Accuracy of the merged model when mapping towards one arbitrary model (a, b, c, d, e) versus using our universe space.}
    \label{fig:pairwise_to_reference_barplot}
\end{wrapfigure}
Alternatively, one could achieve cycle consistency by using one of the source models as a reference and learning pairwise maps towards this one. 
This, however, would require arbitrarily selecting one of the models, making the overall merging dependent on an arbitrary choice.
To see why this matters, we merged $5$ \texttt{ResNet20-}$4\times$ by choosing one model as reference and aggregating the models in its basin.
\Cref{fig:pairwise_to_reference_barplot} shows severe oscillations in the results, with one model reaching an accuracy as low as 65\%, while our approach performs as well as the best possible reference. This approach, moreover, does not address multi-model merging, as it is intrinsically pairwise: in a multi-task setting, models optimally mapped to a reference basin would only be able to solve the task solved by the reference model. This would prevent merging from being used for models containing complementary information, such as knowledge fusion~\cite{jin2022dataless} or multi-task merging~\cite{zip-it}. In our setting, instead, the universe model must by design be a function of all the models and act as a midpoint, hence aggregating information from all the models.

\paragraph{Linear mode connectivity in the universe} \label{app:linear-mode-connectivity}
\Cref{fig:losses-universe} shows that the loss curves of models interpolated in the universe are much lower than those interpolated in the original space, suggesting that the models are more connected in the former. These results, together with the loss landscape observed in \cref{fig:cifar_loss_contour},  encourage merging the models in the universe space due to the lower loss barrier.

\begin{figure}
    \begin{subfigure}{0.48\textwidth}
        \centering
        \includegraphics[width=\textwidth]{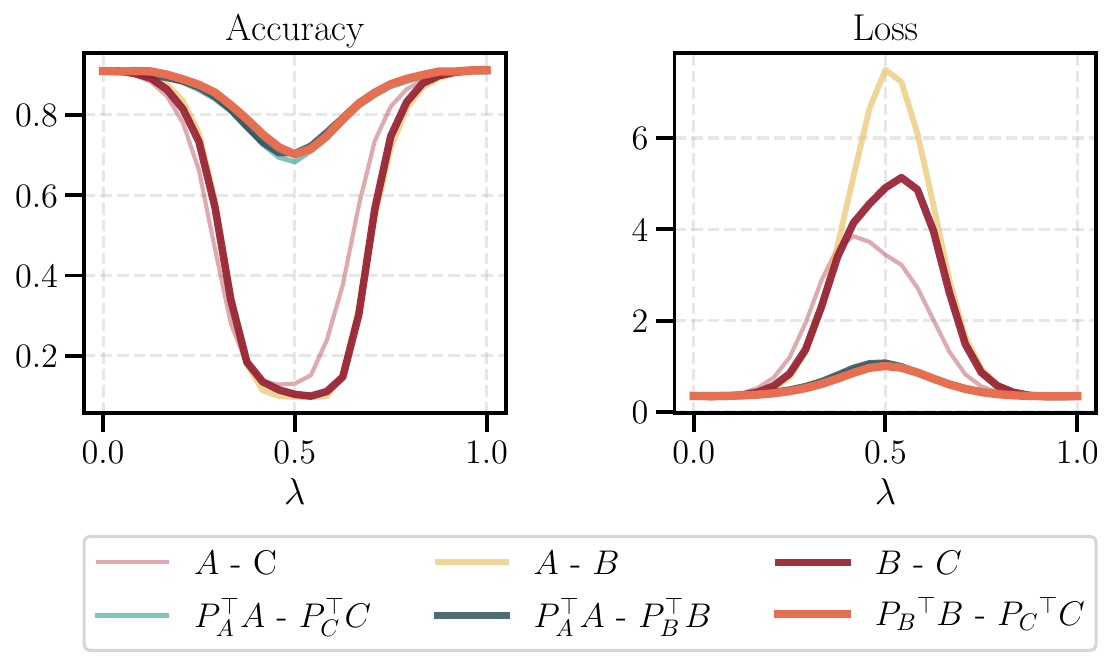}
        \caption{2D visualization of accuracy and loss of the models sampled from the pairwise interpolation lines.}\label{fig:lmc_curves_in_universe}
    \end{subfigure}
    \hfill
    \begin{subfigure}{0.48\textwidth}
        \centering
        \begin{overpic}[width=\textwidth]{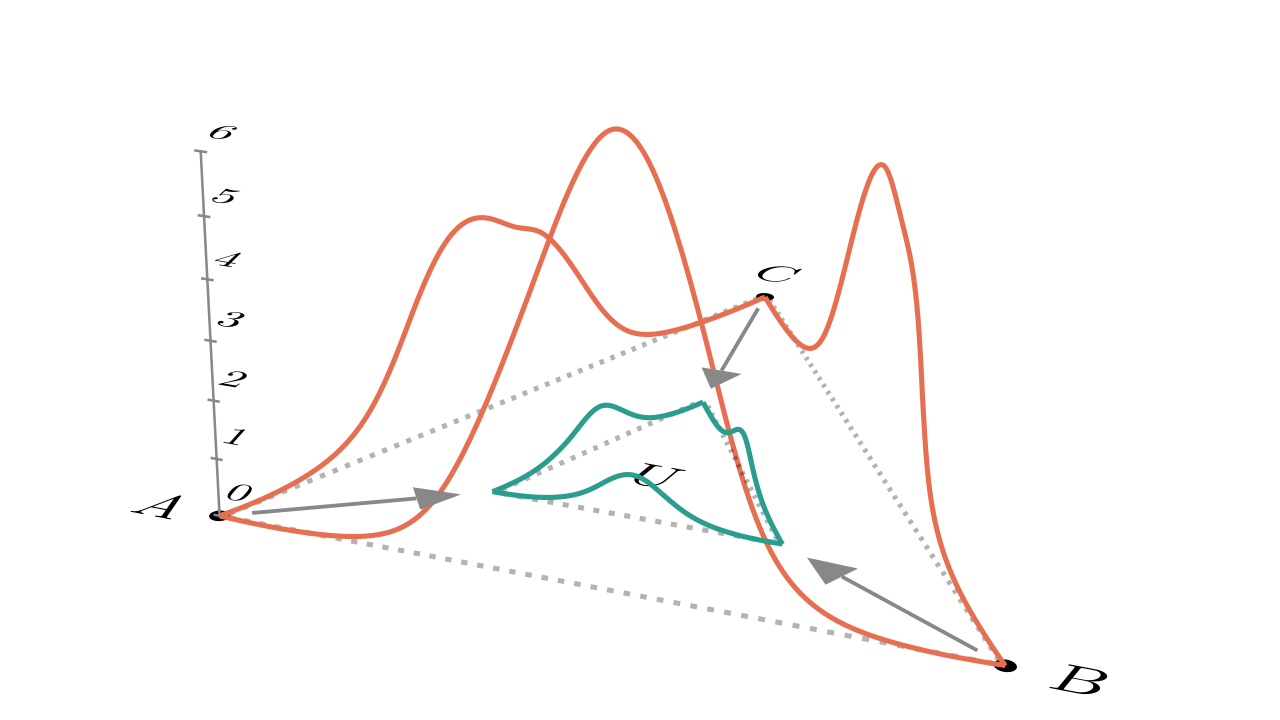}
        \end{overpic}
        \caption{3D visualization of the loss of the models sampled from the pairwise interpolation lines.}
        \label{fig:3D-losses-universe}
    \end{subfigure}
    \caption[Linear mode connectivity before and after universe mapping]{Linear mode connectivity before and after mapping to the universe for $3$ \texttt{ResNet20-2$\times$} models trained over \texttt{CIFAR10} according to \cref{alg:frank-wolfe-generalized}.} \label{fig:losses-universe}
    \vspace{-0.15cm}
\end{figure}

%% file: 3_single_task/1_cycle_consistent/4_Experiments/merge_many.tex

\begin{table}
    \begin{center}
        \resizebox{0.95\textwidth}{!}{%
            \begin{tabular}{clcccccccccccccc}
                \toprule
                \multirow{4}{*}{\textbf{Matcher}} &                                                                                 & \multicolumn{4}{c}{\texttt{EMNIST}}                &                                                  & \multicolumn{4}{c}{\texttt{CIFAR10}} &                                                    & \multicolumn{4}{c}{\texttt{CIFAR100}}                                                                                                                                                                                                                                                                                                                                                                                            \\
                \cmidrule{3-6}                                                                                   \cmidrule{8-11}                                                                                                      \cmidrule{13-16}

                                                  &                                                                                 & \multicolumn{2}{c}{\textbf{Accuracy} $(\uparrow)$} & \multicolumn{2}{c}{\textbf{Loss} ($\downarrow$)} &                                      & \multicolumn{2}{c}{\textbf{Accuracy} ($\uparrow$)} & \multicolumn{2}{c}{\textbf{Loss} ($\downarrow$)}                                                           &               & \multicolumn{2}{c}{\textbf{Accuracy} ($\uparrow$)} & \multicolumn{2}{c}{\textbf{Loss} ($\downarrow$)}                                                                                                                                                                                               \\
                \cmidrule(lr){3-4}                                                                                  \cmidrule(lr){5-6}                                                                                                      \cmidrule(lr){8-9}                                                                                                      \cmidrule(lr){10-11}                                                                                                      \cmidrule(lr){13-14}                                                                                                      \cmidrule{15-16}
                                                  &                                                                                 & train                                              & test                                             & train                                & test                                               &                                                                                                            & train         & test                                               & train                                            & test          &                                                                                                             & train         & test          & train         & test          \\
                \midrule
                \texttt{Naive}                    & \parbox[t]{4mm}{\multirow{5}{*}{ \rotatebox[origin=c]{90}{\texttt{MLP} }}}      & 0.03                                               & 0.03                                             & 3.28                                 & 3.28                                               & \parbox[t]{4mm}{\multirow{5}{*}{ \rotatebox[origin=c]{90}{$\underset{2\times}{\text{\texttt{ResNet}}}$} }} & 0.10          & 0.10                                               & 3.07                                             & 3.07          & \parbox[t]{4mm}{\multirow{5}{*}{ \rotatebox[origin=c]{90}{$\underset{4\times}{\text{\texttt{ResNet}}}$} }}  & 0.01          & 0.01          & 5.30          & 5.30          \\
                \texttt{MergeMany}                &                                                                                 & 0.88                                               & 0.86                                             & 1.11                                 & 1.13                                               &                                                                                                            & 0.38          & 0.37                                               & 2.08                                             & 2.06          &                                                                                                             & 0.31          & 0.28          & 3.01          & 2.76          \\
                \texttt{MergeMany}$^\dagger$      &                                                                                 & 0.88                                               & 0.86                                             & 1.11                                 & 1.13                                               &                                                                                                            & 0.50          & 0.50                                               & 2.34                                             & 2.30          &                                                                                                             & 0.24          & 0.22          & 3.31          & 3.12          \\
                \texttt{$C^2M^3$}                 &                                                                                 & 0.89                                               & 0.87                                             & 1.07                                 & 1.10                                               &                                                                                                            & 0.42          & 0.40                                               & 2.11                                             & 2.05          &                                                                                                             & 0.34          & 0.30          & 2.94          & 2.63          \\
                \texttt{$C^2M^3$}$^\dagger$       &                                                                                 & \textbf{0.89}                                      & \textbf{0.87}                                    & \textbf{1.07}                        & \textbf{1.10}                                      &                                                                                                            & \textbf{0.72} & \textbf{0.69}                                      & \textbf{1.26}                                    & \textbf{1.12} &                                                                                                             & \textbf{0.53} & \textbf{0.46} & \textbf{2.13} & \textbf{1.67} \\
                \midrule
                \texttt{Naive}                    & \parbox[t]{4mm}{\multirow{5}{*}{ \rotatebox[origin=c]{90}{$\underset{2\times}{\text{\texttt{ResNet}}}$} }} & 0.04                                              & 0.04                                            & 4.04                                & 4.04                                              & \parbox[t]{4mm}{\multirow{5}{*}{ \rotatebox[origin=c]{90}{\text{\texttt{VGG16}}}}}                         & 0.10          & 0.10                                               & 2.31                                             & 2.31          & \parbox[t]{4mm}{\multirow{5}{*}{ \rotatebox[origin=c]{90}{$\underset{16\times}{\text{\texttt{ResNet}}}$} }} & 0.01          & 0.01          & 6.22          & 6.22          \\
                \texttt{MergeMany}                &                                                                                 & 0.03                                              & 0.03                                            & 7.17                                & 7.18                                              &                                                                                                            & 0.10          & 0.10                                               & 2.36                                             & 2.36          &                                                                                                             & 0.45          & 0.38          & 2.32          & 3.06          \\
                \texttt{MergeMany}$^\dagger$      &                                                                                 & 0.03                                              & 0.03                                            & 4.74                                & 4.72                                              &                                                                                                            & 0.60          & 0.57                                               & 1.43                                             & 1.32          &                                                                                                             & 0.41          & 0.35          & 2.27          & 2.68          \\
                \texttt{$C^2M^3$}                 &                                                                                 & 0.27                                              & 0.27                                            & 3.43                                & 3.47                                              &                                                                                                            & 0.11          & 0.11                                               & 2.34                                             & 2.34          &                                                                                                             & 0.46          & 0.39          & 2.25          & 3.03          \\
                \texttt{$C^2M^3$}$^\dagger$       &                                                                                 & \textbf{0.60}                                              & \textbf{0.60}                                            & \textbf{1.32}                                & \textbf{1.34}                                              &                                                                                                            & \textbf{0.64} & \textbf{0.62}                                      & \textbf{1.34}                                    & \textbf{1.23} &                                                                                                             & \textbf{0.60} & \textbf{0.49} & \textbf{1.43} & \textbf{2.23} \\
                \bottomrule
            \end{tabular}
        }
    \end{center}
    \caption[Accuracy when merging five models with different initializations]{Accuracy of the merged model when merging $5$ models trained with different initializations. The best results are highlighted in bold. $^\dagger$ denotes models after the \texttt{REPAIR} operation.}
    \label{tab:merge-many}
\end{table}

%% file: 3_single_task/1_cycle_consistent/6_Conclusions/content.tex
The cycle-consistent framework developed in this chapter solves the multi-model matching and merging problem through a single, principled abstraction: factorizing all pairwise permutations through a shared universe space. The Frank-Wolfe-based optimization jointly handles all layers, guaranteeing cycle consistency and eliminating the need for an arbitrary reference model. Across architectures and datasets, the approach consistently outperforms existing multi-model merging methods, and the formalism elegantly satisfies the requirement for the merging operation to unify different models into a cohesive whole, rather than mapping all of them to one arbitrarily chosen member of the set.

The framework is also general beyond the permutation group: by relaxing permutations to orthogonal matrices, the same universe-factorization principle can be applied to align latent representations across models, as explored in~\citet{achara2026multiwayrepresentationalignment}.

The methods developed in \cref{part:setting-one} have addressed the single-task setting, where the central challenge is geometric: aligning independently trained models to overcome permutation symmetries and reduce loss barriers. The approach succeeds precisely because the models share the same objective and differ only in how their neurons are organized. In \cref{part:setting-two}, we turn to a fundamentally different and more practically impactful regime: merging models that have been fine-tuned on \emph{distinct tasks} from a shared pre-trained backbone. Here the challenge is no longer one of geometric alignment (the shared initialization already makes models compatible) but rather one of \emph{task interference}: understanding why aggregating task-specific weight updates works, and how to do so with minimal loss of per-task performance.

%% file: 4_multi_task/0_content.tex

\chapter{Foundations of Multi-Task Merging}\label{ch:multi-task-background}

\subimport{0_background/}{content.tex}
\chapter{Task Vectors as Approximate Gradients}\label{ch:gradient}
\subimport{1_task_vectors_gradients/}{content.tex}
\chapter{Exploiting the Low-Rank Structure of Task Vectors}\label{ch:tsv}
\subimport{2_task_singular_vectors/}{content.tex}
\chapter{Input-Adaptive Merging via Subspace Routing}\label{ch:mass}
\subimport{3_mass/}{content.tex}

\chapter{Democratizing Evolutionary Merging}\label{ch:merge3}
\subimport{5_merge3/}{content.tex}

%% file: 4_multi_task/1_task_vectors_gradients/content.tex
\begin{chapteroverview}
In this chapter, we provide a theoretical foundation for task arithmetic. We show that, under full-batch gradient descent, a single-epoch task vector is exactly the scaled negative gradient of the task loss, so that summing task vectors is equivalent to one step of joint multi-task training. Beyond one epoch, the two approaches diverge by a curvature-controlled $O(\eta^2)$ term, which we bound explicitly for feed-forward networks.
\end{chapteroverview}

\section{Motivation and Overview}\label{sec:gradient-intro}

\input{1_Introduction/content}
\section{Formalizing the task vector--gradient equivalence}\label{sec:tv-as-gradients}
\input{2_TVs_as_Gradients/content.tex}

\section{Scope and limitations}\label{sec:gradient-limitations}
The analysis above assumes full-batch gradient descent (GD), whereas in practice task vectors are produced by stochastic optimizers (SGD, Adam, AdamW). The noise, momentum, and adaptive step sizes of these methods break the algebraic cancellations on which the proof relies, and extending the results to stochastic settings remains open. Likewise, the closed-form bound on $C$ applies to feed-forward networks; CNNs and Transformers introduce architectural elements (weight sharing, attention) that the current framework does not cover. A formal treatment under SGD and extensions to modern architectures are natural next steps; we discuss them further in \cref{part:future-directions}.

\section{Implications for merging practice}\label{sec:gradient-discussion}
\input{5_Discussion/content}

\section{Summary and outlook}\label{sec:gradient-conclusions}

\input{6_Conclusions/content}

%% file: 4_multi_task/1_task_vectors_gradients/2_TVs_as_Gradients/content.tex
This section formalizes the relationship between task vectors and multi-task loss gradients.
For analytical clarity, we model learning as \emph{full-batch} gradient descent with a fixed step size $\eta$, an assumption common in theoretical studies of convergence and generalization \citep{arora2018convergence, nikolakakisbeyond} that simplifies the derivations while preserving key insights into the optimization dynamics.

Recall that vanilla task arithmetic (TA) \citep{task-vectors} constructs a multi-task model by (i) fine-tuning the base network separately on each task, and (ii) adding the resulting task vectors, scaled by a coefficient $\alpha$, to the pre-trained weights.
In the full-batch setting, we show that TA is \emph{exactly} equivalent to one epoch of joint training; thereafter, the two approaches diverge by a curvature-controlled $O(\eta^{2})$ term.

We first introduce the notation used throughout this chapter.

\paragraph{Notation.}
Let $T$ denote the set of tasks, with $|T|$ its cardinality. The pre-trained weights are $\theta_{\text{pre}}$. For a task $t \in T$, let $\theta_t^{(k)}$ be the parameters after $k$ epochs of fine-tuning on $t$, and define the \emph{task vector} $\tau_t^{(k)} := \theta_t^{(k)} - \theta_{\text{pre}}$, i.e., the parameter displacement induced by fine-tuning \citep{task-vectors}. Fine-tuning on task $t$ minimizes the empirical loss
\[
\overline{L}_t(\theta) := \frac{1}{n_t} \sum_{i=1}^{n_t} \ell(x_i, y_i, \theta),
\]
where $n_t$ is the number of training samples for task $t$ and $\ell$ is the per-sample loss. We write $\theta^{(k)}_{\text{MT}}$ for the parameters after $k$ epochs of \emph{joint} training on the combined loss $\sum_{t \in T} \overline{L}_t(\theta)$, and define the \emph{multi-task vector} $\tau^{(k)}_{\text{MT}} := \sum_{t \in T} \tau_t^{(k)}$.

\begin{theorem}
    \label{main_thm}
    Let $ \theta_{\text{TA}}^{(k)} = \theta_{\text{pre}} + {\alpha} \sum_{t \in T} \tau_t^{(k)}$ be the model obtained using vanilla task arithmetic with parameter $\alpha$.
    Let $\{\theta_t^{(k)}\}_{t\in T}$ be produced by running $k$ full-batch GD epochs with step size $\eta$ on each task, and let $\theta_{\text{MT}}^{(k)}$ be the model obtained from $k$ GD epochs with step size $\alpha\eta$ on the aggregated loss $\sum_{t\in T}\overline L_t$.
   Then
    \begin{align}
        \label{eq:thm_equivalence_first_epoch}
         \theta_{\text{TA}}^{(1)} &=  \theta_{\text{MT}}^{(1)},                                                                                                       \\
        \label{eq:thm_equivalence_other_epochs}
         \theta_{\text{TA}}^{(k)} &=  \theta_{\text{MT}}^{(k)}  + \eta^2 \, C\bigl(\{\theta_{\text{MT}}^{(j)}\}_{j=1}^{k-2}\bigr) + O(\eta^3), \qquad k>1,
    \end{align}
    where the curvature term $C$ and the per-task residual $r_t$ are defined as
    \begin{equation} \label{eq:term_c}
    C\bigl(\{\theta_{\text{MT}}^{(j)}\}_{j=1}^{h}\bigr)
    = \sum_{t\in T} \sum_{e=0}^{h}
      \nabla^{2}\overline{L}_{t}\!\bigl(\theta^{(e)}_{\text{MT}}\bigr)\!
      \sum_{m=0}^{e} r_t \!\bigl(\theta^{(m)}_{\text{MT}}\bigr),
    \end{equation}
    \begin{equation} \label{eq:term_r}
    r_t (\theta)
    :=\alpha\!\!\sum_{\substack{t' \neq t \\ t' \in T}}
        \nabla\overline{L}_{t'}(\theta)
       +(\alpha-1)\,
        \nabla\overline{L}_{t}(\theta).
    \end{equation}
\end{theorem}

\begin{remark}
Theorem~\ref{main_thm} assumes full-batch gradient descent with a fixed step size. Extending it to mini-batch SGD with momentum, Adam, or AdamW remains open: stochastic noise, adaptive step sizes, and momentum break the algebraic cancellations exploited in the proof. Nonetheless, the empirical evidence in this chapter suggests that the qualitative insights (first-epoch dominance, curvature-controlled divergence) carry over to practical training regimes.
\end{remark}

Equation~\eqref{eq:thm_equivalence_first_epoch} states that after a single epoch, TA matches multi-task training \emph{exactly}: at $k{=}1$ the task vectors are precisely the negative gradients of the individual task losses, and summing them recovers the joint gradient step with no approximation error (\cref{fig:teaser}).

For $k > 1$ the exact equivalence breaks down. The leading $O(\eta)$ terms cancel, and the gap reduces to a curvature-controlled residual $\eta^2 C(\cdot)$ plus higher-order corrections $O(\eta^3)$. The term $C(\cdot)$ accumulates Hessian--gradient interactions along the joint training trajectory: TA and multi-task training diverge whenever the loss landscape curves differently across tasks. This perturbative view pinpoints the source of the gap and shows that the approximation degrades \emph{gradually} with training, rather than failing abruptly after the first epoch.

\subsection{Dominance of the first epoch}
\begin{wrapfigure}[11]{r}{0.33\textwidth}
    \centering
    \vspace{-3em}
    \includegraphics[width=\linewidth]{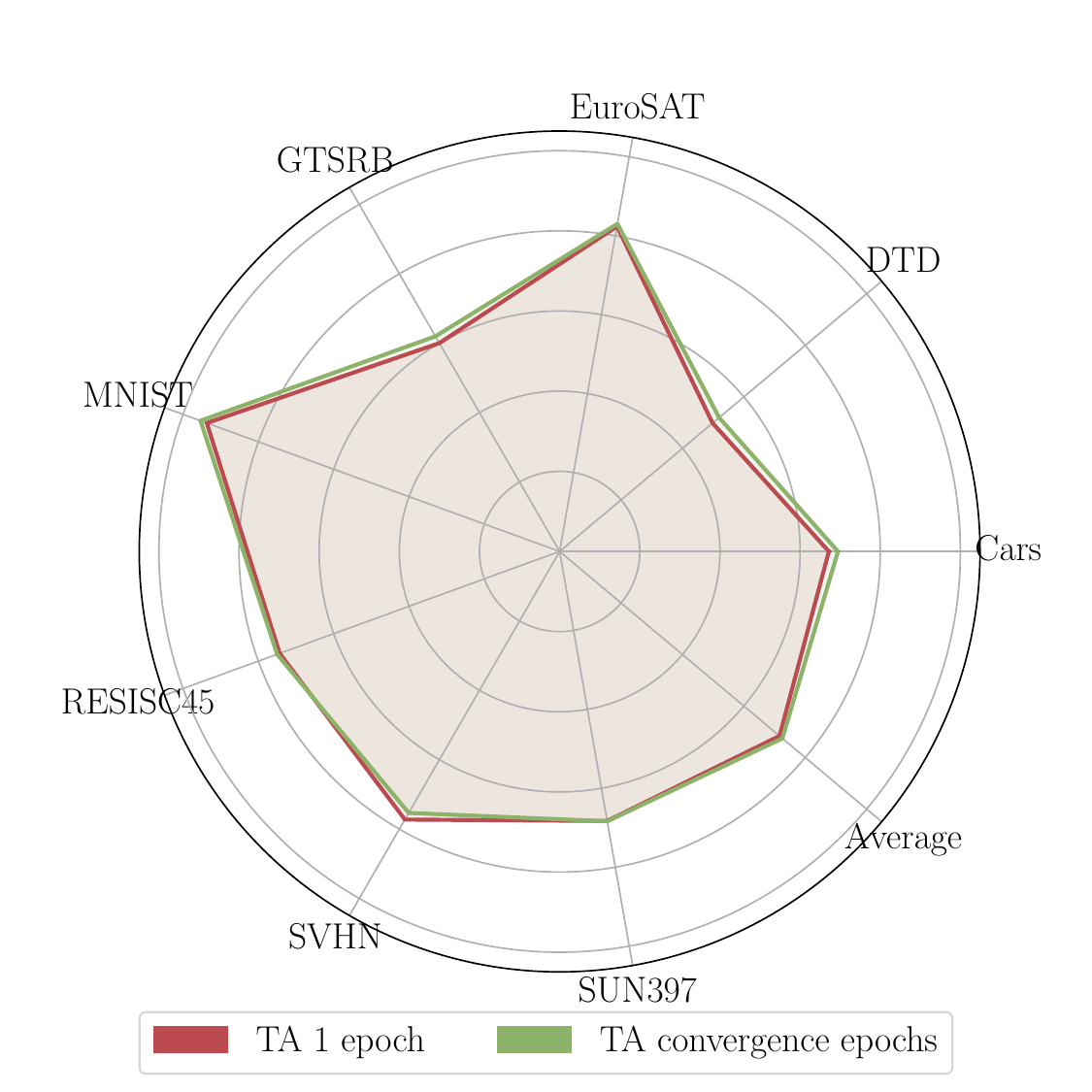}
    \caption{Task arithmetic accuracy: $1$ epoch vs. converged.}
    \label{fig:ta-1-epoch-vs-ta-convergence}
\end{wrapfigure}
Even beyond the first epoch, task vectors remain effective because much of the fine-tuning trajectory is dictated by the initial gradient. We verify this with a controlled experiment.

We compare TA performance when merging models fine-tuned to convergence versus after a single epoch (\cref{fig:ta-1-epoch-vs-ta-convergence}). Remarkably, the single-epoch variant performs competitively across all tasks, consistent with our theoretical finding that early task vectors are the most faithful gradient approximations.

To understand why, we examine how much each epoch contributes to the overall optimization. \Cref{fig:normalized-gradient-norms} plots the epoch-wise normalized gradient norm $\frac{\|\nabla_{\theta}^{(k)} L \|}{\sum_{k^\prime=1}^{K} \|\nabla_{\theta}^{(k^\prime)} L \|}$: the first epoch contributes the largest share across all tasks. Even on datasets where this dominance is less pronounced (e.g., Cars, SVHN), the cumulative gradient through $k$ epochs remains well-aligned with the first-epoch gradient: \cref{fig:cosine-sim-gradients} shows that $\cos(\tau_1, \tau_k)$ exceeds $0.8$ at epoch~2 for all tasks and remains above $0.6$ through epoch~10, decaying gradually as training progresses. Fully converged task vectors favor task-specific performance, while first-epoch vectors are better gradient surrogates; both yield similar multi-task accuracy after merging, but the single-epoch setting is considerably cheaper.

\begin{figure}[htbp]
  \centering
  \begin{subfigure}[t]{0.58\textwidth}
    \centering
    \includegraphics[width=\linewidth]{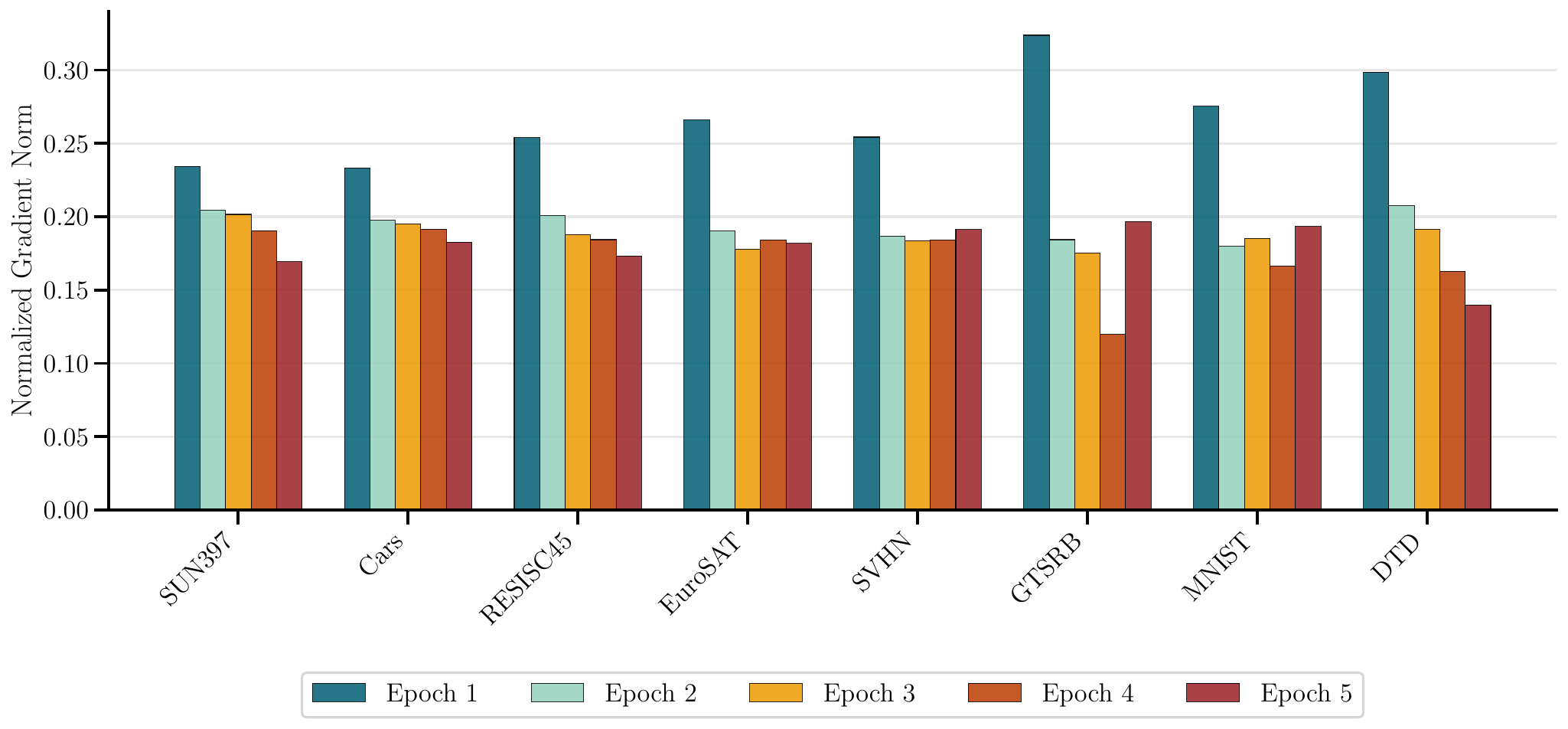}
    \vspace{-1em}
    \caption{Normalized gradient norms for the first 5 epochs of fine-tuning.}
    \label{fig:normalized-gradient-norms}
  \end{subfigure}
  \hfill
  \begin{subfigure}[t]{0.4\textwidth}
    \centering
    \includegraphics[width=\linewidth]{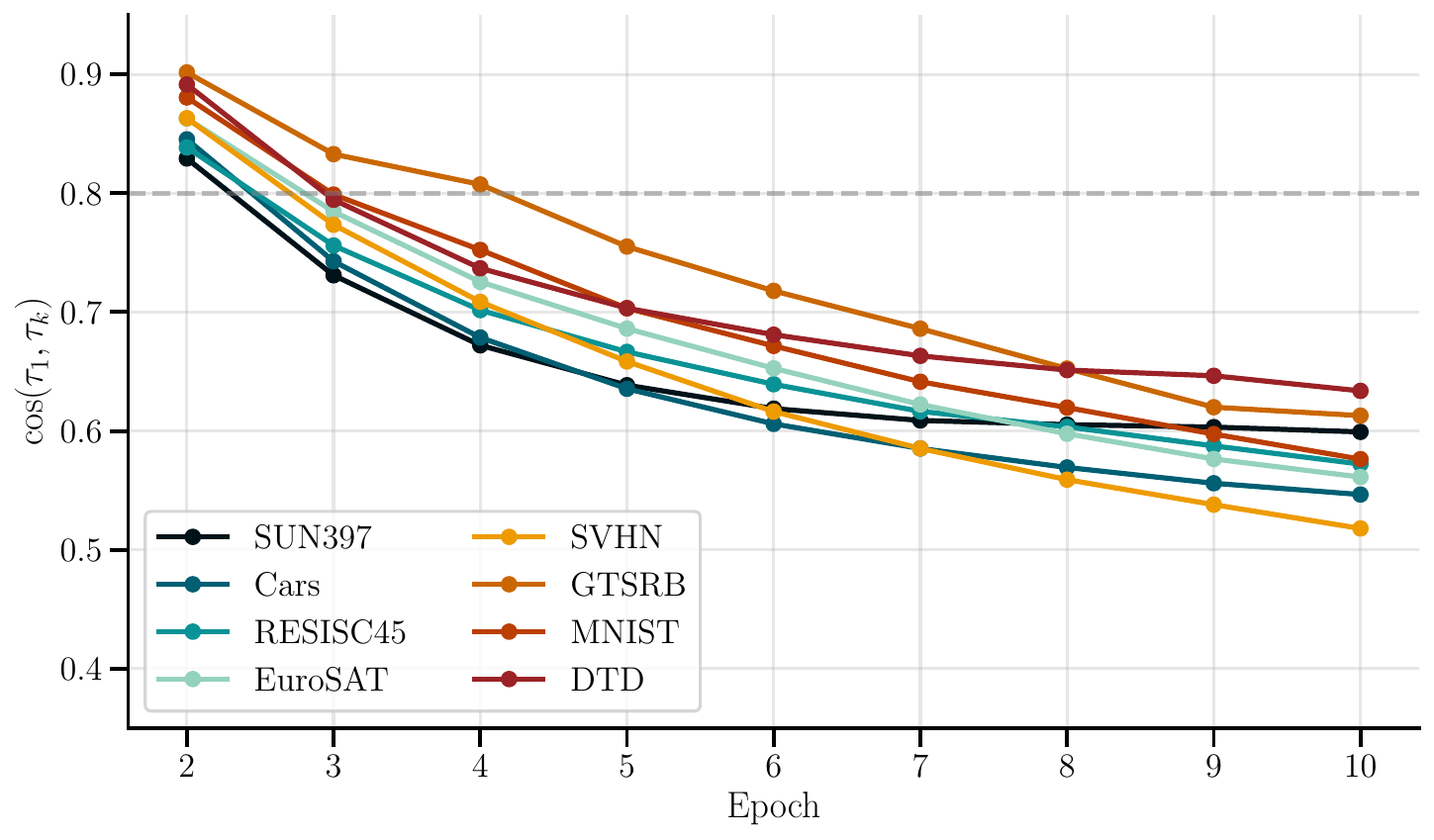}
    \vspace{-1em}
    \caption{Cosine similarity of cumulative gradients with respect to the first-epoch gradient.}
    \label{fig:cosine-sim-gradients}
  \end{subfigure}
  \caption{Analysis of first-epoch gradients.}
\end{figure}

Our formal results assume GD, but treating adaptive optimizers as approximate GD preserves the core intuition: reducing task interference in the merged model corresponds to minimizing gradient conflicts in multi-task learning.

\subsection{Proof sketch}
The full derivation is in \cref{sec:proofs}; we outline the key ideas here.
The technical backbone is the following proposition, which gives a term-by-term Taylor expansion of the multi-task vector; Theorem~\ref{main_thm} follows as a corollary.

\begin{proposition}\label{thm:multitask-vector-gradient} 
   Let $\{\theta_t^{(k)}\}_{t\in T}$ be produced by fine-tuning the base model for $k$ full-batch GD epochs with step size $\eta$ for each task, and let $\theta_{\text{MT}}^{(k)}$ be the model obtained from $k$ GD epochs with step size $\alpha\eta$.
   It holds that
    \begin{align}
         & \tau^{(1)}_{\text{MT}} = -\eta \nabla \sum_{t \in T} \overline{L}_t(\theta_{\text{pre}})                   \\
         & \tau^{(k)}_{\text{MT}} =  - \eta \sum_{t \in T}  \sum_{j=0}^{k-1}   \nabla \overline{L}_t(\theta_{\text{MT}}^{(j)})  + \eta^2 C(\{\theta_{\text{MT}}^{(j)}\}_{j=1}^{k-2}) + O(\eta^3) \; \; \; \forall k \ge 2
    \end{align}
\end{proposition}

\begin{remark}\label{prop:task-vectors-gradients}
    In the single-task case, Proposition~\ref{thm:multitask-vector-gradient} immediately gives $\tau_t = - \eta \nabla \overline{L}_t (\theta_{\text{pre}})$ after one epoch: the task vector is exactly the scaled negative gradient, and adding it to the pre-trained model is equivalent to a single gradient descent step on the task loss.
\end{remark}

The proof compares multi-task training and vanilla task arithmetic, showing that they differ only by a curvature-controlled second-order term. It proceeds in two stages.

\paragraph{Stage \texorpdfstring{$1$}{1}: Decomposing Single-Task Error}
TA fine-tunes each task separately and then adds the resulting task vectors $\tau_t^{(k)}=\theta_t^{(k)}-\theta_{\text{pre}}$.
The key question is how far each single-task iterate $\theta_t^{(k)}$ drifts from the multi-task iterate $\theta_{\text{MT}}^{(k)}$.
Lemma~\ref{lemma:induction_update} shows that this distance decomposes into (i) a linear drift term, (ii) a quadratic curvature correction, and (iii) higher-order remainders negligible for small $\eta$.

The lemma is expressed in terms of three auxiliary quantities $r_t, p_t^k$, and $s_t^k$, each capturing a distinct effect when comparing single-task and multi-task trajectories:

\begin{itemize}
    \item $r_t(\cdot)$ quantifies the mismatch between the two gradients at one step. 
    \begin{align}
        r_t (\theta) &=  \alpha  \sum_{\substack{t' \neq t \\ t' \in T}} \nabla \overline{L}_{t'}(\theta) + (\alpha -1 )  \nabla \overline{L}_t(\theta) = \alpha \sum_{t' \in T} \nabla \overline{L}_{t'}(\theta) -  \nabla \overline{L}_t(\theta)
    \end{align}
    \item $p_t^k(\cdot)$ represents the accumulation of those mismatches over many steps.
    \begin{align}
        &p_t^k (\theta_{\text{pre}}, \theta^{(1)}_{\text{MT}}, \dots , \theta^{(k)}_{\text{MT}}) =   \sum_{j=0}^k  r_t (\theta^{(j)}_{\text{MT}})
    \end{align}
    \item $s_t^k(\cdot)$ accounts for the extra bending introduced by curvature once the paths have parted. 
    \begin{equation}
        s^{k}_t(\theta_{\text{pre}}, \dots, \theta^{(k)}_{\text{MT}})  = \sum_{j=0}^{k} \nabla^2 \overline{L}_t(\theta^{(j)}_{\text{MT}})[p^j_t(\theta_{\text{pre}}, \dots, \theta^{(j-1)}_{\text{MT}})].
    \end{equation}
\end{itemize}

\begin{lemma}\label{lemma:induction_update}
Using the notation of Proposition~\ref{thm:multitask-vector-gradient}, for $m = 1$:
    \begin{equation}
    \theta^{(1)}_t = \theta^{(1)}_{\emph{MT}} + \eta \, p_t^0 (\theta_{\text{pre}}),
    \end{equation}
    and for $m \ge 2$:
    \begin{equation}
    \theta^{(m+1)}_{t} = \theta^{(m+1)}_{\emph{MT}} +  \eta \, p_t^{m}\!\bigl(\theta_{\text{pre}}, \dots,\theta^{(m)}_{\emph{MT}}\bigr)
    - \eta^2 \, s^{m-1}_t\!\bigl(\theta_{\text{pre}}, \dots, \theta^{(m-1)}_{\emph{MT}}\bigr) + O(\eta^3).
    \end{equation}
\end{lemma}

The case $m=1$ is immediate; for $m\geq 2$ the proof proceeds by induction (see \cref{sec:proofs}).
We sketch the base step $m=2$: Taylor-expanding the single-task update around the multi-task iterate yields (i) a first-order linear term ($p_t^{1}$), accumulating the gradient mismatch from previous epochs, and (ii) a second-order curvature correction ($s_t^{0}$), capturing how the loss surface of task $t$ bends around the multi-task trajectory. Rearranging recovers the lemma's structure for the second epoch.

\paragraph{Stage \texorpdfstring{$2$}{2}: Error Aggregation across Tasks}
With the per-task decomposition in hand, we sum over all tasks to obtain $\theta_{\text{TA}} - \theta_{\text{MT}}$. The linear terms in $\eta$ cancel, and the quadratic curvature terms accumulate into the coefficient $C$ (\cref{eq:term_c}).

\paragraph{Epoch 1.} A direct computation shows $\theta^{(1)}_{\text{TA}} = \theta^{(1)}_{\text{MT}}$:
\[ \theta^{(1)}_{\text{TA}} = \theta_{\text{pre}} - \eta \alpha\sum_{t \in T} \nabla \overline{L}_t(\theta_{\text{pre}}) = \theta^{(1)}_{\text{MT}}, \]
where the right-hand side uses learning rate $\alpha\eta$ for the joint loss.

\paragraph{Epochs $k \geq 2$.} Applying Lemma~\ref{lemma:induction_update}, each single-task parameter decomposes into the multi-task parameter plus a first-order drift ($\eta p_t^k$) and a second-order curvature term ($-\eta^2 s_t^{k-1}$). When these are summed over all tasks to form the task arithmetic aggregate, the linear drift terms cancel by symmetry ($\sum_{t} r_t(\theta) = 0$ for $\alpha = 1/|T|$), leaving only the quadratic block $C$. This yields exactly the $O(\eta^2)$ gap in Theorem~\ref{main_thm}.

\subsection{Bounding the second-order deviation term}
The error term $C$ defined in \cref{eq:term_c} captures the second-order gap between TA and multi-task training. To characterize how large this gap can grow, we derive explicit, uniform bounds on $\|C\|_2$ under mild structural assumptions: feed-forward networks with bounded weights, bounded inputs, and activations with controlled first and second derivatives. These assumptions are standard in the neural network theory literature \citep{yarotsky2017error, bartlett2017spectrally, golowich2018size}.

\begin{theorem}[Uniform bound on the coefficient vector \(C\bigl(\{\theta_{\text{MT}}^{(j)}\}\bigr)\)]
    \label{thm:coefficient-bound}
Assume the hypothesis of Proposition~\ref{thm:multitask-vector-gradient} holds and let $C (\{\theta_{\text{MT}}^{(j)}\}_{j=1}^h)$ be the error term defined in \cref{eq:term_c}. We assume the tasks are all classification tasks optimized with cross-entropy loss. 
We also add a few structural constraints on the network itself. Specifically the model is a depth-\(L\) feed-forward network with weight matrices   \(W^{(1)},\dots ,W^{(L)}\). For every layer, there exist positive constants \(s_\ell\) such that
    \(\lVert W^{(\ell)}\rVert_2\le s_\ell\).
    The inputs are also bounded \(\lVert x\rVert_2\le M_x\). Finally, we require the activation functions to have bounded first and second derivatives, i.e., there are \(\beta_\phi,\gamma_\phi>0\) such that
    \( \sup_{z}|\phi'(z)|\le\beta_\phi
    \quad\text{and}\quad
    \sup_{z}|\phi''(z)|\le\gamma_\phi.
    \)
Setting
 \(\Pi=\prod_{\ell=1}^{L}s_{\ell}\),
we obtain the following.

\begin{enumerate}[label=(\roman*)]
\item \emph{\textbf{For general activations}}:\vspace{-0.4em}
\[
  \bigl\|C(\{\theta^{(j)}_{\text{MT}}\}_{j=1}^{h})\bigr\|_2
  \;\le\;
  T\,\binom{h+2}{2}\,
  |\alpha T-1|\,
  \textcolor{ForestGreen}{H^{\phi}_{\max}}\,\textcolor{BrickRed}{G^{\phi}_{\max}},
\]
 \[ \text{with }
  \textcolor{ForestGreen}{H^{\phi}_{\max}}\;\le\;
  2\,\gamma_\phi\,M_x^{2}\,\Pi^{2}\,\beta_\phi^{\,2L-2} \text{ and  }
  \textcolor{BrickRed}{G^{\phi}_{\max}}\;\le\;\sqrt{2}\,M_x\,\Pi\,\beta_\phi^{\,L-1}.\; 
\] 
\item \emph{\textbf{For ReLU activations}} ($\gamma_\phi=0$, $\beta_{\phi} = 1$): \vspace{-0.4em}
\[
  \bigl\|C(\{\theta^{(j)}_{\text{MT}}\}_{j=1}^{h})\bigr\|_2
  \;\le\;
  \frac{T}{2}\,\binom{h+2}{2}\,
  |\alpha T-1|\, H^{\text{ReLU}}_{max} G^{\text{ReLU}}_{max}
\]

\[ \text{with } 
\textcolor{ForestGreen}{H^{\text{ReLU}}_{\max}}\;\leq  \; \tfrac12\sqrt{2}\,M_{x}^{3}\Pi^{3}\beta_\phi^{\,3L-3} \text{ and  } \textcolor{BrickRed}{G^{\text{ReLU}}_{\max}}\;\le\;\sqrt{2}\,M_x\,\Pi\,. \] 
\end{enumerate}\vspace{-0.7em}
\end{theorem}

The bound separates into a \emph{task-dependent} factor
\(T\,\tfrac{(h+1)(h+2)}{2}\,\lvert\alpha T-1\rvert\)
and a \emph{network-dependent} factor controlled by
\(\{M_x,\Pi,\beta_\phi,\gamma_\phi\}\).
For ReLU activations the bound tightens because the piecewise-linear activation has zero second derivative almost everywhere.

The complete derivation is in \cref{sec:proofs2}; the proof proceeds in three steps:
\begin{enumerate}
  \setlength\itemsep{0.05em}
    \item[(i)] We provide a bound over the $\ell_2$ norm of $C(\{\theta_{\text{MT}}^{(j)}\}_{j=1}^{h})$:
    \[
    \|C\|_{2}
  \;\le\;
  T\,\frac{(h+1)(h+2)}{2}\,
  |\alpha T-1|\,
  \textcolor{ForestGreen}{H_{\max}^{\phi}}\textcolor{BrickRed}{G_{\max}^{\phi}}
  \; 
    \]
    where we define the uniform bounds:
    \begin{itemize}
        \item \textbf{Hessian bound }\(\textcolor{ForestGreen}{H_{\max}^{\phi}}:=\displaystyle\max_{t,\ell}
          \lVert\nabla^{2}\overline L_{t}(\theta^{(\ell)}_{\MT})\rVert_{2}\)
        \item \textbf{Gradient bound } \(\textcolor{BrickRed}{G_{\max}^{\phi}}:=\displaystyle\max_{t,m}
          \lVert\nabla\overline L_{t}(\theta^{(m)}_{\MT})\rVert_{2}\)
    \end{itemize}

     \item[(ii)] Then, we bound the Hessian bound \textcolor{ForestGreen}{\(H_{\max}^{\phi}\)} and the gradient bound \textcolor{BrickRed}{\(G_{\max}^{\phi}\)} as follows:
     \begin{itemize}
      \item $\textcolor{ForestGreen}{H_{\max}^{\phi}} \le 2\,\gamma_\phi\,M_{x}^{2}\,\Pi^{2}\beta_\phi^{\,2L-2}$ (for general activations)
      \item $\textcolor{BrickRed}{G_{\max}^{\phi}} \le \sqrt{2}\,M_{x}\,\Pi\,\beta_\phi^{\,L-1}$
    \end{itemize}
          
where we assumed that \(
\|x\|_{2}\le M_{x},
\;
\|W^{(\ell)}\|_{2}\le s_{\ell},
\;
|\phi'(z)|\le\beta_\phi
\).
\item[(iii)] Finally, plugging \textcolor{ForestGreen}{\(H_{\max}^{\phi}\)} and \textcolor{BrickRed}{\(G_{\max}^{\phi}\)} explicitly into the inequality reproduces exactly the two cases stated in
Theorem~\ref{thm:coefficient-bound}.
\end{enumerate}

%% file: 4_multi_task/1_task_vectors_gradients/5_Discussion/content.tex
The formal results above have direct implications for merging practice. We examine how the parameter-space trajectory is shaped by the merging strategy and why task proficiency does not necessarily correlate with mergeability.

\subsection{Parameter space trajectory}
Most task-arithmetic-based approaches construct a merged model in a single shot by aggregating task vectors extracted from the pretrained initialization. From the gradient perspective developed above, this corresponds to taking a single large step in parameter space along the sum of task gradients evaluated at $\theta_{\text{pre}}$. While effective in many cases, this strategy relies on a linear approximation that is only guaranteed to be accurate locally. When task gradients are misaligned or curvature effects accumulate, a single extrapolation can move the parameters away from low-loss regions.

\begin{wrapfigure}[15]{r}{0.35\textwidth}
  \centering
  \includegraphics[width=0.35\textwidth]{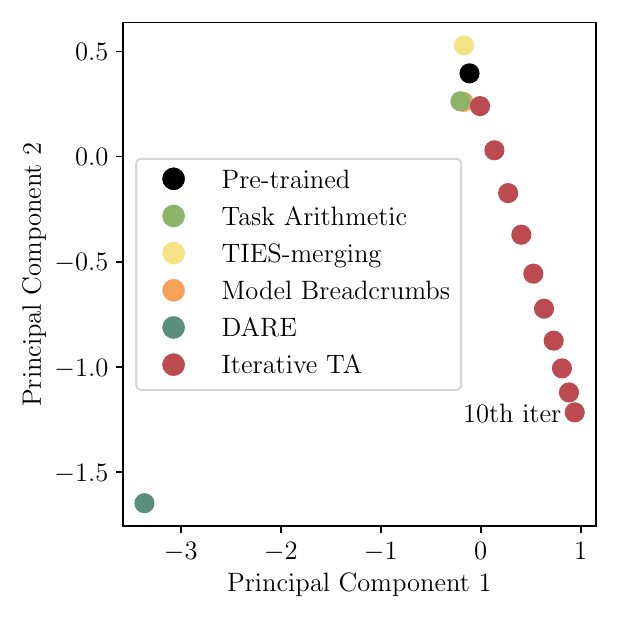}
  \caption{Checkpoint projection of different strategies.}
  \label{fig:pca}
\end{wrapfigure}Several recent works implicitly address this issue by spreading task arithmetic over time rather than performing it in a single step. In particular, iterative task arithmetic~\citep{zhou2025atmimprovingmodelmerging} alternates short fine-tuning phases with repeated merging operations, effectively re-estimating task vectors around the current parameters at each iteration. While originally proposed as a practical heuristic, this procedure can be interpreted as tracking a trajectory in parameter space through a sequence of small, locally valid updates.

\Cref{fig:pca} provides a geometric illustration of this effect via a 2D PCA projection of the resulting checkpoints. TIES and Model Breadcrumbs converge to similar basins, while DARE diverges due to stochastic pruning. In contrast, the trajectory produced by iterative merging progresses smoothly through parameter space, remaining closer to low-loss regions throughout optimization. This empirical behavior is consistent with the gradient interpretation of task vectors developed in this chapter: smaller, repeated steps better respect the local geometry of the loss landscape and mitigate task interference.

\subsection{Task proficiency is not mergeability}
One might expect that more specialized, higher-performing models would produce better merged results, but the evidence points the other way. As shown in \cref{fig:ta-1-epoch-vs-ta-convergence}, merging models fine-tuned to convergence yields no gains over merging single-epoch models.

From the gradient perspective, the explanation is direct: the longer a model fine-tunes, the more its task vector deviates from the true multi-task gradient. The first epoch's task vector is an exact scaled negative gradient; subsequent epochs introduce a curvature-controlled $O(\eta^2)$ error. Highly specialized models also diverge further in parameter space \citep{twinmerging}, making naive aggregation less effective (\cref{fig:pca}).

These findings complement the \emph{weight disentanglement} view of \citet{ortiz2024task} (see \cref{sec:tangent-space}), which shows that task vectors are more reliable when models remain close to the pre-trained initialization, where functional components are more linearly separable. The gradient perspective offers an orthogonal rationale grounded in optimization dynamics: \emph{models that stay closer to the pre-trained base are more mergeable} because their task vectors are more faithful gradient approximations.

%% file: 4_multi_task/2_task_singular_vectors/content.tex

\begin{chapteroverview}
In this chapter, we study the layer-wise matrix structure of weight updates rather than treating them as flat vectors. We decompose per-layer task matrices via SVD, showing they are intrinsically low-rank, and leverage this structure for both compression (\method{TSV-Compress}) and interference reduction (\method{TSV-Merge}) during merging.
\end{chapteroverview}

\section{From flat vectors to layer-wise structure}\label{sec:tsv-intro}

\input{1_Introduction/content.tex}

\section{Task singular vectors and their properties}\label{sec:tsv-def}
\input{2.5_Background/content}

\section{Compression and interference reduction}\label{sec:tsv-approach}

\input{3_Approach/content.tex}
\section{Experimental results}\label{sec:tsv-results}
\input{3.5_Results/content}

\section{Ablations and analysis}\label{sec:tsv-analysis}
\input{4_Analysis/content.tex}

\section{Summary and outlook}\label{sec:tsv-conclusions}

\input{6_Conclusions/content.tex}

%% file: 4_multi_task/2_task_singular_vectors/2.5_Background/content.tex

%
\subsection{Layer-wise task matrices and their SVD}\label{subsec:background}
As introduced in \cref{sec:task-arithmetic}, task arithmetic operates on entire task vectors $\tau_t = \theta_t - \theta_{\text{pre}}$, treating them as flat vectors in parameter space. A finer-grained view considers these operations at the individual layer level. The per-layer \emph{task matrix} for task $i$ at layer $\ell$ is $\Delta_i^{(\ell)} = \theta_{i}^{(\ell)} - \theta_{\text{pre}}^{(\ell)}$, and the merged weights become:
\begin{equation}
    \label{eq:merge_weights_layer}
    \theta_\text{MT}^{(\ell)} = \theta_{\text{pre}}^{(\ell)} + \alpha \frac{\sum_{i=1}^\ntasks \Delta_i^{(\ell)}}{\ntasks}\,.
\end{equation}
When layers have a matrix structure (e.g., linear projections in attention or MLP blocks), this formulation preserves structural information that is lost in the flat-vector view. For brevity, we omit the layer index and refer to $\Delta_i^{(\ell)}$ simply as $\Delta_i$.

\subsubsection*{Decomposing layer task matrices}
Treating layer-wise weights as matrices rather than flattened vectors enables us to analyze their SVD.
Given two tasks $i, j$, we consider the SVD of their task matrices $\Delta_i$ and $\Delta_j$ at a generic layer:
\begin{equation*}
    \Delta_i = U_i \Sigma_i V_i^\top, \qquad 
    \Delta_j = U_j \Sigma_j V_j^\top
\end{equation*}
where \(U_i, U_j\) and \(V_i, V_j\) are the matrices of left and right singular vectors, respectively, and \(\Sigma_i,\Sigma_j\) are diagonal matrices of singular values. We term the singular vectors \emph{Task Singular Vectors} (TSVs).

We can equivalently write the aggregated task matrices in \cref{eq:merge_weights_layer} in terms of their SVDs. Defining $U=[U_1  \cdots  U_{\ntasks}]$ as the column-wise concatenation of all the left TSVs, $V = [V_1 \cdots V_{\ntasks}]$ as the column-wise concatenation of the right TSVs, and $\Sigma$ as the block diagonal matrix with $\{ \Sigma_i \}_{i=1}^\ntasks$ along its diagonal:
\begin{equation}\label{eq:M}
    M = U \Sigma V^\top = \sum_{i=1}^\ntasks U_i \Sigma_i V^{\top}_i = \sum_{i=1}^\ntasks \Delta_i \,.
\end{equation}
In the simple case where $T=2$, $M$ would be given by:
\begin{equation*}
    M = \left[
\begin{array}{c c}
    U_{1} & U_{2} \\
\end{array}
\right]
\left[
\begin{array}{c c}
    \Sigma_{1} &  0 \\
    0 & \Sigma_{2} \\
\end{array}
\right]
\left[
\begin{array}{c}
    V_{1}^\top \\
    V_{2}^\top \\
\end{array}
\right] = \Delta_1 + \Delta_2 \,.
\end{equation*}
Note that concatenating singular components from different tasks violates certain SVD properties: the matrices \( U \) and \( V \) are no longer orthogonal, because singular vectors from different tasks may overlap. Additionally, the singular values \( \Sigma_i \) from different tasks can vary significantly in magnitude, potentially biasing the merge toward tasks with larger singular values.
\subsection{Low-rank nature of layer task matrices} \label{subsec:low-rank}
By Eckart--Young's theorem \citep{Eckart1936TheAO}, the best rank-$k$ approximation (in Frobenius norm) of each task matrix $\Delta_i$ is obtained by retaining only the top-$k$ singular values and their corresponding vectors:
\begin{equation}
\hat{\Delta}_{i} = \sum_{j=1}^{k} \sigma_j^{i} u_j^{i} v_j^{i\top}.
\end{equation}
As shown in \cref{fig:acc_by_rank}, task matrices are inherently low-rank: preserving only 3\% of the singular components per task reduces mean accuracy by merely 1.5\%. Building on this, we propose in \cref{subsec:tsv-compress} a compression algorithm that maintains $99\%$ of accuracy while shrinking task vectors by a factor of $\ntasks$.
\begin{figure}
  \centering
  \includegraphics[width=.9\linewidth]{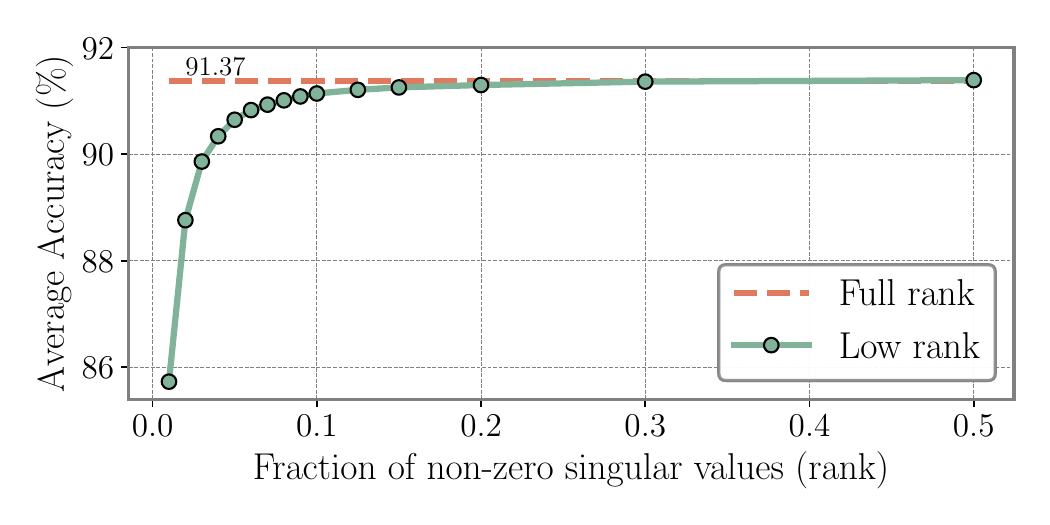}
    \caption[Accuracy vs.\ fraction of retained singular components]{Mean absolute accuracy of the \model{ViT-B-32} model across increasing fractions of retained singular components, averaged over 20 tasks. The red line represents the average accuracy of the original fine-tuned models with full-rank task matrices, while the green line shows the accuracies using low-rank approximations.}
  \label{fig:acc_by_rank}
\end{figure}

Given this low-rank structure, it is natural to aggregate task matrices within their subspaces. We, therefore, obtain a reduced version of the aggregation matrix $M$ (\cref{eq:M}) as:
\begin{equation}\label{eq:sum-of-rank-1-matrices-multi-task}
    \hat{M} = \sum_{i=1}^{\ntasks} \hat{U}_i \hat{\Sigma}_i \hat{V}_i^\top = \sum_{i=1}^{\ntasks} \sum_{j=1}^{k} \sigma_j^{i} u_j^{i} v_j^{i\top},
\end{equation}
where \( \hat{U}_i \) and \( \hat{V}_i \) contain the top-\( k \) left and right singular vectors for task \( i \), respectively, and \( \hat{\Sigma}_i \) is the diagonal matrix of the top-\( k \) singular values \( \sigma_j^{i} \).
\begin{figure}
  \centering
  \includegraphics[width=\linewidth]{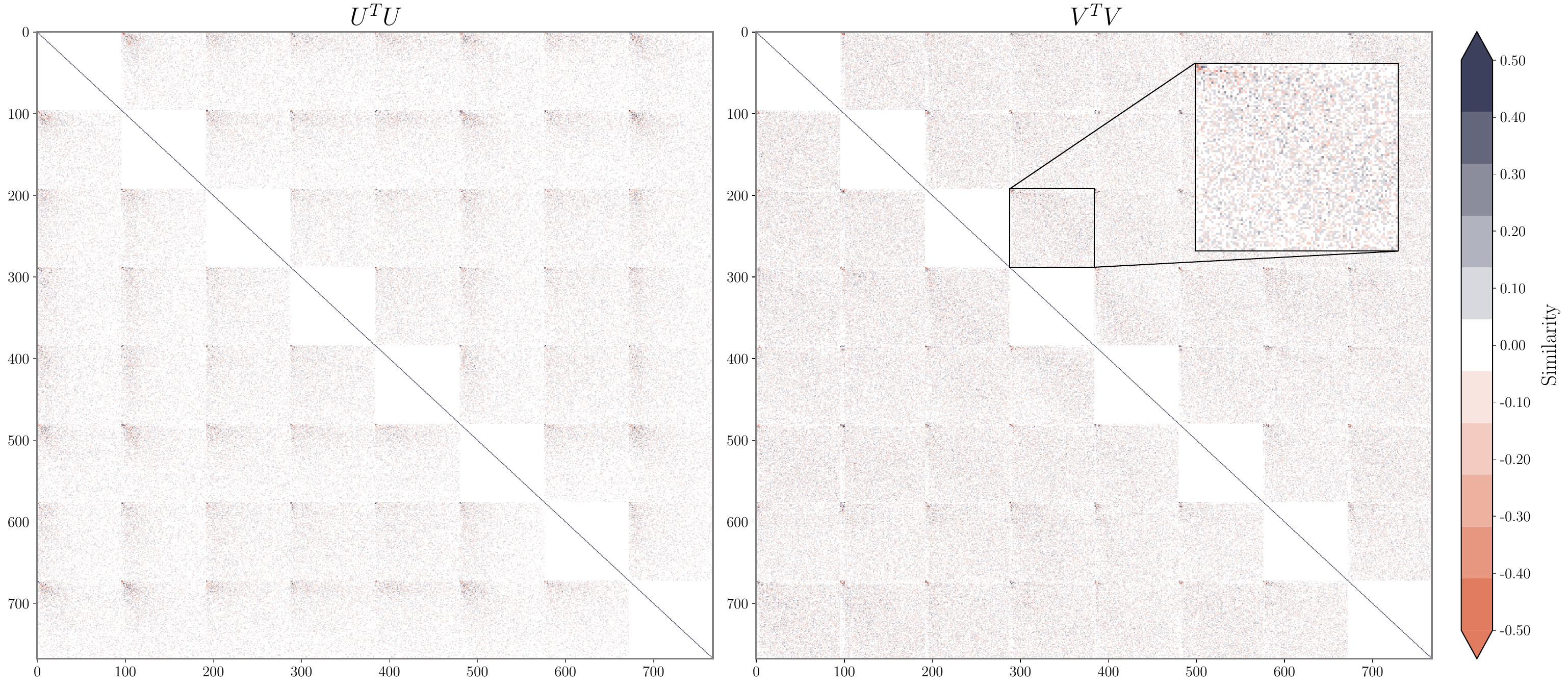}
  \caption[Task interference similarity matrices for 8 tasks]{Visualization of task interference among 8 tasks computed on the first attention layer of a \model{ViT-B-32}. The diagonal blocks display intra-task similarities, while the off-diagonal blocks illustrate inter-task similarities. The zoomed-in section highlights the interaction between the right singular vectors of the 3$^{\text{rd}}$ and 4$^{\text{th}}$ tasks.}
  \label{fig:similarity_matrices}
\end{figure}
\subsection{Singular task interference}\label{subsec:task-interference}
We now introduce a measure of task interference based on TSVs, which we term \emph{Singular Task Interference} (STI).

If all tasks occupied \emph{orthogonal} subspaces, the concatenated matrices $U$ and $V$ would have orthonormal columns, so $U^\top U = I$ and $V^\top V = I$. Deviations of these Gram matrices from the identity measure the degree of subspace overlap. We weight these deviations by $\Sigma$ to reflect each direction's importance (directions with large singular values matter more), and aggregate with the element-wise $\ell_1$ norm:
\begin{equation}\label{eq:task-interference}
    \operatorname{STI}\left(\left\{\Delta_i\right\}_{i=1}^{T}\right) = \lVert (U^\top U - I) \Sigma (V^\top V - I) \rVert_1\,,
\end{equation}
where $U$, $\Sigma$ and $V$ are obtained by concatenating the singular value decompositions $(\left\{U_i\right\}_{i=1}^T$, $\left\{\Sigma_i\right\}_{i=1}^T$, $\left\{V_i\right\}_{i=1}^T)$ of per-layer task matrices $\left\{\Delta_i\right\}_{i=1}^{T}$ as detailed in \cref{subsec:background}.
By construction, $\operatorname{STI} = 0$ when all task subspaces are mutually orthogonal, and grows as tasks share more directions in weight space.

Overlapping singular vectors indicate shared features in weight space across tasks; such overlap introduces interference when models are merged. \Cref{fig:similarity_matrices} visualizes this for eight tasks.

In transformer-based architectures, each transformer block contains multiple weight matrices ($W_Q$, $W_K$, $W_V$, $W_O$, and the two MLP projection matrices). Each of these is treated as a separate per-layer task matrix $\Delta_i^{(\ell)}$, with an independent SVD applied to each.

%% file: 4_multi_task/2_task_singular_vectors/3.5_Results/content.tex
We evaluate our approaches over three different suites of tasks having cardinality 8, 14, and 20, respectively. The first one, introduced in \citep{task-vectors}, consists of datasets: \dataset{Cars} \cite{krause_3d_2013}, \dataset{DTD} \cite{cimpoi_describing_2014}, \dataset{EuroSAT} \cite{helber_eurosat_2019}, \dataset{GTSRB} \cite{stallkamp_german_2011}, \dataset{MNIST} \cite{726791}, \dataset{RESISC45} \cite{cheng_remote_2017}, \dataset{SUN397} \cite{xiao_sun_2016}, and \dataset{SVHN} \cite{netzer_reading_nodate}.
The benchmark with 14 tasks builds on the preceding one, incorporating six additional datasets: \dataset{CIFAR100} \cite{CIFAR}, \dataset{STL10} \cite{coates_analysis_2011}, \dataset{Flowers102} \cite{nilsback_automated_2008}, \dataset{OxfordIIITPet} \cite{parkhi_cats_2012}, \dataset{PCAM} \cite{veeling_rotation_2018}, and \dataset{FER2013} \cite{goodfellow_challenges_2013}.
Finally, the 20-tasks benchmark includes the preceding 14 plus the following six: \dataset{EMNIST} \cite{EMNIST}, \dataset{CIFAR10} \cite{CIFAR}, \dataset{Food101} \cite{bossard_food-101_2014}, \dataset{FashionMNIST} \cite{xiao_fashion-mnist_2017}, \dataset{RenderedSST2} \cite{socher_recursive_nodate}, and \dataset{KMNIST} \cite{clanuwat_deep_2018}.

\input{3.5_Results/merging_results}

\subsection{Model merging results} \label{sec:MM_results}
We evaluate on three \baseline{CLIP} \cite{radford2021learningtransferablevisualmodels} variants with different \model{ViT} \cite{dosovitskiy2021imageworth16x16words} encoders: \model{ViT-B-32}, \model{ViT-B-16}, and \model{ViT-L-14}. Benchmarks involve merging 8, 14, and 20 tasks, following the setup of \citet{wang2024localizing}. We compare against weight averaging, \baseline{Task Arithmetic} \citep{task-vectors}, and \baseline{Consensus Merging} \citep{wang2024localizing}. Zero-shot performance (\Cref{tab:task_acc}) and individually fine-tuned models (\Cref{tab:task_acc_compression}) serve as lower and upper bounds, respectively. We report both average absolute and normalized accuracy (see \Cref{sec:normalized_accuracy}).

As shown in \Cref{tab:task_acc}, \method{TSV-M} achieves state-of-the-art results across all benchmarks, regardless of model size or number of tasks. The largest gains appear on \model{ViT-B-32}, where it outperforms \baseline{Task Arithmetic} and \baseline{Consensus TA} by $+15.45\%$ and $+10.71\%$ absolute accuracy on average. On \model{ViT-L-14} with 8 tasks, the method attains $96.98\%$ normalized accuracy, meaning a single merged model incurs only a $3.02\%$ accuracy reduction compared to eight individual fine-tuned models. Per-dataset accuracies are reported in \cref{fig:radar_chart}.

\input{3.5_Results/compression_results}
\subsection{Compression results} \label{subsec:compression-results}
\Cref{tab:task_acc_compression} reports the results for \method{TSV-C}, compared with \baseline{TALL-masks} \citep{wang2024localizing}, which stores sparse binary masks for each task. \method{TSV-C} consistently retains more than $99\%$ of the original accuracy across all benchmarks and models. The two methods perform comparably (within $1\%$ on average), with \method{TSV-C} slightly trailing on smaller \model{ViTs} but ahead on \model{ViT-L-14}.
A key advantage of \method{TSV-C} is predictable storage: it stores only the top $\frac{1}{T}$ singular vectors per task, yielding a fixed footprint of approximately twice the original model size regardless of the number of tasks. \baseline{TALL-masks}, by contrast, has storage that grows with the number of tasks and depends on mask compressibility.

%% file: 4_multi_task/2_task_singular_vectors/3.5_Results/merging_results.tex
\begin{table}
  \centering
  \resizebox{\textwidth}{!}{
    \begin{tabular}{cccc|ccc|ccc}
      \toprule
      \multirow{2}{*}{\textbf{Method}}         & \multicolumn{3}{c}{\model{ViT-B-32}}                     & \multicolumn{3}{c}{\model{ViT-B-16}}                     & \multicolumn{3}{c}{\model{ViT-L-14}}                    \\ 
                                               \cmidrule{2-10}
                                               & 8 tasks           & 14 tasks          & 20 tasks           & 8 tasks           & 14 tasks          & 20 tasks          & 8 tasks           & 14 tasks          & 20 tasks          \\ 
                                               \midrule
      \multicolumn{1}{c|}{Zero-shot}            & $48.26_{(53.59)}$ & $57.21_{(63.69)}$ & $56.10_{(62.41)}$ & $55.34_{(59.34)}$ & $61.28_{(66.19)}$ & $59.73_{(64.52)}$ & $64.70_{(68.00)}$ & $68.20_{(72.15)}$ & $65.23_{(68.99)}$ \\
      \multicolumn{1}{c|}{Weight Averaging}    & $66.34_{(72.13)}$ & $64.34_{(71.12)}$ & $61.04_{(67.53)}$ & $72.22_{(76.60)}$ & $69.46_{(74.82)}$ & $65.31_{(70.36)}$ & $79.56_{(83.15)}$ & $76.73_{(81.10)}$ & $71.60_{(75.60)}$ \\
      \multicolumn{1}{c|}{Task Arithmetic}     & $70.79_{(76.55)}$ & $65.32_{(72.09)}$ & $60.52_{(66.79)}$ & $75.41_{(79.58)}$ & $70.52_{(75.89)}$ & $65.78_{(70.76)}$ & $84.93_{(88.65)}$ & $79.41_{(83.95)}$ & $74.01_{(78.07)}$ \\
      \multicolumn{1}{c|}{Consensus TA}        & $75.03_{(80.84)}$ & $70.39_{(77.36)}$ & $65.43_{(71.98)}$ & $79.39_{(83.86)}$ & $74.39_{(79.92)}$ & $69.76_{(74.93)}$ & $86.34_{(90.08)}$ & $82.22_{(86.94)}$ & $79.00_{(83.22)}$ \\
      \rowcolor{gray!15}
      \multicolumn{1}{c|}{\textbf{\texttt{TSV-M} (Ours)}} & $\mathbf{85.86_{(92.31)}}$ & $\mathbf{80.06_{(87.88)}}$ & $\mathbf{77.07_{(84.29)}}$ & $\mathbf{89.01_{(93.94)}}$ & $\mathbf{84.58_{(91.01)}}$ & $\mathbf{80.57_{(86.45)}}$ & $\mathbf{92.98_{(96.98)}}$ & $\mathbf{89.17_{(94.43)}}$ & $\mathbf{87.72_{(92.50)}}$ \\
      \bottomrule
    \end{tabular}
  }
  \caption[Model merging benchmark results]{Average absolute accuracy results on model merging benchmarks; subscript (in parentheses) is the normalized average accuracy.}
  \label{tab:task_acc}
\end{table}

%% file: 4_multi_task/2_task_singular_vectors/3.5_Results/compression_results.tex
\begin{table}
  \centering
  \resizebox{\textwidth}{!}{
    \begin{tabular}{cccc|ccc|ccc}
      \toprule
      \multirow{2}{*}{\textbf{Method}}               & \multicolumn{3}{c}{\model{ViT-B-32}} & \multicolumn{3}{c}{\model{ViT-B-16}} & \multicolumn{3}{c}{\model{ViT-L-14}}                                                                                                                           \\ \cmidrule{2-10}
                                                     & 8 tasks                               & 14 tasks                              & 20 tasks                               & 8 tasks            & 14 tasks          & 20 tasks          & 8 tasks            & 14 tasks          & 20 tasks          \\ \midrule
      \multicolumn{1}{c|}{Finetuned}                 & $92.83_{(100)}$                       & $90.88_{(100)}$                       & $91.37_{(100)}$                       & $94.64_{(100)}$    & $92.76_{(100)}$   & $93.17_{(100)}$   & $95.81_{(100)}$    & $94.29_{(100)}$   & $94.73_{(100)}$   \\
      \multicolumn{1}{c|}{TALL-Mask+TIES}            & $93.13_{(100.37)}$                    & $90.92_{(100.04)}$                    & $91.11_{(99.70)}$                     & $94.68_{(100.04)}$ & $92.69_{(99.90)}$ & $93.05_{(99.87)}$ & $95.96_{(100.15)}$ & $93.40_{(99.09)}$ & $93.91_{(99.16)}$ \\
      \multicolumn{1}{c|}{\textbf{\texttt{TSV-C} (Ours)}} & $92.62_{(99.74)}$                     & $90.29_{(99.28)}$                     & $90.64_{(99.14)}$                     & $94.47_{(99.79)}$  & $92.25_{(99.41)}$ & $92.53_{(99.27)}$ & $95.68_{(99.85)}$  & $94.04_{(99.72)}$ & $94.42_{(99.66)}$ \\
      \bottomrule
    \end{tabular}
  }
  \caption[\method{TSV-C} compression benchmark results]{Average absolute accuracy results across all compression benchmarks for different models and varying number of tasks, subscript (in parentheses) the normalized average accuracies. Our TSV-C compression method consistently achieves over 99\% of the original fine-tuned models' accuracy while significantly reducing storage requirements.}
  \label{tab:task_acc_compression}
\end{table}

%% file: 4_multi_task/2_task_singular_vectors/4_Analysis/content.tex
%
%
We ablate the contributions of interference reduction and low-rank compression, study how task interference varies across layers, and show that the method does not require tuning the scaling coefficient.
%
%
\begin{table}
  \centering
  \begin{tabular}{ccccc}
    \toprule
    \vspace{-0.1cm}Low-rank       & Interf.  & \multicolumn{3}{c}{\model{ViT-B-32}}                                           \\
    \cmidrule{3-5}
    approx. & reduction           & 8 tasks                                  & 14 tasks                & 20 tasks                \\
    \midrule
    \negxmark     & \negxmark     & 76.5 (+0.0)                        & 72.1 (+0.0)       & 66.8 (+0.0)       \\
    \poscheckmark & \negxmark     & 75.2 \worse{1.3}                    & 71.0 \worse{1.1}   & 66.3 \worse{0.5}   \\
    \negxmark     & \poscheckmark & 82.6 \improv{7.4}                   & 75.7 \improv{4.7}  & 69.9 \improv{3.6}  \\
    \poscheckmark & \poscheckmark & 92.3 \improv{9.7}                   & 87.9 \improv{12.2} & 84.3 \improv{14.4} \\
    \bottomrule
  \end{tabular}
  \caption[Ablation of low-rank approximation and interference reduction]{Comparison of different versions of \baseline{Task Arithmetic}, comprising either the low-rank approximation step, the interference reduction step, or both. The method performing both corresponds to the proposed \method{TSV-Merge}.}
  \label{tab:ablation-study}
\end{table}
\begin{figure}
    \centering
    \includegraphics[width=\columnwidth]{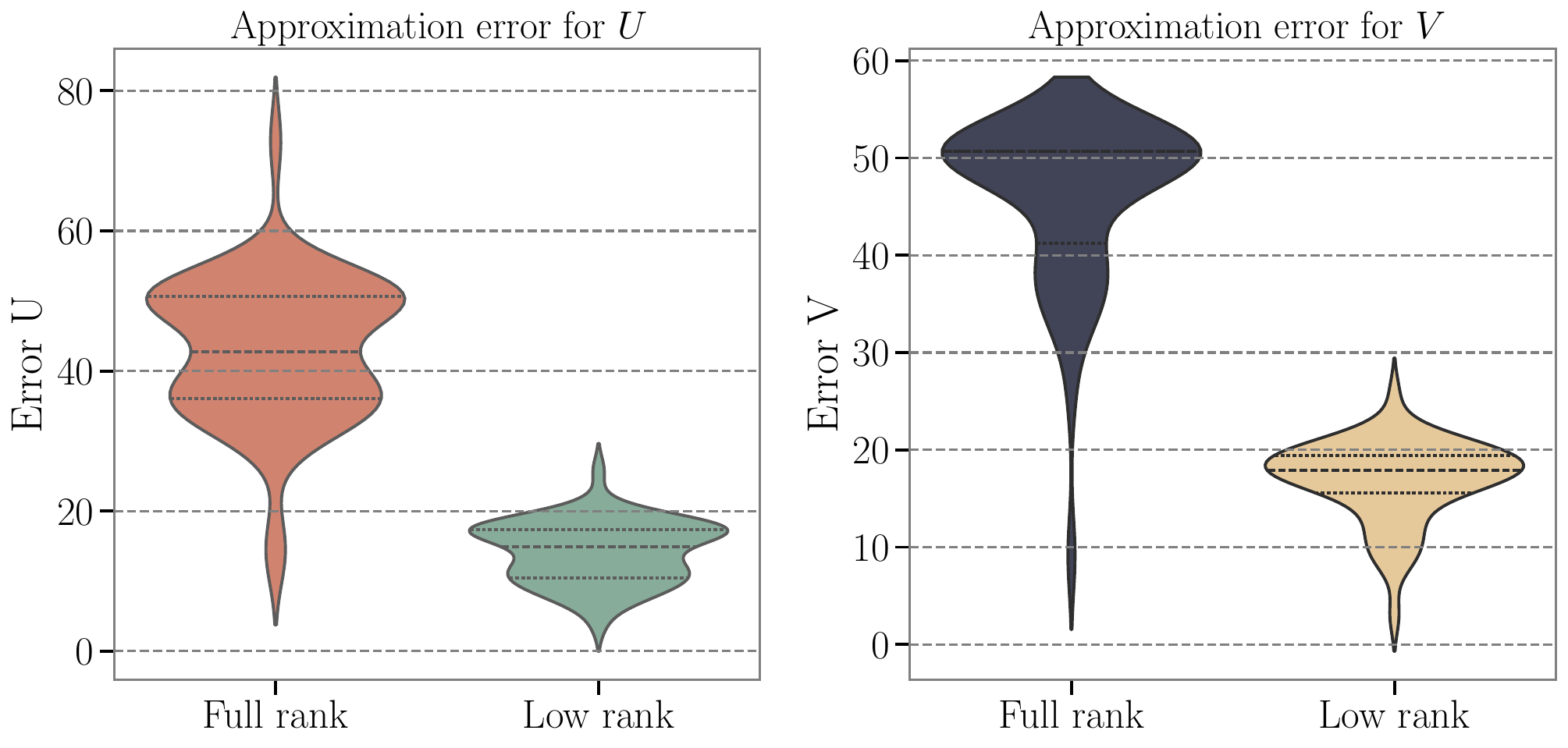}
    \caption[Procrustes orthogonalization approximation error]{Approximation error from the orthogonalization of the TSVs through Procrustes for the \model{ViT-B-32} model across 8 tasks. The violin plots represent layer-wise approximation error distributions for $U$ and  $V$ in both full-rank and low-rank cases.
    }
    \label{fig:approx-error}
\end{figure}
\subsection{Ablation study}
\label{subsec:ablation}
\method{TSV-M} combines low-rank approximation with interference reduction. \Cref{tab:ablation-study} summarizes the contribution of each component. Low-rank approximation alone slightly hurts performance compared to vanilla \baseline{Task Arithmetic} (first row). Interference reduction alone yields gains of $+3.1\%$ to $+6.1\%$. Combining both steps produces improvements of $+15.8\%$ to $+17.5\%$, exceeding the sum of the individual contributions.

This superadditive synergy arises from the interaction between the two steps. Low-rank approximation alone retains the dominant singular directions but leaves inter-task conflicts intact. Interference reduction alone must orthogonalize the concatenated \emph{full-rank} matrices: when each $U_i$ has full ambient dimension, the concatenated matrix $[U_1 \cdots U_T]$ is wide and highly correlated, and forcing it to be nearly orthogonal via Procrustes incurs large approximation error. Applying interference reduction to the low-rank approximations breaks this cycle: with only $k \ll n$ columns per task, Procrustes operates on a compact matrix and achieves tighter orthogonality with less distortion. Compression thus acts as a prerequisite that makes interference reduction effective, and \cref{theo:error} formalizes this.

\Cref{fig:approx-error} confirms this empirically: orthogonalizing full-rank matrices incurs significantly larger approximation errors than their low-rank counterparts, with wider and higher error distributions across layers. The following theorem shows that this is a general property, not specific to the chosen model (proof in \Cref{mini_proof}).
\begin{theorem}
\label{theo:error}
Let \(T \in \mathbb{N}\) such that \(T > 4\). Define \( U = [U_1, \dots, U_T] \) as the matrix obtained by concatenating \(T\) orthogonal matrices \( U_i \), each of shape \(n \times n\). Let \( \widehat{U} = [\widehat{U}_1, \dots, \widehat{U}_T] \) be the matrix formed by truncating each \( U_i \) to its first \(k\) columns. Denote by \(X\) and \(\widehat{X}\) the matrices resulting from Procrustes orthonormalization of \(U\) and \(\widehat{U}\), respectively. If \(k \leq n \frac{T - 2\sqrt{T}}{T}\), then 
    \[
    \| U - X \|_F \geq \| \widehat{U} - \widehat{X} \|_F \,.
    \]
\end{theorem}
In words, Procrustes orthogonalization introduces less error when applied to truncated (low-rank) matrices than to full-rank ones, provided $k$ is small enough.
To verify applicability, consider our most demanding setting, \(T = 20\). The bound becomes \(k \leq \frac{20 - 2\sqrt{20}}{20}\,n \approx 0.55\,n\), and we choose \(k = \lfloor n/T \rfloor = \lfloor n/20 \rfloor\), which satisfies it. For \(T = 8\) and \(T = 14\) the bound is even more permissive.

Note that Gram--Schmidt would be ineffective here, as it sequentially orthogonalizes vectors without minimizing deviation from the original set, potentially introducing large distortions.

Beyond their quantitative role in merging and compression, TSVs also carry an interpretability benefit: because they capture the principal directions of weight change induced by each task, they can be viewed as a \emph{semantic basis} that summarizes what a task ``does'' to a given layer. In \cref{subsec:exp-decoding-text}, we confirm this intuition by decoding individual TSVs into natural-language descriptions via \textsc{TextSpan}, showing that mid-layer singular vectors recover semantically meaningful concepts aligned with each task's visual domain.


\begin{figure}
  \centering
  \includegraphics[width=.9\linewidth]{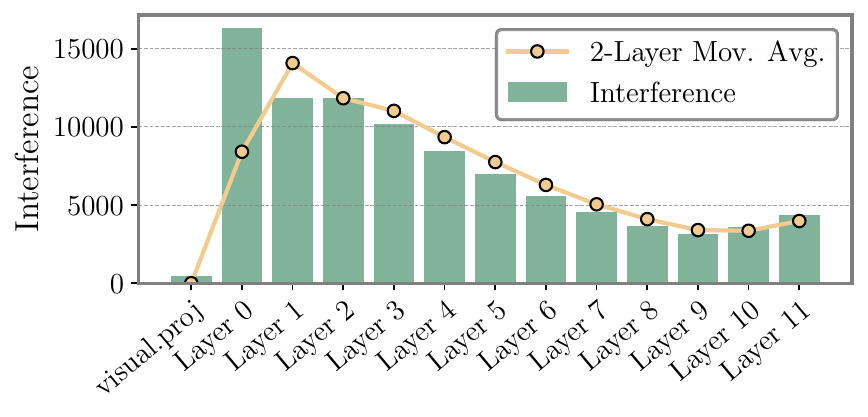}
  \caption[Singular Task Interference across layers]{Singular Task Interference (STI) across layers in a \model{ViT-B-32} for 20 tasks. STI is high in early layers sharing common knowledge and lower in more specialized ones.}
  \label{fig:interference_by_layer}
\end{figure}
\subsection{Per-layer task interference}
Using \cref{eq:task-interference}, we analyze task interference on a per-layer basis. As shown in \cref{fig:interference_by_layer}, interference is highest in the initial transformer layers and decreases in deeper ones. This aligns with the well-known pattern that early layers capture shared features across tasks while deeper layers specialize.
For brevity, we grouped layers within each transformer block in our analysis; each ``Layer $n$'' in \cref{fig:interference_by_layer} includes two attention matrices and two MLP matrices. A layer-by-layer analysis is provided in \cref{fig:detailed_interference_by_layer}.
This layer-wise interference profile motivates the adaptive routing strategy of \cref{ch:mass}, which selects different task subspaces per input rather than committing to a single static combination.

\subsection{Choice of the scaling coefficient}
\label{sec:alpha}
The aggregation in \cref{eq:merge_weights_layer} involves a scaling coefficient $\alpha$, typically tuned on a validation set. As shown in \cref{fig:alpha}, our approach consistently achieves the best results with $\alpha = 1.0$, eliminating the need for hyperparameter tuning and validation data.

\begin{figure}
  \centering
  \includegraphics[width=\linewidth]{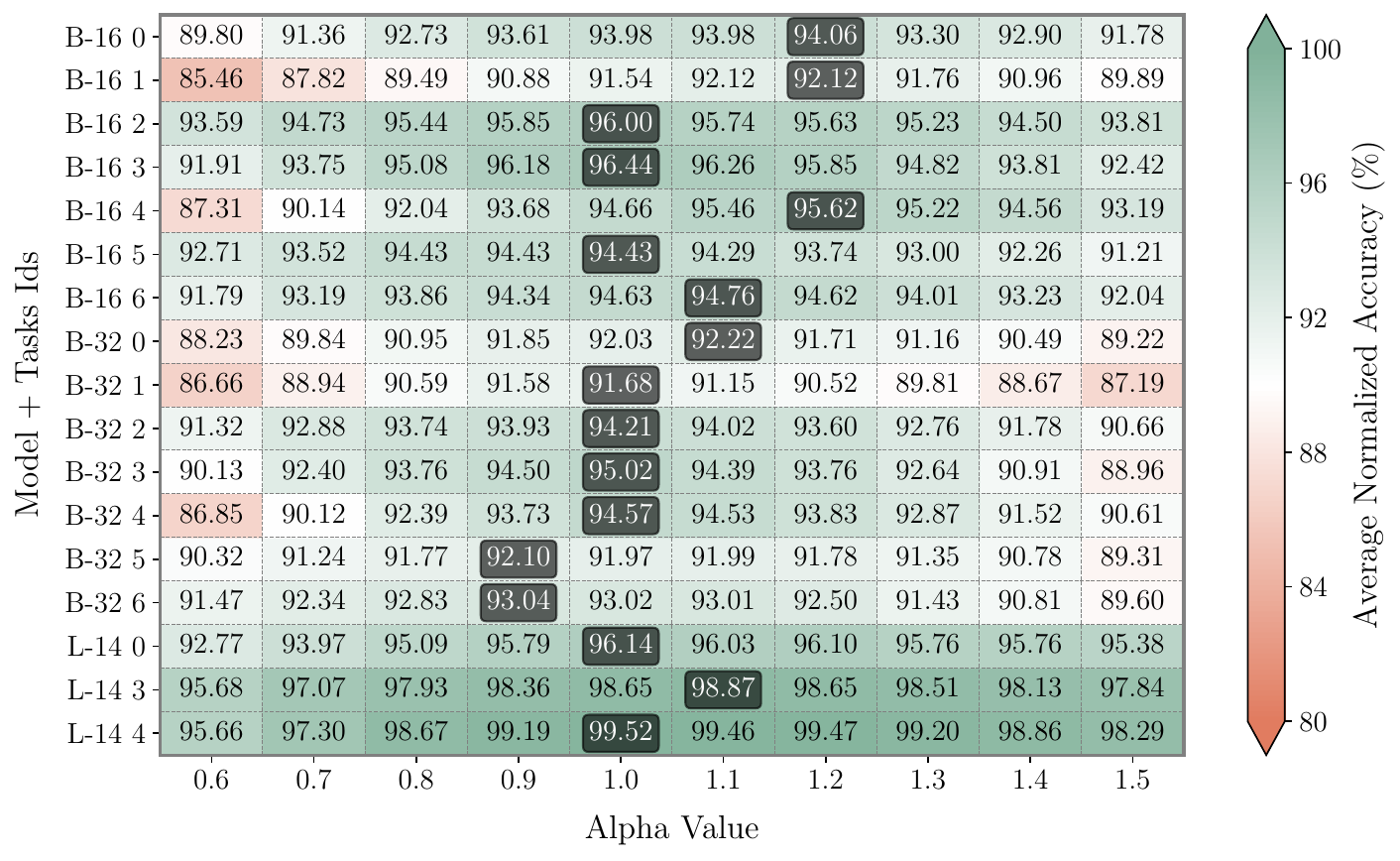}
  \caption[Normalized accuracy for different scaling coefficients]{Best average normalized accuracy for different alpha values. The vertical labels indicate different sets of 8 tasks.}
  \label{fig:alpha}
\end{figure}


\subsection{Effect of task interference reduction}
\begin{wrapfigure}[7]{r}{0.4\linewidth}
  \vspace{-1cm}
  \centering
  \includegraphics[width=\linewidth]{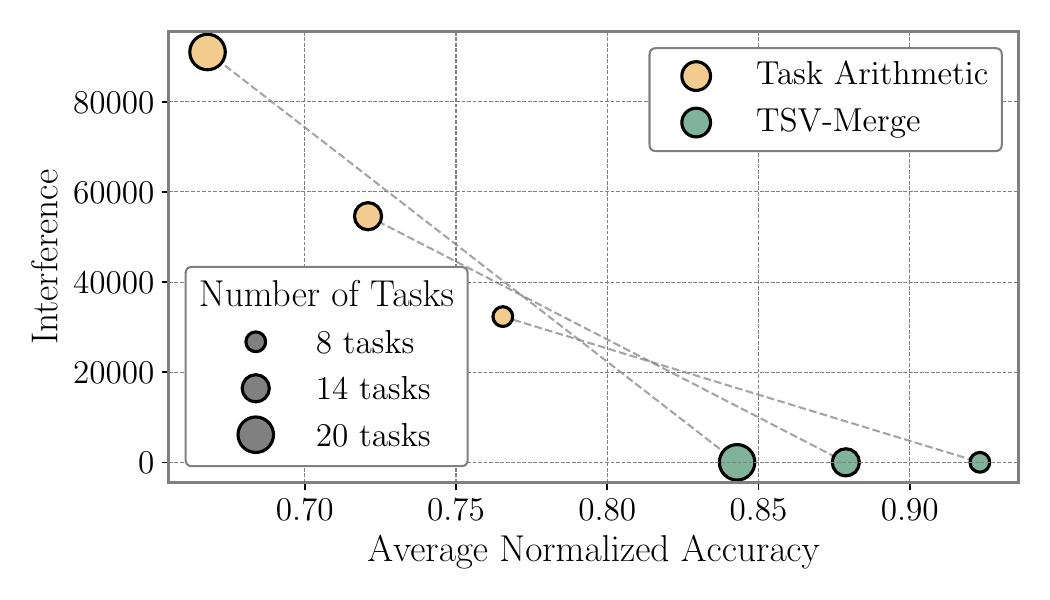}
  \caption[Effect of interference reduction on accuracy]{Singular Task Interference (STI) and average normalized accuracy on the \model{ViT-B-32} model, evaluated across merges of 8, 14, and 20 tasks.}
  \label{fig:correlation_interference_accuracy}
\end{wrapfigure}\Cref{fig:correlation_interference_accuracy} shows the interference of a \model{ViT-B-32} model before and after applying \method{TSV-M}, across task sets of different cardinalities. In all cases, the reduction in interference is accompanied by a significant accuracy gain, confirming that STI captures a quantity directly relevant to merging quality.

%% file: 4_multi_task/3_mass/content.tex
\input{sec/0_abstract}

\section{Beyond static merging} \label{sec:mass-intro}
\input{sec/1_intro}

\section{Adaptive subspace selection} \label{sec:mass-approach}
\input{sec/3_approach}

\section{Experimental results} \label{sec:mass-experiments}
\input{sec/4_experiments}


\section{Summary and outlook} \label{sec:mass-conclusions}
\input{sec/7_conclusions.tex}

%% file: 4_multi_task/3_mass/sec/0_abstract.tex
\begin{chapteroverview}
In this chapter, we move beyond static merging by introducing \mass{} (\emph{MoErging through Adaptive Subspace Selection}), which adapts the merged model to each input at inference time. Building on the low-rank structure identified in \cref{ch:tsv}, \mass{} uses a training-free router to select the most relevant task subspaces for a given input, recovering up to ${\sim}98\%$ of fine-tuned accuracy at a fraction of the storage cost.
\end{chapteroverview}

%% file: 4_multi_task/3_mass/sec/1_intro.tex

\begin{figure}
    \centering
    \includegraphics[width=0.75\linewidth]{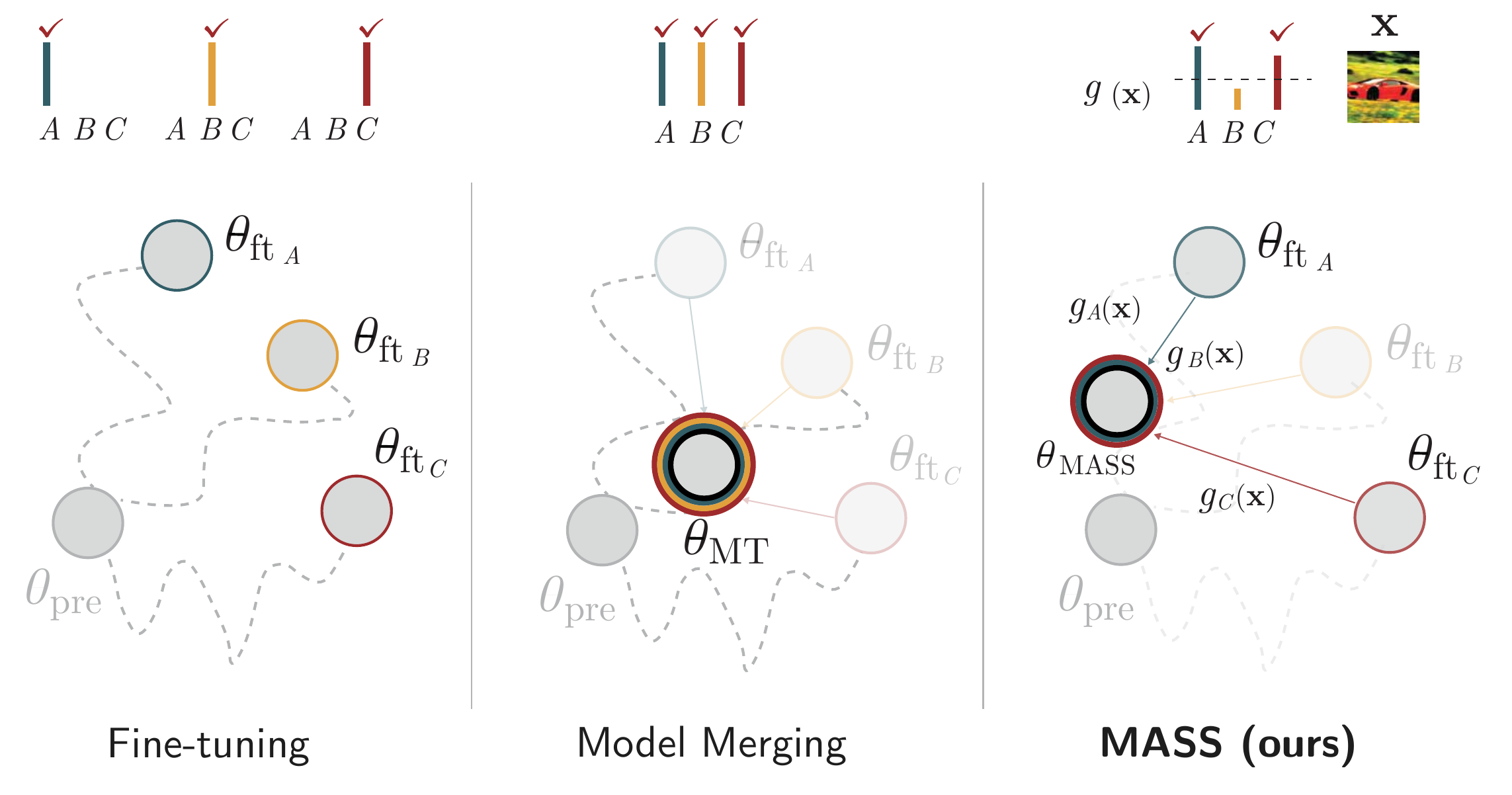}
    \caption[Overview of the \mass{} adaptive merging framework]{\emph{(left)} Fine-tuning yields three separate models on tasks A, B and C. \emph{(middle)} Model merging produces a single model incorporating the task vectors using a constant function of the input.  \emph{(right)} \mass{} stores the pre-trained model $\theta_\text{pre}$ and the task singular vectors $V^\top$ across tasks. At test time, \mass{} adaptively performs merging using a routing mechanism that chooses appropriate task vectors for the input $\mathbf{x}$, using a thresholded gating function $g(\mathbf{x})$. The gate is the residual between the activations of $\mathbf{x}$ and their projections onto the span of the right singular vectors $V_{\perp}$.}
    \label{fig:main-diagram}
\end{figure}

The previous chapter showed that preserving the layer-wise structure of task vectors leads to stronger merging results. However, all methods considered so far share a common limitation: they assume that the task identity is known at inference time. When this oracle is available, the problem reduces to compression, and the low-rank method \texttt{TSV-C} (\cref{ch:tsv}) already achieves $99.5\%$ normalized accuracy with only $2\times$ storage, effectively solving it.

Mixture-of-Experts Merging (MoErging) methods~\citep{moerging} such as \texttt{SMILE}~\citep{tang2024smile}, \texttt{WeMoE}~\citep{tang2024merging}, and \texttt{TwinMerging}~\citep{twinmerging} incorporate a router into the merging process, yet \emph{still assume that the correct classification head is provided at test time}, inheriting the same unrealistic constraint.

The setting considered here is more challenging: the task is \emph{not} known at inference time, and both the encoder subspaces and the classification head must be determined \emph{automatically}. To address it, we introduce \mass{} (\emph{MoErging through Adaptive Subspace Selection}), which conditions the merging process on the input itself (\cref{fig:main-diagram}). \mass{} routes each input to the most relevant task subspaces, as defined by the task singular vectors of \cref{ch:tsv}, and jointly selects the corresponding classification head, all without requiring training data or oracle supervision.

Evaluated on \vitsmall{}, \vitmedium{}, and \vitlarge{} across up to 20 vision tasks and 8 language tasks, \mass{} achieves up to $98\%$ of fine-tuned accuracy for a modest overhead (${\sim}2\times$ forward passes and ${\sim}2\times$ storage, regardless of the number of tasks), while handling the full union of label spaces without oracle guidance.

%% file: 4_multi_task/3_mass/sec/3_approach.tex

Our approach consists of a one-time pre-processing step followed by an inference-time step. The former produces an encoder model $\theta_\text{MT}$ via \cref{alg:fixed-merging}. We refer to this as the ``fixed'' merging step, as it is performed only once and remains independent of the input.
During inference, $\theta_\text{MT}$ is used in a dynamic process, outlined in \cref{alg:adaptive-merging}, consisting of $4$ steps:
\begin{enumerate}[label=(\roman*)]
  \item \textbf{First pass}: forward the input through $\theta_\text{MT}$ and extract its embedding $\mathbf{z}_\ell$ at a chosen layer $\ell$;
  \item \textbf{Routing}: project $\mathbf{z}_\ell$ onto the task subspaces, selecting those having lowest projection error;
  \item \textbf{Adaptive merge}: merge the selected task subspaces into $\Delta_{\text{ada}}$;
  \item \textbf{Second pass}: classify the input using the final merged model $\theta_\text{MT} = \theta_\text{pre}+\alpha \Delta_{\text{ada}}$.
\end{enumerate}
%
\input{algorithms/adaptive.tex}

\subsection{Fixed merging}
The fixed step produces a model capable of task discrimination, whose intermediate activations feed the router in the first pass. We use \baseline{TSV-M} (\cref{ch:tsv}) for this step due to its subspace-aware aggregation; \cref{tab:routers} confirms that routing from the pre-trained backbone yields lower task classification accuracy.

\input{algorithms/fixed.tex}

\subsection{Integrating routing} \label{subsec:routing}
We extend the aggregation step in \cref{eq:task-arithmetic} to include a routing mechanism. Given an input $\mathbf{x}$, the merged model can be adaptively determined by:
\begin{equation}
    \theta_\text{MT} = \theta_\text{pre} + \alpha \sum_{i=1}^\ntasks  {\mathds{1}}_{[   g_i(\mathbf{x}) =1]}( \mathbf{x})  \tau_i = \theta_\text{pre} + \alpha \sum_{i=1}^\ntasks {\mathds{1}}_{[  g_i(\mathbf{x}) =1]}( \mathbf{x})  \sum_{j=1}^{k} \sigma_j^{i} u_j^{i} v_j^{i\top}, \label{eq:routing}
\end{equation}
where $g_i(\mathbf{x})$ is a per-task gating function that adaptively selects which task subspaces to activate, and subsequently merge, depending on the input at hand. 

Routers typically require task-specific \emph{data} (for non-parametric methods such as nearest-neighbor routing) or both data and \emph{training} (for parametric routers). This clashes with the merging scenario, where the fine-tuning data is generally unavailable because models may be downloaded from public repositories. We therefore introduce a completely \textbf{{\em data}- and {\em tuning}-free} approach.

\subsubsection{Projection-based routing} \label{subsec:proj-routing}
\begin{wrapfigure}[9]{r}{0.4\linewidth}
  \centering
  \input{figures/projection.tex}
  \vspace{-0.1cm}
  \caption[Projection of activations onto task singular vector spans]{Projection of the activations $\mathbf{z}_{\ell}$ onto the span of TSVs $\mathbf{v}_1, \mathbf{v}_2$.}
\end{wrapfigure}
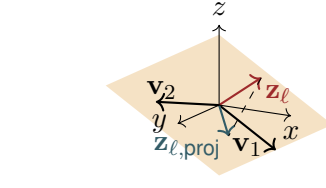
Given an input intermediate representation $\mathbf{z}_{\ell}$ for a predetermined layer $\ell$, we want to identify which task subspace (or set of subspaces) is the most relevant. Concretely, one way to do this is to compute the Euclidean residual of $\mathbf{z}_{\ell}$ after projecting onto $\operatorname{span}(V_i^{(\ell)})$:
\begin{equation}
  r_i \;=\;\bigl\lVert\,\mathbf{z}_{\ell}
  \;-\;\operatorname{Proj}_{V_i^{(\ell)}}(\mathbf{z}_{\ell})\bigr\rVert_2 \;,
  \label{eq:router-residual}
\end{equation}
where $ \operatorname{Proj}_{V_i^{(\ell)}}(\mathbf{z}_{\ell})  = V_i^{(\ell)} \bigl( V_i^{(\ell)} \bigr)^\top \mathbf{z}_{\ell}$ is the optimal $L_2$ projector (see Proposition~\ref{prop:optimal_projection}). 
At this point, the additive inverse of the residuals is normalized through a softmax to obtain the coefficients. The router then picks those exceeding a predetermined threshold $\eta$, limiting the selection to the top-$k$ when more tasks surpass it.
For details regarding the choice of the layer used to compute the residual, see \cref{sec:choosing_routing_layer}.
The effectiveness of this projection-based routing rests on the assumption that different tasks occupy geometrically distinct subspaces. The gradient perspective of \cref{ch:gradient} provides a theoretical basis for this: since task vectors approximate scaled gradients of their respective task losses (\cref{main_thm}), tasks with distinct loss landscapes produce task vectors that point in different directions, making their singular-vector subspaces separable by projection.

\subsubsection{Accounting for redundant directions} \label{sec:accounting_redundant}
For projection-based routing to be effective, no task should overshadow the others.
Consider, for example, three tasks: \dataset{MNIST} \methodimagetag{figures/MNIST.png}, \dataset{EMNIST} \methodimagetag{figures/EMNIST.png}, and \dataset{KMNIST} \methodimagetag{figures/KMNIST.png}.  Being trained on very similar datasets covering the same classes, $\Delta_{\text{MN}}$ and $\Delta_{\text{EMN}}$ share a large portion of their right-singular directions:
\(
  \operatorname{span}\bigl(V_{\text{MN}}^{(\ell)}\bigr)\,\approx\,
  \operatorname{span}\bigl(V_{\text{EMN}}^{(\ell)}\bigr),
\)
while \dataset{KMNIST} has some distinct directions $V_{\text{KM}}^{(\ell)}$ capturing more Japanese \emph{kana}-like shapes. However, all three tasks may agree on certain generic ``black background, white glyph'' features.  Because \dataset{MNIST} and \dataset{EMNIST} partially \emph{reinforce} these directions (they both include them), the \emph{union} of their subspaces can appear ``wider'' or more dominant in that region of feature space. Consequently, for many test samples $\mathbf{z}_{\ell}$ with black backgrounds and centered shapes:
\begin{equation*}
  \bigl\lVert\,\mathbf{z}_{\ell}
  - \operatorname{Proj}_{\,V_{\text{MN}}\cup V_{\text{EMN}}}(\mathbf{z}_{\ell})\bigr\rVert_2
  <
  \bigl\lVert\,\mathbf{z}_{\ell}
  - \operatorname{Proj}_{\,V_{\text{KM}}}(\mathbf{z}_{\ell})\bigr\rVert_2.
\end{equation*}
Hence, the router sees a smaller residual for \dataset{MNIST}/\dataset{EMNIST}, declaring those tasks more suitable even if the glyph belongs to \dataset{KMNIST}.

During the fixed merging step, instead of aggregating all task matrices, we only keep those that are \emph{sufficiently distinct}. Starting from a single task matrix, each remaining candidate $\Delta_i$ is accepted only if its cosine similarity with every previously accepted update stays below a threshold $\varepsilon$. Concretely, we flatten each update as ${\delta_{i} = \mathrm{vec}(\Delta_{i})}$ and discard $\Delta_i$ whenever
\[
    \max_{1\le m \le r}\,\mathrm{sim}(\delta_{i}, \delta_{a_m}) \;>\;\varepsilon\,,
\]
where $\{\Delta_{a_1}, \dots, \Delta_{a_r}\}$ are the already-accepted updates and $\varepsilon$ is a user-specified threshold (e.g., $\varepsilon=0.3$). This prevents highly similar subspaces from overshadowing less common ones (see lines \ref{line:start-discard-directions}--\ref{line:end-discard-directions} of \cref{alg:fixed-merging}).

\subsection{Adaptive merging and inference}
Once the router has selected a subset $\Omega$ of relevant tasks, we merge their subspaces via \baseline{TSV-M} into a single model $\theta_{\text{MASS}}$. Because the task is unknown, we evaluate every selected head: after obtaining the shared representation $\mathbf{z}_{L-1} \in \mathbb{R}^{d}$ from $\theta_{\text{MASS}}$, we compute the logits of each head $h_i : \mathbb{R}^{d} \to \mathbb{R}^{C_i}$ for $i \in \Omega$:
\[
  \mathbf{z}_i = h_i(\mathbf{z}_{L-1}), 
  \quad 
  \mathbf{z}_i \in \mathbb{R}^{C_i}.
\]
We then pick the highest logit among all heads in \(\Omega\). Formally:
\[
  (i^{\star}, c^{\star}) 
  \;=\;
 \underset{(i,c)\,\in\,\Omega\times\{1,\dots,C_i\}}{\arg\max}\; z_i[c]\,.
\]
That is, we pick the most confident head $i^{\star}$ and its predicted class $c^{\star}$, allowing the model to determine both the label space and the prediction on a per-input basis. Because the cost scales with $|\Omega|$ rather than the total number of tasks, inference overhead remains modest even with large task libraries.

\subsection{Residual minimization as maximum a posteriori estimation}
The task selection process in \mass{} admits a probabilistic interpretation as \textit{maximum a posteriori} (MAP) estimation. If residuals follow an isotropic Gaussian, the likelihood of a feature vector given a task decays exponentially with the squared $\ell_2$ reconstruction error, so choosing the task with the smallest residual is equivalent to the MAP estimate. This parallels probabilistic PCA \citep{tipping1999probabilistic}, where reconstruction-error minimization corresponds to maximum likelihood under the same Gaussian assumption. Because \mass{} lacks training data to fit richer distributions, this least-informative model is a natural choice: the isotropic prior treats all directions equally, avoiding bias toward any particular task.

\begin{proposition} \label{thm:MAP_l2}
Let \(\mathbf{z}_\ell \in \mathbb{R}^d\) be a feature vector, and for each task \(i\), decompose it as  
\[
  \mathbf{z}_\ell = V_i V_i^{\top} \mathbf{z}_\ell + \varepsilon_i,
  \qquad
  \varepsilon_i = \bigl(I - V_i V_i^{\top}\bigr) \mathbf{z}_\ell .
\]  
Assume \(\varepsilon_i \sim \mathcal{N}(0, \sigma^2 I)\) and a uniform prior over tasks: \(p(\text{task}=i) = \frac{1}{K}\) for all \(i \in \{1, 2, \ldots, K\}\). Then the maximum a posteriori estimate of the task reduces to  
\[
  \hat{\imath}_{\mathrm{MAP}}
  = \arg\max_i p(\text{task}=i \mid \mathbf{z}_\ell)
  = \arg\min_i \|\varepsilon_i\|_2^2 .
\]  
\end{proposition}

%% file: 4_multi_task/3_mass/algorithms/adaptive.tex
\begin{algorithm}
    \caption{Adaptive Merging Step}
    \label{alg:adaptive-merging}
    \begin{algorithmic}[1]
      \REQUIRE Pretrained model weights $\theta_{\text{pre}}$, task-specific updates $\{\Delta_i\}_{i=1}^\ntasks$, fixed merged model $\theta_{\text{MT}}$, top-$k$ parameter $k$, threshold $\eta$, task-specific classification heads $\{h_i\}_{i=1}^\ntasks$, sample $\mathbf{x}$
      \ENSURE Predicted class $c^*$
      \STATE $\mathbf{z}_{\ell} \gets \text{ForwardPass}(\theta_{\text{MT}}, \mathbf{x})$ \Comment{first pass}
      \FOR{$i=1,\dots,\ntasks$}
      \STATE $\quad r_i \gets \|\mathbf{z}_{\ell} - V_i\,V_i^\top\,\mathbf{z}_{\ell} \|_2$ \Comment{residual as per \cref{eq:router-residual}}
      \ENDFOR
      \STATE $w \gets \operatorname{softmax}(-r)$
      \STATE $\Omega \gets \{i \,:\, w_i \geq \eta\}$ \Comment{Select tasks above threshold}
      \STATE $\Omega \gets \text{TopK}(\Omega, w, k)$ \Comment{Keep only top-$k$ weighted tasks}
      \STATE \textbf{Merge selected subspaces:} $\quad \Delta_{\text{ada}} \gets \sum_{i\in\Omega} \,U_i\,\Sigma_i\,V_i^\top$
      \STATE \textbf{Compute adaptive model:} $\quad \theta_{\text{MASS}} \gets \theta_{\text{pre}} + \alpha\,\Delta_{\text{ada}}$
      \STATE \textbf{Classification procedure}
      \STATE $\mathbf{z}_{L-1} \gets \text{ForwardPass}(\theta_{\text{MASS}},\mathbf{x})$ \Comment{Compute shared representation}
      \STATE $\mathbf{z}_{i} \gets h_i(\mathbf{z}_{L-1})$ \Comment{Evaluate each head}
      \STATE $\displaystyle (i^\star, c^\star) \gets \hspace{-0.1cm}  \argmax_{\scaleto{(i,c)\in\Omega\times\{1,\dots,C_i\}}{5pt}} \hspace{-0.1cm} z_i[c]$ \Comment{Highest logit across heads}
        \STATE \textbf{return} $c^\star$
    \end{algorithmic}
  \end{algorithm}

%% file: 4_multi_task/3_mass/algorithms/fixed.tex
\begin{algorithm}
    \caption{Fixed Merging Step}
    \label{alg:fixed-merging}
    \begin{algorithmic}[1]
      \REQUIRE Pretrained model weights $\theta_{\text{pre}}$, task-specific updates $\{\Delta_i\}_{i=1}^\ntasks$,
      user-specified threshold $\varepsilon$
      \ENSURE Fixed merged model weights 
      $\theta_{\text{MT}}$ 
      \\
   \STATE \textbf{Accounting for redundant directions} (\cref{sec:accounting_redundant})
      \STATE $\mathcal{M}=\{\}$ \label{line:start-discard-directions}
      \FOR{$i=1,\dots,\ntasks$}
      \STATE $\delta_i \gets \text{vec}(\Delta_i)$ 
  \IF{
  $ \max_{\{j \in \mathcal{M}\}} \text{sim}(\delta_i, \delta_j) < \varepsilon$
  }
          \STATE $\mathcal{M} \gets \mathcal{M} \cup \{i\}$
        \ENDIF
      \ENDFOR \label{line:end-discard-directions}
 \STATE \textbf{Merging step using} \baseline{TSV-M} (\cref{ch:tsv}) \textbf{on the} $\{ \Delta_{i}\}_{i \in \mathcal{M}}$
      \FOR{$i \in \mathcal{M}$} \label{line:start-fixed_merging}
      \STATE $\Delta_i = U_i\,\Sigma_i\,V_i^\top$ 
      \STATE $\tilde{U}_i \gets U_{i{[:,1:k]}}$, $\tilde{\Sigma}_i \gets \Sigma_{i{[1:k,1:k]}}$, $\tilde{V}_i \gets V_{i{[:,1:k]}}$
      \ENDFOR
      \STATE \quad $U \gets [\tilde{U}_1\,|\,\tilde{U}_2\,|\,\cdots\,|\,\tilde{U}_\ntasks]$
      \STATE \quad $\Sigma \gets \operatorname{block\_diag}(\tilde{\Sigma}_1,\tilde{\Sigma}_2,\dots,\tilde{\Sigma}_\ntasks)$
      \STATE \quad $V \gets [\tilde{V}_1\,|\,\tilde{V}_2\,|\,\cdots\,|\,\tilde{V}_\ntasks]$
      \STATE \quad  $U_\perp \gets \operatorname{orthogonalize}(U)$
      \STATE \quad  $V_\perp \gets \operatorname{orthogonalize}(V)$
      \STATE \quad $\hat{\Delta} \gets U_\perp\,\Sigma\,V_\perp^\top$
    \STATE \quad  $\theta_{\text{MT}} \gets \theta_{\text{pre}} + \alpha\ \hat{\Delta}$ \label{line:end-fixed_merging}
       \STATE \textbf{return} $\theta_{\text{MT}}$ 
    \end{algorithmic}
  \end{algorithm}

%% file: 4_multi_task/3_mass/figures/projection.tex
  \tdplotsetmaincoords{105}{-30}
  \tdplotsetrotatedcoords{0}{30}{0}

  \begin{tikzpicture}[tdplot_main_coords,font=\sffamily,scale=0.5]

    \begin{scope}[tdplot_rotated_coords]

      \begin{scope}[canvas is xy plane at z=0]
        \fill[myyellow,fill opacity=0.3] (-2,-3) rectangle (2,3);
        \draw[->,thick] (0,0) -- (2,0) node[midway,below] {$\mathbf{v}_1$};
        \draw[->,thick] (0,0) -- (-0.9, 2) node[midway,left, yshift=5pt] {$\mathbf{v}_2$};
      \end{scope}

      \coordinate (O)      at (0,0,0);
      \coordinate (Xabove) at (1,1,2);  
      \coordinate (Xproj)  at (1,1,0);  

      \draw[->,thick,myred] (O) -- (Xabove)
      node[pos = 0.4,right,xshift=7pt] {$\mathbf{z}_{\ell}$};

      \draw[->,thick,myblue] (O) -- (Xproj)
      node[midway,below left,xshift=3pt, yshift=-2pt] {$\mathbf{z}_{\ell, \text{proj}}$};
      \draw[dashed] (Xabove) -- (Xproj);


    \end{scope}

    \pgfmathsetmacro{\R}{2.2}
    \draw[->] (0,0,0) -- (\R,0,0) node[pos=1.0,below] {$x$};
    \draw[->] (0,0,0) -- (0,\R,0) node[pos=1.0,left ] {$y$};
    \draw[->] (0,0,0) -- (0,0,\R) node[pos=1.0,above]{$z$};

  \end{tikzpicture}

%% file: 4_multi_task/3_mass/sec/4_experiments.tex
\input{tables/main_results.tex}

\input{tables/llms.tex}

\subsection{Merging performance}

\paragraph{Models, baselines and datasets}
We experiment with three \baseline{CLIP}~\citep{radford2021learningtransferablevisualmodels} variants using \model{ViT-B-32}, \model{ViT-B-16}, and \model{ViT-L-14} visual encoders~\citep{dosovitskiy2021imageworth16x16words}. Baselines include weight averaging, \baseline{Task Arithmetic}~\citep{task-vectors}, and \baseline{Consensus Merging}~\citep{wang2024localizing}; zero-shot performance and the mean accuracy of individually fine-tuned models serve as lower and upper reference points, respectively.
We evaluate on collections of 8, 14, and 20 tasks (see \cref{app:datasets} for details), reporting both absolute and normalized accuracy.

\paragraph{MoErging results}
\cref{tab:main-results} reports the average absolute accuracy and the corresponding normalized accuracy (subscript, also in percentage) for each method, model size, and number of vision tasks.
To adapt \texttt{SMILE} and \texttt{WeMoE} to our general setting, we employ a majority-voting-based heuristic to route the classification head. 
\begin{wrapfigure}[14]{l}{0.4\linewidth}
  \vspace{-0.5cm}
  \centering
  \includegraphics[width=\linewidth]{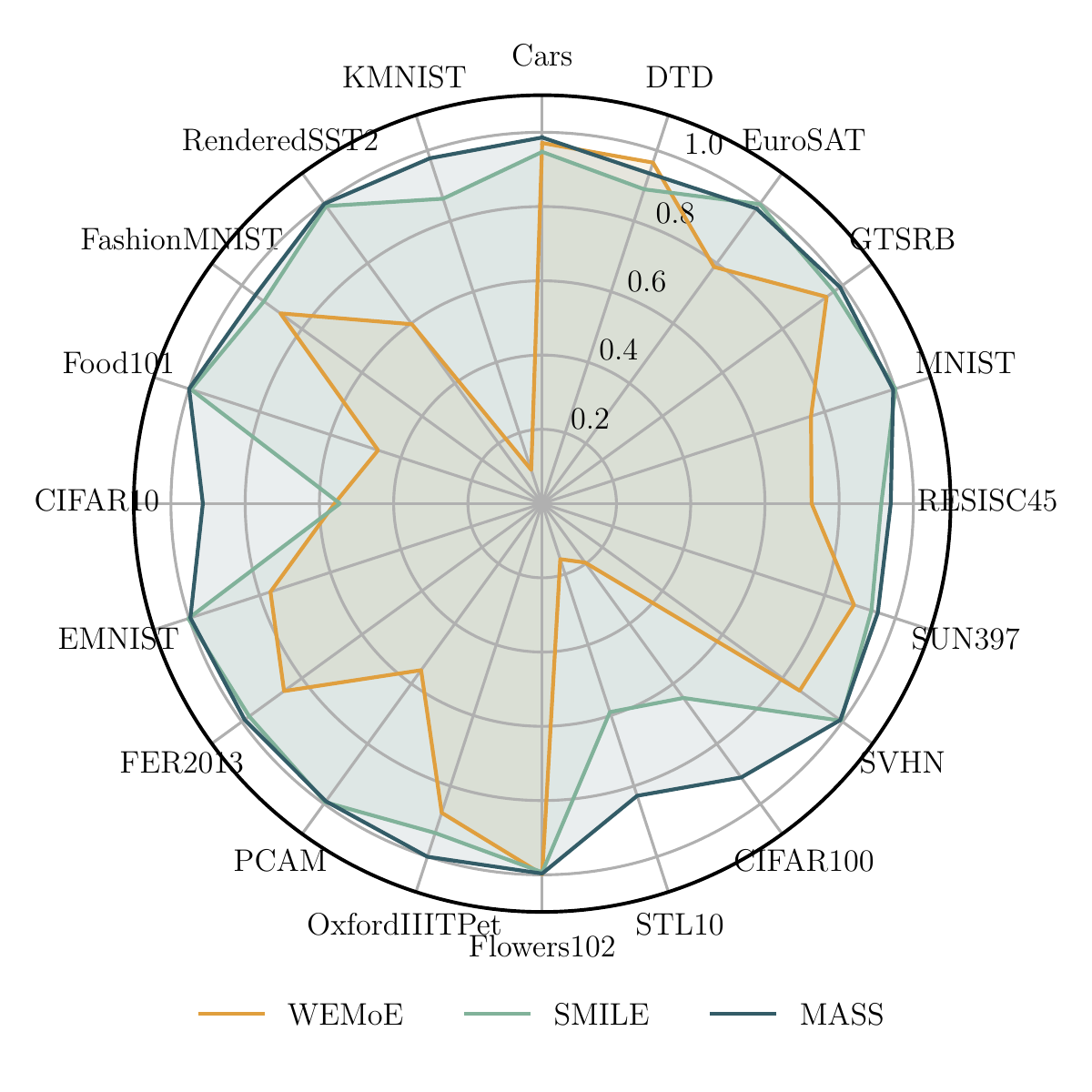}
  \caption{Per-dataset accuracies.}
  \label{fig:per-dataset-accuracies-l14}
\end{wrapfigure}
Details in \cref{subsec:adapting-moerging-methods}.
\mass{} sets a new state of the art for MoErging in 8 out of 9 benchmarks, with gains up to $\approx \mathbf{6\%}$ over the best-performing baseline. \Cref{fig:per-dataset-accuracies-l14} shows a per-dataset breakdown, confirming that accuracy is consistently high across all datasets rather than skewed toward easier ones. Storage-wise, \mass{} requires a constant $2\times$ parameter increase, whereas other MoErging baselines range from ${\sim}2.5\times$ to ${\sim}14\times$ (see \cref{subsec:storage-compute-overhead}).

For reference, we also report fixed merging methods that assume oracle access to the correct classification head. Despite operating in a strictly harder setting, \mass{} outperforms them by a consistent margin: compared with \texttt{TSV-M}, the routing mechanism yields an $\approx \mathbf{5\%}$ accuracy increase. On the 8-task benchmark (the only vision setting with reported results from \baseline{TwinMerging}~\citep{twinmerging}), \mass{} achieves \textbf{97.6\%} normalized accuracy versus 95.3\% for \baseline{TwinMerging}, despite the latter assuming oracle knowledge of the classification head.
\Cref{tab:llms-results} extends the evaluation to 8 \model{Flan-T5-Base} models~\citep{chung2024scaling} fine-tuned on GLUE~\citep{wangglue} language tasks, where \mass{} again yields the best results, confirming that its benefits are modality-agnostic.

\begin{figure}
  \begin{subfigure}{0.5\textwidth}
    \centering
    \includegraphics[width=\textwidth]{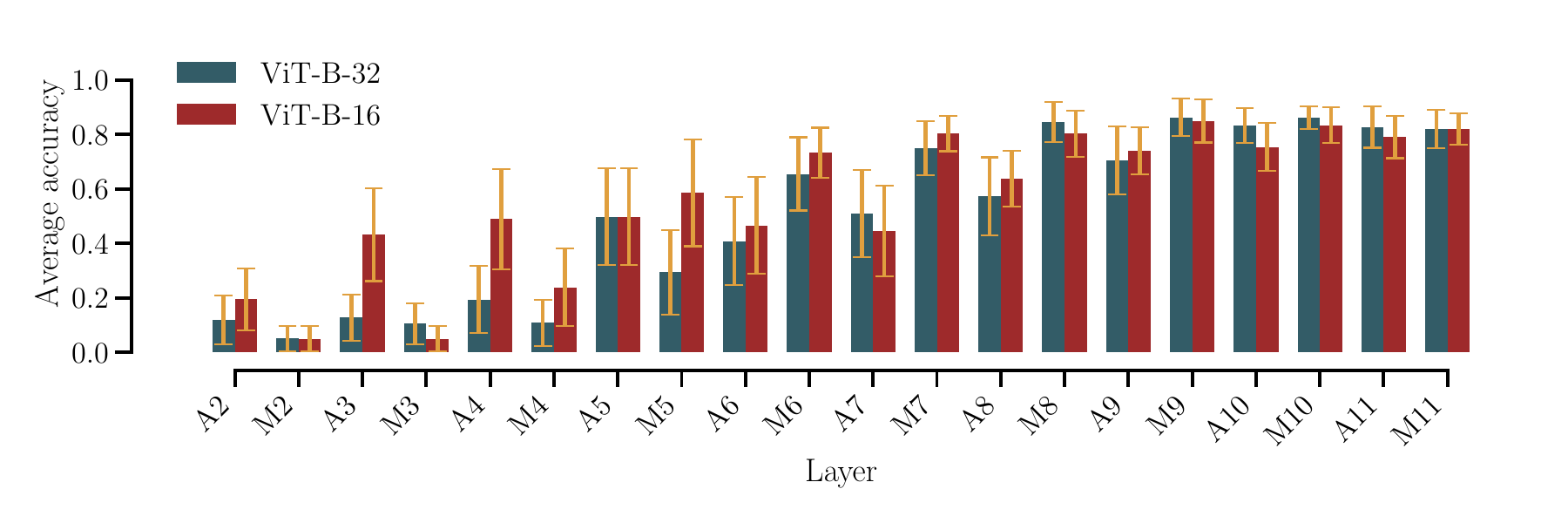}
    \caption{Averaged across all tasks for two models.}
    \label{fig:layer-task-accuracies-b32b16}
  \end{subfigure}
  \hfill
  \hfill
  \begin{subfigure}{0.5\textwidth}
    \centering
    \includegraphics[width=\textwidth]{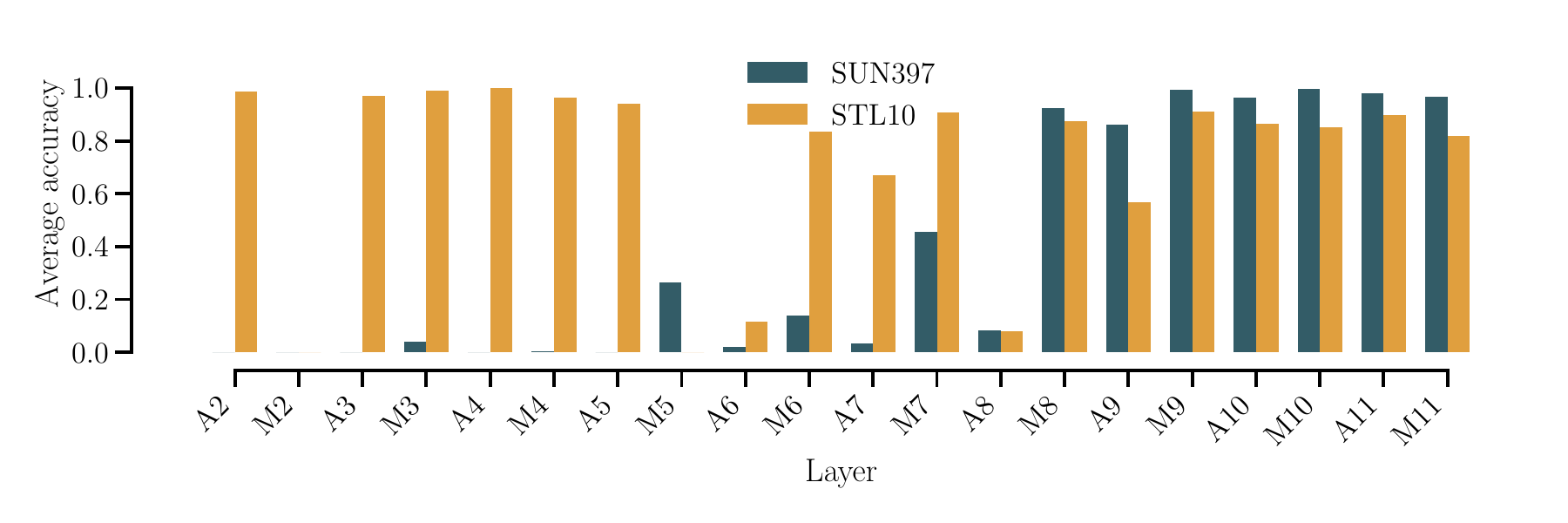}
    \caption{Focusing on \dataset{SUN397} and \dataset{STL10} for a \vitsmall{}.}
    \label{fig:layer-task-accuracies-sun-stl}
  \end{subfigure}
  \caption[Per-layer task accuracies for \model{ViT-B-32}]{Per-layer task accuracies for \model{ViT-B-32} on the 20-task benchmark. Layers starting with `A' indicate attention layers, while those starting with `M' refer to MLPs.}
  \label{fig:layer-task-accuracies-sun397stl10}
\end{figure}

\paragraph{Batched inference}
\begin{wrapfigure}[10]{r}{0.36\linewidth}
  \vspace{-0.6cm}
  \centering
  \includegraphics[width=\linewidth]{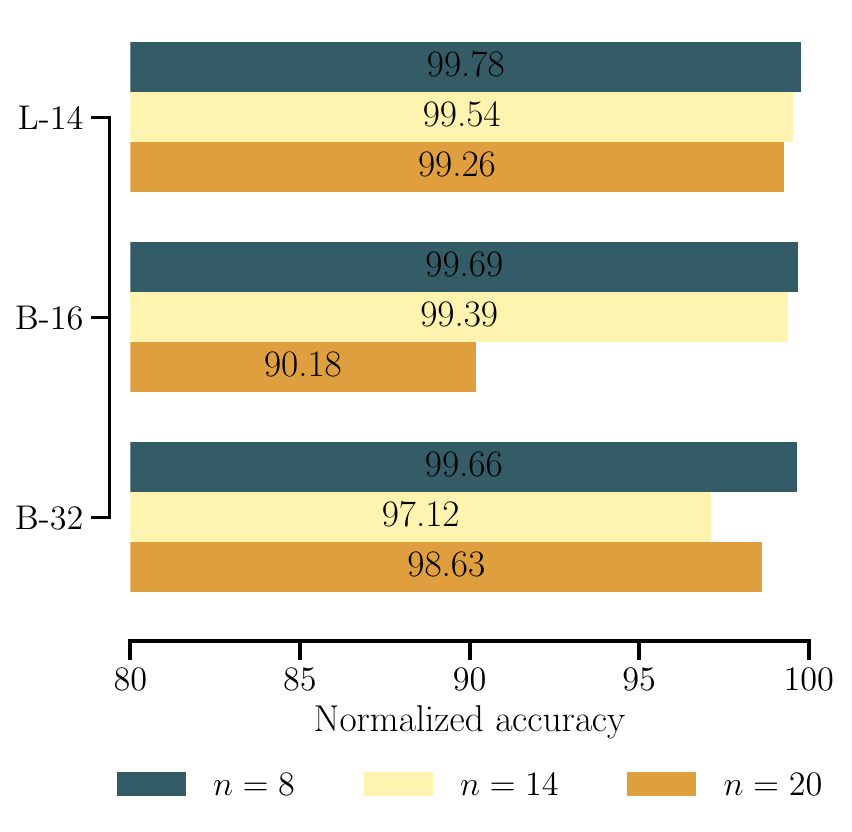}
  \caption{Batched accuracy.}
  \label{fig:batched_performance}
\end{wrapfigure}
The main experiments consider the hardest scenario, where each sample may belong to a different task. In practice, batched requests often share the same task, allowing the router to select a single merged model per batch. As \Cref{fig:batched_performance} shows, this closes the gap almost entirely: \mass{} reaches at least $97\%$ normalized accuracy in 8 out of 9 settings, effectively matching individually fine-tuned models.

\subsection{Choosing a routing layer} \label{sec:choosing_routing_layer}
\Cref{fig:layer-task-accuracies-b32b16} shows task prediction accuracy when routing at different layers for \vitsmall{} and \vitmedium{}. The best layer is consistent across architectures (layer 9 for both), with MLP layers slightly outperforming self-attention layers. However, the best layer varies substantially across tasks, with per-layer variance reaching $40\%$. \Cref{fig:layer-task-accuracies-sun-stl} illustrates this for a \vitsmall{}: \dataset{STL10} benefits from earlier layers ($\ell=3$--$5$), while \dataset{SUN397} peaks at later ones ($\ell=9$--$11$). This suggests that adaptive, per-task layer selection is a promising direction for future work. An analogous analysis for \vitlarge{} is provided in the appendix.

\subsection{Comparison with other routers}
We compare our router with two alternatives that require progressively more supervision.

\emph{Nearest Neighbor} (\texttt{nn}) stores a small support set of validation embeddings per task and classifies each test sample by cosine similarity to its nearest stored neighbor. It requires no learnable parameters but assumes access to per-task validation data.

\emph{MLP router} (\texttt{mlp}) trains a small MLP $f_\theta$ on the union of per-task validation sets to predict task identity from $\mathbf{z}_{\ell}$ via cross-entropy (details in \cref{app:additional_details}).

\paragraph{Results}
\begin{wraptable}[11]{r}{0.3\linewidth}
    \centering
    \vspace{-0.3cm}
    \resizebox{1\linewidth}{!}{%
        \begin{tabular}{c ccc}
            \toprule
            \mass{}                      & \multicolumn{3}{c}{\model{ViT-L-14}}                                \\
            \cmidrule(lr){2-4}
            +                                & 8 tasks                              & 14 tasks & 20 tasks          \\
            \midrule
            \method{nn}                      &                                       94.0     & 92.1     & 92.0   \\
            \method{mlp}                     &                                       98.9     & 99.5     & 98.3   \\
            \cdashlinelr{1-4}
            \method{proj}$_{\texttt{PRE}}$   &                                       99.1     & 97.7     & 91.9   \\ \rowcolor{mygreen!50}
            \method{proj}$_{\texttt{TSV-M}}$ &                                      $98.6$   & $97.3$   & $96.6$ \\
            \bottomrule
        \end{tabular}
    }
    \caption{Average normalized accuracy for different routers.}
    \label{tab:routers}
\end{wraptable}
\cref{tab:routers} compares routing strategies. The \texttt{nn} router performs well but remains slightly below our best results. The \texttt{mlp} router achieves the highest accuracy overall, yet it relies on labeled validation data, contradicting the data-free premise of merging.

Among projection-based routers, routing from \texttt{TSV-M} (\texttt{proj}$_{\texttt{TSV-M}}$) outperforms routing from the pre-trained backbone (\texttt{proj}$_{\texttt{PRE}}$) as the number of tasks grows, with the gap widening to $\approx 10\%$ for a \vitsmall{} (\cref{tab:remaining-routers}). This highlights a key insight: because \texttt{TSV-M} embeds each task's top singular vectors into distinct, orthogonal subspaces within a single model, a simple projection suffices to identify the relevant subspace, requiring no labels or additional training.

\subsection{Interpreting task singular vectors}\label{subsec:exp-decoding-text}
Finally, we interpret the task singular vectors used for routing. Since mid-layer embeddings in CLIP-like models are known to preserve semantic content, routing at these layers should rely on meaningful features. To verify this, we decode TSVs using \textsc{TextSpan}~\citep{gandelsman2024interpreting, basile2024residual}, which iteratively identifies the most influential text directions explaining how weight changes affect the representation. While originally designed for image embeddings, we apply it here to singular vectors derived from weight updates.

As \Cref{fig:decoding-exp} shows, TSVs capture meaningful visual-language associations: the singular vectors for \dataset{Cars}~\citep{krause_3d_2013} activate the description \textsc{``Image of a car''}, while those for \dataset{DTD}~\citep{cimpoi_describing_2014} align with \textsc{``Close-up of a textured mesh''}. Notably, the decoded concepts are consistent across architectures: \model{ViT-L-14} at layer 21 and \model{ViT-B-32} at layer 10 yield nearly identical phrases for the same tasks, hinting at the possibility of transferring semantic information across architectures via text embeddings.

Interpretability depends strongly on the routing layer. Mid-to-late layers (e.g., layer 10 in \model{ViT-B-32}, layer 21 in \model{ViT-L-14}) produce semantically coherent descriptions, while early layers yield irrelevant concepts such as \textsc{``Image with a penguin''} for \dataset{DTD} or \textsc{``Photo taken in Santorini''} for \dataset{MNIST}. This confirms that early layers encode generic features, whereas deeper layers specialize toward domain-specific semantics.

These analyses validate the routing strategy: mid-layer TSVs capture precisely the semantic differences the router exploits for task discrimination. More broadly, they reveal a direct correspondence between TSVs and the semantic structure of the underlying data.

\begin{figure}[t]
  \centering
  \begin{subfigure}[t]{0.49\textwidth}
    \centering
    {\scriptsize \textcolor{myblue}{\texttt{L-14}, L21: \textsc{Image of a car}}}\\
    {\scriptsize \textcolor{myblue}{\texttt{B-32}, L10: \textsc{Image of a car}}}\\
    {\scriptsize \textcolor{myred}{\texttt{B-32}, L3: \textsc{Aerial view of a hamlet}}}\\
    [3pt]
    \begin{tikzpicture}
      \node[draw=myblue!90,rounded corners=3pt,line width=0.4pt]
      {\includegraphics[width=0.16\linewidth]{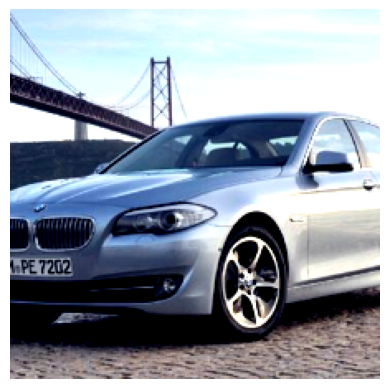}};
    \end{tikzpicture}
    \hspace{2pt}
    \begin{tikzpicture}
      \node[draw=myblue!90,rounded corners=3pt,line width=0.4pt]
      {\includegraphics[width=0.16\linewidth]{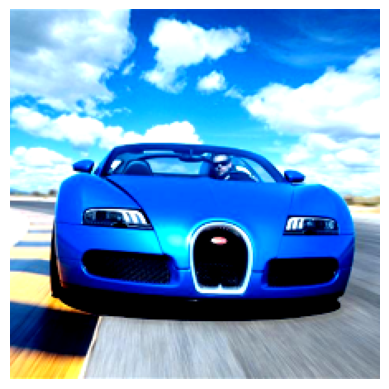}};
    \end{tikzpicture}
    \hspace{2pt}
    \begin{tikzpicture}
      \node[draw=myblue!90,rounded corners=3pt,line width=0.4pt]
      {\includegraphics[width=0.16\linewidth]{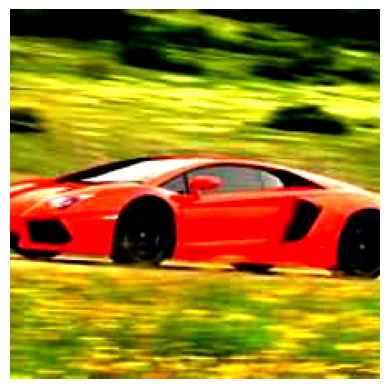}};
    \end{tikzpicture}
    \caption{\dataset{Cars}}
  \end{subfigure}
  \hfill
  \begin{subfigure}[t]{0.49\textwidth}
    \centering
    {\scriptsize \textcolor{myblue}{\texttt{L-14}, L21: \textsc{Close-up of a textured synthetic mesh}}}\\
    {\scriptsize \textcolor{myblue}{\texttt{B-32}, L10: \textsc{Close-up of a textured mesh}}}\\
    {\scriptsize \textcolor{myred}{\texttt{B-32}, L3: \textsc{Image with a penguin}}}\\
    [3pt]
    \begin{tikzpicture}
      \node[draw=myblue!90,rounded corners=3pt,line width=0.4pt]
      {\includegraphics[width=0.16\linewidth]{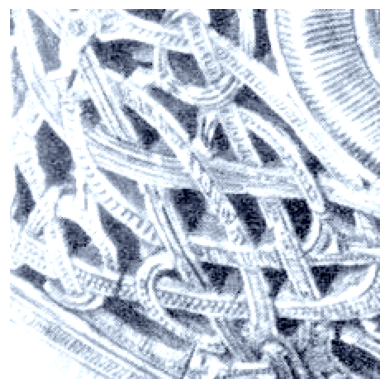}};
    \end{tikzpicture}
    \hspace{2pt}
    \begin{tikzpicture}
      \node[draw=myblue!90,rounded corners=3pt,line width=0.4pt]
      {\includegraphics[width=0.16\linewidth]{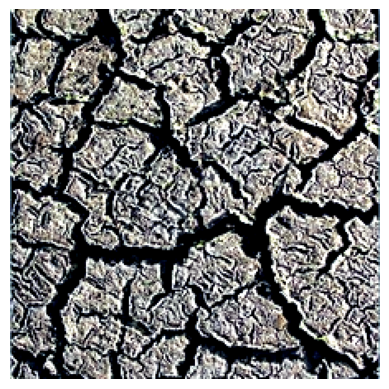}};
    \end{tikzpicture}
    \hspace{2pt}
    \begin{tikzpicture}
      \node[draw=myblue!90,rounded corners=3pt,line width=0.4pt]
      {\includegraphics[width=0.16\linewidth]{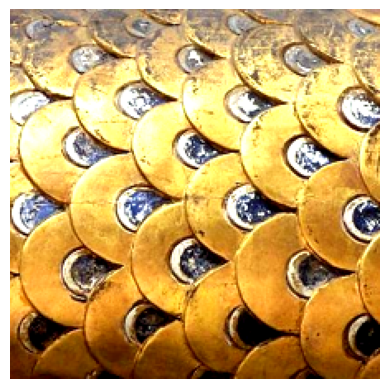}};
    \end{tikzpicture}
    \caption{\dataset{DTD}}
  \end{subfigure}

  \vspace{0.2cm}

  \begin{subfigure}[t]{0.475\textwidth}
    \centering
    {\scriptsize \textcolor{myblue}{\texttt{L-14}, L21: \textsc{An image of the number 9}}}\\
    {\scriptsize \textcolor{myblue}{\texttt{B-32}, L10: \textsc{An image of the number 8}}}\\
    {\scriptsize \textcolor{myred}{\texttt{B-32}, L3: \textsc{Photo taken in Santorini}}}\\[3pt]
    \begin{tikzpicture}
      \node[draw=myblue!90,rounded corners=3pt,line width=0.4pt]
      {\includegraphics[width=0.16\linewidth]{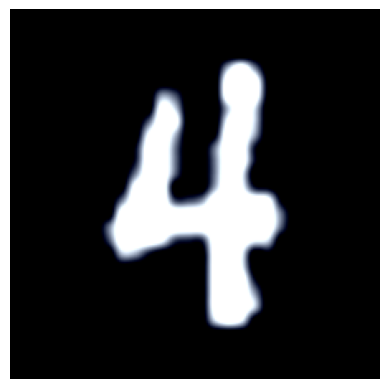}};
    \end{tikzpicture}
    \hspace{1pt}
    \begin{tikzpicture}
      \node[draw=myblue!90,rounded corners=3pt,line width=0.4pt]
      {\includegraphics[width=0.16\linewidth]{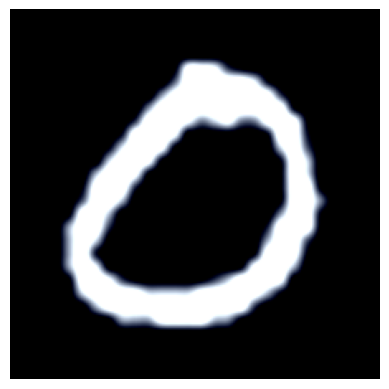}};
    \end{tikzpicture}
    \hspace{1pt}
    \begin{tikzpicture}
      \node[draw=myblue!90,rounded corners=3pt,line width=0.4pt]
      {\includegraphics[width=0.16\linewidth]{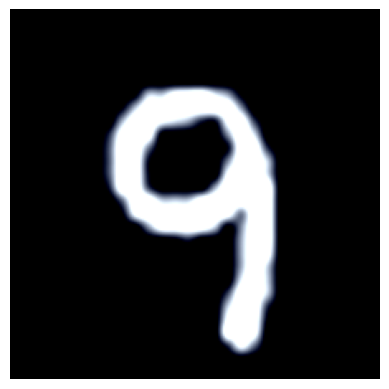}};
    \end{tikzpicture}
    \caption{\dataset{MNIST}}
  \end{subfigure}
  \hfill
  \begin{subfigure}[t]{0.475\textwidth}
    \centering
    {\scriptsize \textcolor{myblue}{\texttt{L-14}, L21:  \textsc{Aerial view of an agricultural field}}}\\
    {\scriptsize \textcolor{myblue}{\texttt{B-32}, L10: \textsc{Aerial view of a farmland}}}\\
    {\scriptsize \textcolor{myred}{\texttt{B-32}, L3: \textsc{Rural windmill silhouette}}}\\
    [3pt]
    \begin{tikzpicture}
      \node[draw=myblue!90,rounded corners=3pt,line width=0.4pt]
      {\includegraphics[width=0.16\linewidth]{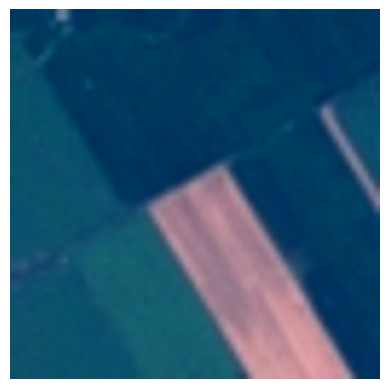}};
    \end{tikzpicture}
    \hspace{1pt}
    \begin{tikzpicture}
      \node[draw=myblue!90,rounded corners=3pt,line width=0.4pt]
      {\includegraphics[width=0.16\linewidth]{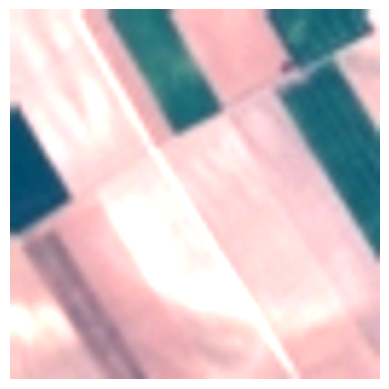}};
    \end{tikzpicture}
    \hspace{1pt}
    \begin{tikzpicture}
      \node[draw=myblue!90,rounded corners=3pt,line width=0.4pt]
      {\includegraphics[width=0.16\linewidth]{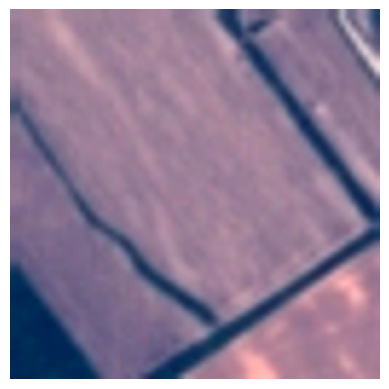}};
    \end{tikzpicture}
    \caption{\dataset{EuroSAT}}
  \end{subfigure}

  \caption[Decoding task singular vectors as text]{Captions obtained by decoding task singular vectors as text as described in \cref{subsec:exp-decoding-text}, accompanied by task representative images. Captions produced by the task singular vectors of predictive layers reflect the task content, those obtained by non-predictive ones do not.}
  \label{fig:decoding-exp}
\end{figure}

%% file: 4_multi_task/3_mass/tables/main_results.tex
\renewcommand{\arraystretch}{1.4}
\begin{table}
    \centering
    \resizebox{\textwidth}{!}{
        \begin{tabular}{cc ccc ccc ccc}
            %
            %
            \toprule
                                                                       & \multirow{2}{*}{\textbf{Method}}              & \multicolumn{3}{c}{\model{ViT-B-32}} & \multicolumn{3}{c}{\model{ViT-B-16}} & \multicolumn{3}{c}{\model{ViT-L-14}}                                                                                                                                                                   \\
            \cmidrule(lr){3-5} \cmidrule(lr){6-8} \cmidrule(lr){9-11}
                                                                       &                                               & 8 tasks                              & 14 tasks                             & 20 tasks                             & 8 tasks                  & 14 tasks                 & 20 tasks                 & 8 tasks                  & 14 tasks                 & 20 tasks                 \\

            \cmidrule(r){2-2} \cmidrule(lr){3-3}\cmidrule(lr){4-4} \cmidrule(lr){5-5} \cmidrule(lr){6-6} \cmidrule(lr){7-7} \cmidrule(lr){8-8} \cmidrule(lr){9-9} \cmidrule(lr){10-10} \cmidrule(lr){11-11} \rowcolor{gray!15}
            %
            %
            \cellcolor{white}                                          & \multicolumn{1}{c}{\method{Zero-shot}}         & $48.1_{(54.8)}$ & $56.9_{(64.5)}$ & $57.5_{(65.2)}$ & $55.3_{(60.6)}$ & $61.9_{(67.9)}$ & $62.5_{(68.3)}$ & $64.9_{(69.2)}$ & $69.1_{(73.8)}$ & $68.2_{(72.7)}$                        \\ \rowcolor{gray!15}
            \cellcolor{white} \multirow{-2}{*}{\small \vertical{Base}} & \multicolumn{1}{c}{\method{Fine-tuned}}        & $90.3_{(100)}$                      & $89.0_{(100)}$                      & $89.5_{(100)}$                      & $92.4_{(100)}$          & $91.3_{(100)}$          & $91.9_{(100)}$          & $94.2_{(100)}$          & $93.4_{(100)}$          & $94.0_{(100)}$          \\
            \cdashlinelr{1-11}
            %
            %
            \multirow{6}{*}{\small \vertical{Fixed}}                   & \multicolumn{1}{c}{\method{Weight Averaging}} & $67.1_{(74.6)}$ & $65.5_{(73.8)}$ & $64.4_{(72.6)}$ & $73.0_{(78.8)}$ & $70.8_{(77.3)}$ & $69.2_{(75.3)}$ & $80.5_{(85.2)}$ & $78.5_{(83.8)}$ & $76.1_{(80.9)}$                        \\
                                                                       & \multicolumn{1}{c}{\method{Task Arithmetic}}  &$68.8_{(75.7)}$ &$64.6_{(72.5)}$ &$64.0_{(71.9)}$ &$73.0_{(78.3)}$ &$70.6_{(77.0)}$ &$69.0_{(75.0)}$ &$84.4_{(89.3)}$ &$80.4_{(85.8)}$ &$76.9_{(81.7)}$ \\
                                                                       & \multicolumn{1}{c}{\method{Consensus TA}}     & $72.6_{(80.1)}$ & $70.3_{(78.9)}$ & $68.5_{(76.8)}$ & $75.9_{(81.7)}$ & $74.9_{(81.7)}$ & $72.2_{(78.4)}$ & $85.5_{(90.5)}$ & $82.0_{(87.6)}$ & $78.9_{(83.8)}$ \\
                                                                       & \multicolumn{1}{c}{\method{TSV-M}}            & $83.2_{(91.8)}$                      & $78.6_{(88.0)}$                      & $75.6_{(84.3)}$                      & $85.5_{(92.2)}$          & $81.4_{(88.8)}$          & $78.8_{(85.5)}$          & $91.2_{(96.7)}$          & $88.8_{(94.9)}$          & $87.5_{(93.0)}$          \\
                                                                       & \multicolumn{1}{c}{\method{Iso-C}}            &$82.8_{(91.7)}$ & $78.4_{(88.0)}$ & $73.2_{(81.9)}$ & $87.5_{(94.4)}$ & $79.8_{(87.0)}$ & $75.3_{(81.6)}$ & $92.6_{(98.2)}$ & $89.6_{(95.8)}$ & $86.8_{(92.3)}$ \\
                                                                       & \multicolumn{1}{c}{\method{Iso-CTS}}          & $82.0_{(90.9)}$ & $80.6_{(90.4)}$ & $77.0_{(86.2)}$ & $88.7_{(95.9)}$ & $84.1_{(91.8)}$ & $80.7_{(87.7)}$ & $92.8_{(98.5)}$ & $\mathbf{91.1_{(97.4)}}$ & $89.2_{(94.9)}$                       \\
            \cdashlinelr{1-11}

            %
            %
            \cellcolor{white}                                          & \multicolumn{1}{c}{\method{WeMoE}}            & $\mathbf{88.8_{(97.5)}}$             & $74.3_{(82.8)}$                      & $68.2_{(76.3)}$                      & $89.1_{(96.4)}$          & $76.6_{(83.2)}$          & $65.0_{(70.5)}$          & $88.7_{(94.2)}$          & $72.3 _{(76.8)}$         & $65.0 _{(69.4)}$         \\
            \multirow{-1}{*}{\small \vertical{MoE}}
                                                                       & \multicolumn{1}{c}{\method{SMILE-1}}          & $83.2_{(92.1)}$                      & $75.4_{(84.5)}$                      & $72.8_{(82.3)}$                      & $87.8_{(94.9)}$          & $81.7_{(89.5)}$          & $79.5_{(86.7)}$          & $91.2_{(96.7)}$          & $86.6_{(92.7)}$          & $84.9_{(90.5)}$          \\
                                                                       & \multicolumn{1}{c}{\method{SMILE-2}}          & $84.4_{(93.5)}$                      & $76.4_{(85.6)}$                      & $74.1_{(83.8)}$                      & $89.0_{(96.2)}$          & $82.7_{(90.7)}$          & $80.4_{(87.7)}$          & $92.0_{(97.6)}$          & $87.1_{(93.4)}$          & $85.5_{(91.1)}$          \\
            \rowcolor{mygreen!50}

                                                                       \cellcolor{white} & \textbf{\mass{}}                          & $87.0_{(96.5)} $                     & $\mathbf{82.9_{(93.2)}}$             & $\mathbf{81.1_{(90.9)}}$             & $\mathbf{90.6_{(98.0)}}$ & $\mathbf{87.8_{(96.1)}}$ & $\mathbf{81.1_{(88.7)}}$ & $\mathbf{92.9_{(98.6)}}$ & $90.9_{(97.3)}$ & $\mathbf{90.8_{(96.6)}}$ \\
            \bottomrule
        \end{tabular}
    }
    \caption[\mass{} model merging benchmark results]{Average absolute accuracy results on model merging benchmarks; subscript (in parentheses) is the normalized average accuracy.}
    \vspace{-0.2cm}
    \label{tab:main-results}
\end{table}

%% file: 4_multi_task/3_mass/tables/llms.tex

\begin{table}[t]
    \centering
    \small
    \resizebox{\textwidth}{!}{
        \begin{tabular}{cccccccccc}
            \toprule
            \textbf{Method}                           & {CoLA}                            & {MNLI}                             & {MRPC}                            & {QNLI}                            & {QQP}                             & {RTE}                             & {SST-2}                            & {STS-B}                            & {Avg.}                            \\
            \cmidrule(r){1-1} \cmidrule(r){2-2} \cmidrule(lr){3-3}\cmidrule(lr){4-4} \cmidrule(lr){5-5} \cmidrule(lr){6-6} \cmidrule(lr){7-7} \cmidrule(lr){8-8} \cmidrule(lr){9-9} \cmidrule(lr){10-10}
            \rowcolor{gray!15}
            \method{Finetuned}                        & 75.0$_{(100)}$                    & 83.4$_{(100)}$                     & 87.5$_{(100)}$                    & 91.5$_{(100)}$                    & 85.4$_{(100)}$                    & 85.9$_{(100)}$                    & 93.6$_{(100)}$                     & 88.7$_{(100)}$                     & 86.4$_{(100)}$                    \\ \cdashlinelr{1-10}
            \method{Weight Averaging}                 & 69.1$_{(92.1)}$                   & 62.6$_{(75.1)}$                    & 79.4$_{(90.7)}$                   & 89.8$_{(98.1)}$                   & 83.9$_{(98.2)}$                   & 81.2$_{(94.5)}$                   & 91.7$_{(97.9)}$                    & 73.2$_{(82.5)}$                    & 78.9$_{(91.3)}$                   \\
            \method{Task Arithmetic}                  & 70.5$_{(94.0)}$                   & 57.8$_{(69.3)}$                    & 78.4$_{(89.6)}$                   & 90.2$_{(98.6)}$                   & 83.6$_{(97.9)}$                   & 80.5$_{(93.7)}$                   & 92.3$_{(98.6)}$                    & 77.8$_{(87.7)}$                    & 78.9$_{(91.3)}$                   \\
            \method{Ties-Merging}                     & 70.3$_{(93.7)}$                   & 65.0$_{(77.9)}$                    & 78.9$_{(90.2)}$                   & 90.2$_{(98.6)}$                   & 83.5$_{(97.8)}$                   & 81.6$_{(95.0)}$                   & 91.7$_{(97.9)}$                    & 78.3$_{(88.3)}$                    & 79.9$_{(92.5)}$                   \\
            \method{Twin-Merging}               & \textbf{74.9}$_{\mathbf{}{(99.8)}}$                   & 83.7$_{(100.3)}$                    & \textbf{87.5}$_{\mathbf{(100.0)}}$                   & 90.3$_{(98.6)}$                   & 84.6$_{(99.0)}$                   & \textbf{86.6}$_{\mathbf{(100.8)}}$                   & 93.6$_{(100.0)}$                    & 88.8$_{(100.1)}$                    & \textbf{86.2}$_{\mathbf{(99.7)}}$                 \\
            \method{WeMoE}                            & 72.5$_{(96.6)}$                   & 79.0$_{(94.7)}$                    & 51.9$_{(59.3)}$                   & 89.3$_{(97.5)}$                   & 69.6$_{(81.4)}$                   & 81.5$_{(94.8)}$                   & 88.6$_{(94.6)}$                    & 82.1$_{(92.5)}$                    & 76.8$_{(88.9)}$                 \\
            \method{SMILE-1}                          & 72.0$_{(96.0)}$                   & 84.2$_{(101.0)}$                   & 84.3$_{(96.3)}$                   & 91.3$_{(99.8)}$                   & 84.7$_{(99.2)}$                   & 84.1$_{(97.9)}$                   & 93.3$_{(99.7)}$                    & 87.0$_{(98.1)}$                    & 85.1$_{(98.5)}$                   \\
            \method{SMILE-2}                          & 73.2$_{(97.6)}$                   & \textbf{84.2}$_{(\mathbf{101.0})}$ & 85.0$_{(97.1)}$                   & \textbf{91.3}$_{(\mathbf{99.8})}$ & 84.9$_{(99.4)}$                   & 84.8$_{(98.7)}$                   & 93.5$_{(99.9)}$                    & 87.3$_{(98.4)}$                    & 85.5$_{(99.0)}$                   \\
            \rowcolor{mygreen!50}\textbf{\mass{}} & 74.1$_{(98.8)}$ & 83.2$_{(99.8)}$                    & 85.8$_{(98.1)}$ & 90.9$_{(99.3)}$                   & \textbf{85.1}$_{(\mathbf{99.6})}$ & 84.9$_{(98.8)}$ & \textbf{94.2}$_{(\mathbf{100.6})}$ & \textbf{88.9}$_{(\mathbf{100.2})}$ & 85.9$_{(99.4)}$ \\
            \bottomrule
        \end{tabular}
    }
    \caption[Merging results on GLUE benchmark]{Accuracy when merging 8 fine-tuned models on the GLUE \citep{wangglue} benchmark; Normalized average accuracy in subscript. }
    \label{tab:llms-results}
\end{table}

%% file: 4_multi_task/3_mass/sec/7_conclusions.tex
\mass{} demonstrates that the low-rank task subspaces identified in \cref{ch:tsv} are useful not only at merge time but also at inference time: by routing each input to the most relevant subspaces and jointly selecting encoder and classification head, a single model handles all tasks without test-time supervision. The projection-based router is entirely training- and data-free, making it directly applicable to checkpoints downloaded from public repositories, and recovers nearly the full accuracy of individually fine-tuned experts at a fraction of their combined storage cost.

Refining the router for more precise subspace selection and extending \mass{} to out-of-distribution settings, where unseen tasks could be composed on the fly from existing singular vectors, are promising directions for future work.

A complementary question, largely orthogonal to the merging mechanism itself, is how the merging \emph{coefficients} should be chosen. The methods in \cref{ch:gradient,ch:tsv} and this chapter either fix coefficients analytically or derive them from structural properties of the task vectors. A more powerful alternative is to treat coefficient selection as an optimization problem solved via evolutionary search. \Cref{ch:merge3} addresses the computational bottleneck that has made this approach impractical, enabling high-quality evolutionary merging on a single consumer GPU.

%% file: 4_multi_task/5_merge3/content.tex
\begin{chapteroverview}
In this chapter, we address the computational bottleneck of evolutionary model merging, whose fitness evaluation requires repeated inference over large datasets. We introduce \mergethree, a three-stage framework that combines dataset reduction, Item Response Theory-based ability estimation, and evolutionary search to bring the cost of fitness computation down by roughly $50\times$. We provide theoretical guarantees on the resulting performance estimator and validate the approach on cross-lingual and multilingual merging tasks.
\end{chapteroverview}

\section{The computational bottleneck of evolutionary merging}\label{sec:merge3-intro}

\input{1_Introduction/content}
\section{Evolutionary algorithms and item response theory}\label{sec:merge3-background}
\input{2_Background/content}

\section{The \texorpdfstring{\mergethree{}}{MERGE3} framework}\label{sec:merge3-approach}

\input{3_Approach/content}
\section{Experimental results}\label{sec:merge3-experiments}

\input{4_Experiments/content}
\section{Theoretical guarantees}\label{sec:merge3-theory}
\input{5_Theoretical_analysis/content}

\section{Computational cost}\label{sec:merge3-technical}
\input{6_Technical_Details/content}

\section{Summary and outlook}\label{sec:merge3-conclusions}
\input{7_Conclusions/content}

%% file: 4_multi_task/5_merge3/2_Background/content.tex

\paragraph{Evolutionary Algorithms.}
Evolutionary Algorithms are black-box optimization algorithms operating on a population of potential solutions by evolving them through generations with operators such as selection, mutation, recombination, and crossover \citep{ea, petrowski2017evolutionary, dasgupta1997evolutionary}.
The fitness function evaluates the quality of each solution, guiding selection by favoring higher-scoring solutions for reproduction~\citep{ea}. Closest to our work, \citet{sakana} apply evolutionary algorithms to optimize model merging recipes, eliminating the need for trial-and-error in selecting merge coefficients. In this context, the natural candidate for a fitness function is the performance of the resulting model on a held-out validation set.

\paragraph{Item Response Theory.}
Item Response Theory (IRT) \citep{cai2016item, van2018handbook, brzezinska2020item, lord1968statistical} is a paradigm to design, analyze, and score responses to tests such as SAT or GRE \citep{an2014item, kingston1982feasibility, petersen1982using}. Based on the relationship between individuals' performances on a test item and the test takers' levels of performance on the corresponding required ability, IRT has recently spread from psychometrics to natural language processing. In this direction, \citet{lalor2016building} leverage IRT's latent dimensions to evaluate language models, while \citet{vania2021comparing} use it to analyze benchmark saturation in NLP evaluations. 
More relevant to our work, \citet{zhuang2023efficiently} and \citet{tinybenchmarks} employ IRT-driven adaptive testing to alleviate the computational burden of large-scale evaluations for LLMs. Although their focus is on general LLM evaluation, our work builds on these approaches to design IRT-based estimators specifically tailored for evolutionary model merging, exploiting the structure of the merging process to obtain more accurate fitness estimates from smaller data subsets.

%% file: 4_multi_task/5_merge3/5_Theoretical_analysis/content.tex
\label{sec:theory}

In this section, we provide theoretical guarantees for our \textit{performance estimator}, analyzing its stability under dataset reduction and showing that it remains a reliable proxy for full-dataset accuracy.

The section is structured as follows: first (\cref{subsec:eps-stability}), we establish a formal connection between the accuracy of the performance estimator and the quality of the minimum found when using it as the objective function; second (\cref{subsec:mp-irt-theory}), we study the asymptotic properties of the estimator, formalizing it as unbiased; and finally (\cref{subsec:mp-irt-eps-stability-analysis}), we show that the estimator behaves in expectation within an \(\epsilon\)-bound of the true optimum. The proofs are given in \cref{app:proofs}.

\subsection{Part~I: \texorpdfstring{\(\epsilon\)}{ε}-stable estimators and
\texorpdfstring{\(\epsilon\)}{ε}-optimality preservation}
\label{subsec:eps-stability}

We first consider a performance metric \(F(\theta;\mathcal{D})\) for 
\(\theta \in \Theta \subset \mathbb{R}^n\), where \(\mathcal{D}\) is a dataset.  
If we choose a smaller subset \(\bar{\mathcal{D}} \subset \mathcal{D}\) to approximate this metric, 
denoted \(F(\theta;\bar{\mathcal{D}})\), we wish to control the loss in optimality incurred 
by replacing \(F(\theta;\mathcal{D})\) with \(F(\theta;\bar{\mathcal{D}})\).  

\begin{definition}[\(\epsilon\)-Stability.]
    Given two datasets \(\mathcal{D}\) and \(\bar{\mathcal{D}}\), we say \(F(\cdot;\bar{\mathcal{D}})\) is 
    \emph{\(\epsilon\)-stable with respect to} \(F(\cdot;\mathcal{D})\) if, for all 
    \(\theta \in \Theta\),
    \[
      \bigl|F(\theta;\mathcal{D})\;-\;F(\theta;\bar{\mathcal{D}})\bigr|\;\le\;\epsilon
    \]
\end{definition}

Under this condition, minimizing \(F(\cdot;\bar{\mathcal{D}})\) yields an objective value 
within \(2\epsilon\) of minimizing \(F(\cdot;\mathcal{D})\).  Formally:

\begin{theorem}[\(\epsilon\)-Optimality Preservation]
\label{thm:eps-opt-preserve}
Let \(D\) be a dataset, let \(\bar{\mathcal{D}}\subset \mathcal{D}\) be a subset, and let
\(F(\cdot;\bar{\mathcal{D}})\) be \(\epsilon\)-stable with respect to \(F(\cdot;\mathcal{D})\),
with a fixed \(\epsilon>0\).  Define
\[
  \theta^\star
  \;=\;
  \arg\!\min_{\theta\in \Theta}\;F(\theta;\mathcal{D})
  \quad\text{and}\quad
  \hat{\theta}
  \;=\;
  \arg\!\min_{\theta\in \Theta}\;F(\theta;\bar{\mathcal{D}})
\]
Then
\[
    F(\hat{\theta};\,D)
    \;\le\;
    F(\theta^\star;\,D) + 2\epsilon
\]
\end{theorem}

\begin{proof}
By \(\epsilon\)-stability applied to \(\hat{\theta}\): \(F(\hat{\theta};\mathcal{D}) \le F(\hat{\theta};\bar{\mathcal{D}}) + \epsilon\).
Since \(\hat{\theta}\) minimizes \(F(\cdot;\bar{\mathcal{D}})\): \(F(\hat{\theta};\bar{\mathcal{D}}) \le F(\theta^\star;\bar{\mathcal{D}})\).
By \(\epsilon\)-stability applied to \(\theta^\star\): \(F(\theta^\star;\bar{\mathcal{D}}) \le F(\theta^\star;\mathcal{D}) + \epsilon\).
Chaining these inequalities gives \(F(\hat{\theta};\mathcal{D}) \le F(\theta^\star;\mathcal{D}) + 2\epsilon\).
\end{proof}
Thus, \(\epsilon\)-stability ensures that any global minimizer on \(\bar{\mathcal{D}}\)
is \(2\epsilon\)-optimal on the true full dataset \(D\).
%
Nevertheless, uniformly bounding \(\bigl|F(\theta;\mathcal{D})-F(\theta;\bar{\mathcal{D}})\bigr|\) for all \(\theta\) may be too strong in practice. For this reason, we introduce:
\begin{definition}[\(\epsilon\)-Stability in expectation]\label{subsec:eps-stability-expectation}
    Given two datasets \(\mathcal{D}\) and \(\bar{\mathcal{D}}\), we say \(F(\cdot;\bar{\mathcal{D}})\) is 
    \emph{\(\epsilon\)-stable in expectation with respect to} \(F(\cdot;\mathcal{D})\) if
    \[    \mathbb{E}_{\bar{\mathcal{D}}}\bigl[\bigl|F(\theta;\mathcal{D})\;-\;F(\theta;\bar{\mathcal{D}}) \bigr|\bigr]
      \;\le\;\epsilon
    \]
    where the expectation is over the (random) choice of \(\bar{D}\)
\end{definition}
Under this relaxed notion, we still obtain a similar control on the \emph{expected} suboptimality gap:
\begin{theorem}[Expected \(\epsilon\)-Stability of the Minimum]
\label{thm:expected-eps-stability}
Suppose \(F(\cdot;\bar{\mathcal{D}})\) is \(\epsilon\)-stable in expectation with respect 
to \(F(\cdot;\mathcal{D})\).  Let 
\[
  m^\star \;:=\;\min_{\theta\in\Theta}\,F(\theta;\mathcal{D})
  \quad\text{and}\quad
  \widehat{m}(\bar{\mathcal{D}}) \;:=\;\min_{\theta\in\Theta}\,F(\theta;\bar{\mathcal{D}})
\]
Then
\[
  \bigl|
    m^\star
    \;-\;
    \mathbb{E}_{\bar{\mathcal{D}}}\bigl[\widehat{m}(\bar{\mathcal{D}})\bigr]
  \bigr|
  \;\le\;
  \epsilon
\]
\end{theorem}

Hence, even if stability only holds \emph{on average}, the expected gap between 
the global optimum on \(\mathcal{D}\) and the optimum on \(\bar{\mathcal{D}}\) remains at most \(\epsilon\).

\subsection{Part~II: Theoretical guarantees for \texorpdfstring{\mpirt{}}{mp-IRT}}
\label{subsec:mp-irt-theory}

We now apply these ideas to our proposed \mpirt{} estimator  (cf.\ \cref{sec:estimate}).  We first show that \mpirt{} is asymptotically  unbiased, and then combine this fact with \Cref{thm:expected-eps-stability}  to argue that \mpirt{}-based minimizers remain close to those that minimize  the full-dataset performance measure.

\paragraph{Asymptotic unbiasedness.}
\label{par:asymptotic-unbiasedness}
The following proposition establishes that, as \(\bar{D}\) grows,  \(\hat{Z}^{\mathrm{mp\text{-}IRT}}\) converges in probability to the true performance \(Z\). Its proof relies on classical limit arguments for unbiased estimators.

\begin{proposition}[Asymptotic unbiasedness of \mpirt{}]
\label{prop:estimator-unbiased}
Assume:
\begin{enumerate*}[label=(\roman*)]
    \item \(\hat{\xi} \to \xi\) in probability as
    \(\lvert \hat{I}\rvert \to \infty\),
    \item for each \(i\in I\), the true values \(a_i,\beta_i,\theta_1,\theta_2\)
          are known, with \(\sup_{i\in I}\|a_i\|_2 \le c\) for a fixed \(c\),
    \item linear inheritance of abilities 
          (cf.\ \Cref{conj:latent-abilities-linear}) holds.
\end{enumerate*}
Then, for all \(j,l\),
\begin{align*}
    &\Bigl|\mathbb{E}\bigl[\hat{Z}_{jl}\,\bigm|\,
      Y_{i_0l},\dots,Y_{i_kl}
    \bigr]
    \;-\;
    \mathbb{E}\bigl[Z_{jl}\,\bigm|\,
      Y_{i_0l},\dots,Y_{i_kl}
    \bigr]
  \Bigr| \;\to\; 0
\end{align*}
in probability as $\lvert \hat{I}\rvert \to \infty$.
\end{proposition}
Thus, for sufficiently large subsets \(\bar{D}\), the discrepancy between 
\(\hat{Z}_{\tilde{m}}\) and \(Z_{\tilde{m}}\) can be made arbitrarily small 
with high probability.

\subsection{Part~III: Performance preservation via \mpirt{}}
\label{subsec:mp-irt-eps-stability-analysis}
We now conclude that \mpirt{} preserves near-optimality when we optimize over 
a suitably large \(\bar{\mathcal{D}}\subset \mathcal{D}\).  Since 
Proposition~\ref{prop:estimator-unbiased} asserts that \(\hat{Z}\) approximates 
\(Z\) well for large \(\lvert \bar{\mathcal{D}}\rvert\), it follows (under mild conditions) 
that \(\mpirt{}\) remains \(\epsilon\)-stable in expectation.  Hence, 
\Cref{thm:expected-eps-stability} shows that minimizing \(\hat{Z}\) on 
\(\bar{D}\) yields, on average, a solution within \(\epsilon\) of the 
full-dataset optimum.

\begin{theorem}[Asymptotic performance preservation of \mpirt{}]
\label{thm:mpirt-asymptotic-eps-preserve}
Let \(\bar{\mathcal{D}}\subset \mathcal{D}\) be a random subset used to compute 
\(\hat{Z}^{\mathrm{mp\text{-}IRT}}\).  Suppose that, as \(\lvert \bar{\mathcal{D}}\rvert\to\infty\), 
\(\hat{Z}^{\mathrm{mp\text{-}IRT}}\) converges in probability to \(Z\) (the true 
performance on \(\mathcal{D}\)), and that \(\hat{Z}^{\mathrm{mp\text{-}IRT}}\) is 
\(\epsilon\)-stable in expectation for sufficiently large \(\lvert \bar{\mathcal{D}}\rvert\).  
Then the expected global optimum of \(\hat{Z}^{\mathrm{mp\text{-}IRT}}\) on \(\bar{\mathcal{D}}\) 
differs from that of \(Z\) on \(\mathcal{D}\) by at most~\(\epsilon\).  As 
\(\lvert \bar{\mathcal{D}}\rvert\to \infty\), \(\epsilon\to 0\).
\end{theorem}

\paragraph{Finite-sample analysis via the Law of Large Numbers.}
In practice, we rarely have \(\lvert \bar{\mathcal{D}}\rvert\to\infty\).  Instead, one can 
appeal to \emph{expected} \(\epsilon\)-stability 
(\Cref{thm:expected-eps-stability}) and then \emph{estimate} the corresponding 
expectation empirically.  For instance, one may draw multiple subsets 
\(\bar{D}_1,\ldots,\bar{D}_S\) at random from \(\mathcal{D}\) and compute 
\[
  \frac{1}{S}\,\sum_{s=1}^S 
  \Bigl\lvert F(\theta;\mathcal{D}) \;-\; F(\theta;\bar{\mathcal{D}}_s)\Bigr\rvert
\]
as an empirical approximation to 
\(\mathbb{E}_{\bar{\mathcal{D}}}\bigl[\lvert F(\theta;\mathcal{D})-F(\theta;\bar{\mathcal{D}})\rvert\bigr]\).
By the Law of Large Numbers, if this empirical average remains small (say, 
\(\approx \tilde{\epsilon}\)), then the true expectation is also small.  
Consequently, \Cref{thm:expected-eps-stability} implies that the optimal 
solution on each \(\bar{\mathcal{D}}_s\) is within \(\tilde{\epsilon}\) of the global 
optimum on \(\mathcal{D}\), on average.

\paragraph{Conclusion.}
In summary, \mpirt{} inherits asymptotic consistency from \(\mathrm{p\text{-}IRT}\) 
while requiring only a subset \(\bar{\mathcal{D}}\subset \mathcal{D}\).  By showing it is 
\(\epsilon\)-stable (in expectation) for large \(\lvert \bar{\mathcal{D}}\rvert\), we conclude 
that \emph{optimizing on \(\bar{\mathcal{D}}\) yields (on average) a solution close to the 
true optimum on \(\mathcal{D}\).}  In finite-sample regimes, multiple random draws of 
\(\bar{D}\) can be used to empirically verify that the discrepancy remains small, 
thereby justifying the practical use of \mpirt{} on moderately sized subsets.

%% file: 4_multi_task/5_merge3/6_Technical_Details/content.tex
\begin{table}
    \centering
    \caption[Comparison of Evolve methods by cost and accuracy]{Comparison of Evolve methods by number of trials, estimated total time on a single NVIDIA 4090, sample size used for Fitness computation, and final accuracy on \dataset{GSM8K}. The number of trials equals $\text{population size} \times \text{iterations}$ and represents the total number of merged models evaluated during the entire Evolve run.}
    \label{tab:evolution_comparison}
    \vspace{10pt}
    \resizebox{0.75\textwidth}{!}{
    \begin{tabular}{lcccc}
        \toprule
        \textbf{Method} & \textbf{$N_{\text{models}}$} & \textbf{Estimated total time} & \textbf{Sample size} & \textbf{Accuracy} \\
        \midrule
        EvoLLM-JP-7B & 1000 & 62 days & 1000 & 0.49 \\
        \method{MERGE$^3_{100}$} & 175 & 21h & 100 & 0.42 \\
        \method{MERGE$^3_{50}$} & 175 & 12h 20m & 50 & 0.38 \\
        \method{MERGE$^3_{30}$} & 175 & 10h 30m & 30 & 0.38 \\
        \method{MERGE$^3_{20}$} & 175 & 10h 15m & 20 & 0.34 \\
        \bottomrule
    \end{tabular}
    }
\end{table}

We summarize the GPU timing results for \mergethree in \cref{tab:evolution_comparison}, comparing evaluation and merge times across different hardware setups. These findings highlight the practical feasibility of our approach even on older GPUs. For additional experimental details, refer to \cref{app:add-exp-4090}.

%% file: 4_multi_task/5_merge3/7_Conclusions/content.tex
\mergethree{} makes high-quality evolutionary merging feasible on a single GPU, requiring roughly one day on an NVIDIA 4090 with 24\,GB of VRAM. By combining a subset-based approach with IRT-driven performance estimation, the framework reduces merging costs by up to $50\times$ compared to prior methods while maintaining competitive accuracy. The experiments demonstrate successful cross-lingual transfer in mathematics (e.g., from English to Japanese), as well as the synthesis of new multilingual models that outperform each of their language-specific endpoints. \mergethree{} thus expands the practical reach of evolutionary merging, allowing everyday practitioners to benefit from advanced multi-task and multilingual model merges at a fraction of the usual computational cost.

Together, the four chapters of \cref{part:setting-two} form a coherent progression from theory to practice. \Cref{ch:gradient} established \emph{why} task arithmetic works by grounding task vectors in gradient theory. \Cref{ch:tsv} revealed the \emph{low-rank structure} of these vectors and showed how to exploit it for compression and interference reduction. \Cref{ch:mass} extended this structural insight to the inference stage, enabling \emph{input-adaptive} routing through task subspaces. Finally, this chapter democratized \emph{evolutionary search} over merging configurations: while \cref{ch:tsv} reduces interference structurally at the parameter level, \mergethree{} navigates interference by optimizing merging coefficients via evolution, removing the computational barrier that had confined evolutionary merging to large-scale compute environments. Taken together, these contributions advance the theoretical understanding, structural exploitation, adaptive deployment, and practical accessibility of multi-task model merging.

%% file: 6_future_directions/content.tex

\chapter{Applications and Impact} \label{ch:applications}
\subimport{1_applications/}{content.tex}

\chapter{Future Directions} \label{ch:future-directions}
\subimport{2_future_directions/}{content.tex}

\chapter{Conclusions}\label{ch:conclusions}
\subimport{3_conclusions/}{content.tex}

%% file: 6_future_directions/1_applications/content.tex
The methods developed throughout this thesis have applications that extend well beyond the benchmarks on which they were evaluated. We outline here the most immediate ones.

\paragraph{Federated learning.}
The single-task merging methods for models initialized from different checkpoints explored in \cref{part:setting-one} lend themselves naturally to federated learning.
In federated learning~\citep{mcmahan2017communication}, clients collaboratively train a shared model without exchanging their private data: each client fine-tunes a local copy of the model on its own data, and a central server periodically aggregates the local updates into a global model that is redistributed to the clients.
The standard aggregation strategy, \method{FedAvg}~\citep{mcmahan2017communication}, simply averages the client parameters, which works well when clients start from the same initialization and perform few local steps, but degrades as local training diverges.
The cycle-consistent alignment developed in \cref{ch:cycle-consistent} directly addresses this failure mode: by mapping all client models into a shared universe space before averaging, it can compensate for the drift induced by prolonged local training or heterogeneous data distributions.
Preliminary experiments (\cref{app:cycle-cons-fed-learning}) already show that this approach outperforms \method{FedAvg} when clients are initialized from different random seeds, and that the gains increase with the number of local epochs, precisely the regime where standard averaging breaks down.
While these results are limited to small-scale settings, they point to a promising direction in which permutation-based alignment could serve as a drop-in replacement for naive averaging in federated protocols.

\paragraph{Ensembling from the universe space.}
Beyond simple averaging, the universe space provides a geometrically principled sampling ground for ensembling.
As discussed in \cref{ch:background}, prior work has shown that sampling models from low-loss connected regions and averaging their predictions yields ensemble-quality performance at a fraction of the cost of independently trained ensembles~\citep{Garipov2018-pz, benton_loss_2021}.
The universe mapping from \cref{ch:cycle-consistent} places all aligned models into a common low-loss basin (\cref{fig:cifar_loss_contour}), creating exactly the kind of connected manifold from which diverse checkpoints can be cheaply sampled for prediction averaging.
This opens the possibility of obtaining ensemble-quality predictions from independently trained models without requiring a shared initialization, by first aligning them into the universe and then sampling along the interpolation paths within it.

\paragraph{Transferring fine-tuning across model releases.}
When a foundation model is updated or retrained, previously fine-tuned variants become obsolete, and practitioners must repeat the fine-tuning process from scratch.
An emerging line of work investigates whether task vectors can instead be \emph{transported} to a new backbone.
\citet{rinaldiupdate} propose a data-free, two-level permutation method for Transformers that re-basins task vectors onto an updated checkpoint by first permuting attention heads and then adjusting intra-head parameters.
\citet{rinaldi2025gradient} take a complementary approach, using a small number of target-model gradients to mask the source task vector so that it aligns with the loss landscape of the new backbone.
While both methods show promising initial results, fully closing the gap with native fine-tuning remains an open challenge, making this a natural application area for the alignment and merging techniques developed in this thesis.

\paragraph{State-of-the-art multi-task merging.}
The multi-task merging methods for models fine-tuned from the same foundation model developed in \cref{part:setting-two} provide immediate practical utility.
The structured, layer-wise approach introduced in \cref{ch:tsv}, which preserves the matrix structure of per-layer task vectors and exploits their low-rank geometry, established a new state of the art in vision model merging, surpassing prior methods by an average of ${\sim}15$ percentage points.
The input-adaptive extension \mass{} (\cref{ch:mass}) pushes this further by routing each input to the most relevant task subspaces at inference time, recovering up to $98\%$ of individual expert accuracy with a fixed $2\times$ storage overhead and no oracle knowledge of the task at hand.

\paragraph{Efficient compression of task-specific experts.}
From a complementary angle, \method{TSV-C} (\cref{ch:tsv}) provides a practical solution for scenarios where multiple task-specific experts must be stored and deployed.
By retaining only the top $\frac{1}{T}$ singular vectors per task per layer, \method{TSV-C} compresses $T$ task vectors down to a fixed footprint of approximately $2\times$ the size of the original pre-trained model, regardless of the number of tasks.
Concretely, this means that 20 task-specific experts can be stored in the same space as 2 full models, while retaining over $99\%$ of the original fine-tuned accuracy across all benchmarks (\cref{tab:task_acc_compression}).
This is especially relevant for serving systems that need to maintain a library of specialized models: \method{TSV-C} either substantially reduces the storage required for a fixed set of experts, or allows many more experts to be stored within the same storage budget.

\paragraph{Democratizing evolutionary merging for new capabilities.}
The evolutionary merging framework \mergethree{} (\cref{ch:merge3}) offers a practical and accessible way to obtain models exhibiting new skills by composing existing ones, without requiring any training.
We demonstrated this by merging a math-specialized English model with language-specific fine-tuned variants of \model{Mistral-7B}, obtaining models that solve mathematical reasoning tasks in Romanian, Dutch, German, and Japanese, surpassing both endpoint models by 10--20\% in accuracy on the translated \dataset{GSM8K} (\cref{fig:cross-transfer}).
We further showed that \mergethree{} enables the synthesis of multilingual models that outperform all of their monolingual constituents: starting from individually fine-tuned models for Italian, English, German, and Dutch, the evolved multilingual model surpassed each language-specific baseline on \dataset{ARC-Challenge} by up to 19\% (\cref{tab:multilingual}), demonstrating genuine cross-lingual positive transfer rather than a compromise between languages.
Crucially, by introducing efficient performance estimators based on Item Response Theory, \mergethree{} reduces the computational cost of evolutionary merging by $50\times$ compared to full-dataset evaluation, making it feasible in approximately one day on a single NVIDIA 4090 rather than requiring month-long runs on specialized hardware.
This democratization of evolutionary merging puts model composition within reach of individual researchers and practitioners who lack access to large-scale compute infrastructure.

\paragraph{Applicability and challenges.}
The applications above differ substantially in their level of maturity. Some are immediate consequences of the methods developed in this thesis, while others should be understood as promising directions that require additional adaptation before deployment at the scale of current foundation models. Table~\ref{tab:applications_challenges} summarizes this distinction.

\begin{table}[h]
\centering
\small
\begin{tabular}{p{0.22\linewidth} p{0.22\linewidth} p{0.26\linewidth} p{0.20\linewidth}}
\toprule
\textbf{Scenario} & \textbf{Current applicability} & \textbf{Main challenges} & \textbf{Most relevant methods} \\
\midrule
Federated learning & Medium. Most natural when clients share an architecture but local training causes drift. & Scaling permutation alignment beyond small CNN/MLP settings; handling transformer symmetries; communication overhead; privacy constraints. & C$^2$M$^3$ \\
Universe-space ensembling & Medium-low. Conceptually natural, but currently closer to an algorithmic direction than a deployed use case. & Efficiently sampling diverse checkpoints; validating gains on large architectures; avoiding ensemble-time inference cost. & C$^2$M$^3$ \\
Fine-tuning transfer across releases & Medium-high. Highly relevant when users have many fine-tuned variants of an older base model. & Backbone mismatch; tokenizer and architecture changes; incomplete recovery compared to native fine-tuning. & Alignment methods, task vectors \\
Multi-task expert merging & High for open-source and specialized model ecosystems; lower for monolithic frontier models. & Task interference; loss of specialist performance; dependence on a shared base model; evaluation of emergent combined capabilities. & TSV-Merge, MASS, MERGE$^3$ \\
Expert-library compression & High when many task-specific deltas, adapters, or fine-tuned checkpoints must be stored. & Maintaining per-task fidelity; deciding the rank budget; extending compression reliably to very large language models. & TSV-C \\
Evolutionary merging & High for open-source LLM composition; medium for production settings. & Expensive evaluation; benchmark overfitting; robustness under prompt and distribution shifts; safety evaluation of merged models. & MERGE$^3$ \\
\bottomrule
\end{tabular}
\caption{Applicability and open challenges for the deployment scenarios discussed in Chapter~9.}
\label{tab:applications_challenges}
\end{table}

\paragraph{Relevance to current foundation models.}
A further qualification concerns the nature of modern foundation models. The clearest setting for multi-task merging is one in which several task-specific experts are fine-tuned from the same pretrained backbone and then combined. This setting is central to much of the model-merging literature and to Part~III of this thesis, but it is not an exact description of the strongest contemporary foundation models. Frontier base models and post-trained assistants are increasingly general-purpose systems rather than narrow task experts, and their capabilities are often shaped by broad instruction tuning, preference optimization, tool use, and large-scale data mixtures rather than by isolated task-specific fine-tuning. As a result, the literal ``merge one expert per task'' framing is less salient for monolithic frontier models than for open-source ecosystems of specialized checkpoints, adapters, and task vectors.

This does not make model merging obsolete for foundation models; rather, it changes the relevant unit of composition. In current practice, the objects being merged are often not narrow classifiers but model variants: instruction-tuned models, code models, reasoning models, multilingual models, safety-tuned models, domain-adapted models, or parameter-efficient deltas derived from a common base. In this regime, the methods studied in this thesis remain directly relevant. Task-vector methods provide a language for reusing and composing fine-tuning updates; TSV-style decompositions reduce the storage footprint of large collections of variants; routing methods such as MASS suggest ways to activate only the subspaces needed for a given input; and MERGE$^3$ addresses the practical cost of searching over combinations of large language model variants. Thus, the immediate deployment target is not necessarily the internal training pipeline of a closed frontier model, but the broader foundation-model lifecycle: maintaining families of related checkpoints, transferring capabilities across releases, compressing specialized variants, and composing open-source models without retraining.

%% file: 6_future_directions/2_future_directions/content.tex
We now outline several directions that, in our view, represent the most promising avenues for advancing the field.

\paragraph{More expressive alignment frameworks.}
The single-task alignment methods developed in \cref{part:setting-one} rely on the Frank-Wolfe algorithm to solve a relaxation of the weight matching problem over the Birkhoff polytope of doubly stochastic matrices.
While effective, this formulation treats the matching problem layer by layer and captures inter-layer dependencies only through gradient coupling.
A natural extension is to formulate the problem directly as a \emph{Quadratic Assignment Problem} (QAP)~\citep{wang08}, which jointly optimizes all permutations while explicitly encoding the quadratic coupling between adjacent layers.
QAP is NP-hard in general, but modern solvers based on semidefinite programming relaxations or learned graph matching networks~\citep{wang08} could yield tighter solutions than the current linear relaxation, particularly for deep architectures where inter-layer dependencies are strongest.
An orthogonal and potentially complementary direction comes from \emph{functional maps}~\citep{ovsjanikov2012functional}, a framework from geometry processing that replaces explicit point-to-point correspondences with a small matrix mapping between spectral bases.
Rather than finding a permutation between neurons, one could work in the spectral domain of the weight matrices, representing the correspondence as a compact linear map between the eigenbases of the two models.
This perspective has already been brought into the representation alignment literature by \citet{fumero2024latent}, who showed that spectral correspondences between neural representation spaces can be computed efficiently and used for model stitching and retrieval.
Extending this framework to weight-space alignment could yield soft, differentiable, and naturally low-dimensional matching that sidesteps the combinatorial hardness of permutation search entirely.

\paragraph{Heterogeneous merging.}
Nearly all methods in this thesis, and in the broader literature, assume \emph{homologous} source models: same architecture, same pre-trained initialization.
This is a significant limitation in practice: one may wish to combine a vision transformer with a convolutional network, merge language models with different tokenizers and vocabulary sizes, or even bridge across fundamentally different layer types (e.g., Transformer and state-space models).
Preliminary work has begun to address these settings.
\citet{zip-it} propose a ``zipping'' operation that merges internal features within and across models trained on disjoint tasks, even without a shared initialization, while \citet{nguyen_cross-layer_2023} tackle the case of networks with different depths by aligning representations across layers. On the language side, \citet{minixhofer2024zero} further demonstrate that tokenizers can be transferred across models in a zero-shot fashion, removing one of the key obstacles to merging language models with mismatched vocabularies.
Despite these advances, heterogeneous merging remains largely unsolved for large-scale models: merging models from different pre-training corpora or RLHF pipelines requires not only architectural alignment but also semantic alignment of the learned representations. 
The functional maps framework discussed above, which naturally accommodates spaces of different dimensionality by working in a shared spectral basis, could provide a principled foundation for this line of work. 

\paragraph{From merging to optimization.}
The connection between task vectors and gradients established in \cref{ch:gradient} suggests that insights from model merging could inform optimizer design, and vice versa.
A striking example is the \method{Muon} optimizer~\citep{jordan2024muon, liu2025muon}, which applies an orthogonalization step to the momentum at each training step: it computes the SVD of the gradient matrix $G = U \Sigma V^\top$ and uses $U V^\top$ as the update direction, effectively discarding the singular values and retaining only the directional structure.
This is remarkably close to the structured merging paradigm developed in \cref{ch:tsv}, where task vectors are decomposed via SVD and the singular vectors are identified as the structurally meaningful components, with interference reduction achieved through whitening of the singular value spectrum.
Both approaches share the same insight: that the \emph{directions} in weight space carry the essential information, while the \emph{magnitudes} should be normalized or reweighted.
\method{Muon} has already shown strong results at scale, achieving $2\times$ computational efficiency over \method{AdamW} when training 3B-parameter language models~\citep{liu2025muon}.
This convergence between the merging and optimization communities suggests a rich space for cross-pollination: structured merging techniques might yield better update rules for multi-task and continual learning, while insights from spectral optimizer design could inform how task vectors should be combined.

\paragraph{Scaling structured merging to large language models.}
The layer-wise SVD-based methods developed in \cref{ch:tsv,ch:mass} have been validated extensively on vision transformers up to \vitlarge{} (300M parameters) across up to 20 tasks.
Extending these methods to large language models with 8B parameters and beyond poses both computational and methodological challenges.
On the computational side, the SVD of weight matrices at LLM scale is expensive, though recent advances in randomized and streaming SVD algorithms could make it tractable.
Methodologically, the structure of LLM task vectors may differ from vision models: the attention and MLP blocks in transformer language models have different rank profiles and interference patterns than the vision encoder layers studied in this thesis.
Understanding whether the low-rank structure and the effectiveness of interference reduction via whitening transfer to this regime is a necessary first step before the layer-wise structural paradigm can be applied at the frontier of language modeling.

\paragraph{Beyond classification: merging for generative models.}
The multi-task merging methods in this thesis have been evaluated primarily on vision classification benchmarks.
However, model merging has already found enthusiastic adoption in the generative modeling community, particularly for diffusion models.
\citet{hamalainen2024diffusionsoup} demonstrate that simple weight averaging of Stable Diffusion models trained on different data shards yields substantial improvements in image quality, and enables training-free continual learning by adding or removing models from the average.
\citet{gandikota2024concept} introduce concept sliders, LoRA adapters for diffusion models that function as directional edits in weight space, directly analogous to task vectors: over 50 such sliders can be composed simultaneously without degrading output quality.
These results suggest that the structured merging techniques from \cref{ch:tsv}, which reduce interference between task directions while preserving their individual contributions, could substantially improve upon the naive averaging currently used in practice.
Extending the SVD-based interference analysis and the input-adaptive routing of \mass{} to diffusion model architectures, where ``tasks'' correspond to styles, concepts, or domains rather than classification categories, is a natural and high-impact direction.

\paragraph{Towards a predictive theory of mergeability.}
Finally, the field still lacks a \emph{predictive theory of merge success}: there is no standard, cheaply computable ``mergeability score'' that generalizes across architectures and modalities, analogous to how scaling laws guide pre-training decisions.
Current practice relies on grid search over scaling coefficients, pruning thresholds, and mask parameters.
Recent work has begun to address this gap from complementary angles.
\citet{rahamim2026will} propose a concrete definition of mergeability and identify the base model's existing knowledge as a dominant factor: models fine-tuned on instances that the base model already handles well are substantially more mergeable than those trained on difficult instances.
\citet{zhou2026demystifying} take a method-aware perspective, showing that mergeability is not an intrinsic model property but depends jointly on the merging method and the partner tasks; nonetheless, subspace overlap and gradient alignment consistently emerge as method-agnostic prerequisites for compatibility, echoing the gradient analysis of \cref{ch:gradient}.
On the practical side, \citet{bolton2026simmerge} introduce \method{SimMerge}, which learns to predict the best merge operator and merge order from cheap, task-agnostic similarity signals between checkpoints.
These results are encouraging, but a unified theory that connects the structural, spectral, and gradient perspectives developed in this thesis into a single predictive framework remains an open challenge whose resolution would transform model merging from a largely empirical practice into a principled engineering discipline.

%% file: 6_future_directions/3_conclusions/content.tex
This thesis began with a simple but consequential observation: the history of deep learning
has been shaped by a quiet assumption that knowledge lives inside individual models, trained
in isolation and discarded when no longer needed. The goal of this work has been to
question that assumption and develop alternatives grounded in theory and algorithms.
The result is a set of contributions that together advance the formal foundations of
model merging, drawing on geometry, algebra, and provable guarantees.

\section{Summary of Contributions}

\paragraph{Single-task merging: geometry and cycle consistency.}
The first part of this thesis addressed the problem of merging independently trained
instances of the same architecture. The central obstacle is well understood: the permutation
symmetries of neurons create a combinatorial number of functionally equivalent
parameterizations, causing independently trained models to occupy seemingly distant regions
of weight space even when they compute the same function. Prior alignment methods
such as Git Re-Basin resolved this pairwise and layer by layer, but offered no guarantees
on consistency when more than two models were involved. Cyclic compositions of
pairwise permutations would drift, accumulating error and placing the re-permuted model
in a completely different basin from which it started.

\Cref{ch:cycle-consistent} addressed this limitation by introducing \textsc{C\textsuperscript{2}M\textsuperscript{3}}, a cycle-consistent
multi-model merging algorithm. The key insight is a factorization of pairwise permutations
through a shared \emph{universe} space: rather than seeking a direct mapping between
every pair of models, the method learns a single permutation from each model to a
canonical reference, so that all pairwise correspondences are obtained by composition and
cycle consistency is enforced by construction. The Frank-Wolfe-based optimization handles
all layers simultaneously, capturing inter-layer dependencies that are invisible to
layer-by-layer methods. Empirical evaluation across architectures and datasets demonstrated
consistent and often substantial accuracy improvements over the prior state of the art,
with accuracy gains as high as 20\% when merging five models simultaneously. Crucially,
the universe space acts not only as an alignment target but as a natural aggregation ground:
models mapped into a shared low-loss basin can be meaningfully averaged, and the
resulting merged model can serve as a principled midpoint from which diverse checkpoints
are cheaply sampled for ensembling. This addresses a fundamental limitation of
reference-based approaches, where the choice of reference introduces an arbitrary bias
that the merged model cannot overcome.

\paragraph{Multi-task merging: theory.}
Having handled the geometric challenge of single-task alignment, the thesis turned to a
fundamentally different and more practically impactful regime: merging models fine-tuned
on distinct tasks from a shared pre-trained backbone. Here, the shared initialization already
ensures weight-space compatibility, so the problem is not one of geometric alignment but
of task interference: understanding why combining task-specific weight differences works
at all, and when it fails.

\Cref{ch:gradient} provided the theoretical foundation that had been largely missing from
the literature. By analyzing full-batch gradient descent, the chapter established a formal
equivalence: a task vector computed from a single epoch of fine-tuning is precisely the
negative loss gradient, scaled by the learning rate. This makes the sum of task vectors in
task arithmetic mathematically equivalent to a single gradient descent step on the aggregated
multi-task loss. For models fine-tuned beyond a single epoch, the equivalence degrades
gracefully, with a second-order deviation of $O(\eta^2)$ controlled by the curvature of the
loss surface. These findings reframe task arithmetic as approximate multi-task learning,
providing a principled explanation for both its effectiveness and its failure modes.
A key practical implication follows directly: merging models fine-tuned for a single epoch can match or exceed the performance of merging fully converged models. This reflects the fact that early-epoch task vectors are the best approximations of the true task gradients, and that the fine-tuning trajectory is largely determined by the gradient information accrued in the first epoch. The result clarifies why task proficiency and mergeability are often at odds, and offers a principled justification for the widespread empirical observation that shorter fine-tuning intervals often produce better merged models.

\paragraph{Multi-task merging: low-rank structure and interference reduction.}
\Cref{ch:tsv} departed from the convention of treating task vectors as flat parameter vectors and instead examined them at the layer level, preserving their natural matrix structure.
Singular value decomposition of per-layer task matrices revealed a consistent structural pattern: task matrices are intrinsically low-rank. The effective rank of a task matrix is a small fraction of its ambient dimension, and the dominant singular vectors concentrate the task-relevant information. This empirical finding has two immediate algorithmic consequences.

The first is compression. By retaining only the top-$k$ singular components per task
per layer, \textsc{TSV-Compress} stores $T$ task-specific experts in the same space as
approximately two full models, regardless of $T$. In practice, twenty fine-tuned experts can be compressed to twice the size of the base model while retaining over 99\% of per-task accuracy, a result that significantly expands the practical scalability of multi-task model libraries.

The second consequence is interference reduction. The concept of Singular Task
Interference (STI) quantifies the overlap between the singular vector subspaces of
different tasks: high STI signals that two tasks share directions in weight space, making their task vectors prone to destructive interaction when summed. \textsc{TSV-Merge}
exploits this structure by disentangling task directions via a Procrustes-based whitening transformation before aggregation, producing a merged model that significantly outperforms prior methods. On multi-task vision benchmarks spanning up to twenty tasks across three ViT architectures, TSV-Merge established a new state of the art, with average accuracy improvements of approximately 15 percentage points over existing approaches.

\paragraph{Multi-task merging: input-adaptive routing.}
All methods discussed to this point produce a single, static merged model that is applied
identically to every input regardless of its content. This uniformity leaves performance
on the table: a merged model that is optimal on average may be suboptimal for any
individual input if the input belongs clearly to one task's distribution.

\Cref{ch:mass} introduced \textsc{MASS} (Merging via Adaptive Subspace Selection), which
extends the low-rank structure of \cref{ch:tsv} from merge time to inference time. The
central observation is that the task singular vectors embedded into the merged model by
\textsc{TSV-Merge} form a set of orthogonal subspaces, each encoding the directional
signature of a specific task. At inference, a projection-based router measures the
reconstruction quality of each input's intermediate representation against these subspaces
and selects the task with minimal residual, a decision that is equivalent, under an isotropic
Gaussian noise model, to maximum a posteriori task estimation. The router requires no
labels, no additional training, and no access to the original task data; it operates entirely
from the geometry embedded in the merged model.

Experimentally, \textsc{MASS} sets a new state of the art for mixture-of-experts merging
across eight out of nine benchmarks, recovering up to 98\% of the accuracy of individually
fine-tuned experts while requiring only a fixed $2\times$ storage overhead. The method
generalizes across visual and language modalities, and a qualitative analysis of the task
singular vectors confirms that they capture interpretable, domain-specific semantics aligned
with the visual content of each task. The result is an adaptive merged model that is both principled and practical, deployable directly on top of any collection of fine-tuned checkpoints without any access to their original training data.

\paragraph{Efficient evolutionary merging at scale.}
A complementary strategy for model merging bypasses the design of hand-crafted merging
rules entirely and instead optimizes merging coefficients through evolutionary search.
This approach has been highly effective, but at a computational cost
(thousands of fitness evaluations, each requiring full inference over large datasets) that
has placed it far beyond the reach of practitioners working on standard hardware.

\Cref{ch:merge3} introduced \textsc{MERGE3}, a framework that reduces the computational cost
of evolutionary merging by $50\times$, making it feasible on a single consumer
GPU in approximately one day rather than requiring month-long runs on specialized
infrastructure. The approach combines three components: a dataset reduction strategy
that selects a small subset for fitness evaluation; an Item Response Theory (IRT)-based
ability estimator that infers the latent capabilities of endpoint models from their
correctness patterns; and a novel performance estimator, MP-IRT, that predicts the
fitness of a merged model as a linear combination of the endpoints' estimated abilities,
exploiting the linearity of weight-space merging. Theoretical analysis shows that
MP-IRT is asymptotically unbiased and $\epsilon$-stable in expectation with respect to
full-dataset evaluation, ensuring that evolutionary search on the reduced dataset yields
solutions within $\epsilon$ of the true global optimum.

Experimentally, \textsc{MERGE3} demonstrates successful cross-lingual transfer of
mathematical reasoning skills: merging a math-specialized English model with
language-specific fine-tunings to obtain models that solve GSM8K in Romanian, German,
Dutch, and Japanese, surpassing all endpoint models by 10--20\%. The framework
further evolves a single multilingual model from four language-specific endpoints that
outperforms each constituent model by up to 19\% on ARC-Challenge, demonstrating
genuine cross-lingual positive transfer rather than a mere compromise. This application extends the scope of the thesis from the vision domain to large language models.

\section{Overarching Themes}

Reading across the contributions of this thesis, three overarching themes emerge.

\paragraph{Making weight space legible.}
For most of the history of deep learning, the weight space of a neural network has been treated as fundamentally illegible: a high-dimensional numerical object whose internal organization is an accidental byproduct of optimization, meaningful only in the aggregate sense that the network computes something useful. Model merging quietly challenged this assumption: the very act of adding task vectors and observing coherent behavior implied that weight space had structure, that directions in it were interpretable, that arithmetic over parameters was not arbitrary. What \cref{ch:gradient} and \cref{ch:tsv} together establish is precisely this legibility in formal terms. Task vectors are not arbitrary displacements; they are approximate gradient directions, and their geometry in parameter space reflects the geometry of the loss landscape. The dominant singular vectors of a per-layer task matrix are not random directions; they are the axes along which a task has reorganized the pre-trained representations, and they are interpretable enough that decoding them through a language model yields accurate semantic descriptions of the task's visual content. This last observation, demonstrated in \cref{ch:mass}, where task singular vectors decoded as text recover phrases like \emph{image of a car} or \emph{close-up of a textured mesh}, is a particularly suggestive finding: it indicates that the structure uncovered by the theory corresponds to interpretable properties of the learned representations.

\paragraph{Structure begets efficiency.}
A second theme is that exploiting the geometric and algebraic structure of task vectors
yields algorithms that are simultaneously more accurate and more efficient than their
structure-agnostic counterparts. The low-rank structure of task matrices enables compression
that is essentially lossless. The singular vector subspace geometry enables routing without
labels or training. The linearity of merging enables IRT-based performance estimation
that reduces evolutionary search costs by $50\times$. In each case, the
structural insight is not merely a theoretical observation but the direct source of
computational gain.

\paragraph{Accessibility as a scientific goal.}
A third, perhaps less conventional theme is the deliberate pursuit of accessibility. The democratization already underway in the community (evidenced by the fraction of top-ranked models that are the product of merging rather than training) demands that the tools for doing merging well be as widely accessible as the practice of merging itself.
This thesis argues, and attempts to demonstrate, exactly that. The \textsc{MERGE3}
framework makes evolutionary merging feasible on hardware available to any individual
researcher. The compression guarantees of \textsc{TSV-Compress} make multi-task expert
libraries practical at modest storage budgets. The data-free routing of \textsc{MASS}
makes adaptive inference available without training supervision or access to task data.
Together, these contributions strive to keep merging the accessible tool it was meant to be.

\section{Limitations and Open Questions}

An honest account of the contributions in this thesis must also acknowledge their limits.
The single-task alignment methods of \cref{part:setting-one} have been validated on convolutional and MLP architectures; their behavior on large transformer networks, where the interaction between attention and MLP layers creates more complex symmetry structures, remains to be established. The multi-task methods of \cref{part:setting-two} assume homologous source models sharing the same pre-trained backbone, a condition that holds in the fine-tuning regime but excludes the increasingly common scenario of merging models with different architectures, tokenizers, or pre-training distributions. The gradient equivalence of
\cref{ch:gradient} is derived under full-batch gradient descent with a fixed step size, and while its empirical implications transfer to stochastic settings, a formal analysis under SGD with momentum and learning rate schedules remains an open problem.

More broadly, the field still lacks a predictive theory of mergeability, a cheaply
computable quantity, analogous to a scaling law, that would tell a practitioner in advance whether two given models can be merged effectively and what performance to expect.
The tools developed in this thesis (gradient alignment, spectral interference analysis) each offer partial windows onto this question, but a unified
predictive framework has not yet emerged. Developing such a theory would transform
model merging from a largely empirical practice into a fully principled engineering
discipline.

\section*{Closing Remarks}
The thesis has argued that the weights of a neural network are not merely the byproduct of an optimization process to be discarded once training is complete. They are a structured encoding of knowledge, one that can be decomposed, composed, routed, and compared, enabling capabilities that go beyond those of any single model. Achieving this requires understanding the geometry and algebra of weight space, and this thesis has taken concrete steps in that direction.

As model repositories continue to grow at an exponential rate, and as the computational cost of training from scratch becomes increasingly prohibitive for most of the world's researchers, the ability to reuse, combine, and compose existing models will only grow in importance. The contributions of this thesis aim to make such composition not only practically effective, but grounded in principled theoretical understanding.

%% file: 99_appendix/content.tex

\chapter{Additional material on Chapter 3}
\subimport{A_cycle_consistent/}{content.tex}

\chapter{Additional material on Chapter 5}
\subimport{C_task_vectors_gradients/}{content.tex}

\chapter{Additional material on Chapter 6}
\subimport{C_TSV/}{content.tex}

\chapter{Additional material on Chapter 7}
\subimport{E_mass/}{content.tex}

\chapter{Additional material on Chapter 8}
\subimport{F_merge3/}{content.tex}

%% file: 99_appendix/A_cycle_consistent/content.tex
\section{Additional details}\label{cycle-cons-appendix-a}

\input{A_details/content.tex}

\section{Additional experiments}\label{cycle-cons-appendix-b}
\input{B_experiments/content.tex}

\section{Additional analysis}\label{cycle-cons-appendix-c}
\input{C_analysis/content.tex}

\section{Discussion}\label{cycle-cons-appendix-d}
\input{D_Discussion/content.tex}

%% file: 99_appendix/A_cycle_consistent/A_details/content.tex
Here we report in-depth explanations and additional experimental details. In particular, \cref{app:extended-related-work} extensively outlines the most related works,
\cref{app:pairwise-frankwolfe} shows the \method{Frank-Wolfe} algorithm for the pairwise case, while \cref{app:merge-many} describes the \texttt{MergeMany} procedure presented in~\cite{git-rebasin} for merging multiple models. Finally, we show how the matching algorithm empirically converges in \cref{app:convergence}.

\subsection{Extended related work}\label{app:extended-related-work}
We report here a thorough review of works that are relevant to our research, providing a comprehensive understanding of the context of our work.
\paragraph{Linear mode connectivity}
Mode connectivity is interested in modes, i.e., model parameters at convergence. In this regard, \citet{linear-mode-connectivity} first studied the connectivity of the parameters of models that were trained for a few epochs from the same initialization, while \citet{Garipov2018-pz} investigated whether these can be connected through a high-accuracy path without requiring the same initialization.
Simultaneously, \citet{Draxler2018-vr} proposed an algorithm to find a \emph{Minimum Energy Path} (MEP) between two modes of a neural network, showing that these paths are mostly flat in both the training and test landscapes. This implies that many minima actually live in a shared low-loss valley rather than in distinct basins.
From a different perspective, \citet{permutationequivariance} proposed to study a class of neural functionals which are permutation-equivariant by design.
Recent research proposes to study model behavior in the weight space beyond linear mode connectivity: \citet{mechanisticmc} show that different ``mechanisms'' in related models prevent simple paths of low loss in the weight space, while \citet{beyondlmc} studied the linear connections between the linear features of each layer of differently trained models.

\paragraph{Model merging}
Model merging~\cite{git-rebasin,rebasin-implicit-sinkhorn,model-fusion,jin2022dataless,robots,zip-it} has seen a surge of interest in the last years as a means to ensemble models without incurring the added computational cost. One of the first works in this direction is \citet{model-fusion}, who proposed an optimal-transport based weight-matching procedure. Later, \citet{git-rebasin} proposed three matching methods, one of which being data-free.
Closer to our global optimization, \citet{rebasin-implicit-sinkhorn} proposed a gradient-descent based procedure that iteratively updates soft permutation matrices maintaining their bistochasticity via a differentiable Sinkhorn routine.
When the models to match have been trained on different tasks, \citet{zip-it} introduce a more general ``zip'' operation that accounts for features that may be task-specific and further allow obtaining multi-headed models. Most recently, \citet{navon2023equivariant} proposed aligning models in the embedding space of a deep weight-space architecture. Finally, weight merging proved useful for large language models~\cite{jin2022dataless} and robotics~\cite{robots}.
For a complete survey of mode connectivity and model merging, we refer the reader to~\cite{surveydmf}.

\paragraph{Cycle consistency}
Cycle consistency is a recurrent idea in computer vision and pattern recognition, where it appears under different names (e.g., ``synchronization'', ``loop constraints'', or ``multi-way matching'') depending on the task.
In the area of multi-view 3D reconstruction, \citet{5539801} were probably the first to make an explicit attempt at finding solutions meeting the cycle-consistency requirement, although without ensuring theoretical guarantees on the result. In geometry processing, \citet{Cosmo2017209} ensured cycle-consistent alignment of collections of 3D shapes using an $n$-fold extension of the Gromov-Wasserstein distance with sparsity constraints. Overall, cycle consistency is a recurring idea in the computer vision~\cite{wang2013exact, 5539801, sync-CV} graph matching~\cite{pachauri, convex-relaxation, 7780916} and geometry processing literature~\cite{10.1145/2601097.2601111, Cosmo2017209, Bernard_2019_ICCV}.

\subsection{Pairwise Frank-Wolfe algorithm}\label{app:pairwise-frankwolfe}
As introduced in \cref{subsec:pairwise-matching}, we optimize a layer-global objective by iteratively optimizing its linear approximation via the Frank-Wolfe algorithm~\cite{frank-wolfe}. We compute the gradient of \cref{eq:weight-matching-obj} with respect to each permutation ${P}_i$, as the sum of two contributions for each $\nabla_{P_i}$:  one from permuting the rows of ${W}_i$ and another from permuting the columns of ${W}_{i+1}$:
\begin{align}
    \nabla_{P_i}f & = \underbrace{W^A_i P_{i-1} {(W_i^B)}^\top}_{\text{from permuting rows}} + \underbrace{{(W^A_{i+1})}^\top P_{i+1}  W_{i+1}^B}_{\text{from permuting columns}}.
\end{align}
We report in \Cref{alg:frank-wolfe-pairwise} the Frank-Wolfe algorithm for the pairwise case.
\input{A_details/frank_wolfe_pairwise}

\subsection{MergeMany algorithm}\label{app:merge-many}
\Cref{alg:merge-many} reports the MergeMany procedure originally proposed by \citet{git-rebasin} for merging multiple models, mainly consisting in alternating matching and aggregation until convergence. In practice, at each iteration, the procedure picks a reference model at random and matches all the other models to it. Then, they are all aggregated by averaging the weights.
\begin{algorithm}
    \caption{\textsc{MergeMany}}\label{alg:merge-many}
    \begin{algorithmic}[1]
        \REQUIRE{} Model weights $\theta_1, \dots, \theta_N$
        \ENSURE{} A merged set of parameters $\tilde{\theta}$.
        \REPEAT{}
        \FOR{$i \in \textsc{RandomPermutation}(1,\dots,N)$}
        \STATE{} $\theta' \gets \frac{1}{N-1}\sum_{j\in \{1,\dots,N\} \setminus \{i\}} \theta_j$
        \STATE{} $\pi \gets \textsc{PermutationCoordinateDescent}(\theta', \theta_i)$
        \STATE{} $\theta_i \gets \pi(\theta_i)$
        \ENDFOR{}
        \UNTIL{convergence}
        \STATE{} \textbf{return} $\frac{1}{N}\sum_{j=1}^N \theta_j$
    \end{algorithmic}
\end{algorithm}

\subsection{Convergence and efficiency}\label{app:convergence}
\begin{wrapfigure}[15]{r}{0.4\textwidth}
    \centering
    \includegraphics[width=0.4\textwidth]{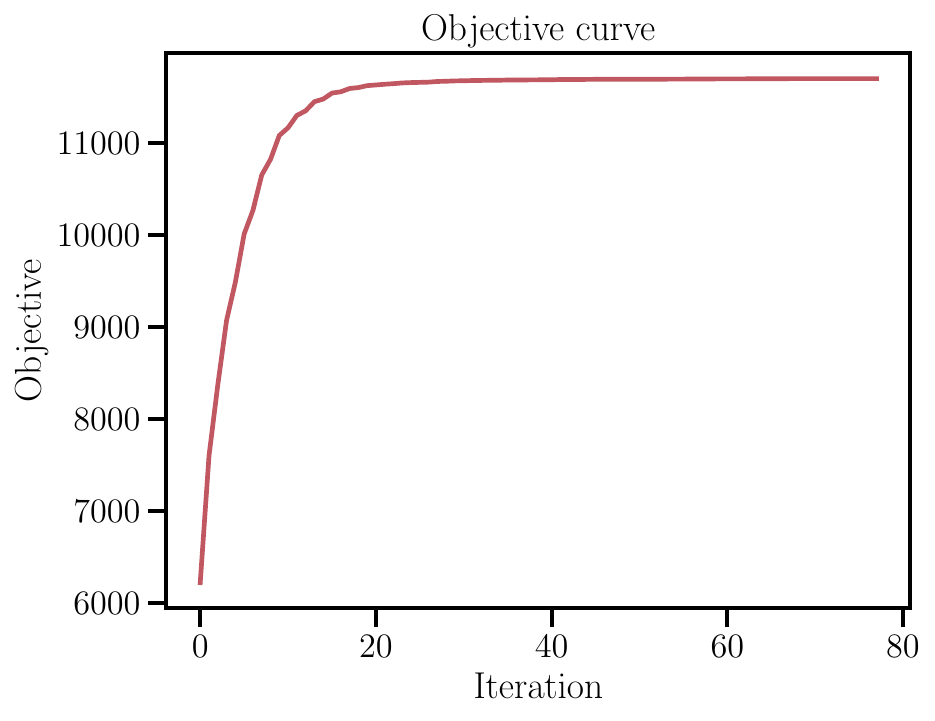}
    \caption[Objective values during optimization]{Objective values during the optimization. As guaranteed by the \method{Frank-Wolfe} algorithm, the objective value increases monotonically.}
    \label{fig:convergence-obj}
\end{wrapfigure}
We report here the convergence of our matching algorithm. In particular, \cref{fig:convergence-obj} shows the objective values during the optimization, exhibiting the expected monotonic increase, while \cref{fig:convergence-step-sizes} shows the step sizes result of the line search at each iteration.
Interestingly, \cref{fig:conv-step-sizes} shows that the step sizes are generally decreasing, but descend in an alternating manner. This is likely due to the fact that the permutations are obtained as consecutive interpolations, where even steps result in a soft permutation matrix that is the average of the current and next permutation matrix, while odd steps generally result in a hard permutation matrix with entries in $[0,1]$.
\Cref{fig:frank-wolfe-steps} finally shows the intermediate permutation values during the optimization: at each step, the entries of the permutation matrix are the linear interpolation of the current solution and the projected gradient with factor $\alpha$ given by the step size. The red values in the figure represent entries currently being updated, which are neither 1 (blue) nor 0 (yellow).

We report in \cref{tab:model_perf} the wall-clock time when merging $n=2,3$ \model{ResNet20} models having $1\times$, $2\times$, $4\times$, $8\times$ and $16\times$ width, together with their number of parameters.

\begin{table}[H]
    \centering\label{tab:model_perf}
    \begin{tabular}{cccccc}
        \toprule
                           & 1x    & 2x     & 4x     & 8x      & 16x     \\ \midrule
        \# params          & 292k  & 1.166m & 4.655m & 18.600m & 74.360m \\
        \midrule
        \multicolumn{6}{c}{n=2}                                          \\
        \midrule
        $C^2M^3$           & 33.4s & 33.5s  & 40.5s  & 80.8s   & 367.8s  \\
        \texttt{MergeMany} & 0.24s & 0.4s   & 3.4s   & 8.9s    & 59.4s   \\
        \midrule
        \multicolumn{6}{c}{n=3}                                          \\
        \midrule
        $C^2M^3$ time      & 32.9s & 83.18s & 91.0s  & 162.0s  & 715.8s  \\
        \texttt{MergeMany} & 1.2s  & 4.1s   & 19.5s  & 105.8s  & 892.3s  \\
        \bottomrule                                                      \\
    \end{tabular}
    \caption[Wall-clock time for merging ResNet20 models]{Wall-clock time for merging $n=3$ ResNet20 models with different widths.}
\end{table}
As can be inferred from the table, the scaling laws depend on the complexity of the resulting matching problem and cannot be predicted merely from the number of parameters, with a 4-fold increase in parameters resulting in no increase in runtime for the first three columns, a double increase in the second-last column and a 5-fold increase in the last. Compared to MergeMany, our approach enjoys a milder increase in running time when increasing the number of parameters. For simpler settings, however, \texttt{MergeMany} is significantly faster.
Being the two approaches on the same order of magnitude and given the one-time nature of model merging, we believe this aspect to be of secondary importance, especially considering merging to be, in many cases, an alternative to training a model from scratch.

\begin{figure}
    \begin{subfigure}{.31\textwidth}
        \centering
        \includegraphics[width=\textwidth]{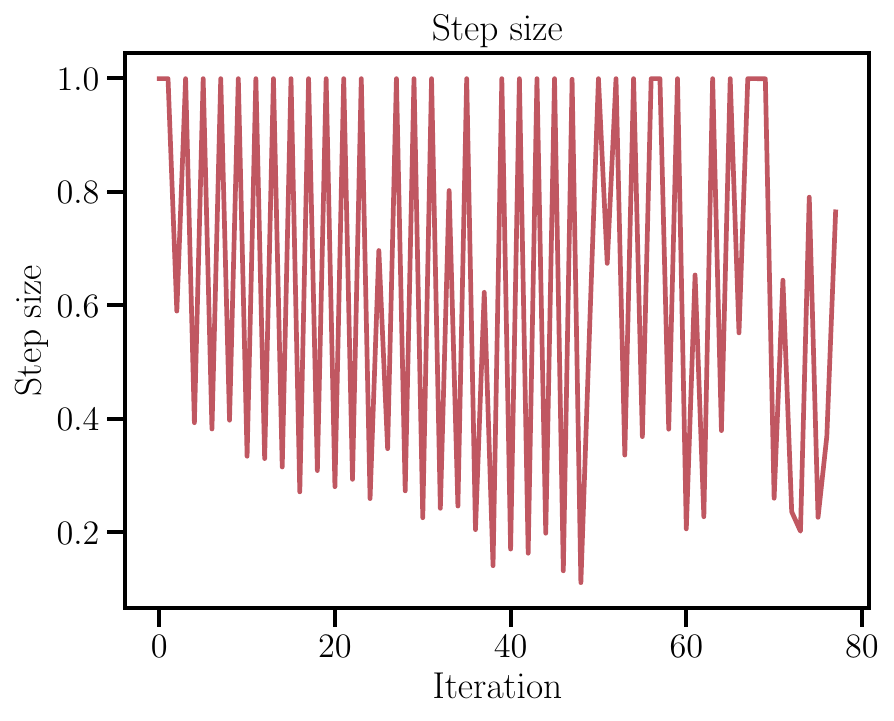}
        \caption{Step sizes for all iterations.}\label{fig:conv-step-sizes}
    \end{subfigure}
    \begin{subfigure}{.31\textwidth}
        \centering
        \includegraphics[width=\textwidth]{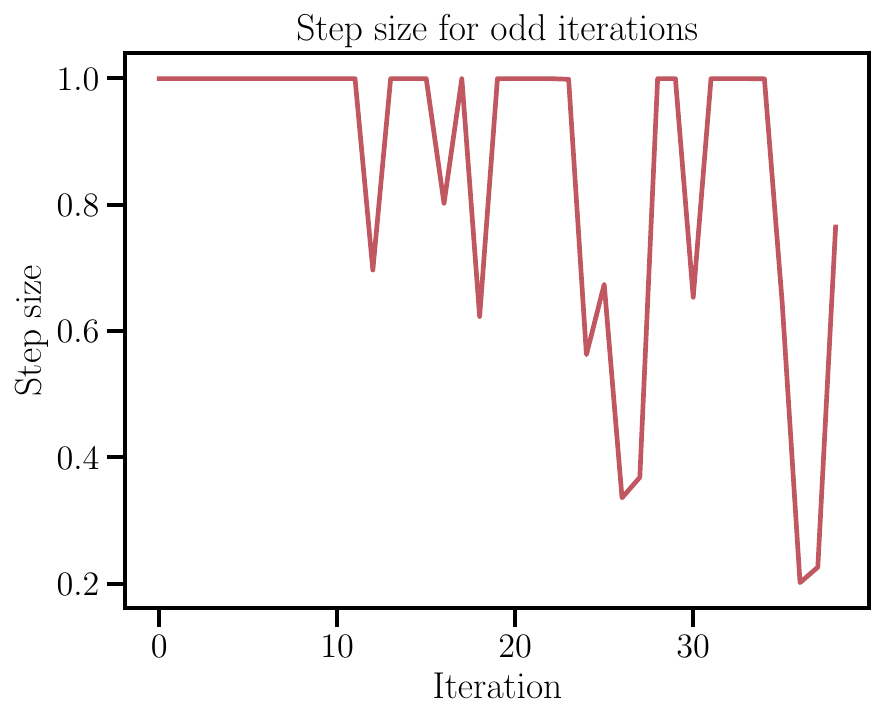}
        \caption{Step sizes for odd iterations.}\label{fig:conv-step-sizes-odd}
    \end{subfigure}
    \begin{subfigure}{.31\textwidth}
        \centering
        \includegraphics[width=\textwidth]{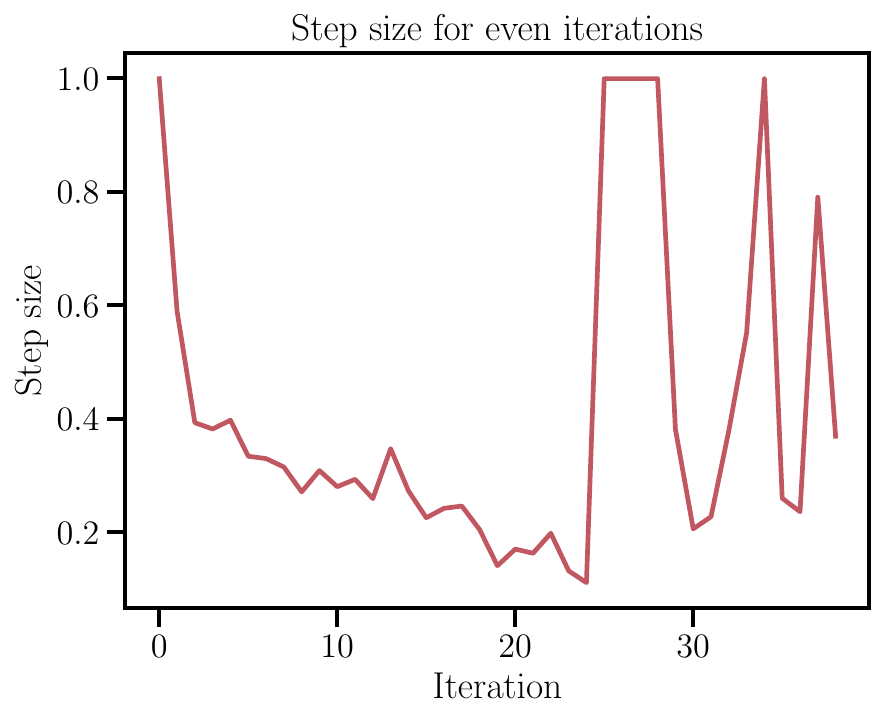}
        \caption{Step sizes for even iterations.}\label{fig:conv-step-sizes-even}
    \end{subfigure}
    \caption{Step sizes during the optimization.}\label{fig:convergence-step-sizes}
\end{figure}
\begin{figure}
    \centering
    \includegraphics[width=.8\textwidth]{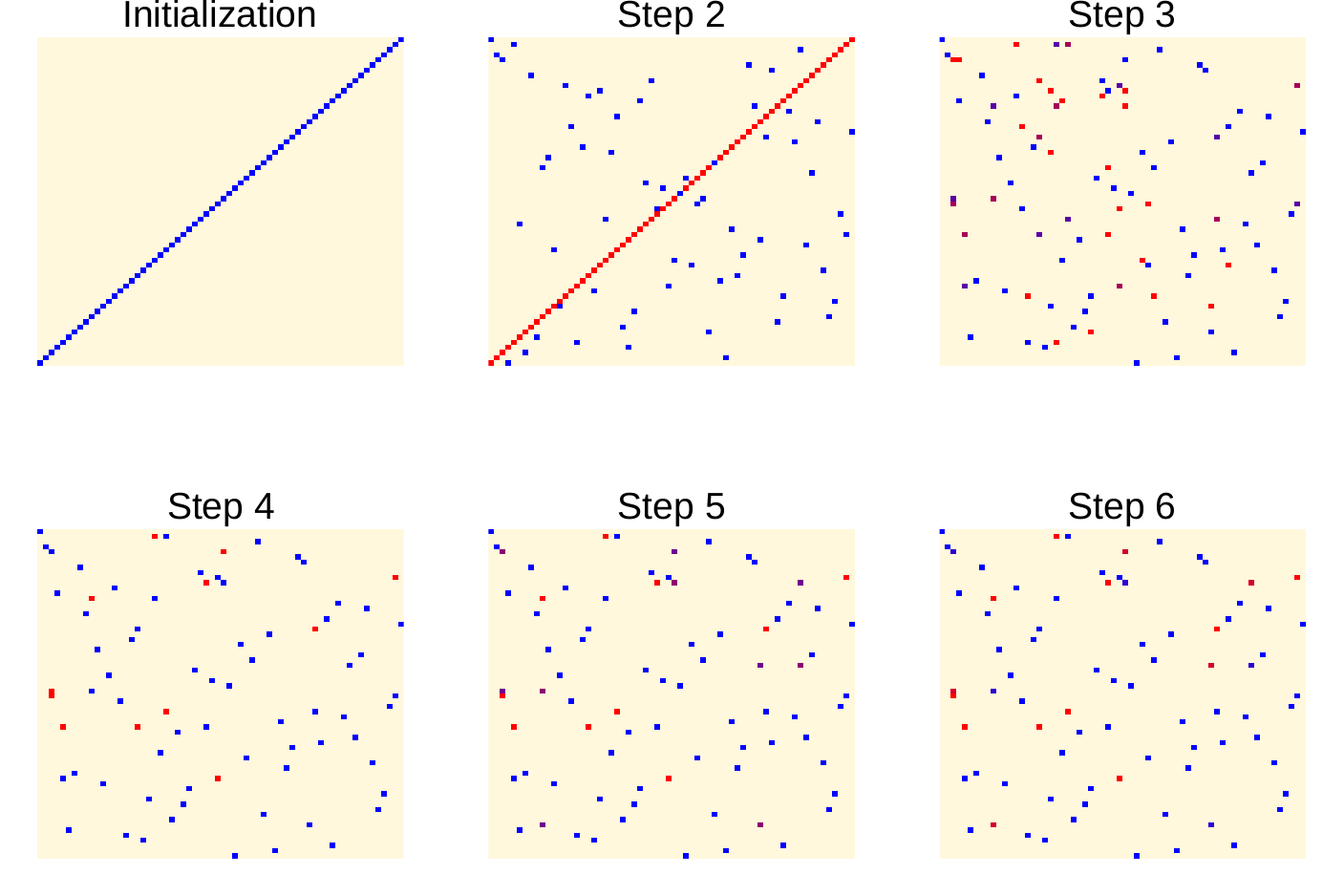}
    \caption[First 6 steps of Frank-Wolfe for one permutation matrix]{First 6 steps of \cref{alg:frank-wolfe-pairwise} for one permutation matrix. At each step, the new solution is given by the linear interpolation of the current solution and the gradient of \cref{eq:weight-matching-obj}. }\label{fig:frank-wolfe-steps}
\end{figure}

\subsection{Architectural details}\label{app:architecture-details}
We report here the architectural details of all the architectures we have used in the experiments.

\paragraph{Multi-Layer Perceptrons}
We use a simple MLP mapping input to a $256$-dimensional space followed by 3 hidden layers of 512, 512 and 256 units respectively, followed by an output layer mapping to the number of classes. We use \emph{ReLU} activations for all layers except the output layer, where we use a \emph{log\_softmax} activation.

\paragraph{ResNet}
We consider a \model{ResNet20}~\cite{resnet} architecture composed by three ResNet block groups, each containing three residual blocks. The model starts with an initial convolutional layer followed by normalization and \emph{ReLU} activation. It then passes through the three block groups with increasing channel sizes (determined by the widen factor) and varying strides, followed by global average pooling and a fully connected layer that outputs class logits. As normalization layers, we consider both the most commonly used \emph{BatchNorm}~\cite{batchnorm} and, for the sake of comparing with \method{Git Re-Basin}, also \emph{LayerNorm}~\cite{layernorm}. The results in the main manuscript are all obtained with \emph{LayerNorm}, while we report the results with \emph{BatchNorm} in \cref{app:batch-norm}.

\paragraph{VGG}
We employ a \model{VGG16}~\cite{simonyan2015deep} architecture with \emph{LayerNorm}~\cite{layernorm} normalization layers. The model has the following convolutional layer dimensions, with ``M'' indicating the presence of a max-pooling layer
\begin{equation}
    64, 64, M, 128, 128, M, 256, 256, 256, M, 512, 512, 512, M, 512, 512, 512, M
\end{equation}
The convolutional layers are organized in 5 blocks, each containing 2 or 3 convolutional layers, followed by a max-pooling layer. The final classifier is composed of three fully connected layers with $512$ hidden dimension and ReLU activations.

\subsection{Datasets, hyperparameters and hardware details}
We employ the most common datasets for image classification tasks: \dataset{MNIST}~\cite{MNIST}, \dataset{CIFAR-10}~\cite{CIFAR}, \dataset{EMNIST}~\cite{EMNIST} and \dataset{CIFAR-100}~\cite{CIFAR}, having $10$, $10$, $26$ and $100$ classes respectively. We use the standard train-test splits provided by \emph{torchvision} for all datasets.

We use the same hyperparameters as \method{Git Re-Basin} where possible to ensure a fair comparison. In particular, we train most of our models with a batch size of $100$ for $250$ epochs, using SGD with momentum $0.9$, a learning rate of $0.1$, and a weight decay of $10^{-4}$. We use a cosine annealing learning rate scheduler with a warm restart period of $10$ epochs and a minimum learning rate of $0$. We report each and every one of the hyperparameters used for each experiment, as well as all the trained models, in the project's WandB dashboard (available in the released codebase).

All of the experiments were carried out using consumer hardware, in particular mostly on a 32GiB RAM machine with a \emph{12th Gen Intel(R) Core(TM) i7-12700F} processor and an \textit{Nvidia RTX} 3090 GPU, except for some of the experiments that were carried out on a 2080. Our modular and reusable codebase is based on \textit{PyTorch}, leveraging \textit{PyTorch Lightning} to ensure reproducible results and modularity and \textit{NN-Template}\footnote{https://github.com/grok-ai/nn-template} to easily bootstrap the project and enforce best practices.

\subsection{Proofs}

\begin{theorem}
    The gradient of the objective function
    \[
        \sum_{p=1}^{n-1} \sum_{q=p+1}^{n}  \sum_{\ell=1}^L \langle (P_{\ell}^p )^\top W_\ell^p P_{\ell -1}^p, (P_{\ell}^q)^\top W_{\ell}^q P^q_{\ell -1} \rangle
    \]
    is Lipschitz continuous, implying our algorithm obtains a stationary point at a rate of $\mathcal{O}(1 / \sqrt{t})$~\cite{lacoste2016convergence}.
\end{theorem}
\begin{proof}
    We recall that, for each layer permutation $P^A = \{P_1^A, P_2^A, \ldots, P_{L}^A\}$ of model $A$, we can define the gradient of our objective function relatively to the model $B$ we are matching towards:
    \begin{align*}
        f(P_\ell^A) & = \nabla^{\text{rows}}_{P_\ell^A} + \nabla^{\text{cols}}_{P_\ell^A} + \nabla^{\text{rows},\leftrightarrows}_{P_\ell^A} + \nabla^{\text{cols},\leftrightarrows}_{P_\ell^A} = \\
                    & \left[W^A_{\ell} P_{\ell-1}^A (P^B_{\ell-1})^\top (W^B_{\ell})^\top + (W^A_{\ell+1})^\top P_{\ell+1}^A (P^B_{\ell+1})^\top W^B_{\ell+1}\right] P^B_{\ell}
        +                                                                                                                                                                                         \\
                    & \left[W^B_{\ell} P_{\ell-1}^B(P^A_{\ell-1})^\top (W^A_{\ell})^\top + (W^B_{\ell+1})^\top P_{\ell+1}^B (P^A_{\ell+1})^\top W^A_{\ell+1}\right] P^A_{\ell}
    \end{align*}
    To prove Lipschitz continuity, we need to show there exists a constant $C$ such that $\forall\; p=1,\dots,n,\; \ell=1,\dots,L\;\;\; \| f(P_\ell^p) - f(Q_\ell^p) \| \leq C \| P_\ell^p - Q_\ell^p \|$. To simplify passages, we only consider a fixed $\ell$ and perform a generic analysis. We begin by observing that
    \begin{align*}
        f(P_\ell^p)                                                                  & - f(Q_\ell^p) =                                                                                              \\
        \sum_{q\in[1,n]\setminus \{p\}}
        \left[ W^p_{\ell} P_{\ell-1}^p (P^q_{\ell-1})^\top (W^q_{\ell})^\top \right. & + \left. (W^p_{\ell+1})^\top P_{\ell+1}^p (P^q_{\ell+1})^\top W^q_{\ell+1}\right](P^q_{\ell} - Q^q_{\ell}) + \\
        \left[ W^q_{\ell} P_{\ell-1}^q(P^p_{\ell-1})^\top (W^p_{\ell})^\top \right.  & + \left. (W^q_{\ell+1})^\top P_{\ell+1}^q (P^p_{\ell+1})^\top W^p_{\ell+1}\right]  (P^p_{\ell}-Q^p_{\ell})
    \end{align*}
    The last form of the above equation can be rewritten as a sum of the two sums:
    \begin{align*}
        \sum_{q\in[1,n]\setminus \{p\}}  \left[W^p_{\ell} P_{\ell-1}^p (P^q_{\ell-1})^\top (W^q_{\ell})^\top \right. & +  \left.(W^p_{\ell+1})^\top P_{\ell+1}^p(P^q_{\ell+1})^\top W^q_{\ell+1}\right] (P^q_{\ell} - Q^q_{\ell})  + \\
        \sum_{q\in[1,n]\setminus \{p\}}  \left[W^q_{\ell} P_{\ell-1}^q(P^p_{\ell-1})^\top (W^p_{\ell})^\top \right.  & + \left.(W^q_{\ell+1})^\top P_{\ell+1}^q (P^p_{\ell+1})^\top W^p_{\ell+1}\right] (P^p_{\ell}-Q^p_{\ell})
    \end{align*}
    Since the first term does not depend on either $P_\ell^p$ or $Q_\ell^p$, we assume as a worst case that its norm is 0. Then, we remove transposes (since $\lVert M \rVert = \lVert M^\top \rVert$) and apply the triangle inequality and the sub-multiplicative property of matrix norms:
    \begin{align*}
        \| f(P_\ell^p) - f(Q_\ell^p) \|                                                                                                         & \leq                                                                                 \\
        \sum_{q\in[1,n]\setminus \{p\}} \|P^p_{\ell}-Q^p_{\ell}\| \left( \|W^q_{\ell}\| \|P_{\ell-1}^q\|\|P^p_{\ell-1}\| \|W^p_{\ell}\| \right. & + \left. \|W^q_{\ell+1}\| \|P_{\ell+1}^q\| \|P^p_{\ell+1}\| \|W^p_{\ell+1}\| \right)
    \end{align*}
    \noindent Let $C = \max_{q\in[1,n]\setminus \{p\}}\left\{ \|W^q_{\ell}\| \|P_{\ell-1}^q\|\|P^p_{\ell-1}\| \|W^p_{\ell}\| + \|W^q_{\ell+1}\| \|P_{\ell+1}^q\| \|P^p_{\ell+1}\| \|W^p_{\ell+1}\|\right\}$. Then,
    \begin{equation*}
        \| f(P_\ell^p) - f(Q_\ell^p) \| \leq C  \sum_{q\in[1,n]\setminus \{p\}} \|P^p_{\ell}-Q^p_{\ell}\| = C (n-1) \|P^p_{\ell}-Q^p_{\ell}\|
    \end{equation*}

    we conclude that  $f(P_\ell^p)$  is Lipschitz continuous for all models and all layers, with Lipschitz constant $C(n-1)$ depending on both the norm of the weights matrices and the number of models.

\end{proof}

%% file: 99_appendix/A_cycle_consistent/A_details/frank_wolfe_pairwise.tex
\begin{algorithm} \label{alg:frankwolfepw}
\caption{Frank-Wolfe for pairwise Weight Matching}
\label{alg:frank-wolfe-pairwise}
    \begin{algorithmic}[1]
    \REQUIRE Weights of two models $A$ and $B$ with $L$ layers, tolerance $\epsilon > 0$
    \ENSURE Approximate solution to \cref{eq:weight-matching-obj}
    \STATE $\mathbf{P}^k \gets $ identity matrices
    \REPEAT{}
        \FOR{$i=1$ to $L$}  
            \STATE $P_i^k \gets$ permutation acting on rows of $W_i$
            \STATE $P_{i-1}^k \gets$ permutation acting on columns of $W_i$
            \STATE $\nabla_{P_i^k} f \mathrel{+}= W^A_i P_{i-1}^k (W_i^B)^\top$ 
            \STATE $\nabla_{P_{i-1}^k} f \mathrel{+}= (W_i^A)^\top P_{i}^k W_i^B$
        \ENDFOR
        \FOR{$P^k_i \in \mathbf{P}^k$}  
            \STATE $\Pi_i \gets \operatorname{LAP}(\nabla_{P_i^k} f)$ 
        \ENDFOR
        \STATE $\alpha \gets \textsc{LineSearch}(f, \mathbf{P}^k, \mathbf{\Pi})$
        \FOR{$P^k_i \in \mathbf{P}^k$}
            \STATE $P_i^{k+1} = (1-\alpha) P_i^k + \alpha  \ \Pi_i$
        \ENDFOR
    \UNTIL $\|f(A, B, \mathbf{P}^{k+1}) - f(A, B, \mathbf{P}^{k})\| < \epsilon$
    \STATE \textbf{return} $\mathbf{P}^k$
    \end{algorithmic}
\end{algorithm}

%% file: 99_appendix/A_cycle_consistent/B_experiments/content.tex
We report additional experiments and results in this section. In particular, \cref{subsec:exp-pairwise-matching} presents a complete evaluation of our matching method for the pairwise case, showing it to be generally competitive with the state-of-the-art \texttt{Git Re-Basin} algorithm~\cite{git-rebasin} and to outperform it on architectures employing \emph{BatchNorm}~\cite{batchnorm} normalization. We then discuss different permutation initialization strategies in \cref{app:init-strategies}.

\subsection{Pairwise model matching and merging}\label{subsec:exp-pairwise-matching}
\begin{wraptable}{l}{0.4\textwidth}
    \begin{center}
        \resizebox{0.4\textwidth}{!}{%
            \begin{tabular}{clccc}
                \toprule
                                                                                                                           & \multirow{2}{*}{\textbf{Matcher}} & \multicolumn{2}{c}{\textbf{Barrier}}                            \\
                \cmidrule{3-4}
                                                                                                                           &                                   & Train                                & Test                     \\
                \midrule
                \parbox[t]{4mm}{\multirow{3}{*}{ \rotatebox[origin=c]{90}{$\underset{8\times}{\text{\texttt{ResNet}}}$} }} & \texttt{Naive}                    & 7.00 $\pm$ 1.24                      & 8.37 $\pm$ 1.23          \\
                                                                                                                           & \texttt{Git-Rebasin}              & 1.04 $\pm$ 0.10                      & 1.54 $\pm$ 0.13          \\

                                                                                                                           & \texttt{Frank-Wolfe}              & $\mathbf{0.92 \pm 0.06}$             & $\mathbf{1.42 \pm 0.10}$ \\
                \midrule
                \parbox[t]{4mm}{\multirow{3}{*}{ \rotatebox[origin=c]{90}{\texttt{VGG16}} }}
                                                                                                                           & \texttt{Naive}                    & 5.79 $\pm$ 0.39                      & 7.36 $\pm$ 0.38          \\
                                                                                                                           & \texttt{Git-Rebasin}              & 0.44 $\pm$ 0.03                      & 0.64 $\pm$ 0.03          \\
                                                                                                                           & \texttt{Frank-Wolfe}              & $\mathbf{0.44 \pm 0.05}$             & $\mathbf{0.63} \pm 0.06$ \\
                \bottomrule
            \end{tabular}
        }
    \end{center}
    \caption[Pairwise loss barriers on \texttt{CIFAR100}]{Mean and standard deviation of the test and train loss barriers for each method when matching $n=2$ models on \texttt{CIFAR100}.}
    \label{tab:pairwise_barriers_table}
\end{wraptable}
As described in \cref{subsec:pairwise-matching}, our formalization can readily be used to match $n=2$ models. In this case, the energy is given by \cref{eq:weight-matching-obj} and the permutations are not factorized. We compare the performance of our approach against the \texttt{Git Re-Basin} algorithm~\cite{git-rebasin} and the \texttt{naive} baseline that aggregates the models by taking an unweighted mean on the original model weights without applying any permutation. From the data presented in \cref{tab:pairwise_barriers_table}, we observe that the approach is competitive with \texttt{Git Re-Basin}~\cite{git-rebasin}, with the two methods exhibiting analogously low test barrier on \texttt{CIFAR10}.
Focusing on the \texttt{ResNet20} architecture, we can see that width plays the same role in both approaches, with the barrier decreasing as it increases. We can also appreciate how, while the same architecture resulted in similar barriers for the two approaches on \texttt{CIFAR10}, the barrier is significantly lower for \texttt{Frank-Wolfe} in \texttt{CIFAR100}, possibly suggesting that the latter is more robust to the complexity of the dataset.

\input{B_experiments/pairwise_barriers_table.tex}

\subsubsection{ResNet with BatchNorm}\label{app:batch-norm}
\begin{wraptable}[8]{r}{0.4\textwidth}
    \vspace{-0.6cm}
    \begin{center}
        \resizebox{0.4\textwidth}{!}{%
            \begin{tabular}{lccc}
                \toprule
                \multirow{2}{*}{Matcher} & \multicolumn{2}{c}{loss barrier $(\downarrow)$}                            \\
                \cmidrule{2-3}
                                         & train                                           & test                     \\
                \midrule
                \texttt{Naive}           & 4.72 $\pm$ 0.86                                 & 4.99 $\pm$ 0.86          \\
                \texttt{Git Re-Basin}    & 4.33 $\pm$ 0.64                                 & 4.62 $\pm$ 0.65          \\
                \texttt{Frank-Wolfe}     & \textbf{3.53 $\pm$ 0.58}                        & \textbf{3.79 $\pm$ 0.57} \\
                \bottomrule
            \end{tabular}
        }
    \end{center}
    \caption[Loss barriers for \texttt{ResNet20} with \textit{BatchNorm}]{Mean and stddev of the test and train loss barriers on $2$ \texttt{ResNet20-2$\times$} models with \textit{BatchNorm} normalization.}
    \label{tab:ResNet-BatchNorm-pairwise}
\end{wraptable}
We also report the results of a ResNet20 with $2\times$ width using \emph{BatchNorm} \cite{batchnorm} layers instead of \emph{LayerNorm}~\cite{layernorm} ones. This version, as noted in~\cite{repair}, is in fact harder to match but also the one that is commonly used in practice. We can see in \cref{tab:ResNet-BatchNorm-pairwise} that the \texttt{Frank-Wolfe} matcher is able to achieve a lower barrier than \texttt{Git Re-Basin}, indicating the approach to be more robust to architectures using different normalization layers.

\subsection{Initialization strategies} \label{app:init-strategies}
As introduced in \cref{alg:frank-wolfe-generalized}, we initialize each $N$-dimensional permutation to be the $N \times N$ identity matrix. We now compare this strategy against two alternatives that provide doubly stochastic
\begin{wraptable}[13]{l}{0.3\textwidth}
    \centering
    \resizebox{0.3\textwidth}{!}{%
        \begin{tabular}{cccc}
            \toprule
            \multirow{2}{*}{models} & \multicolumn{3}{c}{loss barrier $(\downarrow)$}                                              \\
            \cmidrule{2-4}
                                    & id                                              & barycenter      & Sinkhorn                 \\
            \midrule
            (a, b)                  & $0.52$                                          & $\mathbf{0.47}$ & $0.60 \pm 0.04$          \\
            (b, c)                  & $0.65$                                          & $0.70$          & $\mathbf{0.64 \pm 0.06}$ \\
            (a, c)                  & $0.97$                                          & $0.95$          & $\mathbf{0.92 \pm 0.07}$ \\
            \bottomrule                                                                                                            \\
        \end{tabular}
    }
    \caption[Test barrier with different initialization strategies]{Test barrier of the interpolations of 3 \texttt{ResNet20-2$\times$} models using different initializations.} \label{tab:init-strategies}
\end{wraptable}
matrices, \emph{i.e.}, such that their rows and columns sum to one: i) the Sinkhorn initialization~\cite{sinkhorn} that initializes the permutation matrix as the solution of the Sinkhorn-Knopp algorithm \cite{sinkhorn}; ii) the barycenter of doubly stochastic matrices, \emph{i.e.} the matrix where each element is given by $1/N$. \Cref{tab:init-strategies} shows the test barrier of the interpolations of three \texttt{ResNet20-2$\times$} models $a, b$, and $c$ when using the different strategies over 10 different trials.  We can see that the constant initializations (identity and barycenter) work well in general, with the additional benefit of having 0 variance in the results. On the other hand, if computational cost is not a concern, one can still choose to run a pool of trials with different Sinkhorn initializations and finally select the best one, trading this way efficiency with some extra accuracy points.

\subsection{Variance of the results in Git Re-Basin}
As introduced in \cref{subsec:exp-pairwise-matching-brief}, \texttt{Git Re-Basin}~\cite{git-rebasin} depends on a random choice of layers, resulting in variations of up to $10\%$ in accuracy depending on the optimization seed, while our method shows zero variance. While we have already seen the results for a model pair in \cref{fig:git-re-basin-variance}, we report, for completeness, the results of matching and averaging models with \texttt{Git Re-Basin} using different optimization seeds for additional pairs. As can be seen in \cref{tab:git-rebasin-variance}, the trend is confirmed over these ones, with results significantly oscillating and our approach always above or on par with their mean.

\begin{table}
    \begin{center}
        \resizebox{\textwidth}{!}{%
            \begin{tabular}{ccccccccccccccc}
                \toprule
                models                 &       & 1    & 2    & 3    & 4    & 5    & 6    & 7    & 8    & 9    & mean & stddev & max gap & Frank-Wolfe   \\
                \midrule
                \multirow{2}{*}{(1,2)} & train & 0.76 & 0.78 & 0.78 & 0.80 & 0.77 & 0.76 & 0.78 & 0.75 & 0.81 & 0.78 & 0.018  & 0.057   & 0.78          \\
                                       & test  & 0.73 & 0.75 & 0.75 & 0.77 & 0.74 & 0.73 & 0.74 & 0.72 & 0.78 & 0.74 & 0.018  & 0.060   & \textbf{0.75} \\
                \midrule
                \multirow{2}{*}{(1,3)} & train & 0.67 & 0.69 & 0.69 & 0.69 & 0.62 & 0.69 & 0.66 & 0.71 & 0.68 & 0.68 & 0.023  & 0.085   & 0.68          \\
                                       & test  & 0.64 & 0.66 & 0.67 & 0.65 & 0.60 & 0.66 & 0.63 & 0.67 & 0.65 & 0.65 & 0.020  & 0.071   & 0.65          \\
                \midrule
                \multirow{2}{*}{(2,3)} & train & 0.75 & 0.74 & 0.75 & 0.72 & 0.76 & 0.74 & 0.70 & 0.73 & 0.78 & 0.74 & 0.020  & 0.074   & \textbf{0.76} \\
                                       & test  & 0.70 & 0.71 & 0.71 & 0.68 & 0.72 & 0.70 & 0.67 & 0.70 & 0.74 & 0.70 & 0.020  & 0.071   & \textbf{0.72} \\
                \bottomrule
            \end{tabular}
        }
    \end{center}
    \caption[Variance of \texttt{Git Re-Basin} across optimization seeds]{Accuracy of the interpolated model using \texttt{Git Re-Basin}~\cite{git-rebasin} over different pairs of models $(1,2), (1,3), (2,3)$ by changing random seed $i=1,\dots,9$ in the weight matching procedure.} \label{tab:git-rebasin-variance}
\end{table}


\subsection{Large-scale matching: ResNet50s trained over ImageNet}
For this experiment, we matched three different \texttt{ResNet50}s trained over \texttt{ImageNet}. We used three publicly available pre-trained checkpoints from \emph{timm}, namely \texttt{a1\_in1k}\footnote{\url{https://huggingface.co/timm/resnet50.a1_in1k}}, \texttt{c1\_in1k}\footnote{\url{https://huggingface.co/timm/resnet50.c1_in1k}} and \texttt{ram\_in1k} \footnote{\url{https://huggingface.co/timm/resnet50.ram_in1k}}.
As \cref{tab:merge-many-resnet50} shows, \texttt{$C^2M^3$} underperforms the baseline in this case. To see why, we report in \cref{fig:resnet50-heatmap} the pairwise accuracies obtained using pairwise weight matching over all the \texttt{ResNet50} checkpoints available in \textit{timm}.
\begin{wraptable}[13]{r}{0.3\textwidth}
    \begin{center}
        \resizebox{0.3\textwidth}{!}{%
            \begin{tabular}{ccc}
                \toprule
                \textbf{Matcher}             & \textbf{Accuracy} $(\uparrow)$ & \textbf{Loss} ($\downarrow$) \\
                \midrule
                \texttt{Naive}               & 0.001                          & 6.91                         \\
                \texttt{MergeMany}           & 0.001                          & 6.91                         \\
                \texttt{MergeMany}$^\dagger$ & 0.30                           & 4.87                         \\
                \texttt{$C^2M^3$}            & 0.001                          & 6.91                         \\
                \texttt{$C^2M^3$}$^\dagger$  & 0.07                           & 6.13                         \\
                \bottomrule
            \end{tabular}
        }
    \end{center}
    \caption[Merging \texttt{ResNet50} models on \texttt{ImageNet}]{Accuracy and loss of the interpolated model using different matchers over three \texttt{ResNet50} models trained on \texttt{ImageNet}. }
    \label{tab:merge-many-resnet50}
\end{wraptable}
Let us focus on the triplet (am, a2, ram) and replace the model names with (a, b, c) for clarity. We see that, while the mergings (a, b) and (b, c) result in high-accuracy models, the merging (a, c) yields poor results.  Given the cycle consistency of our method, we inherit the difficulty of the hardest pair, which in this case is (a, c). It is worth noting that this behavior is not present in the other cases we investigated in this thesis, and might be due to the considered models being trained with different training schedules and hyperparameters. Future research could investigate new strategies to handle such cases, \emph{e.g.} by iteratively merging models by following a max-accuracy path in an accuracy weighted graph.

\begin{figure}
    \centering
    \includegraphics[width=0.8\textwidth]{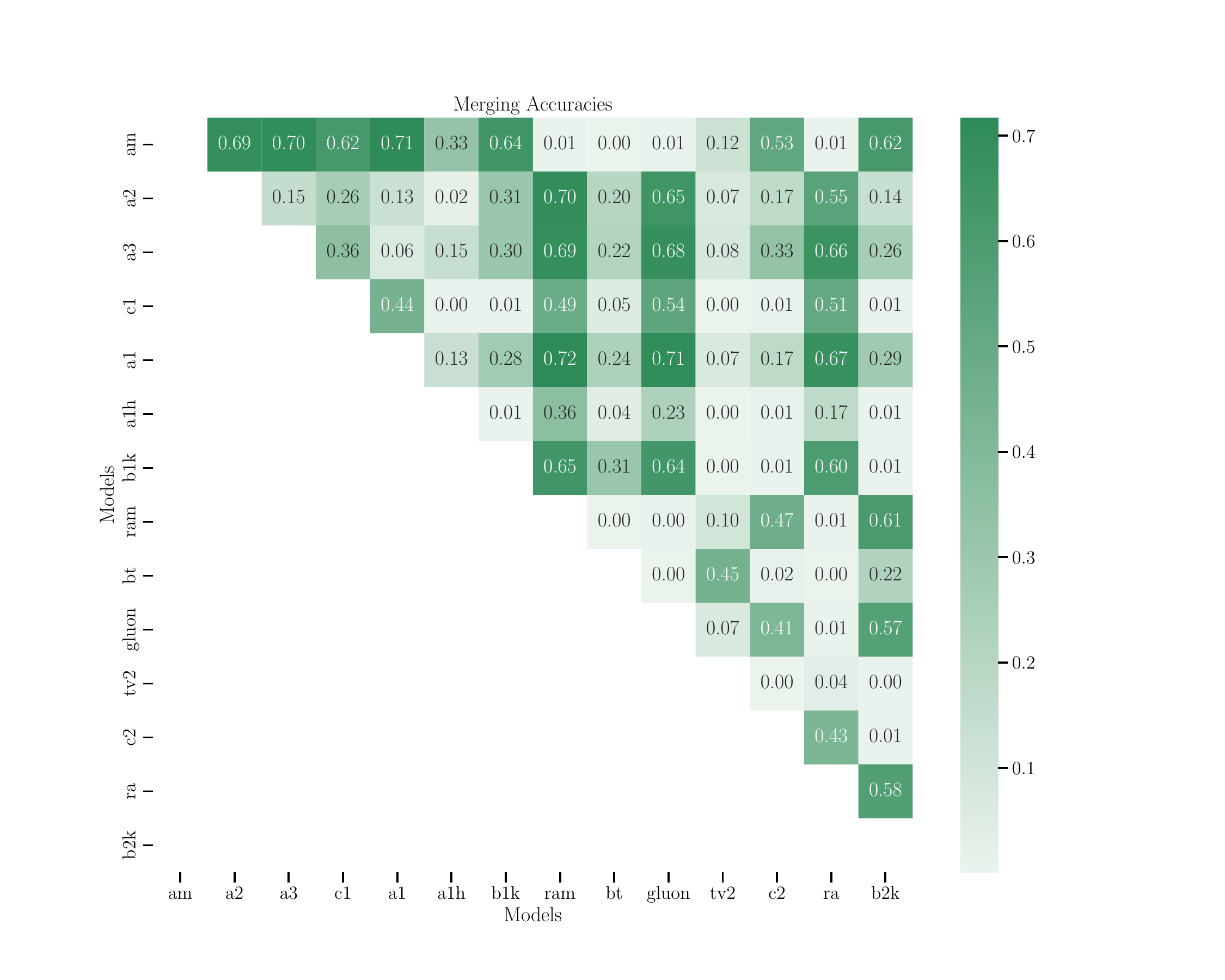}
    \caption[Pairwise merging accuracies for \texttt{ResNet50} models]{Pairwise accuracies obtained using \texttt{Git Re-Basin} \citep{git-rebasin} over different \texttt{ResNet50} models trained on \texttt{ImageNet}. Models are available from the \emph{timm} library.}
    \label{fig:resnet50-heatmap}
\end{figure}

\subsection{Federated learning}\label{app:cycle-cons-fed-learning}
We here report the results of a preliminary experiment where we ran our framework in a federated learning scenario. To this end, we have used the state-of-the-art federated learning library Flower~\footnote{\url{https://flower.ai/}}~\citep{beutel2020flower} and employed our matching scheme over a set of 10 clients over \texttt{CIFAR10}, each adopting a small CNN model. We observe the following:
\begin{itemize}
    \item When all the clients start from the same initialization, our approach has no benefit and falls back to standard averaging. In fact, the optimization process quickly returns identity matrices as permutations, suggesting the models already share the same basin.
    \item When instead we initialize the clients from different random initializations, \cref{tab:federated-learning,tab:federated-learning-2} show that our approach visibly outperforms FedAVG. In particular, the benefits get more pronounced when increasing the number of local epochs. This is in line with the intuition that standard averaging becomes less effective when clients drift due to prolonged local training and too infrequent aggregation.
\end{itemize}

\begin{table}
    \begin{center}
        \resizebox{0.9\textwidth}{!}{%
            \begin{tabular}{cccccccccccc}
                \toprule
                \multirow{2}{*}{Round} & \multicolumn{10}{c}{Accuracy}                                                                                                                                                                 \\
                \cmidrule{2-11}
                                       & 1                             & 5               & 10              & 15              & 20              & 25              & 30              & 35              & 40             & 45             \\
                \midrule
                FedAvg                 & 0.0942                        & 0.394           & 0.4972          & 0.5517          & 0.5699          & 0.5893          & 0.6018          & 0.6063          & 0.6099         & 0.6136         \\
                $C^2M^3$               & 0.0941                        & \textbf{0.4234} & \textbf{0.5193} & \textbf{0.5555} & \textbf{0.5783} & \textbf{0.5978} & \textbf{0.6077} & \textbf{0.6165} & \textbf{0.618} & \textbf{0.622} \\
                \bottomrule
            \end{tabular}
        }
    \end{center}
    \caption[Federated learning accuracy over 50 rounds with 20 local epochs]{Accuracy over 10 clients in a federated learning scenario. We report the accuracy for 50 aggregation rounds, with each client training for 20 local epochs. We report one every five rounds for the sake of clarity.}
    \label{tab:federated-learning}

\end{table}

\begin{table}
    \begin{center}
        \resizebox{0.9\textwidth}{!}{%
            \begin{tabular}{cccccccccccc}
                \toprule
                \multirow{2}{*}{Round} & \multicolumn{10}{c}{Accuracy}                                                                                  \\
                \cmidrule{2-11}
                                       & 1                             & 2      & 3      & 4      & 5      & 6      & 7      & 8      & 9      & 10     \\
                \midrule
                FedAvg                 & 0.0942                        & 0.2638 & 0.3543 & 0.3825 & 0.4165 & 0.4505 & 0.4742 & 0.4994 & 0.5169 & 0.5317 \\
                $C^2M^3$               & \textbf{0.0947}                        & \textbf{0.3303} & \textbf{0.3899} & \textbf{0.4441} & \textbf{0.4764} & \textbf{0.4968} & \textbf{0.5184} & \textbf{0.5334} & \textbf{0.5434} & \textbf{0.5536} \\
                \bottomrule
            \end{tabular}
        }
    \end{center}
    \caption[Federated learning accuracy over 10 rounds with 30 local epochs]{Accuracy over 10 clients in a federated learning scenario. We report the accuracy for 10 aggregation rounds, with each client training for 30 local epochs.}
    \label{tab:federated-learning-2}
\end{table}
While these results are not sufficient to claim an overall supremacy of the approach for the task due to the limited evaluation and choice of models, they show the approach to be promising for the problem and encourage further research.

%% file: 99_appendix/A_cycle_consistent/B_experiments/pairwise_barriers_table.tex
\begin{table}
    \begin{center}
        \resizebox{0.9\textwidth}{!}{%
            \begin{tabular}{lccccccccccc}
                \toprule
                \multirow{3}{*}{\textbf{Matcher}}      & \multicolumn{10}{c}{\textbf{Barrier}} \\
                \cmidrule{2-11}
                                        & \multicolumn{2}{c}{$\underset{2\times}{\text{\texttt{ResNet}}}$} & \multicolumn{2}{c}{$\underset{4\times}{\text{\texttt{ResNet}}}$} & \multicolumn{2}{c}{$\underset{8\times}{\text{\texttt{ResNet}}}$} & \multicolumn{2}{c}{$\underset{16\times}{\text{\texttt{ResNet}}}$} & \multicolumn{2}{c}{$\text{\texttt{VGG16}}$} \\
                \midrule
                & Train                                                   & Test                                                    & Train                                                   & Test                                                    & Train                              & Test   & Train  & Test   & Train  & Test   \\
                \cmidrule{2-3} \cmidrule{4-5} \cmidrule{6-7} \cmidrule{8-9} \cmidrule{10-11}
                \texttt{Naive}        & 5.16 $\pm$ 1.83                                                  & 5.45 $\pm$ 1.83                                                  & 2.94 $\pm$ 0.27                                                  & 3.26 $\pm$ 0.27                                                  & 2.12 $\pm$ 0.03                             & 2.40 $\pm$ 0.03 & 1.84 $\pm$ 0.18 & 2.12 $\pm$ 0.17 & 1.85 $\pm$ 0.00 & 2.31 $\pm$ 0.00 \\
                \texttt{Git Re-Basin} & 0.73 $\pm$ 0.16                                                  & 0.86 $\pm$ 0.17                                                  & \textbf{0.74 $\pm$ 0.35}                                         & 0.80 $\pm$ 0.40                                                  & 0.19 $\pm$ 0.03                             & 0.13 $\pm$ 0.02 & 0.17 $\pm$ 0.02 & 0.07 $\pm$ 0.02 & 0.08 $\pm$ 0.03 & 0.24 $\pm$ 0.03 \\
                \texttt{Frank-Wolfe}  & 0.73 $\pm$ 0.19                                                  & 0.85 $\pm$ 0.19                                                  & 0.78 $\pm$ 0.33                                                  & 0.81 $\pm$ 0.38                                                  & 0.19 $\pm$ 0.03                             & 0.12 $\pm$ 0.02 & 0.16 $\pm$ 0.02 & 0.06 $\pm$ 0.02 & 0.08 $\pm$ 0.03 & 0.25 $\pm$ 0.03 \\
                \midrule
            \end{tabular}
        }
    \end{center}
    \caption[Pairwise loss barriers on \dataset{CIFAR10}]{Mean and standard deviation of the test and train loss barrier for each method when matching $n=2$ models on \dataset{CIFAR10}.}
    \label{tab:pairwise_barriers_cifar10}
\end{table}


%% file: 99_appendix/A_cycle_consistent/C_analysis/content.tex
In this section, we report additional analyses that complement the results presented in the main text. We first analyze in \cref{subsec:exp-similarity} how mapping to universe affects the similarity of the models; then, we evaluate how the composition of the match set affects the accuracy of the merged model in \cref{app:subsets}.

\subsection{Similarity of models}\label{subsec:exp-similarity}
We analyze here how similar are models before and after being mapped to the universe space, first by comparing their representations and then by comparing their weights.
\subsubsection{Representation-level similarity}
\Cref{fig:repr-cka-similarities-orig,fig:repr-cka-similarities-univ} show the Centered Kernel Alignment (CKA) \cite{Kornblith2019-vr} of the representations of 5 \model{ResNet20} models trained on \dataset{CIFAR10} with $2\times$ width. The linear version of CKA is defined as
\begin{equation}
    \label{eq:cka}
    \operatorname{CKA}(X,Y) = \frac{\operatorname{HSIC}(X, Y)}{\sqrt{\operatorname{HSIC}(X, X) \operatorname{HSIC}(Y, Y)}},
\end{equation}
where $\operatorname{HSIC}(X, Y) = \frac{1}{(N-1)^2} \operatorname{tr}(\mathbf{X} \mathbf{H} \mathbf{X}^{\top} \mathbf{H})$, $\mathbf{H} = \mathbf{I} - \frac{1}{N} \mathbf{1}\mathbf{1}^\top$ is a centering matrix, and $\mathbf{1}$ is a vector of $N$ ones.
The denominator is introduced to scale CKA between zero and one, where a value of one indicates equivalent representations. CKA is invariant to orthogonal transformations and isotropic scaling. Being permutations orthogonal transformations, CKA stays exactly the same after mapping the models to the universe. On the contrary, the Euclidean distance of the representations of the models significantly decreases after mapping to the universe, as shown in \cref{fig:repr-eucl-similarities-orig,fig:repr-eucl-similarities-univ}.
\begin{figure}

    \begin{subfigure}{0.4\textwidth}
        \centering
        \includegraphics[width=\textwidth]{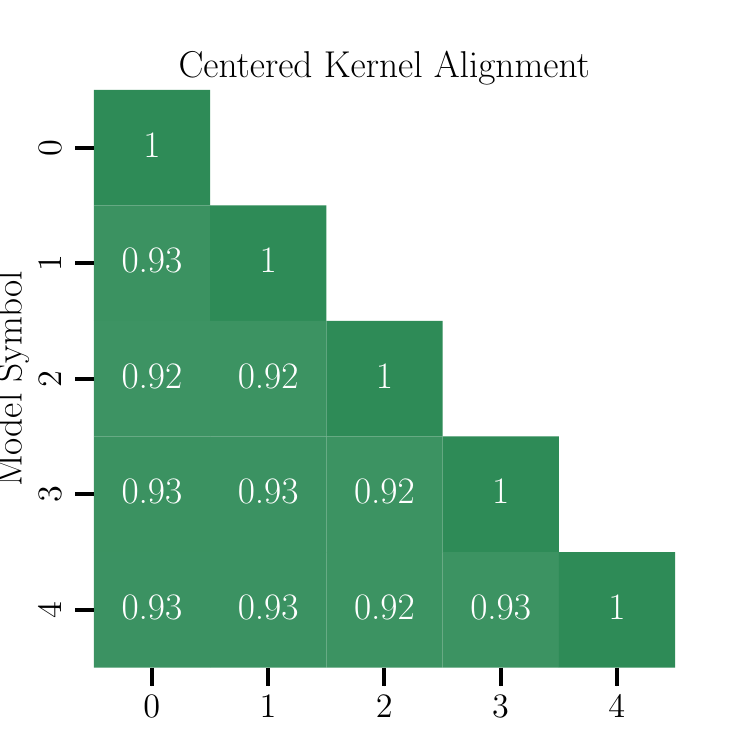}
        \caption{Before mapping to universe.}
        \label{fig:repr-cka-similarities-orig}
    \end{subfigure}
    \hfill
    \begin{subfigure}{0.4\textwidth}
        \centering
        \includegraphics[width=\textwidth]{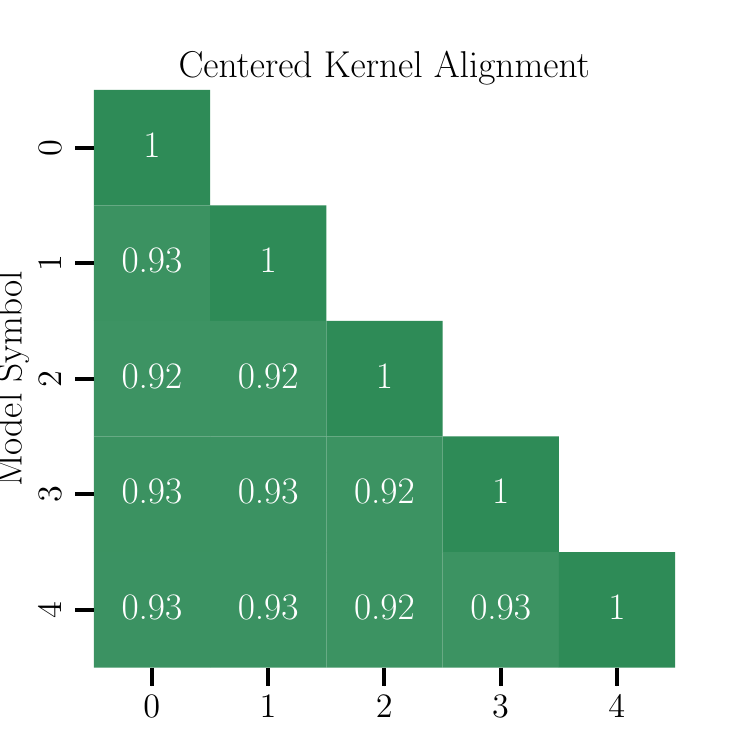}
        \caption{After mapping to universe.}
        \label{fig:repr-cka-similarities-univ}
    \end{subfigure}

    \begin{subfigure}{0.4\textwidth}
        \centering
        \includegraphics[width=\textwidth]{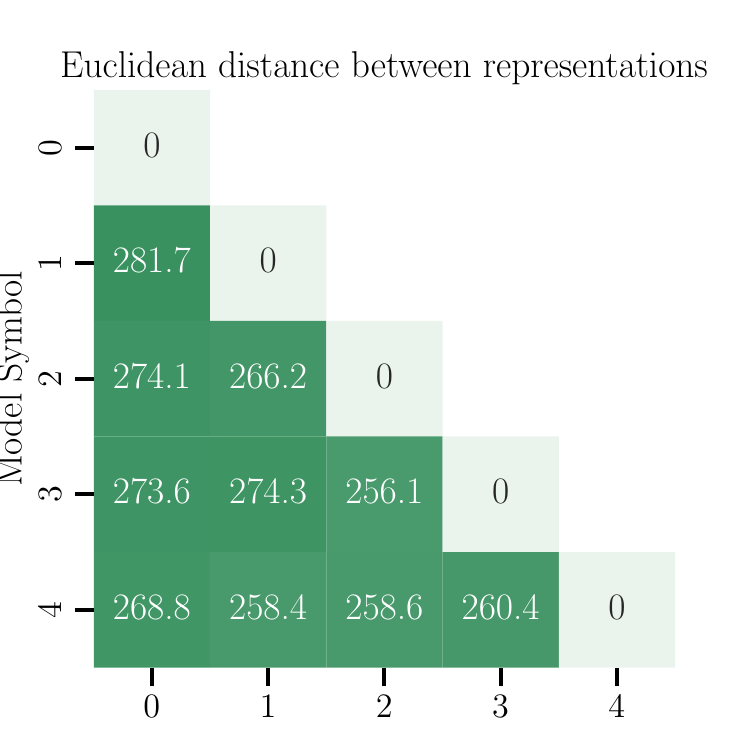}
        \caption{Before mapping to universe.}
        \label{fig:repr-eucl-similarities-orig}
    \end{subfigure}
    \hfill
    \begin{subfigure}{0.4\textwidth}
        \centering
        \includegraphics[width=\textwidth]{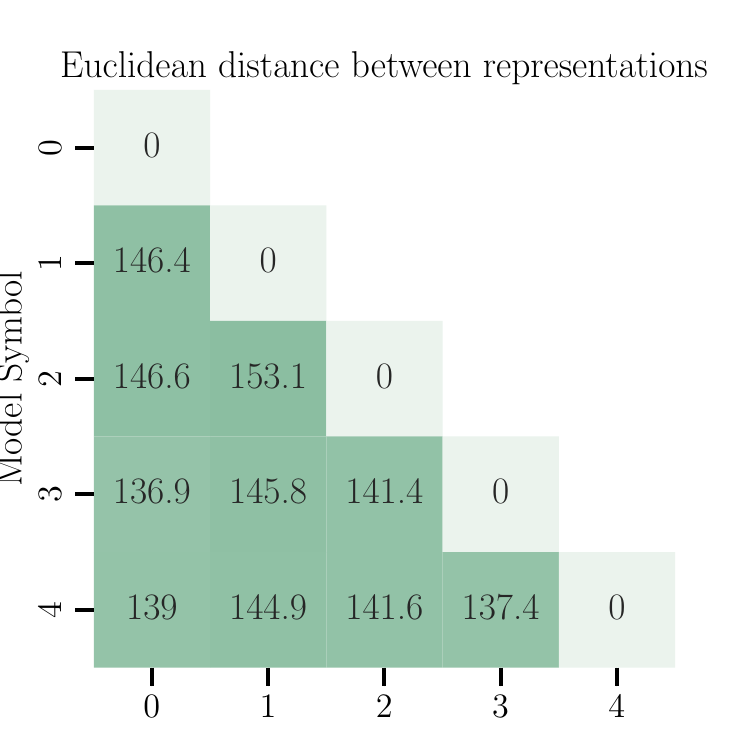}
        \caption{After mapping to universe.}
        \label{fig:repr-eucl-similarities-univ}
    \end{subfigure}

    \caption[CKA and Euclidean distances of \model{ResNet20} representations]{Centered Kernel Alignment and Euclidean distances of the representations of 5 \model{ResNet20} trained on \dataset{CIFAR10} with $2\times$ width.}
    \label{fig:repr-euclidean}
\end{figure}

\subsubsection{Weight-level similarity}

\label{subsec:exp-weight-similarity}

\begin{figure}
    \centering

    \begin{subfigure}{0.4\textwidth}
        \centering
        \includegraphics[width=\textwidth]{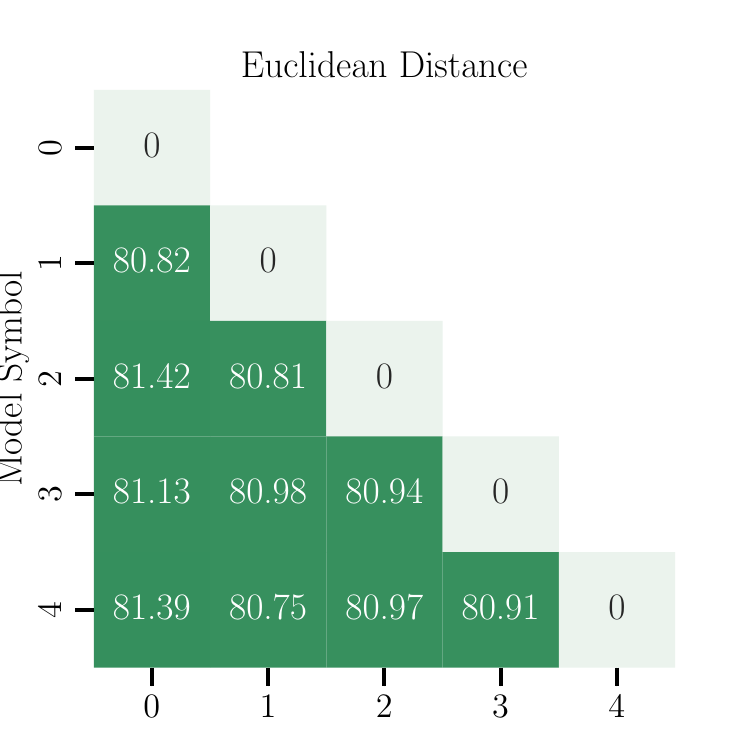}
        \caption{Before mapping to universe.}
        \label{fig:weights-eucl-distance-orig}
    \end{subfigure}
    \hfill
    \begin{subfigure}{0.4\textwidth}
        \centering
        \includegraphics[width=\textwidth]{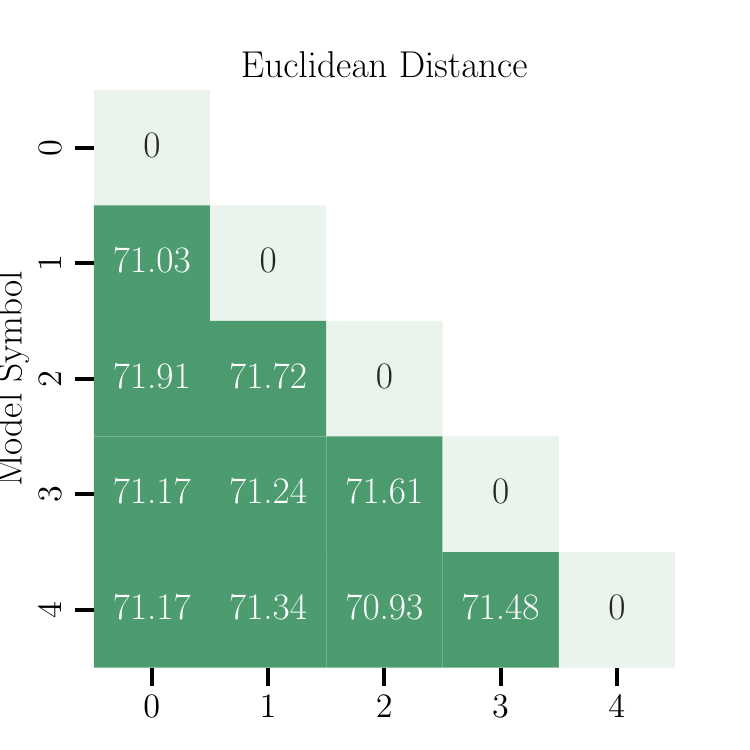}
        \caption{After mapping to universe.}
        \label{fig:weights-eucl-distance-univ}
    \end{subfigure}

    \caption[Euclidean distance of \model{ResNet20} weights]{Euclidean distance of the weights of 5 \model{ResNet20} trained on \dataset{CIFAR10} with $2\times$ width.}
    \label{fig:weights-eucl-similarities}
\end{figure}
We have seen in \cref{fig:weights-cos-similarities} that the cosine similarity of the weights is higher after mapping the weights to the universe. This suggests that the models are more similar in the universe, which is consistent with the fact that it constitutes a convenient space to merge them. We report here for completeness in \cref{fig:weights-eucl-similarities} the Euclidean distance of the weights of 5 \model{ResNet20} models trained on \dataset{CIFAR10} with $2\times$ width, showing the same trend as the cosine similarity.

\subsection{Merging different subsets}\label{app:subsets}
We merge subsets of $k < 5$ models from the set of $5$ models $a,b,c,d,e$ to gauge the effect of the match set composition over the accuracy of the merged model. As shown in \cref{fig:accuracy-model-subsets}, we run two different merging schemes: in the former (left column), we globally match all the $5$ models jointly and then consider subsets only at the aggregation step. In the second analysis (right column), we instead consider model subsets from the start, therefore running the whole matching procedure on the $k$ models before averaging them. This way, we aim to disentangle the error resulting from imperfect matching from the one naturally resulting from the aggregation. 
We highlight a few notable aspects: 
\begin{enumerate}
    \item While the accuracies are expectedly higher when matching a subset with permutations expressly optimized for that same subset (right column), this is not the case for $n=2$, in which the permutations resulting from matching the superset of $5$ models yield better results when merging pairs of them. This hints at the added constraint of cycle consistency over a wide number of models adding in some cases an advisable prior over the search space.
    \item The particular composition of the match set has a significant impact over the matching and subsequent merge operation, yielding differences of up to $\approx 20$ accuracy points for the downstream model.
    \item The standard deviations before the repair operation (red semi-transparent bars in the plots) are way lower when optimizing for the permutations over the superset of all $5$ models; this suggests that the matching difficulty is spread over all the maps jointly, eventually yielding more stable results. 
\end{enumerate}

\begin{figure}
    %
    \begin{subfigure}{0.48\textwidth}
        \centering
        \includegraphics[width=\textwidth]{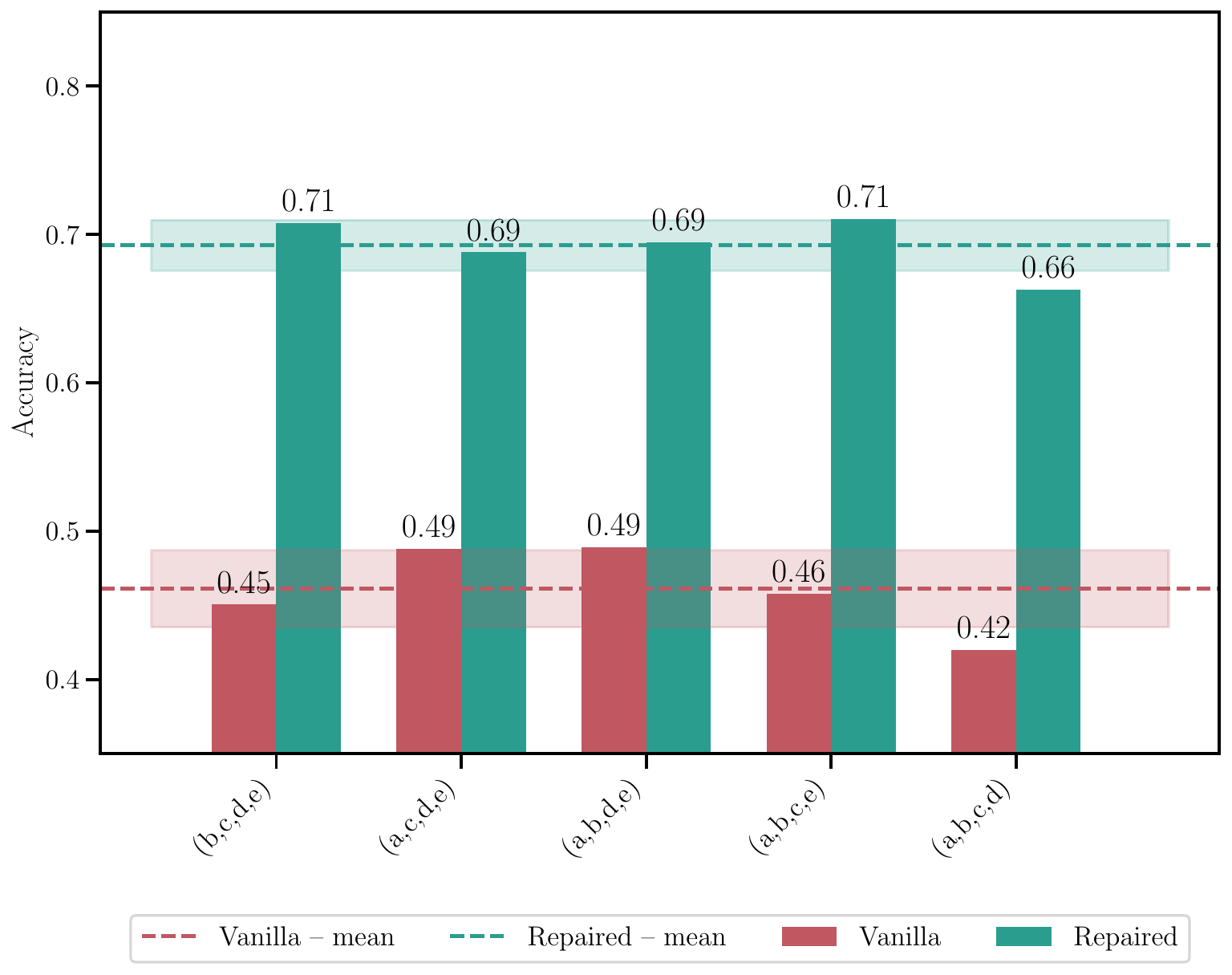}
        \caption{Subsets of $4$ out of $5$ jointly matched models.}
        \label{fig:accuracy-model-4-subsets-not-matched}
    \end{subfigure}
    \hfill
    \begin{subfigure}{0.48\textwidth}
        \centering
        \includegraphics[width=\textwidth]{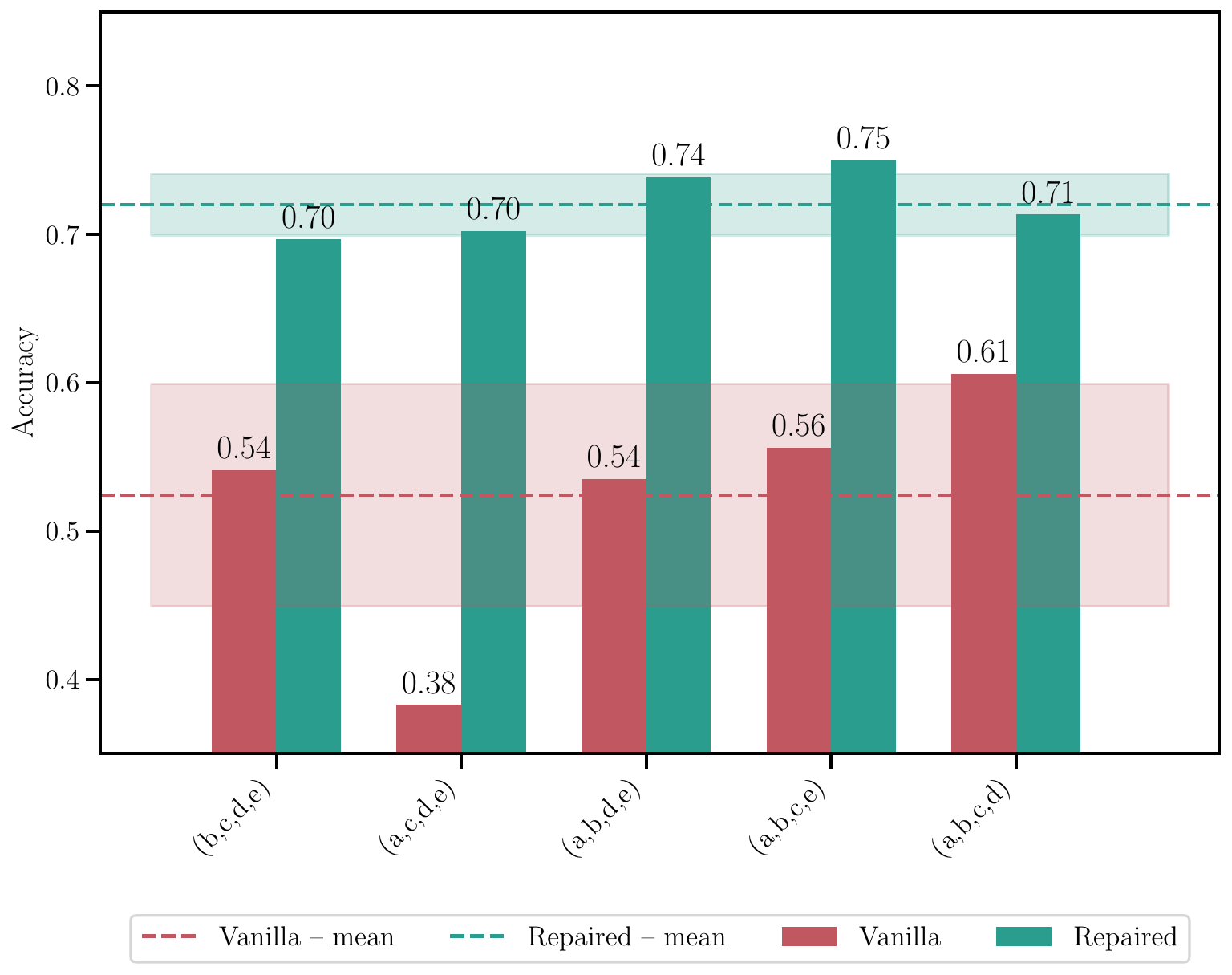}
        \caption{Subsets of $4$ matched models out of $5$ models.}
        \label{fig:accuracy-model-4-subsets-matched}
    \end{subfigure}

    \vspace{0.5cm}

    \begin{subfigure}{0.48\textwidth}
        \centering
        \includegraphics[width=\textwidth]{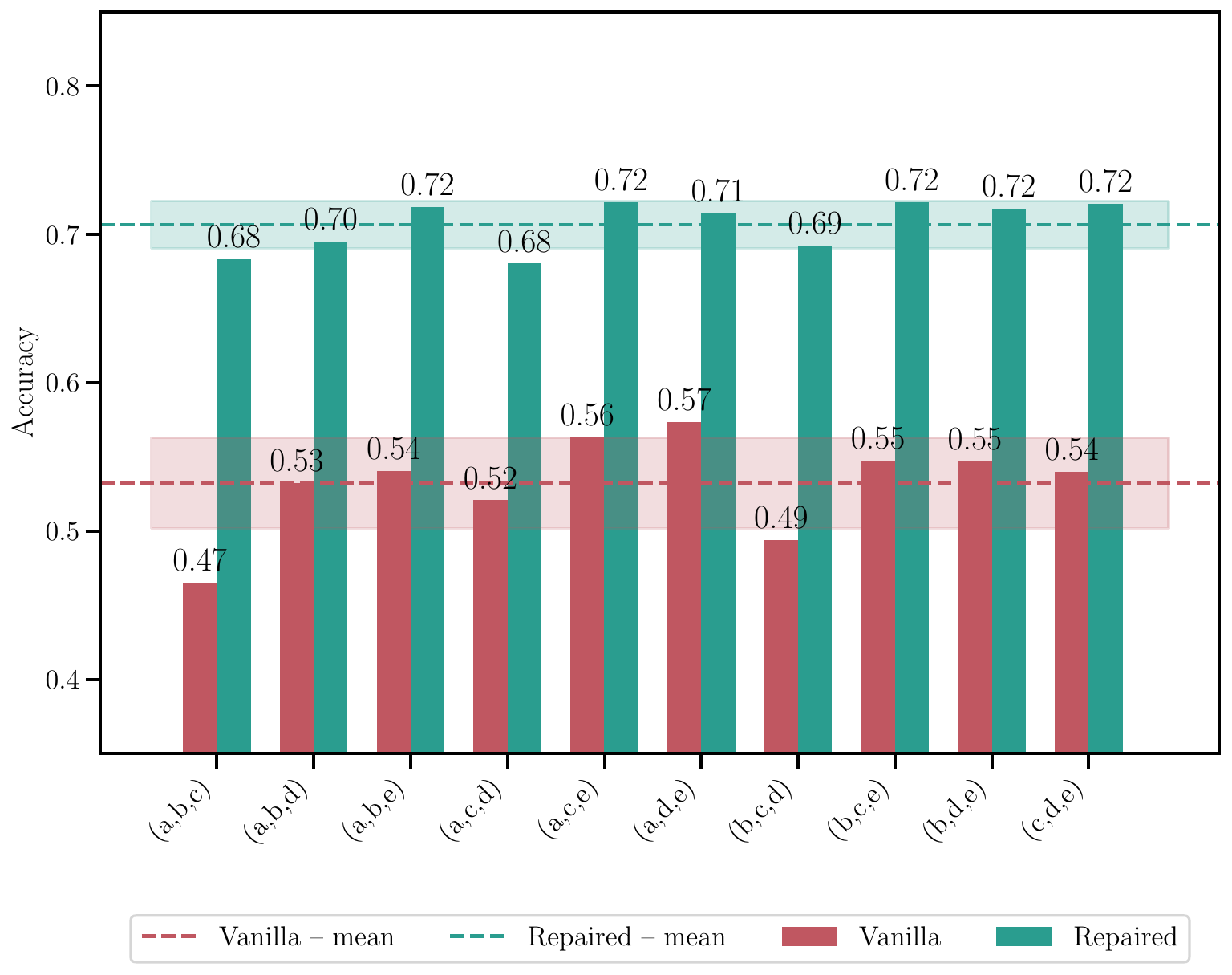}
        \caption{Subsets of $3$ out of $5$ jointly matched models.}
        \label{fig:accuracy-model-3-subsets-not-matched}
    \end{subfigure}
    \hfill
    \begin{subfigure}{0.48\textwidth}
        \centering
        \includegraphics[width=\textwidth]{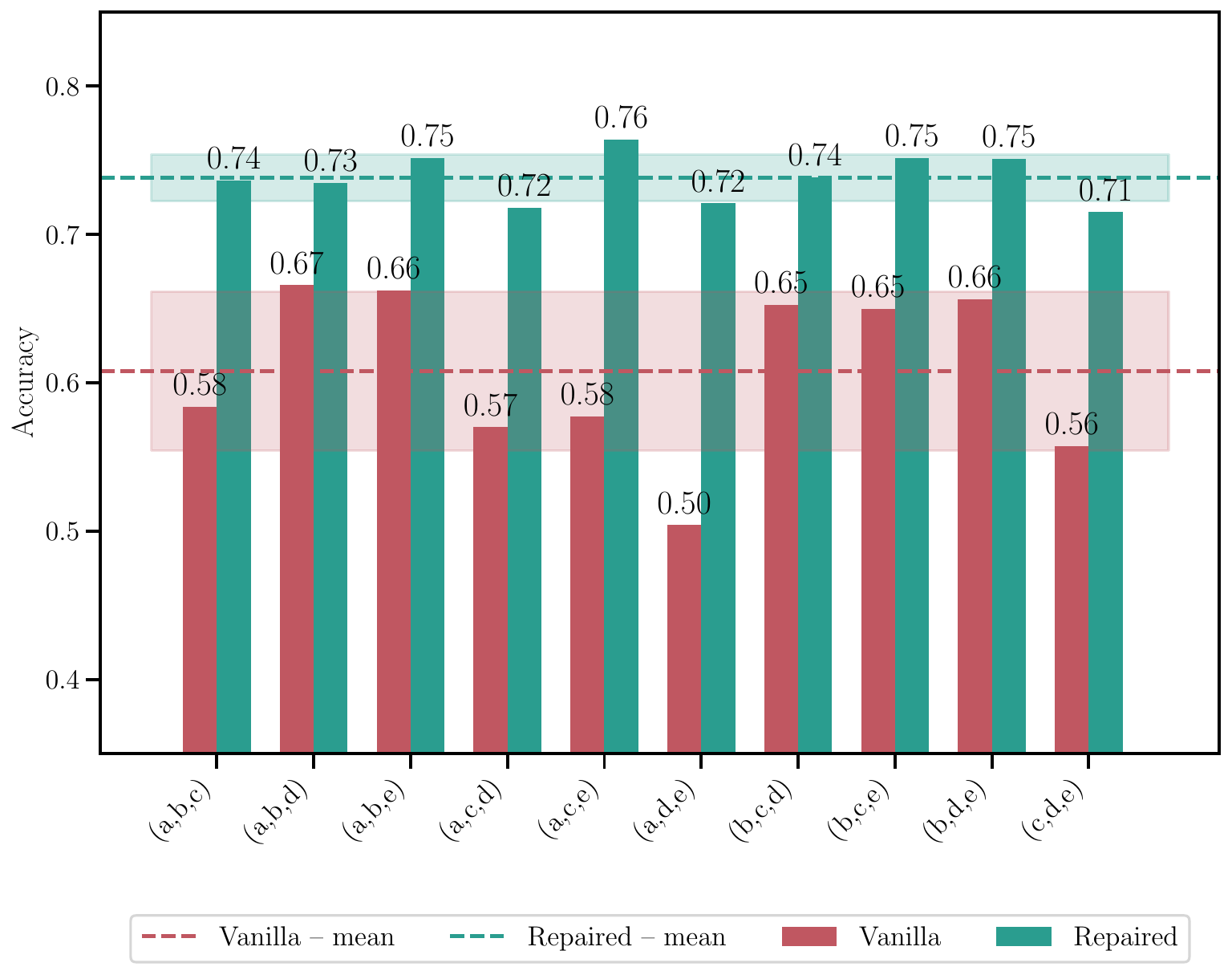}
        \caption{Subsets of $3$ matched models out of $5$ models.}
        \label{fig:accuracy-model-3-subsets-matched}
    \end{subfigure}

    \vspace{0.5cm}

    \begin{subfigure}{0.48\textwidth}
        \centering
        \includegraphics[width=\textwidth]{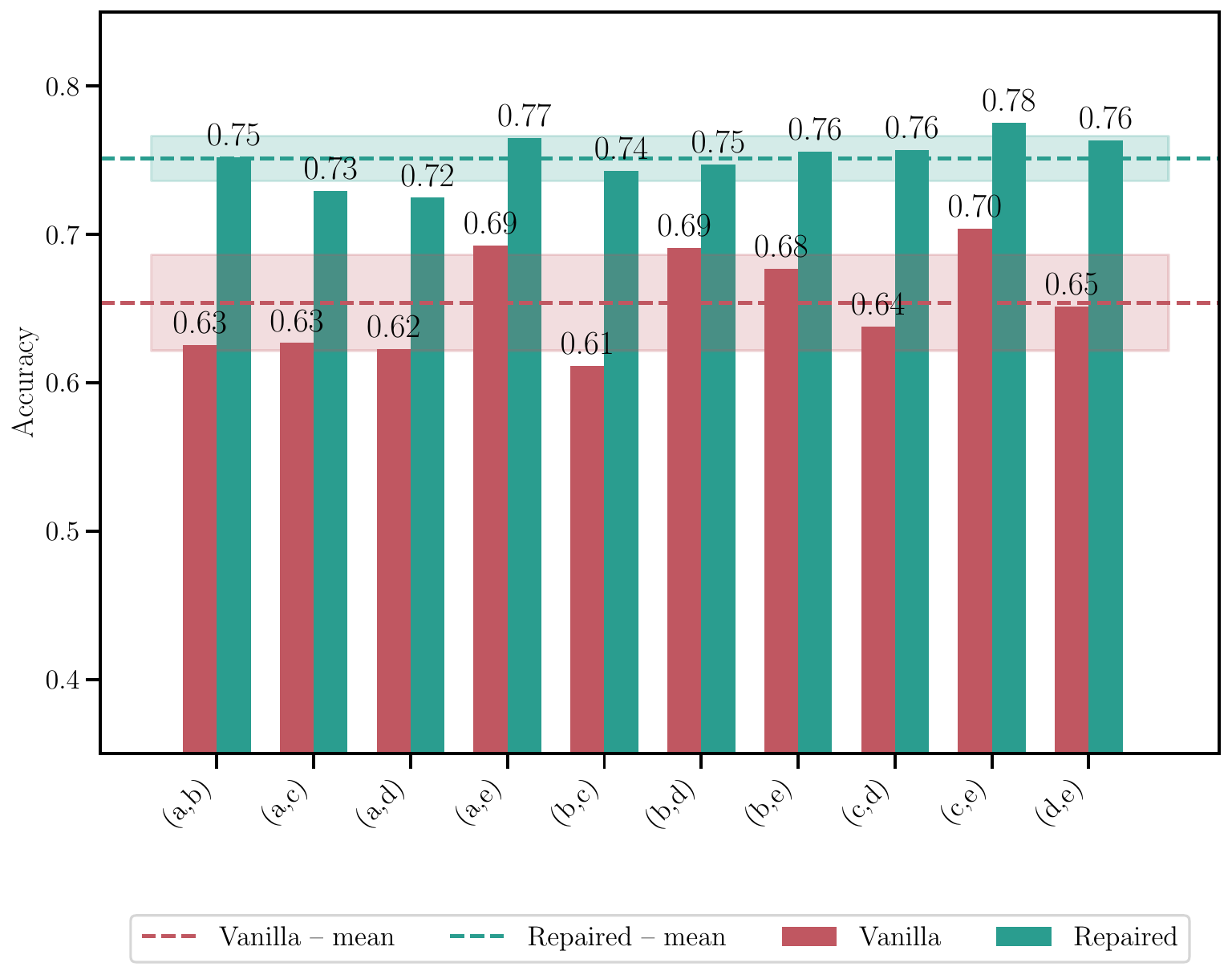}
        \caption{Subsets of $2$ out of $5$ jointly matched models.}
        \label{fig:accuracy-model-2-subsets-not-matched}
    \end{subfigure}
    \hfill
    \begin{subfigure}{0.48\textwidth}
        \centering
        \includegraphics[width=\textwidth]{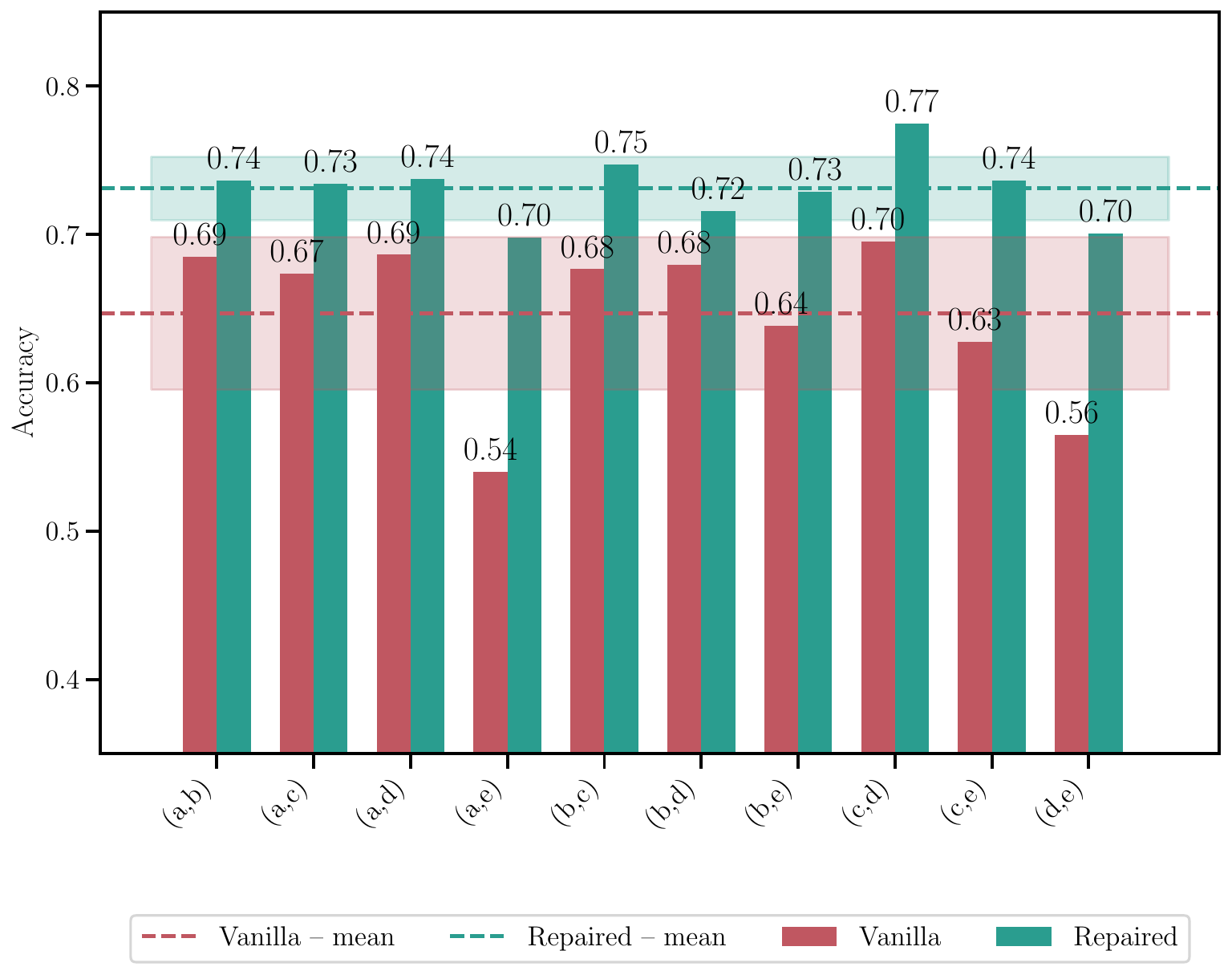}
        \caption{Subsets of $2$ matched models out of $5$ models.}
        \label{fig:accuracy-model-2-subsets-matched}
    \end{subfigure}
    \caption[Accuracy when merging different model subsets]{Accuracy of the resulting model when merging different model subsets.  \textbf{(left)} performance of models obtained from aggregating subsets of $k < 5$ models that were matched jointly. \textbf{(right)} analogous results for subsets of $k$ models that are instead matched independently, \emph{i.e.}, by only optimizing for the permutations that align those $k$ models and discarding the remaining ones. The semi-transparent bands represent the standard deviation of the accuracy.} \label{fig:accuracy-model-subsets}
\end{figure}

%% file: 99_appendix/A_cycle_consistent/D_Discussion/content.tex

We discuss in this section the limitations of our work, as well as potential future societal impact. 

\subsection{\texorpdfstring{On the cycle-consistency of $C^2M^3$}{On the cycle-consistency of C2M3}}
Our method is natively cycle-consistent due to the mathematical formulation of the optimization problem. If we were to not desire cycle consistency, the matching method would fall back to the $n=2$ Frank-Wolfe (FW) case presented in \cref{subsec:pairwise-matching}. One would then have to transform the pairwise matching problem to a $n$-way matching problem, \emph{e.g.} by using the $n=2$ FW procedure as matching step in the \texttt{MergeMany}~\citep{git-rebasin} algorithm. Results for the $n=2$ FW matching are reported in \cref{tab:pairwise_barriers_table}.

\subsection{Limitations}
From what we have observed in our experiments, permutations satisfying linear mode connectivity of the models are hard to find for most architectures and datasets. In fact, given that there is no practical way to prove or disprove the conjecture for which most models end up in the same basin modulo permutations of the neurons, we cannot be sure that a certain set of models even allows finding such permutations, let alone that the permutations found are the optimal ones. We therefore encourage the community not to rely on the existence of such permutations in general. However, we have also shown that we can always find permutations that improve the resulting aggregated model, which is a promising practical result for model merging. As for all the existing works concerning linear mode connectivity and model merging, the resulting models that we obtain are sensible to a wide variety of factors, from training hyperparameters to the optimization algorithm used. Several works have already observed the phenomenon in practice: among these, \citet{git-rebasin} mention among the known failure modes of their approaches models trained with SGD and too low learning rate, or ADAM coupled with too high learning rate. \citet{repair} show that the chosen normalization layer incredibly affects the accuracy of the resulting merged model, while \citet{qu2024rethinkingmodelrebasinlinear} observe learning rate, weight decay, and initialization method to play a strong role as well.
Being a mostly empirical field, most of the technical choices that we make in our work mirror the ones made in previous works and are not based on a solid theoretical foundation. We therefore release all our code and encourage the community to investigate further on what training and optimization hyperparameters affect linear mode connectivity and model merging.

\subsection{Societal impact and broader vision}
The work presented in this chapter serves as an additional tool for the community to improve the efficiency of deep learning models. By merging models, we can reduce the computational cost of training and inference, as well as the memory footprint of the models. In fact, by aggregating the information of a set of models into a single one with the same architecture, practitioners can benefit from the effects of ensembling without incurring its computational cost. Moreover, merging is in many cases a practical necessity to guarantee confidentiality and privacy of user data, as it allows to train models on different subsets of the data, \emph{e.g.} originating from different clients, and then merge them to obtain a single model integrating all the information. This is particularly important in the context of federated learning, where the data is distributed among different clients and cannot be shared. We believe that the work presented in this chapter can be a stepping stone towards more efficient and privacy-preserving deep learning models, and we encourage the community to further investigate the potential of model merging in these contexts.

%% file: 99_appendix/C_task_vectors_gradients/content.tex
\section{Proofs}
\input{B_Proofs/proof_new_thm.tex}
\input{B_Proofs/proof_bound_C.tex}
\input{B_Proofs/content.tex}

%% file: 99_appendix/C_task_vectors_gradients/B_Proofs/proof_new_thm.tex
\subsection{Proofs of Proposition~\ref{thm:multitask-vector-gradient} and Theorem~\ref{main_thm}}\label{sec:proofs}
In this section, we provide proofs for Proposition~\ref{thm:multitask-vector-gradient} and Theorem~\ref{main_thm}.
For clarity, we restate both the proposition and theorem.
In the Appendix, when summing the individual task losses over all tasks, we sometimes index them by their label $t \in T$ and other times list them as $i = 1, \dots, T$. Both notations denote the exact same aggregate over every task.

\begin{proposition*}
  Let $\left\{ \theta^{(k)}_{t} \right\}_{t=1}^{|T|}$ be a set of models obtained by fine-tuning the base model $\theta_{\text{pre}}$ for $k$ epochs on tasks $T$ using GD with a learning rate $\eta$, where fine-tuning task $t\in T$ minimizes the loss $\overline{L}_t(\theta) = \frac{1}{n_t} \sum_{i =1}^{n_t}\ell(x_i, y_i, \theta)$.
  Additionally, let $\left\{ \tau_t^{(k)} \right\}_{t=1}^{|T|}$ denote the corresponding set of task vectors, with each $\tau^{(k)}_t = \theta^{(k)}_{t} - \theta_{\text{pre}}$. Let $\tau^{(k)}_{\text{MT}}$ be the multi-task vector $\tau^{(k)}_{\text{MT}} = \sum_{t \in T} \tau_t^{(k)}$.
  Finally, let $\theta_{\text{MT}}^{(k)} $ represent the model obtained by minimizing the combined loss \(  \sum_{i=1}^{|T|} \overline{L}_i \) for $k$ epochs using GD with a learning rate of $\alpha \eta$. It holds that
    \begin{align}
        &\tau^{(1)}_{\text{MT}} = -\eta \nabla \sum_{t \in T} \overline{L}_t(\theta_{\text{pre}}) \\ 
    &\tau^{(k)}_{\text{MT}} =  - \eta \sum_{t \in T}  \sum_{j=0}^{k-1}   \nabla \overline{L}_t(\theta_{\text{MT}}^{(j)})  + \eta^2 C(\{\theta_{\text{MT}}^{(j)}\}_{j=1}^{k-2}) + O(\eta^3)
    \end{align}
with
\begin{equation}
C (\{\theta_{\text{MT}}^{(j)}\}_{j=1}^h) =  \sum_{t \in T}  \sum_{\ell=0}^{h} \nabla^2 \overline{L}_t(\theta^{(\ell)}_{\text{MT}}) \sum_{m=0}^{\ell}  \left[ \alpha  \sum_{t' \neq t , t' \in T } \nabla \overline{L}_t'(\theta^{(m)}_{\text{MT}}) + (\alpha -1 )  \nabla \overline{L}_t(\theta^{(m)}_{\text{MT}}) \right] 
\end{equation}
\end{proposition*}

\begin{theorem*}
    Let $ \theta_{\text{TA}}^{(k)} = \theta_{\text{pre}} + {\alpha} \sum_{t=1}^T \tau_t^{(k)}$ be the model obtained using vanilla task arithmetic.
    Using the same notation of Theorem~\ref{thm:multitask-vector-gradient}, it holds that  
    \begin{align}
    &\theta_{\text{TA}}^{(1)} =  \theta_{\text{MT}}^{(1)} \\
&\theta_{\text{TA}}^{(k)} =  \theta_{\text{MT}}^{(k)}  + \eta^2 C(\{\theta_{\text{MT}}^{(j)}\}_{j=1}^{k-2}) + O(\eta^3) \; \; \; \text{for} \ k>1
    \end{align}
\end{theorem*}

We recall that $\theta^{(k)}_{i}$ is the model obtained by fine-tuning on task $i$ for $k$ epochs, and that both the fine-tuning on different tasks and the training on the average loss start from a pre-trained model $\theta_{\text{pre}}$.

To prove the statement of the theorem and of the corollary, we need an intermediate result.
We introduce the following notation:
\begin{align}
    &r_i (\theta, \alpha) =   \alpha  \sum_{j \neq i } \nabla \overline{L}_j(\theta) + (\alpha -1 )  \nabla \overline{L}_i(\theta)  = \alpha \sum_{j=1}^{|T|} \nabla \overline{L}_j(\theta) -  \nabla \overline{L}_i(\theta)\\
    &p_i^k (\theta_{\text{pre}}, \theta^{(1)}_{\text{MT}}, \dots , \theta^{(k)}_{\text{MT}}) =   \sum_{j=0}^k  r_i (\theta^{(j)}_{\text{MT}})\\
    &s^{k}_t(\theta_{\text{pre}}, \dots, \theta^{(k)}_{\text{MT}}):=\sum_{j=0}^{k}
       \nabla^{2}\overline{L}_i\bigl(\theta_{\text{MT}}^{(j+1)}\bigr)
       \bigl[p_i^j (\theta_{\text{pre}}, \theta^{(1)}_{\text{MT}}, \dots , \theta^{(j)}_{\text{MT}})\bigr].
\end{align}  

Since $\alpha$ in this context is fixed, we will not emphasize the dependence on $\alpha$ and we will use the notation $r_i (\theta) = r_i (\theta, \alpha)$. 
We need the following preliminary Lemma. 
\begin{lemma*}
Using the notation introduced in Proposition~\ref{thm:multitask-vector-gradient}, it holds that
  \begin{equation}
  \theta^{(1)}_i= \theta^{(1)}_{\text{MT}} + \eta p_i^0 (\theta_{\text{pre}})
\end{equation}  
and for $m\geq 2$
  \begin{equation}
    \theta^{(m+1)}_{i}=  \theta^{(m+1)}_{\text{MT}} +  \eta p_i^{m} (\theta_{\text{pre}}, \dots,\theta^{(m)}_{\text{MT}}) - \eta^2 s^{m-1}_i(\theta_{\text{pre}}, \dots, \theta^{(m-1)}_{\text{MT}}) + O(\eta^3) 
\end{equation}  
\end{lemma*}

\begin{proof}
\small
We first verify the statement directly for $m=1$, then prove the general case $m\geq 2$ by induction. The induction base case is $m=2$. The induction step then shows that if the statement holds for some $m$ it must also hold for $m+1$.

\paragraph{\texorpdfstring{$m=1$}{m=1}. First  epoch}
For each task $i= 1, \dots,|T| $
\[ \theta^{(1)}_i =\theta_{\text{pre}}- \eta \nabla \overline{L}_i(\theta_{\text{pre}}) \text{ while } \theta^{(1)}_{\text{MT}} = \theta_{\text{pre}}- \alpha \eta \sum_{i \in T} \nabla \overline{L}_i(\theta_{\text{pre}}).\]
Consequently, it holds that  
\begin{align*} \theta^1_{i} &=  \theta^{(1)}_{\text{MT}} + \eta \left[ \alpha  \sum_{j \neq i } \nabla \overline{L}_j(\theta_{\text{pre}}) + (\alpha -1 )  \nabla \overline{L}_i(\theta_{\text{pre}})  \right] \\
&=  \theta^{(1)}_{\text{MT}} +    \eta r_i(\theta_{\text{pre}}) = \theta^{(1)}_{\text{MT}} + \eta p_i^0 (\theta_{\text{pre}}) .\end{align*}

\paragraph{\texorpdfstring{$m=2$}{m=2}. Second  epoch }

\begin{align*}
    \theta^{(2)}_i   &= \theta^{(1)}_{i} - \eta \nabla \overline{L}_i(\theta^{(1)}_{i}) \\
   & =   \theta^{(1)}_{\text{MT}} + 
    {\eta} r_i(\theta_{\text{pre}}) -  \eta \nabla \overline{L}_i \left( \theta^{(1)}_{\text{MT}} + 
    {\eta} r_i(\theta_{\text{pre}}) \right)\\
    &\overset{Taylor}{\approx}  \theta^{(1)}_{\text{MT}} +  {\eta} r_i(\theta_{\text{pre}})  -  \eta \nabla \overline{L}_i ( \theta^{(1)}_{\text{MT}} ) - \frac{ \eta^2}{2} \nabla^2 \overline{L}_i(\theta_{\text{MT}}^{(1)} ) r_i(\theta_{\text{pre}}) + O(\eta^3)\\
 &= \theta^{(1)}_{\text{MT}} -  \eta \nabla \overline{L}_i ( \theta^{(1)}_{\text{MT}} )  + 
\eta r_i(\theta_{\text{pre}}) - \frac{  \eta^2}{2} \nabla^2 \overline{L}_i(\theta_{\text{MT}}^{(1)} ) r_i(\theta_{\text{pre}}) + O(\eta^3)
\end{align*}
Adding and subtracting $\eta \alpha \sum_{t \in T} \nabla \overline{L}_i ( \theta^{(1)}_{\text{MT}} )$
\begin{align*}
\theta^{(2)}_i  &= \underbrace{\theta_{\text{MT}}^{(1)}
     -\eta\alpha\!\sum_{t}\!\nabla\overline{L}_t\bigl(\theta_{\text{MT}}^{(1)}\bigr)}_
       {=\;\theta_{\text{MT}}^{(2)}}  + \underbrace{\eta \alpha \sum_{t \in T} \nabla \overline{L}_i ( \theta^{(1)}_{\text{MT}} ) -  \eta \nabla \overline{L}_i ( \theta^{(1)}_{\text{MT}} ) }_{\eta r_i(\theta^{(1)}_{\text{MT}})}
+ \eta r_i(\theta_{\text{pre}} ) \\
&- \frac{  \eta^2}{2} \nabla^2 \overline{L}_i(\theta_{\text{MT}}^{(1)} ) r_i(\theta_{\text{pre}}) + O(\eta^3)
\\
& = \theta^{(2)}_{\text{MT}} + \eta r_i(\theta^{(1)}_{\text{MT}} ) + \eta r_i(\theta_{\text{pre}} ) - \frac{  \eta^2}{2} \nabla^2 \overline{L}_i(\theta_{\text{MT}}^{(1)} ) r_i(\theta_{\text{pre}}) + O(\eta^3)\\
&= \theta^{2}_{\text{MT}} +  \eta p_i^{1} (\theta_{\text{pre}}, \dots,\theta^{(1)}_{\text{MT}}) - \eta^2 s^{0}_i(\theta_{\text{pre}}) + O(\eta^3) 
\end{align*}

\paragraph{Inductive step}
Let us assume that 
\begin{align*}
     \theta_i^{(m)} &= \theta^{(m)}_{\text{MT}} + \eta p_i^{m-1} (\theta_{\text{pre}}, \dots,\theta^{(m-1)}_{\text{MT}})  - \eta^2 s^{m-2}_i(\theta_{\text{pre}}, \dots, \theta^{(m-2)}_{\text{MT}}) + O(\eta^3) 
\end{align*}
We can derive that 
\begin{align*}
     \theta_i^{(m+1)} &= \theta^{(m)}_{i} - \eta \nabla \overline{L}_i(\theta^{(m)}_{i}) \\
& = \theta^{(m)}_{\text{MT}} + \eta p_i^{m-1} (\theta_{\text{pre}}, \dots,\theta^{(m-1)}_{\text{MT}})  - \eta^2 s^{m-2}_i(\theta_{\text{pre}}, \dots, \theta^{(m-2)}_{\text{MT}}) - \eta \nabla \overline{L}_i(\theta^{m}_{i}) + O(\eta^3)   \\
& = \theta^{(m)}_{\text{MT}} + \eta p_i^{m-1} (\theta_{\text{pre}}, \dots,\theta^{(m-1)}_{\text{MT}})  - \eta^2 s^{m-2}_i(\theta_{\text{pre}}, \dots, \theta^{(m-2)}_{\text{MT}}) \\
&- \eta \nabla \overline{L}_i\left( \theta^{(m)}_{\text{MT}} + \eta p_i^{m-1} (\theta_{\text{pre}}, \dots,\theta^{(m-1)}_{\text{MT}})  - \eta^2 s^{m-2}_i(\theta_{\text{pre}}, \dots, \theta^{(m-2)}_{\text{MT}}) \right) + O(\eta^3) \\
&= \theta^{(m)}_{\text{MT}} + \eta p_i^{m-1} (\theta_{\text{pre}}, \dots,\theta^{(m-1)}_{\text{MT}})  - \eta^2 s^{m-2}_i(\theta_{\text{pre}}, \dots, \theta^{(m-2)}_{\text{MT}})  \\
&- \eta \nabla \overline{L}_i( \theta^{(m)}_{\text{MT}} ) - \frac{\eta}{2} \nabla^2 \overline{L}_i ( \theta^{(m)}_{\text{MT}} ) \left( \eta p_i^{m-1} (\theta_{\text{pre}}, \dots,\theta^{(m-1)}_{\text{MT}})  - \eta^2 s^{m-2}_i(\theta_{\text{pre}}, \dots, \theta^{(m-2)}_{\text{MT}}) \right) + O(\eta^3)\\
&= \theta^{(m)}_{\text{MT}} + \eta p_i^{m-1} (\theta_{\text{pre}}, \dots,\theta^{(m-1)}_{\text{MT}})  - \eta^2 s^{m-2}_i(\theta_{\text{pre}}, \dots, \theta^{(m-2)}_{\text{MT}})  \\
&- \eta \nabla \overline{L}_i( \theta^{(m)}_{\text{MT}} ) - \eta^2 \nabla^2 \overline{L}_i( \theta^{(m)}_{\text{MT}} )  p_i^{m-1} (\theta_{\text{pre}}, \dots,\theta^{(m-1)}_{\text{MT}})  + O(\eta^3)\\
&= \theta^{(m+1)}_{\text{MT}} +  \eta p_i^{m} (\theta_{\text{pre}}, \dots,\theta^{(m)}_{\text{MT}}) - \eta^2 s^{m-1}_i(\theta_{\text{pre}}, \dots, \theta^{(m-1)}_{\text{MT}}) + O(\eta^3)
\end{align*}

\end{proof}

\begin{proof}[Proof Proposition and Theorem]
\small
For the first epoch
\[ \theta^{(1)}_{\text{TA}} = \theta_{\text{pre}} + \alpha \sum_{i \in T} \tau^{(1)}_i =  \theta_{\text{pre}}  - \eta \alpha\sum_{i \in T}   \nabla \overline{L}_i(\theta_{\text{pre}})  \]
while, choosing $ \alpha \eta$ as learning rate for the loss $  \sum_{i \in T} \overline{L}_i $  : 
\[ \theta^{(1)}_{\text{MT}} = \theta_{\text{pre}}- \alpha \eta \sum_{i \in T} \nabla \overline{L}_i(\theta_{\text{pre}}).\]

So \( \theta^{(1)}_{\text{MT}} = \theta^{(1)}_{\text{TA}} \).

For $k\geq 2$, notice that 
\begin{equation}
    \theta^{(k)}_{\text{MT}} = \theta_{\text{pre}}  - \alpha \eta \sum_{j=0}^{k-1} \nabla \sum_{t \in T} \overline{L}_t(\theta^{(j)}_{\text{MT}}).
\end{equation}
Now, using Lemma~\ref{lemma:induction_update}, we get:
\begin{align*}
- \alpha \eta& \sum_{j=0}^{k-1} \nabla \sum_{t \in T} \overline{L}_t(\theta^{(j)}_{\text{MT}}) + \eta p_i^{k-1} ( \eta^0, \dots, \eta^{k-1}_{\text{MT}}) \\
&= 
  - \alpha \eta \sum_{j=0}^{k-1} \nabla \sum_{t \in T} \overline{L}_t(\theta^{(j)}_{\text{MT}}) +   \sum_{j=0}^{k-1}  r_i (\theta^{k}_{\text{MT}})  \\
& = - \alpha \eta \sum_{j=0}^{k-1} \nabla \sum_{t \in T} \overline{L}_i(\theta^{(j)}_{\text{MT}})  +  \sum_{j=0}^{k-1} \alpha \sum_{j \in T} \nabla \overline{L}_j(\theta^{(j)}_{\text{MT}}) -  \nabla \overline{L}_i(\theta^{(j)}_{\text{MT}})
\\
&=  - \eta \sum_{j=0}^{k-1} \nabla \overline{L}_i(\theta^{(j)}_{\text{MT}})\,.
\end{align*}  

Namely:

\begin{align*}
    \theta_i^{(m+1)} &=\theta^{(m+1)}_{\text{MT}} +  \eta p_i^{m} (\theta_{\text{pre}}, \dots,\theta^{(m)}_{\text{MT}}) - \eta^2 s^{m-1}_i(\theta_{\text{pre}}, \dots, \theta^{(m-1)}_{\text{MT}}) + O(\eta^3)\\
    &= \theta_{\text{pre}}  - \alpha \eta \sum_{j=0}^{m} \nabla \sum_{t \in T} \overline{L}_i(\theta^{(j)}_{\text{MT}})+  \eta p_i^{m} (\theta_{\text{pre}}, \dots,\theta^{(m)}_{\text{MT}}) - \eta^2 s^{m-1}_i(\theta_{\text{pre}}, \dots, \theta^{(m-1)}_{\text{MT}}) + O(\eta^3)\\
    &=  \theta_{\text{pre}} - \eta \sum_{j=0}^{m} \nabla \overline{L}_i(\theta^{(j)}_{\text{MT}})- \eta^2 s^{m-1}_i(\theta_{\text{pre}}, \dots, \theta^{(m-1)}_{\text{MT}}) + O(\eta^3)
\end{align*}

we can rewrite the tasks vectors as 
\begin{align}
    \tau_i^{(k)} &= \theta^{(k)}_{i} - \theta_{\text{pre}}\\
    &= - \eta \sum_{j=0}^{k-1} \nabla \overline{L}_i(\theta^{(j)}_{\text{MT}})- \eta^2 s^{k-2}_i(\theta_{\text{pre}}, \dots, \theta^{(k-2)}_{\text{MT}}) + O(\eta^3)
\end{align}

Consequently the model obtained with TA is 
\begin{align*}
 \theta^{(k)}_{\text{TA}} &= \theta_{\text{pre}} + \alpha \sum_{i \in T} \tau^{(k)}_i \\
 &=  \theta_{\text{pre}}  - \eta \alpha \sum_{j=0}^{k-1} \sum_{i \in T}   \nabla \overline{L}_i(\theta^{(j)}_{\text{MT}}) -  \alpha \sum_{i \in T}\eta^2 s^{k-2}_i(\theta_{\text{pre}}, \dots, \theta^{(k-2)}_{\text{MT}}) +O(\eta^3) \\
&=\theta^{(k)}_{\text{MT}}  -  \alpha \sum_{i \in T}\eta^2 s^{k-2}_i(\theta_{\text{pre}}, \dots, \theta^{(k-2)}_{\text{MT}}) +O(\eta^3) \,.
\end{align*}

\end{proof}

%% file: 99_appendix/C_task_vectors_gradients/B_Proofs/proof_bound_C.tex
\subsection{Proofs of Theorem~\ref{thm:coefficient-bound}}\label{sec:proofs2}
In this section, we provide the proof for Theorem~\ref{thm:coefficient-bound}, which is restated for clarity.
\begin{theorem*}[Uniform bound on the coefficient \(C\)]
 Let $C (\{\theta_{\text{MT}}^{(j)}\}_{j=1}^h)$ be the error term obtained in Theorem~\ref{main_thm}, with the condition of the theorem above. Additionally, we add that the tasks are all classification tasks for which we used cross entropy loss. We also assume that the network is a depth-\(L\) feed-forward network with weight matrices   \(W^{(1)},\dots ,W^{(L)}\), with the property that there exist constants \(s_\ell>0\) such that 
      \(\lVert W^{(\ell)}\rVert_2\le s_\ell\).
      We abbreviate \(\displaystyle\Pi:=\prod_{\ell=1}^{L}s_\ell.
      \) We also assume that every input vector satisfies \(\lVert x\rVert_2\le M_x\). We assume that the activation functions have bounded first and second derivatives: there are \(\beta_\phi,\gamma_\phi>0\) such that  
      \( \sup_{z}|\phi'(z)|\le\beta_\phi
           \quad\text{and}\quad
    \sup_{z}|\phi''(z)|\le\gamma_\phi.
      \)
      For the ReLU activation we have \(\gamma_\phi=0\).
\begin{enumerate}[label=(\roman*)]
\item{General activations.}
      \[
      \bigl\lVert C(\{\theta^{(j)}_{\mathrm{MT}}\}_{j=1}^{h})\bigr\rVert_2
             \;\le\;
             T\,\frac{(h+1)(h+2)}{2}\,
             \lvert\alpha T-1\rvert\,
             H_{\max}\,G_{\max},
      \]
      with
      \[
      G_{\max}\;\le\;
         \sqrt{2}\,M_x\,\Pi\,\beta_\phi^{\,L-1},
      \qquad
      H_{\max}\;\le\;
         2\,\gamma_\phi\,M_x^{2}\,\Pi^{2}\,
         \beta_\phi^{\,2L-2}.
      \]
\item {ReLU activations (\(\gamma_\phi=0\)).}
      \[
      \bigl\lVert C(\{\theta^{(j)}_{\mathrm{MT}}\}_{j=1}^{h})\bigr\rVert_2
         \;\le\;
         T\,\frac{(h+1)(h+2)}{2}\,
         \lvert\alpha T-1\rvert\,
         \frac12\sqrt{2}\,
         M_x^{3}\,\Pi^{3}\,\beta_\phi^{\,3L-3}.
      \]
\end{enumerate}
\end{theorem*}

\begin{proof}
  We want to bound the 2-norm of the coefficient $C$ in Equation~\ref{eq:term_c}. Let's start by noticing that
\[\| C(\{\theta_{\text{MT}}^{(j)} \}_{j=1}^h) \|_2 \leq \|\sum_{t \in T}  \sum_{\ell=0}^{h} \nabla^2 \overline{L}_t(\theta^{(\ell)}_{\text{MT}}) \|_2 \cdot \big \|  \sum_{m=0}^{\ell}  \left[   \sum_{t' \neq t , t' \in T } \alpha \nabla \overline{L}_{t'}(\theta^{(m)}_{\text{MT}}) + (\alpha -1 )  \nabla \overline{L}_t(\theta^{(m)}_{\text{MT}}) \right] \big \|_2 \]
\[
\leq  \big (\sum_{t \in T}  \sum_{\ell=0}^{h} \| \nabla^2 \overline{L}_t(\theta^{(\ell)}_{\text{MT}})\|_2 \big ) \cdot \big (  \sum_{m=0}^{\ell}   \sum_{t' \neq t , t' \in T } \big \| \alpha \nabla \overline{L}_{t'}(\theta^{(m)}_{\text{MT}}) + (\alpha -1 )  \nabla \overline{L}_t(\theta^{(m)}_{\text{MT}})  \big \|_2  \big )\]

Let us denote by \( G_{max}\) the max of the gradient and by $H_{max}$ the max of the Hessian, the term $C$ can be upper-bounded by 
\begin{align}\| C(\{\theta_{\text{MT}}^{(j)} \}_{j=1}^h) \|_2 &\leq \sum_{t \in T} \sum_{\ell=0}^h H_{max} \sum_{m=0}^{\ell}   | \alpha (T+1) - 1| G_{max} \\
& = T \frac{(h+1)(h+2)}{2}(\alpha (T+1) - 1)  H_{max} G_{max} 
\end{align}

We are left with bounding $H_{max}$ and $G_{max}$. 

Let us start with $G_{max}$, the cross entropy loss on a sample $(x,y)$ is $L(x,y,\theta) = - \log (\sigma(f_{\theta}(x) \cdot e_y))$, with $\sigma( \cdot)$ being the softmax function and $e_y$ being the one-hot vector of class $y$.
We denote $z=f_{\theta}(x) $, we denote by \[  p = \sigma (z) = \left[ \frac{e^{-z_1}}{ \sum_{k=1}^Ke^{-z_k}}, \dots, \frac{e^{-z_K}}{ \sum_{k=1}^Ke^{-z_K}}  \right] . \] 
The gradient of the loss with respect to the output of the network $z$ is: 
\[ g = \nabla_{z} (- \log(\sigma(z) \cdot e_y)) = p - e_y.  \]
\[ H_{logits} = \nabla^2_z (- \log( \sigma(z) \cdot e_y))  = \text{diag}(p) - p p^T \]
$K$ is the number of classes and $P$ the number of parameters. 
Moreover we denote by $J_{\theta} f_{\theta}(x) \in \mathbb{R}^{K \times P}$ the Jacobian of the network with respect to the parameters, $\nabla^2_{\theta} f_{\theta}$ is a $3$-dimensional tensor in $ \mathbb{R}^{K \times P \times P}$. 
By chain rule we can obtain that: 
\begin{align}
   \nabla_{\theta} L(x,y,\theta) &= \underbrace{\nabla_{z} L(x,y,\theta)}_{g^T}\underbrace{\nabla_{\theta} z}_{J_{\theta} f_{\theta}(x)} =   g^T J_{\theta} f_{\theta}(x) \\
\nabla^2_{\theta} L(x,y,\theta) &=  \nabla_{\theta} g^T J_{\theta} f_{\theta}(x) + g^T \nabla^2_{\theta} f_{\theta}(x) \\
   & = \left[ H_{logits}  J_{\theta} f_{\theta}(x) \right]^T J_{\theta} f_{\theta}(x) + g^T \nabla^2_{\theta} f_{\theta}(x) \\
  &  = J_{\theta} f_{\theta}(x)^T H_{logits}   J_{\theta} f_{\theta}(x) + g^T \nabla^2_{\theta} f_{\theta}(x)
\end{align}

\textbf{Bound on the Jacobian of the network}
We assume the activation has bounded first and second derivative, more formally we assume the activation function $\phi(\cdot)$ is such that there exists $\beta_{\phi}$ so that $\sup_{x \in \mathbb{R}} \phi'(x) \leq \beta_{\phi} $ and  there exists $\gamma_{\phi}$ so that $\sup_{x \in \mathbb{R}} \phi''(x) \leq \gamma_{\phi} $ . Notice that this covers both the case of ReLU, the hyperbolic tangent or sigmoid activation function. 

Indeed 
\begin{equation}
 \text{ for the points } x \in \mathbb{R} \text{ where the derivative are defined: }   
 |\text{ReLU}'(x) | \leq 1  \text{ and } \text{ReLU}''(x) =0.
\end{equation}
For the sigmoid activation function 
\begin{equation}
    \sup_{x \in \mathbb{R}} |\sigma'(x) | \leq \frac{1}{4}  \text{ and }  \sup_{x \in \mathbb{R}} |\sigma''(x) | \leq \frac{1}{6 \sqrt{3}}. 
\end{equation}
Finally for the hyperbolic tangent
\begin{equation}
    \sup_{x \in \mathbb{R}}|\tanh'(x) | \leq 1 \text{ and }  \sup_{x \in \mathbb{R}}|\tanh''(x) | \leq \frac{4}{3 \sqrt{3}}. 
\end{equation}

Under the hypothesis of activation function with bounded first and second derivatives almost everywhere, plus the hypothesis that the input of the network lies in a bounded set, $ \| x\| \leq M_x $ , and that for each layer $\ell$ there exists $s_{\ell}$ so that $ || W^{\ell}||_2 \leq s_{\ell}$. We obtain, by the chain rule:
\[ \left\| J_{\theta}f_{\theta}(x) \right\| \leq M_{x} \prod_{l=1}^{L} s_{l} \beta_{\phi}^{L-1}. \]
We denote by $\Pi = \prod_{l=1}^{L} s_{l} $. 
From this we can derive the following 
\textbf{bound on the gradient of the loss function}
\[ || \nabla_{\theta} L ||_2 \leq ||g ||_2 ||J_{\theta} f_{\theta}(x) ||_2 \leq \sqrt{2} M_{x} \Pi\beta_{\phi}^{L-1} \]

\textbf{Bound on the Hessian of the loss.}

If the activation function is the ReLU function then the network is piecewise linear and $\nabla^2_{\theta}f_{\theta}(x)  = 0 $ is almost everywhere. 
So the Hessian of the loss becomes: 
\[ \nabla^2_{\theta} L(x,y,\theta) =  J_{\theta} f_{\theta}(x)^T H_{logits}   J_{\theta} f_{\theta}(x). \]
The matrix $H_{logits}$ is the correct  covariance  of a categorical variable, it is positive semi-definite, so it is diagonalizable with non-negative eigenvalues. 
For any unit vector $x\in\mathbb R^{n}$,
\[
  x^{\mathsf T}H_{logits}x=\sum_{i=1}^{n} p_i x_i^{2}-\Bigl(\sum_{i=1}^{n} p_i x_i\Bigr)^{2}.
\]
We can look for the eigenvector corresponding to the maximum eigenvalues in the space orthogonal to the null-space. 
So, the maximum eigenvalues is given by: 
 \begin{align}
  \max_{x : ||x||_2=1} x^{\mathsf T}H_{logits}x &= \max_{x : ||x||_2=1} \sum_{i=1}^{n} p_i x_i^{2}-\Bigl(\sum_{i=1}^{n} p_i x_i\Bigr)^{2}.
 \end{align}
Now, for $\|x\|_2 = 1$, we can derive that 
\begin{align*}
\sum_{i=1}^{n} p_i x_i^{2}-\Bigl(\sum_{i=1}^{n} p_i x_i\Bigr)^{2}  &= \sum_{i=1}^{n} p_i x_i^{2} -\sum_{i , j}^{n} p_i p_j x_ix_j = \sum_{i=1}^{n} p_i x_i^{2}  -   \sum_{i }^{n} p_i^2 x_i^2 - \sum_{i \neq j}^{n} p_i p_j x_ix_j\\
 &= \sum_{i=1}^n p_i(1-p_i)x_i^2 - \sum_{i\neq j}^{n} p_i p_j x_ix_j \leq \sum_{i=1}^n p_i(1-p_i)x_i^2 + \frac{1}{2} \sum_{i\neq j}^{n} p_i p_j (x_i^2 + x_j^2) \\ & =  \sum_{i=1}^n p_i(1-p_i)x_i^2 +  \sum_{i\neq j}^{n} p_i p_j x_i^2  
 = \sum_{i=1}^n p_i(1-p_i)x_i^2 +  \sum_{i}^{n} p_i (1-p_i) x_i^2   \\
 &= 2 \sum_{i=1}^n p_i(1-p_i)x_i^2 \leq 2 \frac{1}{4}  \sum_{i =1}^{n} x_i^2 = \frac{1}{2} \end{align*}
So the maximum eigenvalue for $H_{logits}$ is $\frac{1}{2}$. 
Consequently, when the activation function is ReLU: 
\[ \| \nabla^2_{\theta} L(x,y,\theta) \|_2 \leq \frac{1}{2} M_{x}^2 \Pi^2\beta_{\phi}^{2L-2}. \]

In the other cases we also have to include the terms coming from the hessian of the network with respect to the parameters. We have to bound that term.

\paragraph*{Bound Hessian of the Network}
We start by noticing the structure of the Jacobian blocks. We use the following layer-wise notation:
\[
h^{(\ell)}:=W^{(\ell)}a^{(\ell-1)}, 
\qquad
a^{(\ell)}:=\phi\!\bigl(h^{(\ell)}\bigr),
\qquad
a^{(0)}:=x .
\]

Let $\Theta$ collect \emph{all} scalar parameters and denote a single one by
$\theta_p$.
We consider $\theta_{p}^r$ is an entry of $W^{(r)}$ with $1\le r\le k$.

The product splits at layer $r$:

\[
\;
\partial_{\theta_{p}^r}a^{(k)}
   =\bigl(D^{(k)}W^{(k)}\dotsm D^{(r)}\bigr)\;
     \bigl(\partial_{\theta_{p}}W^{(r)}\bigr)\;
     a^{(r-1)}
\]

where $D^{(\ell)}=\operatorname{diag}\!\bigl(\phi'(h^{(\ell)})\bigr)$.


For any parameter pair $(\theta_p^r,\theta_q^{r'})$ with $1 \leq r \leq r' \leq k$ (the function implemented by the neural network is continuous, therefore the mixed double derivatives are equal) :
\begin{align*}
\partial_{\theta_q^{r'}}\partial_{\theta_p^r}a^{(k)}
  =
  \underbrace{ \partial_{\theta_q^{r'}}\bigl(D^{(k)}W^{(k)}\dotsm D^{(r)}\bigr)
              (\partial_{\theta_p^r}W^{(r)})a^{(r-1)}}_{(A)}
  \;+\;\\
  \underbrace{D^{(k)}W^{(k)}\dotsm D^{(r)}
              (\partial_{\theta_p^r}W^{(r)})
              \,\partial_{\theta_q^{r'}}a^{(r-1)}}_{=0}
  \;+\;\\
  \underbrace{D^{(k)}W^{(k)}\dotsm D^{(r)}
              \,\partial_{\theta_q^{r'}}(\partial_{\theta_p^r}W^{(r)})
              \,a^{(r-1)}}_{(B)}.
\end{align*}

\begin{align*}
\partial_{\theta_q^{r'}}\partial_{\theta_p^r}a^{(k)}
  \;=\;
  \underbrace{\bigl(\partial_{\theta_q^{r'}}Z_{k\gets r}\bigr)
              \bigl(\partial_{\theta_p^r}W^{(r)}\bigr)a^{(r-1)}}_{(A)}
  \;+\;\\
  \underbrace{Z_{k\gets r}\,
              \partial_{\theta_q^{r'}}\bigl(\partial_{\theta_p^r}W^{(r)}\bigr)
              a^{(r-1)}}_{(B)},
\end{align*}
where
\[
Z_{k\gets r}:=D^{(k)}W^{(k)}\dotsm D^{(r)}, 
\qquad 
D^{(\ell)}=\operatorname{diag}\!\bigl(\phi'(h^{(\ell)})\bigr).
\]

We recall that 
\[
\|D^{(\ell)}\|_{2}\le\beta,\qquad
\|\partial_{\theta}W^{(\ell)}\|_{2}=1,\qquad
\|W^{(\ell)}\|_{2}\le s_\ell,\qquad
\|a^{(r-1)}\|_{2}\le M_x .
\]
Define \(\displaystyle 
     \Pi=\prod_{\ell=1}^{k}s_\ell,
     \quad
     \Pi_{1:r-1}:=\prod_{\ell=1}^{r-1}s_\ell,
     \quad
     \Pi_{r+1:k}:=\prod_{\ell=r+1}^{k}s_\ell .
\)


\paragraph*{ Bound of the term (A)}
\[
\partial_{\theta_q^{r'}}Z_{k\gets r}
=\sum_{\ell=r'}^{k} \big [
   D^{(k)}W^{(k)}\dotsm
   \bigl(\partial_{\theta_q}D^{(\ell)}\bigr)\dotsm
   D^{(r)} + D^{(k)}W^{(k)}\dotsm
   \bigl(\partial_{\theta_q^{r'}}W^{(\ell)}\bigr)\dotsm
   D^{(r)} \big] = \]
   \[=D^{(k)}W^{(k)}\dotsm\bigl(\partial_{\theta_q^{r'}}W^{(r')}\bigr)\dotsm
   D^{(r)} + \sum_{\ell=r'+1}^{k} 
   D^{(k)}W^{(k)}\dotsm
   \bigl(\partial_{\theta_q}D^{(\ell)}\bigr)\dotsm
   D^{(r)} , 
\]with\[
\partial_{\theta_q}D^{(\ell)}
   =\operatorname{diag}\!\bigl(\phi''(h^{(\ell)})\bigr)
     \partial_{\theta_q}h^{(\ell)} =
\]

\[
= \operatorname{diag}\!\bigl(\phi''(h^{(\ell)})\bigr) W^{(\ell)} D^{(\ell-1)}\cdots D^{(r')} \partial_{\theta_q^{r'}}W^{(r')} a^{(r'-1)}
\]
Hence
\[
\bigl\|\partial_{\theta_q^{r'}}D^{(\ell)}\bigr\|_{2}
      \le\gamma\,\beta^{\ell-r'}\,M_x\,\Pi_{r'+1:\ell},
\]
and therefore
\[
  \|(A)\|_{2}
  \;\le\; M_x \big [\beta^{k-r+1}\Pi_{r:r'}\Pi_{r'+1:k} + \sum \limits_{\ell=r'+1}^k \gamma \beta^{\ell - r'} M_x \Pi_{r'+1:\ell} \big]
.
\]

\paragraph*{ Bound of the term (B)}
\[
\partial_{\theta_q^{r'}}\!\bigl(\partial_{\theta_p^r}W^{(r)}\bigr)=0
\]

\paragraph*{Bound Hessian}
\[
\bigl\|\partial_{\theta_q^{r'}}\partial_{\theta_p^r}a^{(k)}\bigr\|_{2}
   \;\le\;
   M_x \big [\beta^{k-r+1}\Pi_{r:r'}\Pi_{r'+1:k} + \sum \limits_{\ell=r'+1}^k \gamma \beta^{\ell - r'} M_x \Pi_{r'+1:\ell} \big]
\]
\end{proof}

%% file: 99_appendix/C_task_vectors_gradients/B_Proofs/content.tex
\subsection{Upper bounding the multi-task loss}
\label{appendix:bounding_loss_ATM}

In this section, we explore the relationship between the PA-ATM loss,  defined as the mean of average losses over all tasks, and the loss of a model trained jointly on all the datasets.  We conduct this analysis using  GD, rather than SGD, thus removing stochasticity and enabling a clearer analysis while retaining the key insights.
Denoting with $t$ both the task and its corresponding dataset with cardinality $n_t$. The total number of samples for all tasks is given by $N = \sum_{t \in T} n_t$.

%
By Theorem~\ref{thm:multitask-vector-gradient}, when merging occurs after one step of fine-tuning on each dataset, the PA-ATM update is given by:
\begin{equation}
    \theta_{\text{pre}}^{(k+1)} = \theta_{\text{pre}}^{(k)} + \frac{\alpha}{|T|} \sum_{t \in T} \tau_{t}^{(k)} = \theta_{\text{pre}}^{(k)} - \alpha \eta  \nabla \left(\frac{1}{|T|} \sum_{t \in T}  \overline{L}_t\right) \,,
\end{equation}
which corresponds to performing a GD step over the loss $L_{\text{ATM}} = \frac{1}{|T|} \sum_{t \in T}  \overline{L}_t $.

Having established that one step of PA-ATM in GD minimizes \( L_{\text{PA-ATM}} \), a crucial question arises: under what conditions does minimizing \( L_{\text{ATM}} \) imply the minimization of \( L_{\text{target}} \)? In other words, when can we be certain that optimizing the PA-ATM loss will also minimize the loss associated with training jointly on all datasets?

To study if optimizing the PA-ATM loss will also minimize the loss associated with training jointly on all datasets we first note that \( L_{\text{PA-ATM}} \) is an unweighted average of the individual dataset losses, while the target loss is a weighted average:
\begin{equation}
    L_{\text{target}}(\theta) = \frac{\sum_{t \in T} n_t \overline{L}_t(\theta)}{\sum_{t \in T} n_t} \,.
\end{equation}
We now analyze the parameter update from \( \theta^{(k)} \) to \( \theta^{(k+1)} \). For both PA-ATM and target methods, we denote the change in loss, \( L_{\text{method}} \), as
\(
\Delta L_{\text{method}} = L_{\text{method}}(\theta^{(k)}) - L_{\text{method}}(\theta^{(k+1)}).
\)
In the following theorem, we prove that if the drop in ATM loss exceeds a threshold $\delta$, the target loss will also decrease. The value of $\delta$ depends on the size of the largest dataset with a decreasing loss and the smallest dataset with an increasing loss. In particular, if the former dataset is larger than the latter, a reduction in PA-ATM loss reduces the target loss.
In practice, this is ensured when the loss is reduced on the largest dataset.
\begin{theorem}
    \label{thm:one}
    Let $D= \{ t \mid \Delta \overline{L}_t > 0 \}$ be the set of datasets where the loss decreases after a parameter update, and  $I = \{ t \mid \Delta \overline{L}_t \leq 0 \}$ be the set of datasets where the loss increases or remains unchanged.
    If the reduction in the PA-ATM loss satisfies \( \Delta L_{\text{PA-ATM}} > \delta \), where
    \[
        \delta = \frac{1}{|T|} \left( 1 - \frac{\min_{t \in I} n_t}{\max_{t \in D} n_t} \right) \sum_{t \in I} |\Delta \overline{L}_t|\,,
    \]
    then the target loss \(  L_{\text{target}} \) will also decrease, i.e., \( \Delta L_{\text{target}} > 0 \).
\end{theorem}

\begin{proof}
    Suppose that when transitioning from parameters \( \theta^{(i)} \) to parameters \( \theta^{(i+1)} \), the change in the average loss for each dataset \( D_k \) is given by \( \Delta \overline{L}_k = \overline{L}_k(\theta^{(i)}) - \overline{L}_k(\theta^{(i+1)}) \).
    We denote by $P$ the set of tasks for which the delta of the loss is positive and by $N$ the set for which it is negative, namely
    $ P = \{ k \in T \text{ s.t. } \Delta \overline{L}_k > 0\}$ and $ N = \{k \in T \text{ s.t. } \Delta \overline{L}_k \leq 0\}$.
    The set $ \{ 1, \dots, n\} = P \cup N$.
    In the following formulas we will use the symbol $|$ for different purposes. For sets, like $ D_i$, $ | D_i|$  denotes the cardinality of the set, while for scalars, such as $ \Delta \overline{L}_k$, $|\Delta \overline{L}_k|$ denotes their absolute value.  Since the tasks in $N$ have negative $\Delta \overline{L}_k  $, it holds that
    \[
        \sum_{j \in N} |D_j| \Delta \overline{L}_j = - \sum_{j \in N} |D_j| |\Delta \overline{L}_j| \,.
    \]
    Under the hypothesis that $\Delta L_{\text{PA-ATM}}> \delta $, this implies $ \sum_{j=1}^n \Delta_j > n \delta   $.  We have that
    $ \sum_{j \in P} \Delta \overline{L}_j + \sum_{j \in N} \Delta \overline{L}_j = \sum_{j \in P} \Delta \overline{L}_j - \sum_{j \in N} |\Delta \overline{L}_j| > n \delta  $, namely  $\sum_{j \in P} \Delta \overline{L}_j > n\delta + \sum_{j \in N} |\Delta \overline{L}_j|  $.

    We want to prove that $\Delta L_{\text{target}}>0$.
    Let us now consider $\Delta L_{\text{target}}>0  $ iff $ \sum_{j=1}^n |D_j| \Delta_j > 0 $.
    \begin{align*}
        \sum_{j=1}^n |D_j| \Delta_j & =  \sum_{j \in P} |D_j|\Delta \overline{L}_j + \sum_{j \in N} |D_j| \Delta \overline{L}_j                                                                  \\
                                    & =  \sum_{j \in P} |D_j|\Delta \overline{L}_j - \sum_{j \in N} |D_j| |\Delta \overline{L}_j|                                                                \\
                                    & >  \min_{j \in P  }|D_j| \sum_{j \in P} \Delta \overline{L}_j - \max_{j \in N  }|D_j| \sum_{j \in N}  |\Delta \overline{L}_j|                              \\
                                    & >  \min_{j \in P  } |D_j| \left[ n\delta  + \sum_{j \in N} |\Delta \overline{L}_j| \right] - \max_{j \in N  }|D_j| \sum_{j \in N}  |\Delta \overline{L}_j| \\
                                    & = n \min_{j \in P  } |D_j| \delta  + (\min_{j \in P  } |D_j| - \max_{j \in N  } |D_j| \sum_{j \in N})  |\Delta \overline{L}_j| \,.
    \end{align*}
    The last line of the previous equation is positive by hypothesis, since we assumed
    $\delta > \frac{1}{n} (1 - \frac{\min_{j \in N  } |D_j| }{\max_{j \in P  } |D_j|}) \sum_{j \in N} |\Delta \overline{L}_j|  $.
    \label{proof}
\end{proof}

\begin{remark}
    If we choose the target loss to be the maximum of the average loss across all datasets \( L_{\text{target}} = \max_{t \in T} \overline{L}_t \), by leveraging the equivalence between the \( L_1 \)-norm and the max norm, we obtain the bound \( L_{\text{target}} \leq T \cdot L_{\text{PA-ATM}} \).
\end{remark}

%% file: 99_appendix/C_TSV/content.tex
\section{Illustrative example: merging two tasks with rank-1 approximation}

Consider merging two distinct tasks by selecting only the first singular vector
and singular value from the SVD for each task. This setting yields the
following setup for each layer $\ell$: \[\left[
    \begin{array}{@{}c c@{}}
      \mid        & \mid        \\
      u_{1_{\ell}} & u_{2_{\ell}} \\
      \mid        & \mid        \\
    \end{array}
    \right]
  \left[
    \begin{array}{c c}
      \sigma_{1_{\ell}} & 0                \\
      0                & \sigma_{2_{\ell}}
    \end{array}
    \right]
  \left[
    \begin{array}{ccc}
      - & v^{T}_{1_{\ell}} & - \\
      - & v^{T}_{2_{\ell}} & - \\
    \end{array}
    \right]\]
In this formulation, $u_{1_{\ell}}$ originates from task 1 and $u_{2_{\ell}}$ from task 2, with analogous assignments for the singular vectors $v$ and singular values $\sigma$.

To elucidate the interaction between tasks, we examine three distinct cases,
considering a single layer, thereby omitting the layer index $\ell$:

\begin{enumerate}
  \item \textbf{Orthogonal Singular Vectors:} when $u_{1}$ and $u_{2}$ (respectively $v$) are orthogonal, the similarity matrix $U^\top U$ (respectively $V^\top V$) is given by:
        \[\left [
            \begin{array}{c c}
              1 & 0 \\
                      0 & 1
            \end{array}
            \right]
        \]
        The zeroes in the off-diagonal elements indicate no interference between the
        tasks. Consequently, the orthogonal components derived from different tasks
        operate independently, ensuring that each task does not affect the other.

  \item \textbf{Collinear Singular Vectors:} when $u_{1}$ and $u_{2}$ (respectively $v$) are collinear, either aligned in the same direction (angle of 0 degrees) or in the opposite direction (angle of 180 degrees), the similarity matrix $U^\top U$ (respectively $V^\top V$) takes the form:
        \[\left [
            \begin{array}{c c}
              1                        & \langle u,\pm u \rangle \\
                      \langle \pm u, u \rangle & 1
            \end{array}
            \right]
        \]
        If the singular vectors are perfectly aligned (0 degrees), then $u_{1} = u_{2}
          = u$, simplifying the diagonal elements to $\langle u, u \rangle = \lVert u
          \rVert^2 =1$. Conversely, if the singular vectors are oppositely aligned (180
        degrees), $u_1 = -u_2$, resulting in $\langle u,-u \rangle = -1$ Thus, the
        similarity matrices becomes:
        \[\left [
            \begin{array}{c c}
              1     & \pm 1 \\
                      \pm 1 & 1
            \end{array}
            \right]
        \]
        This structure reveals complete interference between the tasks: a double
        scaling effect when the vectors agree and complete cancellation when they
        disagree.

  \item \textbf{Partially Collinear Singular Vectors:} when $u_{1}$ and $u_{2}$ (respectively $v$) are partially collinear, with the angle between them ranging from slightly greater than 0 degrees to less than 90 degrees or slightly more than 90 degrees to less than 180 degrees, the similarity matrices are expressed as:
        \begin{align*}
          \left [
            \begin{array}{c c}
              1                        & \langle u_1, u_2 \rangle \\
                      \langle u_2, u_1 \rangle &
              1
            \end{array}
            \right]
        \end{align*}
        In this case, the overlap between singular vectors induces a partial
        interaction between the tasks. The degree of interference, whether it is
        constructive or destructive, is proportional to the cosine of the angle between
        the singular vectors. This partial collinearity leads to subtle interplay,
        where the tasks influence each other to a degree dictated by their vector
        alignment.
\end{enumerate}

This example underscores the critical role of singular vector alignment in
model merging, highlighting how orthogonality ensures independent task
performance, collinearity leads to maximal interference and partial
collinearity results in an intermediate level of task interaction.

\begin{figure}[btp]
  \centering
  \includegraphics[width=0.85\linewidth]{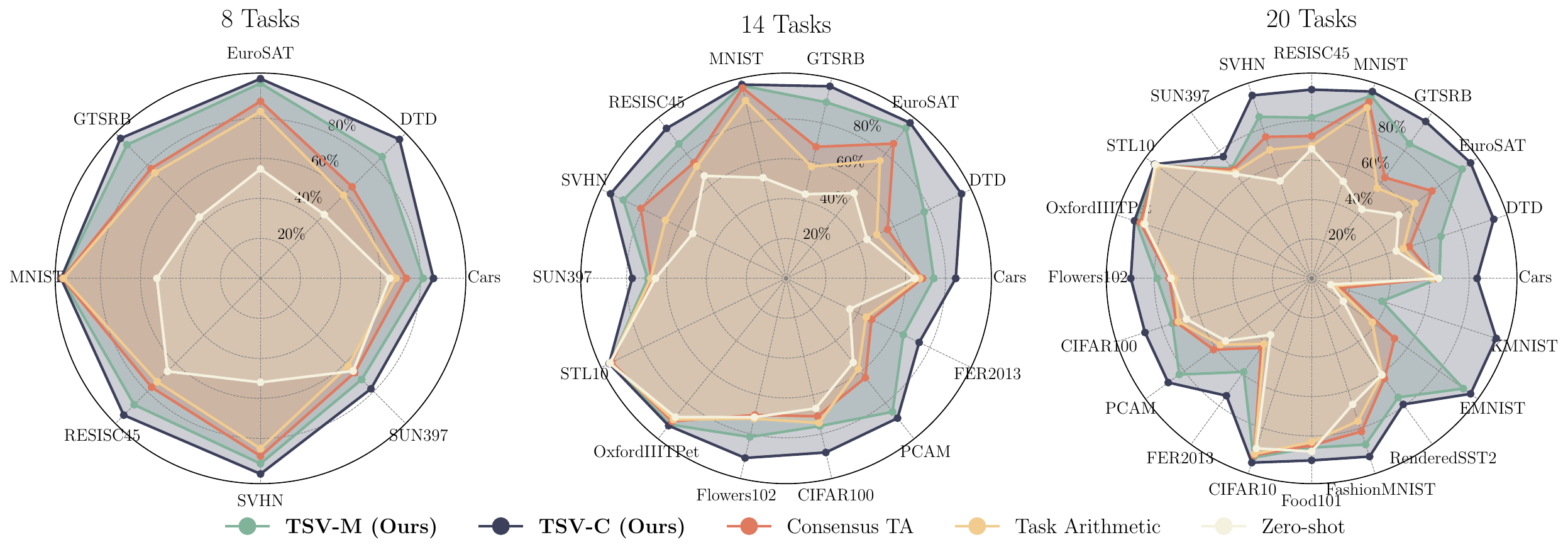}
  \caption[Absolute accuracy of merged \model{ViT-B-16} over 8, 14, and 20 tasks]{Absolute accuracy of a \model{ViT-B-16} merged over 8, 14, and 20 tasks, respectively.}
  \label{fig:radar_chart_B-16}
\end{figure}

\section{Additional details}
\subsection{Implementation details and computational resources}

\paragraph{Normalized Accuracy}
\label{sec:normalized_accuracy}
To address the varying difficulties of the task, we report both normalized and absolute accuracies in our results. The normalized accuracy provides a relative performance metric by comparing the accuracy of the multi-task model to that of individually fine-tuned models. Specifically, the normalized accuracy is calculated as:
\begin{equation}
  \text{Normalized Accuracy} = \frac{1}{\ntasks} \sum_{i=1}^{\ntasks} \frac{\text{accuracy}(\theta_{\text{MT}}, t_i)}{\text{accuracy}(\theta_{ft_i}, t_i)}
\end{equation}
where $\ntasks$ is the total number of tasks, $\theta_{\text{MT}}$ represents the multi-task model and $\theta_{i}$ denotes the individually fine-tuned model for task $t_i$. This metric allows for a more fair comparison by adjusting for the baseline performance of each task.

\paragraph{Datasets for tasks}
All benchmarks were performed by integrating the codebase provided by
\citet{wang2024localizing}. In line with the principles of PEFT, we reused
the already existing model checkpoints in the codebase for both the models and
classification heads without additional fine-tuning.

\begin{figure}[!htbp]
  \centering
  \includegraphics[width=0.85\linewidth]{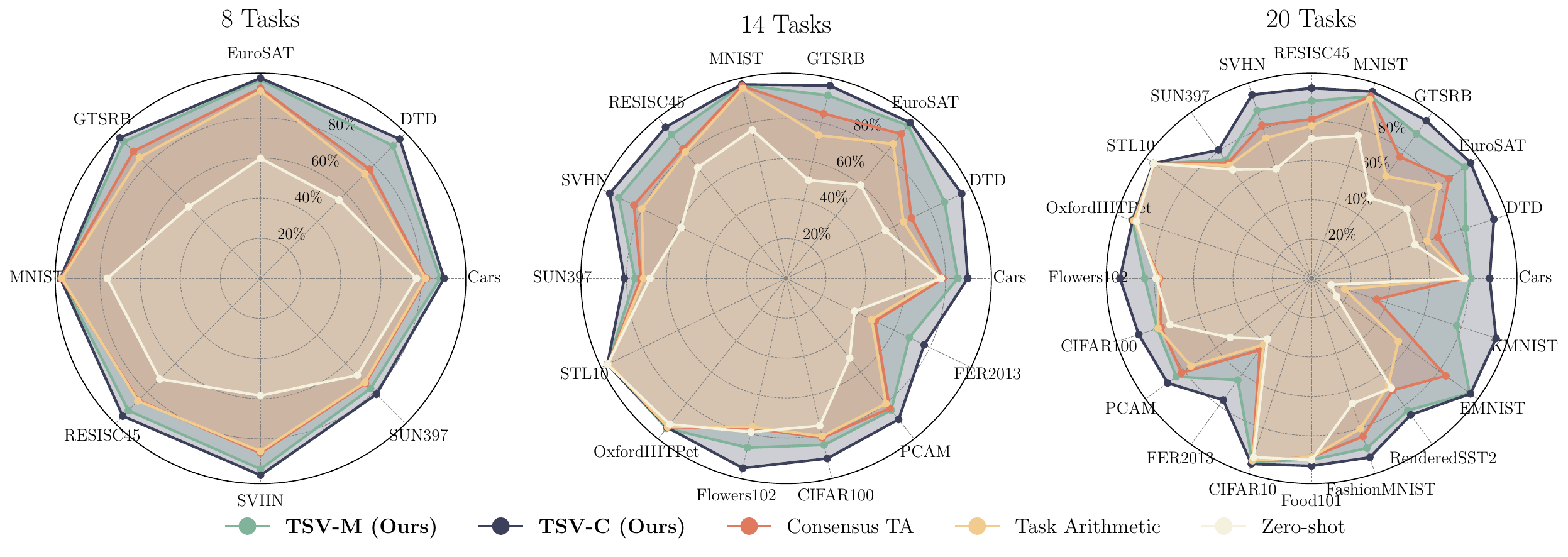}
  \caption[Absolute accuracy of merged \model{ViT-L-14} over 8, 14, and 20 tasks]{Absolute accuracy of a \model{ViT-L-14} merged over 8, 14, and 20 tasks, respectively.}
  \label{fig:radar_chart_L-14}
\end{figure}

\paragraph{Implementation}
Our method utilizes the SVD, a matrix decomposition technique applicable to
two-dimensional matrices. For layers that are not represented as matrices
(e.g., normalization layers) we default to standard \method{Task Arithmetic}.
In particular, we employ Knuth's algorithm
\citep{doi:10.1080/00401706.1962.10490022} to compute the average efficiently.
This ensures that all fine-tuned model task layers, regardless of their
structure, are appropriately integrated into the merged model.

\paragraph{Compute Resources}
We utilize PyTorch as deep learning framework. All the merging and evaluations
were conducted on an \texttt{NVIDIA 4060Ti} GPU with 16GB of memory, and an
\texttt{Intel i7-6800K} CPU equipped with 64GB of RAM. For experiments that
need more than 64GB of RAM, we resort to a shared \texttt{HTCondor} cluster
equipped with \texttt{NVIDIA P6000} GPUs.

\subsection{Hyperparameter settings}
Following \method{Task Arithmetic} \cite{task-vectors} and
\method{Consensus TA} \cite{wang2024localizing}, we apply a single scaling
factor, $\alpha$, to adjust the multi-task vector within the model merging
techniques outlined in \Cref{tab:task_acc}. This scaling factor is optimized,
when feasible, over the range \{0.0, 0.1, ..., 2.9, 3.0\}, with the optimal
value selected based on the average validation performance across all tasks.
However, as discussed in \Cref{sec:alpha}, our experimental findings indicate
that the proposed \method{TSV-Merge} method does not strictly depend on this
hyperparameter, as the marginal performance gains from tuning $\alpha$ do not
justify the computational resources required for the evaluation on the
validation datasets. Consequently, this allows us to eliminate the necessity
for validation datasets and the corresponding labels from the prerequisites of
the method, further simplifying the practicality and resource usage of the
approach. The evaluation is performed on a batch of 32 images. To produce
\cref{fig:alpha} we selected the following subsets of 8 tasks from the 20
available, using the whitening method to speed up computation, the number in
the image is the index indicating the subset:
\begin{enumerate}
  \setcounter{enumi}{-1}

  \item  \dataset{SVHN}, \dataset{SUN397}, \dataset{STL10}, \dataset{OxfordIIITPet}, \dataset{Flowers102}, \dataset{CIFAR100}, \dataset{PCAM}, \dataset{FER2013}

  \item  \dataset{PCAM}, \dataset{FER2013}, \dataset{CIFAR10}, \dataset{Food101}, \dataset{FashionMNIST}, \dataset{RenderedSST2}, \dataset{EMNIST}, \dataset{KMNIST}

  \item  \dataset{EuroSAT}, \dataset{GTSRB}, \dataset{MNIST}, \dataset{RESISC45}, \dataset{SVHN}, \dataset{SUN397}, \dataset{STL10}, \dataset{OxfordIIITPet}

  \item \dataset{Cars}, \dataset{DTD}, \dataset{EuroSAT}, \dataset{GTSRB}, \dataset{MNIST}, \dataset{RESISC45}, \dataset{SVHN}, \dataset{SUN397}

  \item  \dataset{Cars}, \dataset{DTD}, \dataset{EuroSAT}, \dataset{GTSRB}, \dataset{FashionMNIST}, \dataset{RenderedSST2}, \dataset{EMNIST}, \dataset{KMNIST}

  \item  \dataset{MNIST}, \dataset{RESISC45}, \dataset{SVHN}, \dataset{SUN397}, \dataset{STL10}, \dataset{OxfordIIITPet}, \dataset{Flowers102}, \dataset{CIFAR100}

  \item  \dataset{STL10}, \dataset{OxfordIIITPet}, \dataset{Flowers102}, \dataset{CIFAR100}, \dataset{PCAM}, \dataset{FER2013}, \dataset{CIFAR10}, \dataset{Food101}

\end{enumerate}

\begin{figure}[tbp]
  \centering
  \includegraphics[width=\linewidth]{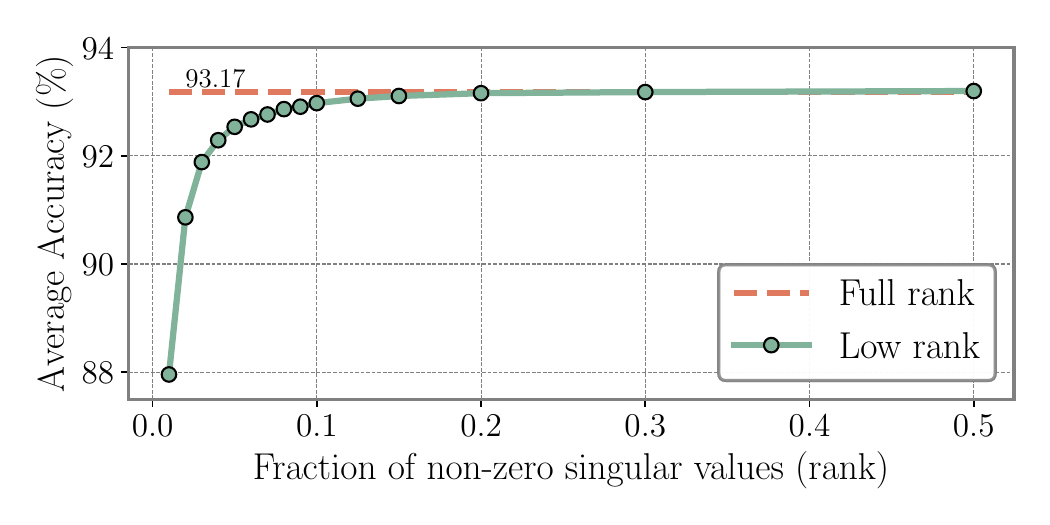}
  \caption[Mean accuracy of \model{ViT-B-16} by fraction of retained singular components]{Mean absolute accuracy of the \model{ViT-B-16} model across increasing fractions of retained singular components, averaged over 20 tasks. The red line represents the average accuracy of the original fine-tuned models with full-rank task matrices, while the green line shows the accuracies using low-rank approximations.}
  \label{fig:acc_by_rank_B-16}
\end{figure}

\subsection{Storage cost calculation}

Suppose we have a neural network comprising $L$ two-dimensional layers, each
of dimension $d \times m$, and $N$ one-dimensional layers of size $c$. The
total number of parameters in the network is therefore:
\[\text{Params(NN)} = L \times (d \times m) + N \times c.\]
In standard \method{Task Arithmetic}, one must store the same number of
parameters to obtain a task vector. In contrast, our approach provides the
flexibility to select the number of parameters to preserve based on storage
constraints or the desired performance level, to adhere to the chosen
constraints. Under the above assumptions, our method applies the truncated SVD
to each two-dimensional layer. This decomposition yields two matrices of
singular vectors, $U$ and $V$, and a vector of singular values, $\sigma$,
specifically:
\begin{itemize}
  \item $U$ of size $d \times k$,
  \item $V$ of size $k \times m$,
  \item $\sigma$ of size $k$,
\end{itemize}
where $k = \text{min}(d,m)$. We select a reduced rank $k' \ll k$ to approximate each layer's task matrix. Consequently, the total number of parameters for TSV becomes:
\[\text{Params(TSV)} = L \times ((d \times k') + k' + (k' \times m)) + N \times c\]
To demonstrate that our method results in fewer stored parameters than the
original parameter count, we require that $k'<\frac{d \times m}{d+m+1}$. This
condition ensures: \[\text{Params(NN)} > \text{Params(TSV)}\] Substituting the
expressions, yields: \[L \times (d \times m) + N \times c > L \times ((d \times
  k') + k' + (k' \times m)) + N \times c\] Simplifying, we obtain:
\begin{equation}
  \begin{split}
    L \times (d \times m) & > L \times ((d \times k') + k' + (k' \times m)) \\
    (d \times m)          & > ((d \times k') + k' + (k' \times m))          \\
    d \times m            & > k' \times (d  + 1 + m)                        \\
    k'                    & < \frac{d \times m}{d+m+1}.
  \end{split}
\end{equation}
This inequality confirms that our method reduces the storage requirements of a task vector when $k'<\frac{d \times m}{d+m+1}$. Empirical evidence from \Cref{fig:acc_by_rank,fig:acc_by_rank_B-16,fig:acc_by_rank_L-14} suggests that selecting $k'< 0.1 \times k = 0.1\times\text{min}(d,m)$ is sufficient to preserve most of the task performance, preserving the main requirement of performance. Furthermore, it is easy to prove that when choosing $k'=\frac{k}{\ntasks}$, the inequality is always satisfied for $\ntasks \ge 3$, respecting the main requirement of limited storage usage.
\begin{figure}[tbp]
  \centering
  \includegraphics[width=\linewidth]{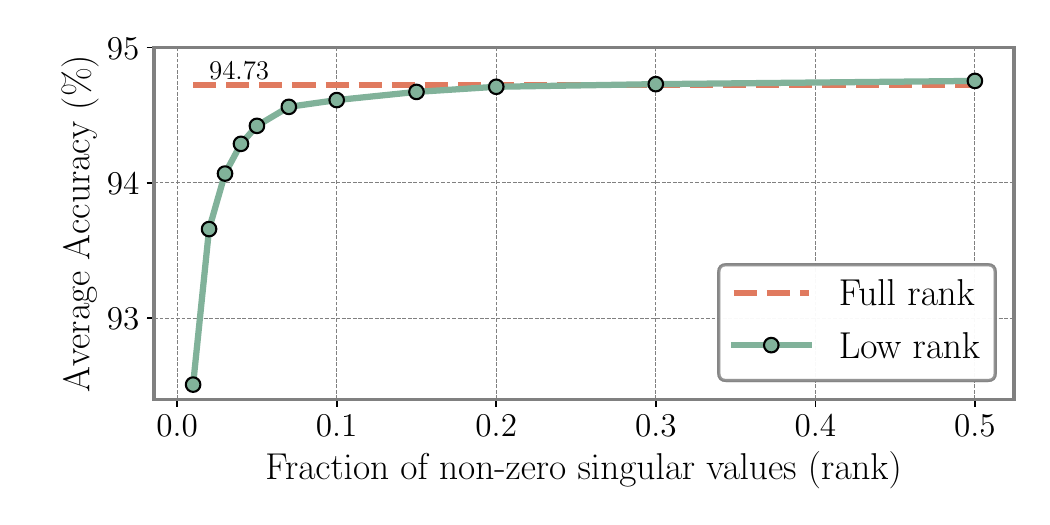}
  \caption[Mean accuracy of \model{ViT-L-14} by fraction of retained singular components]{Mean absolute accuracy of the \model{ViT-L-14} model across increasing fractions of retained singular components, averaged over 20 tasks. The red line represents the average accuracy of the original fine-tuned models with full-rank task matrices, while the green line shows the accuracies using low-rank approximations.}
  \label{fig:acc_by_rank_L-14}
\end{figure}
\begin{table}[b]
  \centering
    \begin{tabular}{ccccc}
      \toprule
      \vspace{-0.1cm}Low-rank & Interf.       & \multicolumn{3}{c}{\model{ViT-B-16}}                                           \\
      \cmidrule{3-5}
      approx.                 & reduction     & 8 tasks                              & 14 tasks           & 20 tasks           \\
      \midrule
      \negxmark               & \negxmark     & 79.6 (+0.0)                          & 75.9 (+0.0)        & 70.8 (+0.0)        \\
      \poscheckmark           & \negxmark     & 79.6 (+0.0)                          & 74.9 \worse{1.0}   & 70.0 \worse{0.8}   \\
      \negxmark               & \poscheckmark & 84.8 \improv{5.2}                    & 79.0 \improv{4.1}  & 73.2 \improv{3.2}  \\
      \poscheckmark           & \poscheckmark & 93.9 \improv{9.1}                    & 91.0 \improv{12.0} & 86.5 \improv{13.3} \\
      \bottomrule
    \end{tabular}
  \caption[Ablation study for \method{TSV-Merge} on \model{ViT-B-16}]{Comparison of different versions of \baseline{Task Arithmetic}, comprising either the low-rank approximation step, the interference reduction step, or both. The method performing both corresponds to the proposed \method{TSV-Merge}.}
  \label{tab:ablation-study-B-16}
\end{table}
\section{Proofs}
We hereby prove the claims outlined in the main manuscript.

\subsection{Characterization of the similarity matrices}

\begin{proposition} \label{prop:psd}
  The matrix \( \hat{U}^\top \hat{U} \) is positive definite.
\end{proposition}

\begin{proof}
  We define \( \hat{U}^\top \hat{U} \), where \( \hat{U} \) is a generic \( d \times k \) rectangular matrix. Consequently, \( \hat{U}^\top \hat{U} \) is a \( k \times k \) square matrix. To establish that \( \hat{U}^\top \hat{U} \) is positive definite, it suffices to demonstrate that for all non-zero vectors \( x \in \mathbb{R}^{k} \), the following inequality holds:
  \[
    x^\top \hat{U}^\top \hat{U} x > 0.
  \]

  This expression can be rewritten as:
  \[
    x^\top (\hat{U}^\top \hat{U}) x = (\hat{U} x)^\top (\hat{U} x) = \lVert \hat{U} x \rVert^2.
  \]
  Here, \( \lVert \hat{U} x \rVert^2 \) denotes the squared Euclidean norm of the
  vector \( \hat{U} x \), which is always non-negative. Moreover, assuming that
  \( \hat{U} \) has full column rank, the norm \( \lVert \hat{U} x \rVert^2 \) is
  strictly positive for any non-zero vector \( x \). Therefore, we have:
  \[
    \lVert \hat{U} x \rVert^2 > 0 \quad \text{for all} \quad x \in \mathbb{R}^{k}, \, x \neq 0.
  \]

  This implies that:
  \[
    x^\top \hat{U}^\top \hat{U} x > 0 \quad \text{for all} \quad x \in \mathbb{R}^{k}, \, x \neq 0,
  \]
  which confirms that \( \hat{U}^\top \hat{U} \) is positive definite.
\end{proof}

\begin{corollary} \label{cor:psd_invertible}
  Since \( \hat{U}^\top \hat{U} \) is positive definite, then \( \hat{U}^\top \hat{U} \) is invertible.
\end{corollary}

\begin{proof}
  From \Cref{prop:psd}, we have established that \( \hat{U}^\top \hat{U} \) is a positive definite matrix. A positive definite matrix, by definition, has all its eigenvalues strictly positive. Let \( \lambda_1, \lambda_2, \ldots, \lambda_k \) denote the eigenvalues of \( \hat{U}^\top \hat{U} \). Therefore, we have:
  \[
    \lambda_i > 0 \quad \text{for all} \quad i = 1, 2, \ldots, k.
  \]

  The determinant of \( \hat{U}^\top \hat{U} \) is the product of its
  eigenvalues:
  \[
    \det(\hat{U}^\top \hat{U}) = \prod_{i=1}^{k} \lambda_i.
  \]
  Since each \( \lambda_i \) is positive, their product is also positive:
  \[
    \det(\hat{U}^\top \hat{U}) > 0.
  \]

  A matrix is invertible if and only if its determinant is non-zero. Given that
  \( \det(\hat{U}^\top \hat{U}) > 0 \), it follows that \( \hat{U}^\top \hat{U}
  \) is invertible.

  Therefore, \( \hat{U}^\top \hat{U} \) is invertible.
\end{proof}

\subsection{Observations}

Since \( \hat{U}^\top \hat{U} \) is a real symmetric matrix, it admits an
eigendecomposition of the form
\[
  \hat{U}^\top \hat{U} = Q \Lambda Q^{-1} = Q \Lambda Q^\top,
\]
where:
\begin{itemize}
  \item \( \Lambda \) is a diagonal matrix containing the real eigenvalues of \( \hat{U}^\top \hat{U} \),
  \item \( Q \) is an orthogonal matrix whose columns are the orthonormal eigenvectors of \( \hat{U}^\top \hat{U} \), satisfying \( Q^\top = Q^{-1} \).
\end{itemize}

The inverse of \( \hat{U}^\top \hat{U} \) exists (see
\Cref{cor:psd_invertible}), and can be expressed using its eigendecomposition
as
\[
  (\hat{U}^\top \hat{U})^{-1} = Q \Lambda^{-1} Q^{-1} = Q \Lambda^{-1} Q^{\top}.
\]
Additionally, since \( \Lambda \) is a diagonal matrix with non-zero diagonal
entries (\Cref{prop:psd}), its inverse \( \Lambda^{-1} \) is straightforward to
compute, with each diagonal element given by
\[
  \Lambda^{-1} = \text{diag}\left( \frac{1}{\lambda_i} \right),
\]
where \( \lambda_i \) are the eigenvalues of \( \hat{U}^\top \hat{U} \).

Furthermore, the eigenvalues of \( (\hat{U}^\top \hat{U})^{-1} \) are \(
\frac{1}{\lambda_i} \), each of which is positive since \( \lambda_i > 0 \) for
all \( i \) (following by the definition in \Cref{prop:psd}). Consequently, \(
(\hat{U}^\top \hat{U})^{-1} \) is also a positive definite matrix.

These observations confirm that not only is \( \hat{U}^\top \hat{U} \) positive
definite, but its inverse inherits this property due to the positivity of its
eigenvalues.
\begin{table}[h!]
  \centering
    \begin{tabular}{ccccc}
      \toprule
      \vspace{-0.1cm}Low-rank & Interf.       & \multicolumn{3}{c}{\model{ViT-L-14}}                                          \\
      \cmidrule{3-5}
      approx.                 & reduction     & 8 tasks                              & 14 tasks          & 20 tasks           \\
      \midrule
      \negxmark               & \negxmark     & 88.6 (+0.0)                          & 84.0 (+0.0)       & 78.1 (+0.0)        \\
      \poscheckmark           & \negxmark     & 87.9 \worse{0.7}                     & 83.4 \worse{0.6}  & 77.2 \worse{0.9}   \\
      \negxmark               & \poscheckmark & 92.1 \improv{4.2}                    & 86.8 \improv{3.4} & 81.0 \improv{3.8}  \\
      \poscheckmark           & \poscheckmark & 97.0 \improv{4.9}                    & 94.4 \improv{7.6} & 92.5 \improv{11.5} \\
      \bottomrule
    \end{tabular}
  \caption[Ablation study for \method{TSV-Merge} on \model{ViT-L-14}]{Comparison of different versions of \baseline{Task Arithmetic}, comprising either the low-rank approximation step, the interference reduction step, or both. The method performing both corresponds to the proposed \method{TSV-Merge}.}
  \label{tab:ablation-study-L-14}
\end{table}
%
\input{mini_proof.tex}
%
\begin{figure}[htbp]
  \centering
  \includegraphics[width=\linewidth]{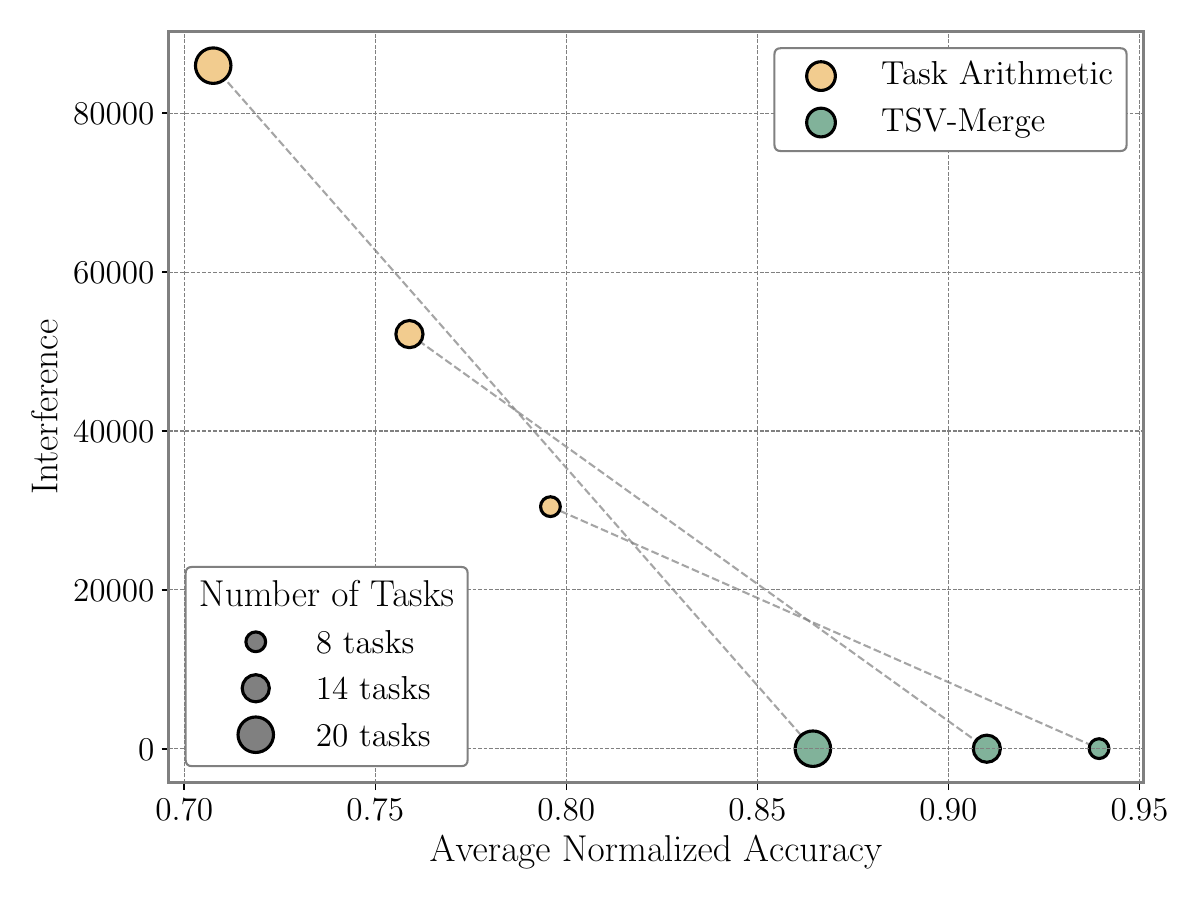}
  \caption[STI and normalized accuracy for \model{ViT-B-16}]{Singular Task Interference (STI) and average normalized accuracy for \baseline{Task Arithmetic} and \method{TSV-Merge} on the \model{ViT-B-16} model, evaluated across merges of 8, 14, and 20 tasks.}
  \label{fig:correlation_interference_accuracy_B-16}
\end{figure}
\begin{figure}[!ht]
  \centering
  \includegraphics[width=0.85\linewidth]{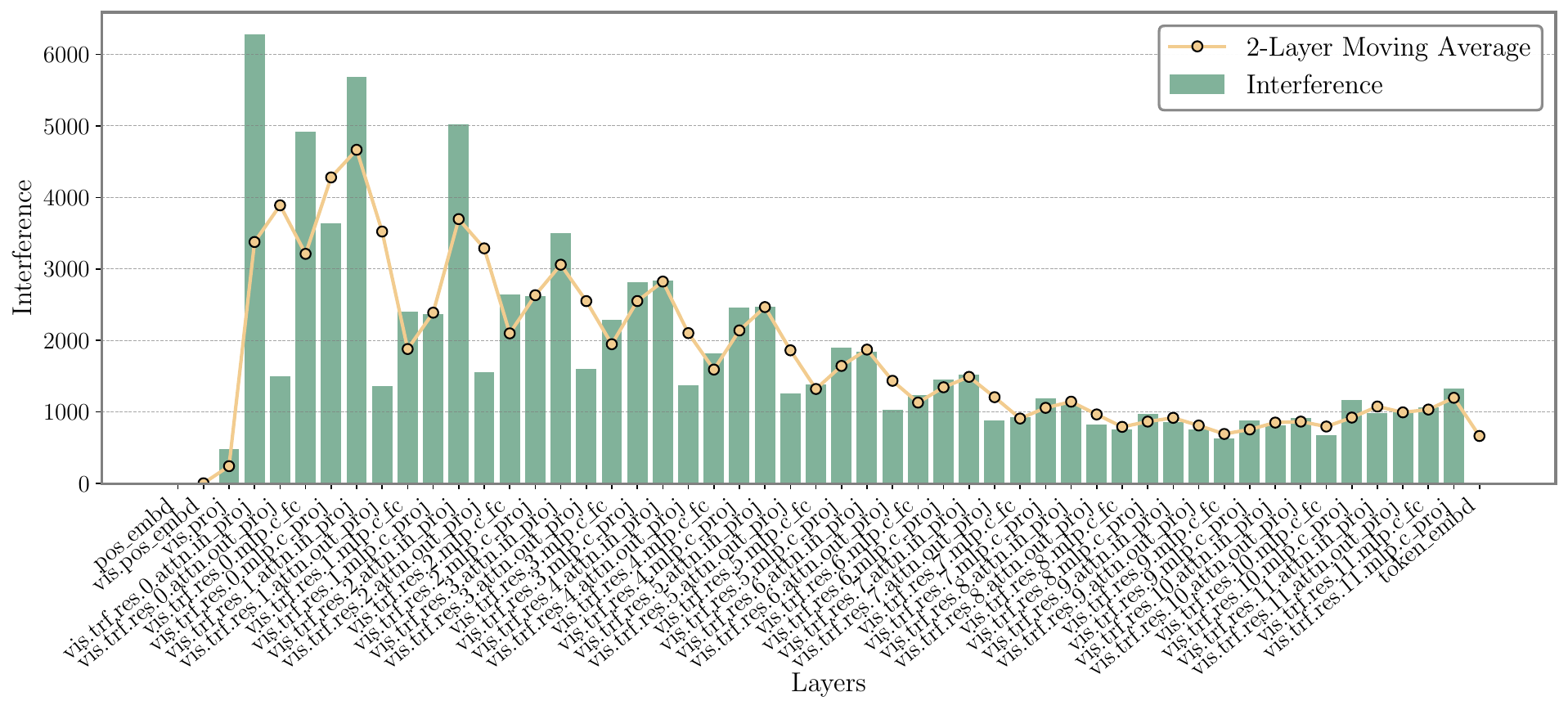}
  \caption[Per-layer Singular Task Interference in \model{ViT-B-32}]{Detailed view of Singular Task Interference (STI) across layers in a \model{ViT-B-32} for 20 tasks. The interference trend is high in early layers and decreases later. Here, the pattern for each transformer block is observable, the interference first increases and then drops in each attention-out layer.}
  \label{fig:detailed_interference_by_layer}
\end{figure}
\section{Additional experimental results}

\subsection{Per-dataset performance metrics}
In \Cref{sec:MM_results}, we present comprehensive results for individual tasks
using the \model{ViT-B-32} model. Here we include analogous radar plots for the
\model{ViT-B-16} model in \Cref{fig:radar_chart_B-16} and the \model{ViT-L-14}
model in \Cref{fig:radar_chart_L-14}. The analyses of these models reveal
findings consistent with those reported for \model{ViT-B-32} in the main text.

\subsection{Extended analysis}

\subsubsection{Whitening vs. SVD}
As we have seen in \Cref{subsec:tsv-merge}, applying a whitening transformation
to the matrices of task singular vectors is mathematically equivalent to
solving the Orthogonal Procrustes problem. However, implementing these two
approaches may yield different results depending on the distinct matrix
decomposition algorithms employed. In this study, we used PyTorch to compute
both eigendecomposition and SVD, observing slightly different results that may
be attributed to numerical errors. To more robustly compute the matrix square
root for the eigendecomposition case, we compute
\[\Lambda^{-\frac{1}{2}} = \text{diag}\left(\frac{1}{\sqrt{|\lambda_i|}+\epsilon}\right)\]
where $\epsilon = 1\mathrm{e}{-12}$ prevents division by $0$ and the absolute
value avoids numerical errors producing small negative values in magnitude less
than $1\mathrm{e}{-6}$.

\subsubsection{Impact of rank}
The \Cref{subsec:low-rank} shows that the task matrices of a \model{ViT-B-32}
are inherently low-rank and a small percentage of TSVs is enough to approximate
each layer with satisfying results. We here provide the same plots for the
models \model{ViT-B-16} (\Cref{fig:acc_by_rank_B-16}) and \model{ViT-L-14}
(\Cref{fig:acc_by_rank_L-14}), observing analogous findings. In fact, the first
shows a decrease of 1.3\% mean accuracy at 3\% of retained TSVs and the second
shows a reduction of 1.1\% mean accuracy at 2\% of maintained TSVs. We refer to
\Cref{fig:rank_analysis_by_dataset} for a breakdown of this analysis.
\begin{figure}[htbp]
  \centering
  \includegraphics[width=\linewidth]{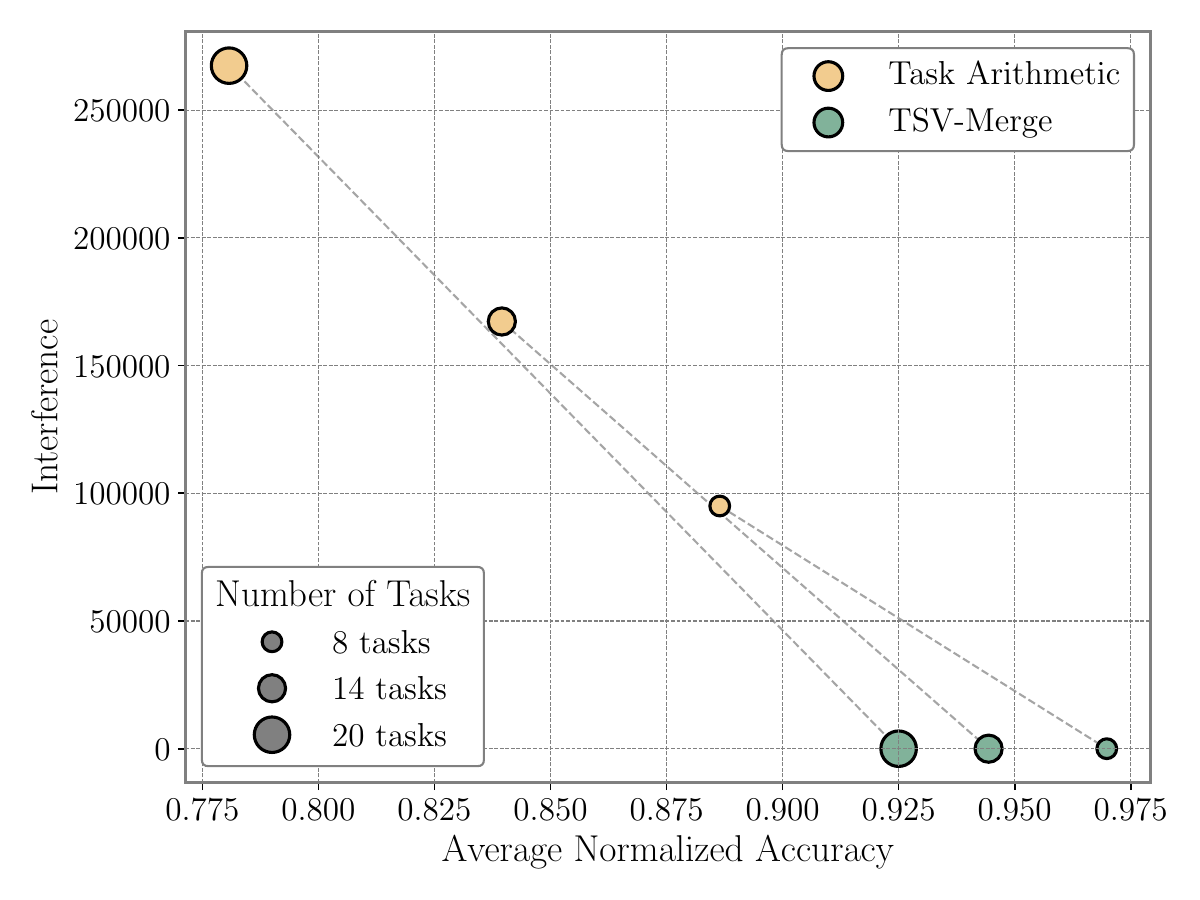}
  \caption[STI and normalized accuracy for \model{ViT-L-14}]{Singular Task Interference (STI) and average normalized accuracy for \baseline{Task Arithmetic} and \method{TSV-Merge} on the \model{ViT-L-14} model, evaluated across merges of 8, 14, and 20 tasks.}
  \label{fig:correlation_interference_accuracy_L-14}
\end{figure}
\subsubsection{Extended ablation study}
In \Cref{subsec:ablation}, we reported an ablation study on the
\model{ViT-B-32} model to evaluate the individual contributions of low-rank
approximation and interference reduction to the overall performance of our
\method{TSV-Merge} method. To further substantiate our findings and demonstrate the
generality of our approach across different model sizes, we report in
\Cref{tab:ablation-study-B-16} the ablation study for the \model{ViT-B-16}
model and in \Cref{tab:ablation-study-L-14} \model{ViT-L-14} model. The
experimental setup follows the one described in \Cref{subsec:ablation}. We
assess the impact of the two key components of \method{TSV-M}, low-rank
approximation and interference reduction, by considering the following four
configurations:
\begin{enumerate}
  \item \textbf{Baseline Task Arithmetic}: the standard TA method without any modification.
  \item \textbf{Low-Rank Approximation}: apply only low-rank approximation to task matrices without any interference reduction step.
  \item \textbf{Interference Reduction}: apply interference reduction to the full-rank task matrices without any pre-step of low-rank approximation.
  \item \textbf{TSV-Merge}: Combining both low-rank approximation and interference reduction.
\end{enumerate}

\subsubsection{Effect of task interference}
We provide here the same plots shown in
\Cref{fig:correlation_interference_accuracy}, for the \model{ViT-B-16} we show
it in \Cref{fig:correlation_interference_accuracy_B-16} and respectively for
\model{ViT-L-14} in \Cref{fig:correlation_interference_accuracy_L-14}. The
finding remains valid for these models as well: in all instances,
a significant gain in accuracy is observed when the interference is removed.

\subsubsection{Detailed per-layer task interference}
We show in \cref{fig:detailed_interference_by_layer} the per-layer task
interference, extending the block-level analysis in
\Cref{fig:interference_by_layer} in the main manuscript.

\subsubsection{Compression analysis}
Our experimental results (see \Cref{tab:task_acc_compression} in main
manuscript) demonstrate that TSV outperforms TALL-Mask on the large-scale
\model{ViT-L-14} model for 14 and 20 tasks benchmarks, signaling a scaling
advantage. With a fixed-budget analysis, we show in \Cref{fig:compression} that
unlike TALL-Mask, which has a fixed requirement defined by model size and
number of tasks, we allow flexible compression rates by allowing rank
selection. This enables more aggressive compression, as highlighted in the
green region in the figure.

\begin{figure}[t]
  \centering
  \includegraphics[width=\linewidth]{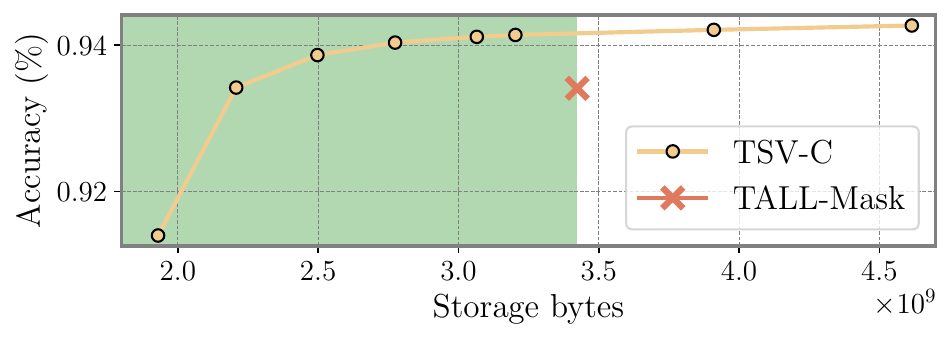}
  \vspace{-0.35cm}
  \caption[Accuracy vs.\ compression budget for \model{ViT-L-14}]{Accuracy with varying compression budgets for \model{ViT-L-14} across 14 tasks.}
  \label{fig:compression}
  \vspace{-0.35cm}
\end{figure}

\subsubsection{Test-time adaptation}
We compare our method with \baseline{AdaMerging}
\citep{yang2023adamerging} for test-time adaptation. On a
subset of 7 tasks from the 8 task benchmark, \baseline{AdaMerging} achieves an
accuracy of 85.43\%, while our \method{TSV-M} attains 88.93\%, an improvement
of approximately 3.5\% without requiring any test-time adaptation.
Additionally, when integrating an \baseline{AdaMerging}-style test-time
adaptation into our framework, the accuracy increases to 89.87\%, demonstrating
the complementary benefits of combining TSV-M with test-time adaptation
techniques.

\section{Theoretical motivations and analysis}

\subsection{Theoretical foundation -- empirical design}
The \method{TSV-C} method is grounded in the well-established framework of
low-rank approximation for compression (e.g., \cite{denton2014exploiting}).
Instead, \method{TSV-M} is motivated by more empirical foundations: it is
designed to achieve noise reduction through low-rank approximation and to
eliminate interference via orthogonalization. Low-rank truncation serves to
filter out insignificant variations, while orthogonalization ensures that
task-specific singular vectors remain independent, preserving individual task
performance.

\subsection{Heuristic interference measure}
Given that a formal definition of interference in model merging is not yet
established, we adopt an operational definition: interference is any cross-task
interaction that hinders the merging process. Our proposed Singular Task
Interference measure is empirically validated by the consistent performance
improvements observed when its value is minimized. Furthermore, we examine the
relationship between overlapping singular vectors and knowledge sharing. Unlike
multi-task learning (MTL), which enables coordinated knowledge sharing through
joint training, the independent task-wise fine-tuning in model merging may
evolve in destructive overlaps in the activations, resulting in interference
rather than beneficial knowledge sharing. By orthogonalizing the singular
vectors, our approach effectively mitigates these overlaps, reducing
interference and enhancing the performance of the merged model.

\begin{figure}
  \centering
  \includegraphics[width=0.90\linewidth, height=0.85\textheight, keepaspectratio]{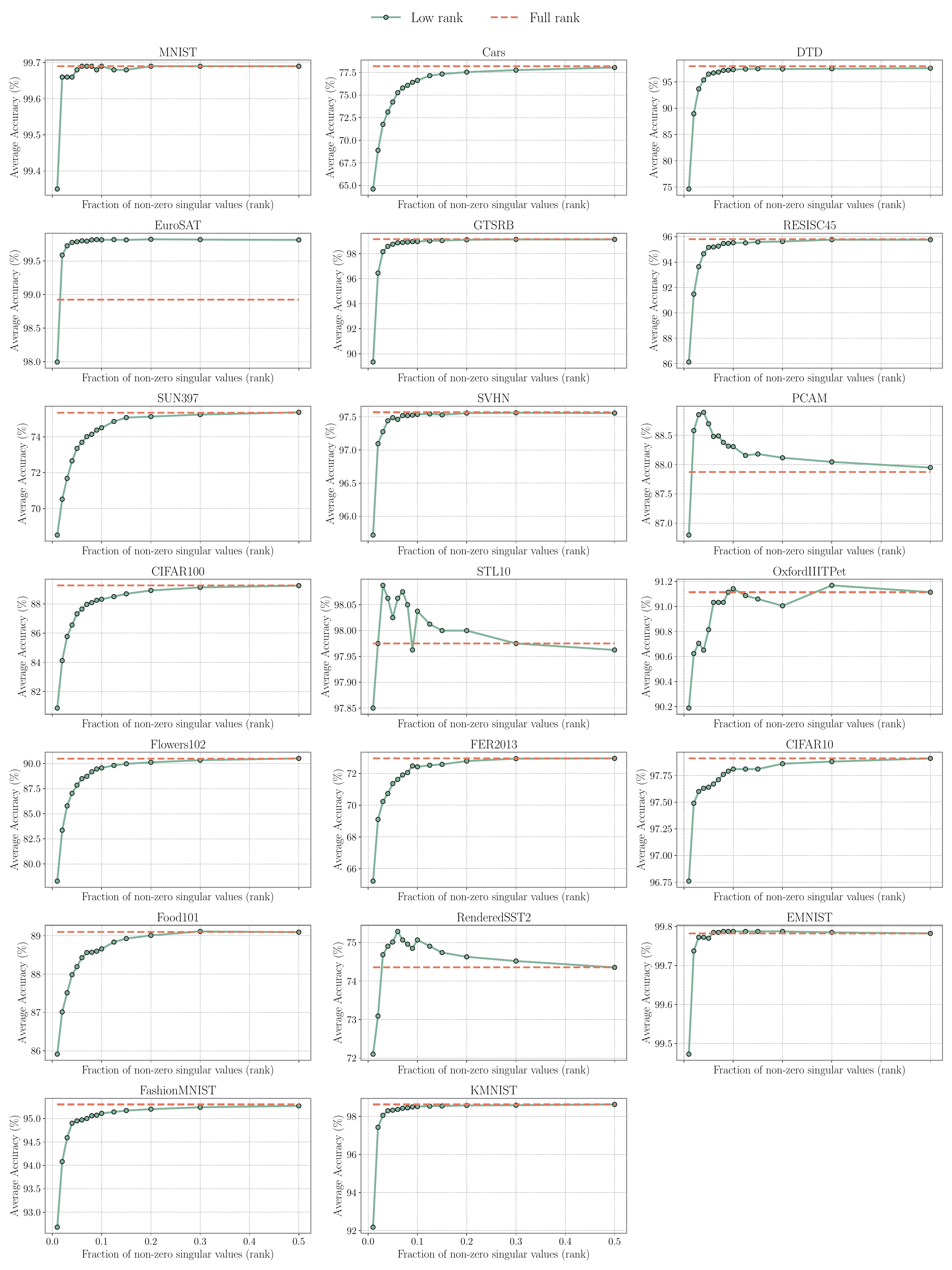}
  \caption[Per-task accuracy of \model{ViT-B-32} by retained singular components]{Absolute accuracy of the \model{ViT-B-32} model across increasing fractions of retained singular components, for each task. The red line represents the accuracy of the original fine-tuned models with full-rank task matrices, while the green line shows the accuracies using low-rank approximations.}
  \label{fig:rank_analysis_by_dataset}
\end{figure}

\begin{figure}
  \centering
  \includegraphics[width=0.97\linewidth, height=0.85\textheight, keepaspectratio]{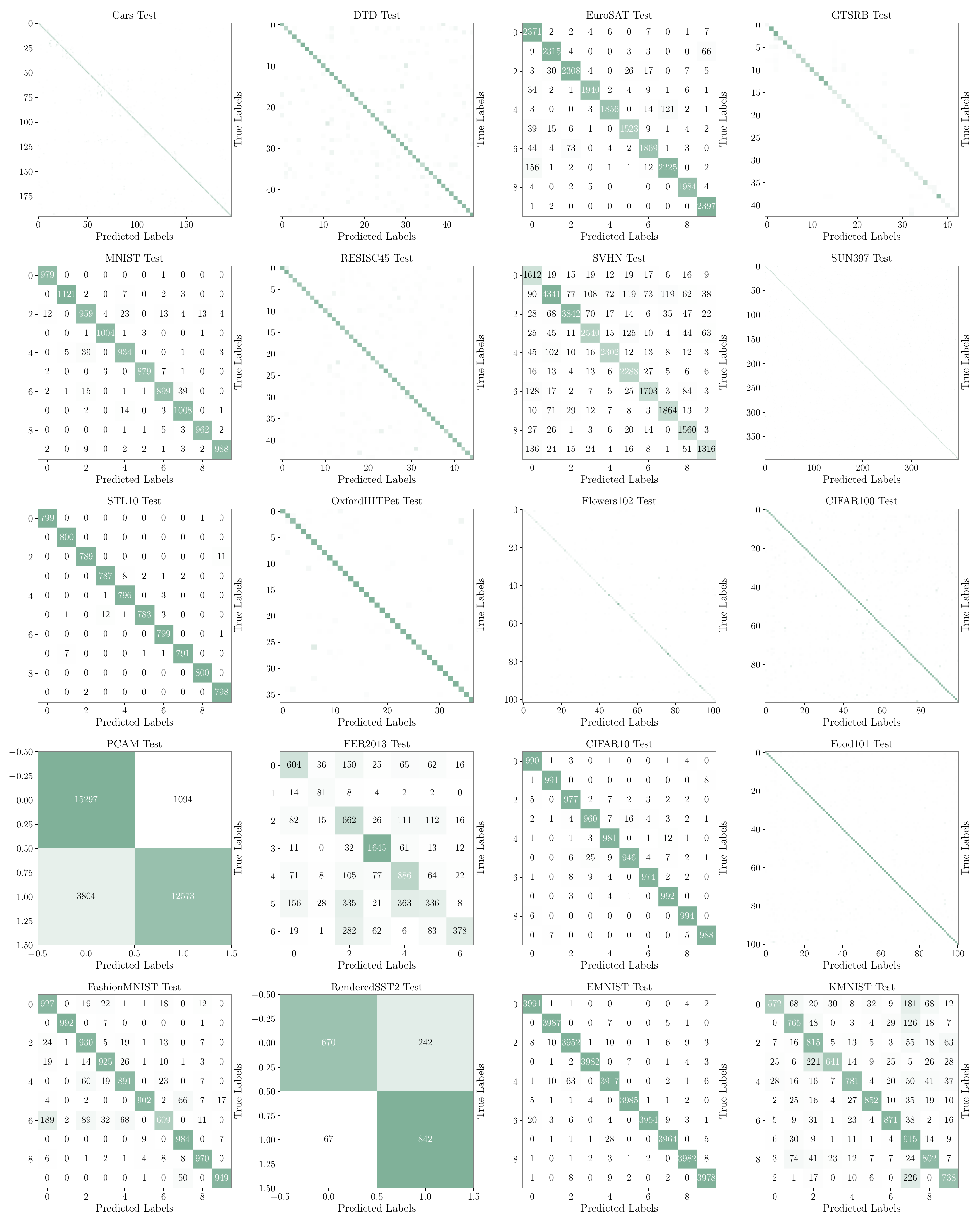}
  \caption[Confusion matrices for merged \model{ViT-L-14} over 20 tasks]{Breakdown of classification accuracy in confusion matrices of a single merged \model{ViT-L-14} model over 20 tasks. The numbers are omitted when they are too small to display.}
  \label{fig:confusion_matrix}
\end{figure}

%% file: 99_appendix/C_TSV/mini_proof.tex
\subsection{\texorpdfstring{Proof of \Cref{theo:error}}{Proof of TSV Truncation Theorem}}
\label{mini_proof}
\begin{theorem*}
Let \(T \in \mathbb{N}\) such that \(T > 4\). Define \( U = [U_1, \dots, U_T] \) as the matrix obtained by concatenating \(T\) orthogonal matrices \( U_i \), each of shape \(n \times n\). Let \( \widehat{U} = [\widehat{U}_1, \dots, \widehat{U}_T] \) be the matrix formed by truncating each \( U_i \) to its first \(k\) columns. Denote by \(X\) and \(\widehat{X}\) the matrices resulting from Procrustes orthonormalization of \(U\) and \(\widehat{U}\), respectively. If \(k \leq n \frac{T - 2\sqrt{T}}{T}\), then
    \[
    \| U - X \|_F \geq \| \widehat{U} - \widehat{X} \|_F \,.
    \]
\end{theorem*}

\begin{proof}
 Let us consider the SVD decomposition of \(U\) and \(\widehat{U}\): \(U = P_U \Sigma_U P_V \) and \(\widehat{U} = R_U \widehat{\Sigma}_U R_V\). 
  \(X\) and \( \widehat{X} \)  obtain as \( X = P_U P_V^\top , \; \widehat{X} = R_U R_V^\top \) respectively 
We first consider the Frobenius norm of \(|| X- U||_F\). 
Notice that the singular values of \(U\) are the square root of the  eigenvalues of \( \Sigma_U = U U^\top \).

\(U U^\top = \sum_{i=1}^N U_i U_i^{\top} = T I_n.  \) As a consequence, the eigenvalues of \(U U^{\top} \) are all equal to \(T\) and the singular values are all equal to $\sqrt{T}$.
\begin{align*}
|| X- U||_F &= ||P_U P_V^\top  - P_U \Sigma_U P_V^\top||_F & \\
&= || P_U (I - \Sigma_U) P_V^\top ||_F  & \\
&= ||I_n - \Sigma_U ||_F & 
\\
&= ||I_n- \sqrt{T} I_n||_F & 
\\
&= \sqrt{n} (\sqrt{T}-1).
\end{align*}

We are now left to compute \( ||\widehat{X}-\widehat{U}||_F \). In this case, we are not able to compute the exact norm without other assumptions, but we can provide an upper bound that gives us a sufficient condition to prove our statement. 
As before 
\begin{align*}
    ||\widehat{X}-\widehat{U}||_F &= ||I_n - \widehat{\Sigma}_U ||_F\\
    &= \sqrt{\sum_{i=1}^n (\hat{\sigma}_i -1)^2 }.
    \end{align*}
    where $\hat{\sigma}$ are the singular values of \( \widehat{U} \).

Notice that 
\( \sum_{i=1}^n \hat{\sigma}^2_i = \text{tr}(\widehat{U} \widehat{U}^{\top}) =  \text{tr}(\widehat{U}^{\top} \widehat{U} ) \) 
and \( \widehat{U}^{\top} \widehat{U}  \) is a \( Tk \times Tk \) matrix with diagonal elements equal to one, so \( \text{tr}(\widehat{U}^{\top} \widehat{U} ) = kT \).


Moreover,  

\begin{align}
  \sum_{i=1}^n \hat{\sigma_i} &=  \sum_{i=1}^n  \sqrt{\lambda_i (\widehat{U} \widehat{U}^{\top})} \\
  & \geq   \sqrt{ \sum_{i=1}^n \lambda_i (\widehat{U} \widehat{U}^{\top})} \\
  &= \sqrt{\text{tr}(\widehat{U} \widehat{U}^{\top})} = \sqrt{T k}  
\end{align}

Notice that this upper bound is tight, indeed the sum of the singular values of $\hat{U}$ must lie within: \( \sqrt{kT} \leq \sum_{i=1}^n \hat{\sigma}_i \leq kT \). 
The minimum \( \sqrt{kT} \) is achieved if all matrices $U_i$ are equals, on the other end, the maximum \( kT \) is achieved if the \(kT \) columns are orthonormal. 

Putting everything together, 
\begin{align}
    || \widehat{X} - \widehat{U} ||_F  &= \sqrt{\sum_{i=1}^n (\hat{\sigma}_i -1)^2} & \\
    &= \sqrt{n + \sum_{i=1}^n \hat{\sigma}^2_i  - 2\sum_{i=1}^n \hat{\sigma}_i } & \\
    & = \sqrt{ n + kT  - 2\sum_{i=1}^n \hat{\sigma}_i } & \\
    & \leq \sqrt{ n + kT -2\sqrt{kT}     }.
\end{align}

So we have to check for what values of $k $ it holds that \(\sqrt{ n + kT - 2 \sqrt{kT}     } \leq \sqrt{n }(\sqrt{T}-1) \).

We have that 
\begin{equation}
    \sqrt{ n + kT - 2 \sqrt{kT}     } \leq  \sqrt{ n + kT    } \leq  \sqrt{n }(\sqrt{T}-1) 
    \label{eq:that_need_to_be_satistfied}
\end{equation}

\cref{eq:that_need_to_be_satistfied} is satisfied if 


\( k \leq n \frac{T- 2 \sqrt{T}}{T}\).
This concludes the proof.
 Since $k$ is a positive number, the inequality is meaningful only for $T > 4$.
\end{proof}



%% file: 99_appendix/E_mass/content.tex
\clearpage

\section{Extended related work} \label{app:mass-extended-related-work}

\subsection{Model merging}
Model Merging has recently gained traction as a computationally efficient alternative to ensembling. Early methods, motivated by linear mode connectivity \citep{linear-mode-connectivity, Entezari2021-me, mirzadeh_linear_2020, Garipov2018-pz}, primarily aligned models trained with different optimization seeds. This was achieved by finding neuron permutations matching these ones before the aggregation~\citep{git-rebasin, repair, model-fusion, zip-it, rebasin-implicit-sinkhorn, navon2023equivariant, horoi_harmony_2024} and \cref{ch:cycle-consistent}. More recent merging approaches aggregate instead multiple fine-tuned models derived from a shared pre-trained backbone~\citep{task-vectors, ties, yu2024language, matena_merging_2022, wortsman_robust_2022, davari2025model, wang2024localizing, zhou2025atmimprovingmodelmerging, ortiz2024task, huang2024emrmerging, daheim2024model, stoicamodel, yang2024representation, tangparameter}, \cref{ch:tsv} and \cref{ch:merge3}. These methods incorporate various strategies to improve the merging, such as finding the optimal combination of task vectors \citep{yang2023adamerging}, mitigating sign disagreement \citep{ties}, randomly dropping a fraction of the updates \citep{yu2024language}, or employing evolutionary strategies \citep{sakana} and \cref{ch:merge3}. The most recent line of work considers task vectors at a layer level, accounting this way for the natural structure of the layers and significantly improving the merging outcome \citep{stoicamodel, Iso-C} and \cref{ch:tsv}. 
Our approach builds upon the latter methodologies by introducing adaptivity through input-driven routing, significantly narrowing the performance gap between the fine-tuned endpoints and their resulting merge.

\subsection{MoErging}
Model Merging with Mixture-of-Experts, often referred to as ``MoErging'' \citep{he2023merging, moerging, jang2023exploring,chronopoulou2023adaptersoup,belofsky2023token,zhao2024loraretriever,muqeeth2024learning,tang2024merging,ostapenko2024towards,cheng2024dam, tang2024smile, twinmerging, sun2025transformer2}, explores how independently trained experts--potentially contributed by a decentralized community--can be combined within a single adaptive model by dynamically selecting which expert(s) should handle a given input. In the LLM domain, specialized modules (e.g., LoRAs or adapters) are merged via parametric or data-driven routers that match the incoming prompt to the most relevant module \citep{jang2023exploring,chronopoulou2023adaptersoup,belofsky2023token,zhao2024loraretriever,muqeeth2024learning,tang2024merging,ostapenko2024towards, cheng2024dam}; some approaches, such as PHATGOOSE~\citep{muqeeth2024learning}, require fine-tuning additional routing parameters, whereas others, like weight-ensembling MoE~\citep{tang2024merging}, may need to train the router at test time. 
Sharing a similar two-pass pipeline, \texttt{Transformer}$^2$ \citep{sun2025transformer2} modifies the fine-tuning procedure to yield task-aligned singular vectors, allowing for expert routing at test time. Differently from the latter, \mass{} works on any independently fine-tuned models, requiring no ad hoc training routines.
A similar strategy was independently introduced by \texttt{SMILE} \citep{tang2024smile}, which, like our method, enables data-free merging of multiple experts. However, \texttt{SMILE} leaves the pre-trained backbone intact and selectively \emph{adds} task-specific low-rank updates at inference, whereas \texttt{MASS} begins with a single model containing all updates and \emph{deactivates} any irrelevant subspaces via a router. The fact that both lines of work arrived at a similar subspace-activation concept, despite differing motivations, highlights the versatility and broad appeal of such an approach.
Lastly, \texttt{TwinMerging}~\citep{twinmerging} similarly merges a shared expert and multiple task-specific experts via a gating function. However, \texttt{TwinMerging} relies on flat task arithmetic \citep{task-vectors} and requires per-task labeled data to train its router, reducing its applicability. In contrast, \mass{} operates in a fully data-free and training-free regime by design.

\section{Additional details}
\label{app:additional_details}
We here describe the details required to implement and reproduce our results. The code is publicly available. In particular, \cref{app:impl} describes implementation details, \cref{sec:mass-normalized_accuracy} specifies the employed evaluation metrics, \cref{app:arch} reports the employed architecture and \cref{app:hyperparams} specifies the hyperparameters and how they were chosen.

\subsection{Storage and compute overhead}\label{subsec:storage-compute-overhead}
The $\sim\!2\times$ \textbf{compute} figure is an approximation: we run the backbone twice (fixed \texttt{TSV-M}~(\cref{ch:tsv}) pass + routed pass) and the router itself adds ${<1\%}$ FLOPs. Since model-merging saves training and storage rather than inference time, this modest overhead is acceptable. \mass{} routes per sample, so each incurs this cost; however, batching over the same task can reduce it. For example, with 32 samples, routing costs only $1.06\times$ forward equivalents.
The \textbf{storage overhead} is exactly $2\times$: alongside the merged model, we store $1/T$ TSVs per task. Summed across tasks and layers, these form a second full set of weights, effectively doubling storage.
A full comparison with the overhead incurred by the baselines is provided in \cref{tab:parameter-count}.

\subsection{Implementation}\label{app:impl}
We used the same model checkpoints as \texttt{Consensus TA} \citep{wang2024localizing}, except for the one for the \dataset{EMNIST} dataset which we had to re-finetune due to an inconsistency in image orientation between \dataset{EMNIST} and \dataset{MNIST}. Shortly, the \emph{torchvision}\footnote{\href{https://pytorch.org/vision}{https://pytorch.org/vision/stable/index.html}} version yields rotated and flipped images, spuriously yielding extremely similar models (same classes, roughly same dataset statistics) that performed very poorly when interchanged. Simply re-rotating and flipping the \dataset{EMNIST} images to match the orientation of \dataset{MNIST} solves the issue.
We further used a single classification head for \dataset{STL10} and \dataset{CIFAR10} due to their 9 shared classes. The final head has the shared classes plus the two dataset-specific ones, \emph{i.e.} monkey and frog.

\subsection{Benchmarks and datasets}\label{app:datasets}
The 8-task benchmark, introduced in \citep{task-vectors}, comprises the following datasets: \dataset{Cars}, \dataset{DTD}, \dataset{EuroSAT}, \dataset{GTSRB}, \dataset{MNIST}, \dataset{RESISC45}, \dataset{SUN397}, and \dataset{SVHN}. Moving to 14 tasks, we add \dataset{CIFAR100}, \dataset{STL10}, \dataset{Flowers102}, \dataset{OxfordIIITPet}, \dataset{PCAM}, and \dataset{FER2013}. The 20-task suite further includes \dataset{EMNIST}, \dataset{CIFAR10}, \dataset{Food101}, \dataset{FashionMNIST}, \dataset{RenderedSST2}, and \dataset{KMNIST}. We provide the specific dataset details in \cref{tab:dataset_image_sizes}.

\input{tables/parameters}

\begin{table}[t]
  \centering
  \resizebox{0.85\linewidth}{!}{%
  \begin{tabular}{cc ccc}
    \toprule
    Dataset                 & image size       & \# train  & \# val   & \# test  \\
    \midrule
    \dataset{Cars}~\citep{krause_3d_2013}           & varies           & 7330   & 814   & 8041  \\
    \dataset{DTD}~\citep{cimpoi_describing_2014}           & varies           & 1692   & 188   & 1880  \\
    \dataset{EuroSAT}~\citep{helber_eurosat_2019}       & $64 \times 64$   & 21600  & 2700 & 2700 \\
    \dataset{GTSRB}~\citep{stallkamp_german_2011}         & varies           & 23976  & 2664  & 12630 \\
    \dataset{MNIST}~\citep{726791}         & $28 \times 28$   & 55000  & 5000  & 10000 \\
    \dataset{RESISC45}~\citep{cheng_remote_2017}      & $256 \times 256$ & 17010  & 1890  & 6300  \\
    \dataset{SUN397}~\citep{xiao_sun_2016}        & varies           & 17865  & 1985  & 19850 \\
    \dataset{SVHN}~\citep{netzer_reading_nodate}          & $32 \times 32$   & 68257  & 5000  & 26032 \\
    \dataset{CIFAR100}~\citep{CIFAR}      & $32 \times 32$   & 45000  & 5000  & 10000 \\
    \dataset{STL10}~\citep{coates_analysis_2011}         & $96 \times 96$   & 4500   & 500   & 8000  \\
    \dataset{Flowers102}~\citep{nilsback_automated_2008}    & varies           & 918    & 102   & 6149  \\
    \dataset{OxfordIIITPet}~\citep{parkhi_cats_2012} & varies           & 3312   & 368   & 3669  \\
    \dataset{PCAM}~\citep{veeling_rotation_2018}          & $96 \times 96$   & 257144 & 5000  & 32768 \\
    \dataset{FER2013}~\citep{goodfellow_challenges_2013}       & $48 \times 48$   & 25839  & 2870  & 7178  \\
    \dataset{EMNIST}~\citep{EMNIST}        & $28 \times 28$   & 235000 & 5000  & 40000 \\
    \dataset{CIFAR10}~\citep{CIFAR}       & $32 \times 32$   & 45000  & 5000  & 10000 \\
    \dataset{Food101}~\citep{bossard_food-101_2014}       & $512 \times 512$ & 70750  & 5000  & 25250 \\
    \dataset{FashionMNIST}~\citep{xiao_fashion-mnist_2017}  & $28 \times 28$   & 55000  & 5000  & 10000 \\
    \dataset{RenderedSST2}~\citep{socher_recursive_nodate}  & varies           & 6228   & 692   & 1821  \\
    \dataset{KMNIST}~\citep{clanuwat_deep_2018}        & $28 \times 28$   & 55000  & 5000  & 10000 \\
    \bottomrule
  \end{tabular}
  }
  \caption[Dataset image sizes and sample counts]{Image sizes, and numbers of train, validation, and test samples for the considered datasets.}
  \label{tab:dataset_image_sizes}
\end{table}

\subsection{Evaluation measures} \label{sec:mass-normalized_accuracy}
To account for differences in task difficulty, we report both \textit{absolute} and \textit{normalized} accuracy in our results. The normalized accuracy serves as a relative performance measure by comparing the multi-task model's accuracy to that of individual fine-tuned models. It is computed as:
\begin{equation}
  \text{Normalized Accuracy} = \frac{1}{\ntasks} \sum_{i=1}^{\ntasks} \frac{\text{accuracy}(\theta_{\text{MT}}, t_i)}{\text{accuracy}(\theta_{ft_i}, t_i)}
\end{equation}
where $\ntasks$ represents the total number of tasks, $\theta_{\text{MT}}$ is the multi-task model, and $\theta_{ft_i}$ corresponds to the fine-tuned model for task $t_i$. By normalizing accuracy in this way, we ensure a fairer comparison that accounts for variations in baseline task performance.

\subsection{Architectures} \label{app:arch}
We use CLIP from the \emph{OpenClip} library\footnote{\href{https://github.com/mlfoundations/open\_clip/}{https://github.com/mlfoundations/open\_clip/}}, using the three different versions described in \cref{tab:vit_comparison}. For the router, we use a small two-layer MLP with a hidden dimension of $1024$. It accepts a $512$-dimensional embedding vector, applies a linear transformation, a ReLU activation, and dropout with a probability of $0.5$, and outputs logits corresponding to task selection probabilities. 

\input{tables/vits}

\subsection{Adapting oracle MoErging methods} \label{subsec:adapting-moerging-methods}
As extensively discussed, we argue that MoErging methods should not rely on an oracle head, since their pipelines implicitly assume that the task label is unknown at inference time (otherwise, merging the correct task vector would be trivial). To respect this assumption, we introduced a routing procedure that exploits the routers already implemented in each method. For SMILE, we extracted the mode of the tokens at each layer and applied a naïve majority-voting scheme across layers to select the head for each sample, which was then used for the final classification. Analogously, for WeMoE, we applied the same logic, implementing a straightforward head-selection strategy that leverages the existing gates and coefficients. All results reported for MoE methods were obtained under this unified setting.
\begin{figure}
  \centering
  \includegraphics[width=\textwidth]{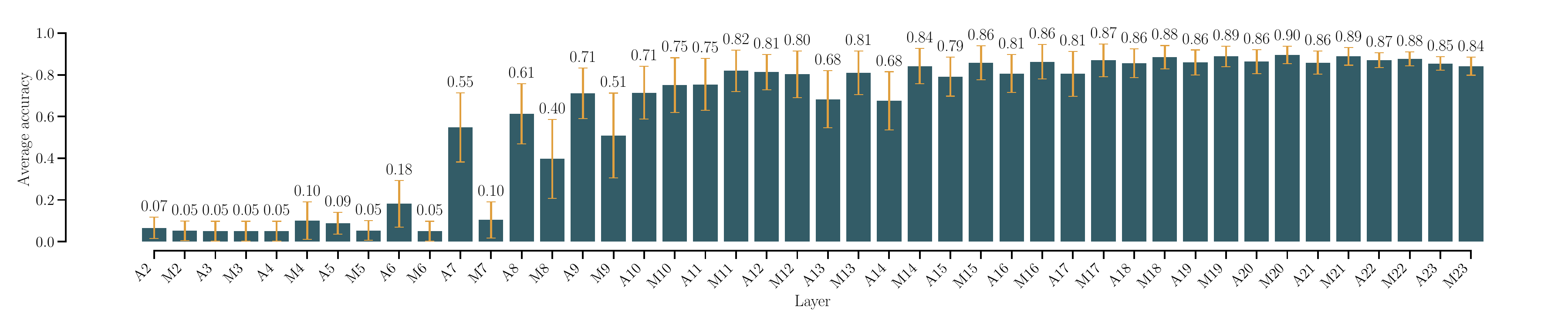}
  \caption{\model{ViT-L-14} per-layer task accuracies.}
  \label{fig:layer-task-accuracies-l14}
\end{figure}

\input{algorithms/fixed}
\paragraph{Radar charts on 8- and 14-task benchmarks.}
In \Cref{sec:mass-experiments}, we present comprehensive results for the approach using the 20 task benchmark. \Cref{fig:radar-charts-8} and \Cref{fig:radar-charts-14} respectively display the normalized accuracies for our method on 8 and 14 tasks across all three model sizes (\model{ViT-B-32}, \model{ViT-B-16}, and \model{ViT-L-14}). In both cases, the approach retains a high fraction of each fine-tuned model's performance, with normalized accuracies often above 80--90\%. Notably, the method scales gracefully as the number of tasks increases from 8 to 14.

\section{Additional experiments and results}
In this section, we provide detailed per-task and per-layer accuracy plots, along with further examples of decoded task vectors, to complement the results presented in the main paper.

\subsection{Hyperparameter settings} \label{app:hyperparams}
Following the recommendation of \cref{ch:tsv}, we use $\alpha$ as a single scaling factor with the suggested value of 1.0 for the TSV-M merging configurations in both \cref{alg:fixed-merging} and \cref{alg:adaptive-merging}. Consistent with TSV, the compression rate assigned to each task space is set to $\frac{1}{T}$. We optimized the similarity threshold $\varepsilon$ over the range \{0.1, 0.2, ..., 0.9\} and determined the router's selection threshold $\eta$ via a Bayesian search within the interval [0.05, 0.5]. As illustrated in \Cref{fig:layer-task-accuracies-b32b16,fig:layer-task-accuracies-l14}, we identify the optimal projection layer by focusing on those revealing the highest task accuracy. Specifically, for \model{ViT-B-32} and \model{ViT-B-16} models, we select the attention and MLP layers within the range \{7, ..., 11\}, while for the \model{ViT-L-14} model, the chosen layers fall in the range \{19, ..., 23\}. The temperature parameter for tuning the behavior of the softmax function at line 6 in \cref{alg:adaptive-merging} is set to 1. 

\paragraph{Sensitivity to Hyperparameters}
\begin{wrapfigure}[17]{r}{0.45\linewidth}
\vspace{-0.9cm}
  \centering
  \setlength{\abovecaptionskip}{2pt}
  \includegraphics[width=\linewidth]{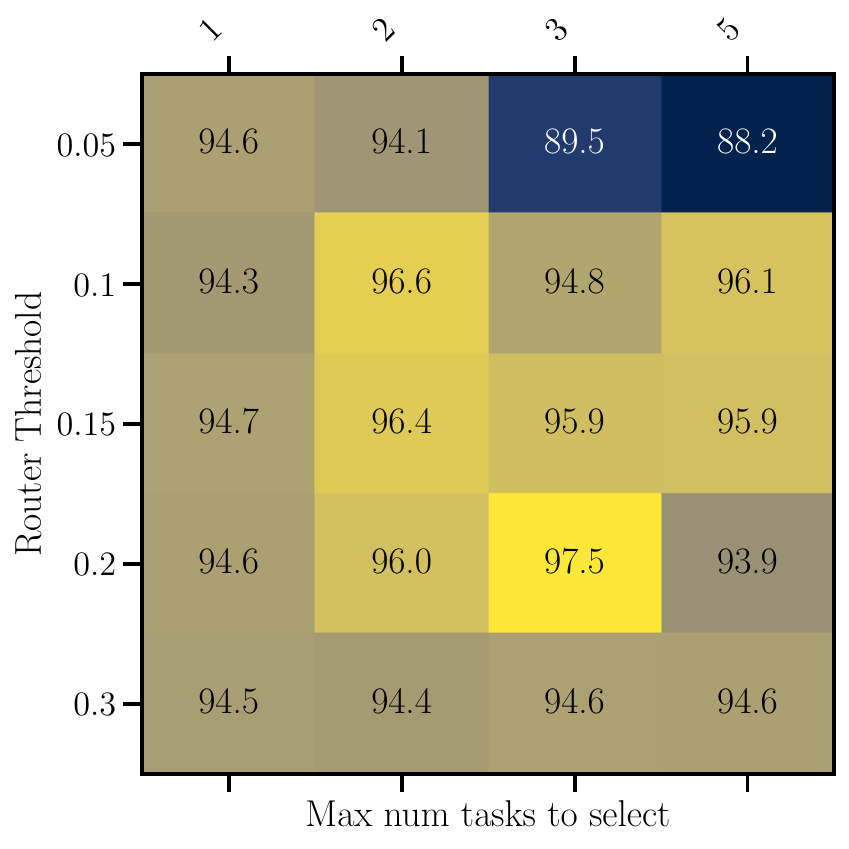}
  \caption{Hyperparameter sensitivity.}
  \label{fig:hparam-sensitivity}
\end{wrapfigure}
The merging coefficient is set to $1$ as in \texttt{TSV-M}~(\cref{ch:tsv}). ~\Cref{fig:hparam-sensitivity} shows how accuracy changes with routing threshold $\eta$ and top-$K$. At low $\eta$ ($0.05$) and large $K$, too many tasks are merged, causing interference. At high $\eta$ ({$\geq$ 0.3}), only the top task is selected, making performance insensitive to $K$ (it is effectively an argmax). The best accuracy occurs in a broad middle range, peaking at ${\eta=0.2}$, ${K = 3}$, where the router balances selectivity and coverage. 
We also vary the cosine threshold $\varepsilon$ used to discard similar task updates before merging. Due to the high dimensionality of the $\Delta$s, large thresholds (${\varepsilon \geq 0.4}$) retain all updates, leaving redundancy unaddressed (accuracy $93.5$).
Small ones (${\varepsilon \leq 0.05}$) instead remove even distinct directions, significantly harming accuracy (${\leq 88.6}$). Intermediate values ($\varepsilon \approx 0.2$) offer a robust filtering, improving performance (${\geq 93.9}$) by suppressing redundancy.

\begin{figure}
  \centering
  \includegraphics[width=\textwidth]{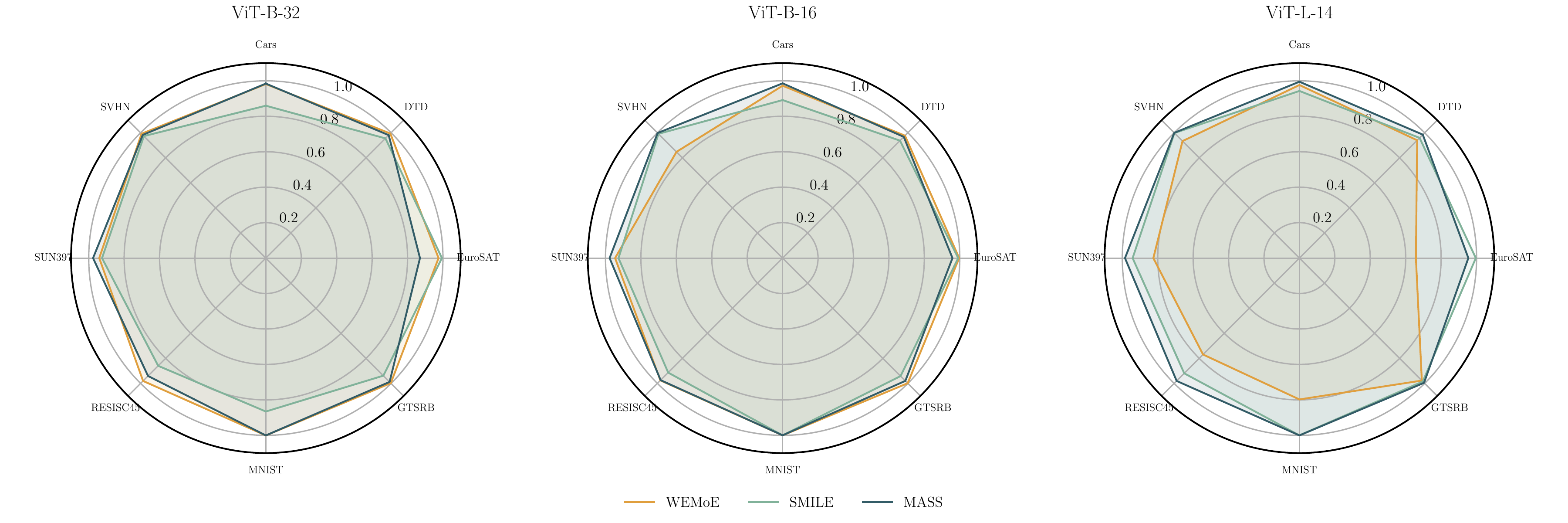}
  \caption[Normalized task accuracies for the 8 tasks benchmark]{Normalized task accuracies over models \model{ViT-B-32}, \model{ViT-B-16} and \model{ViT-L-14} for the 8 tasks benchmark.}
  \label{fig:radar-charts-8}
\end{figure}
\begin{figure}
  \centering
  \includegraphics[width=\textwidth]{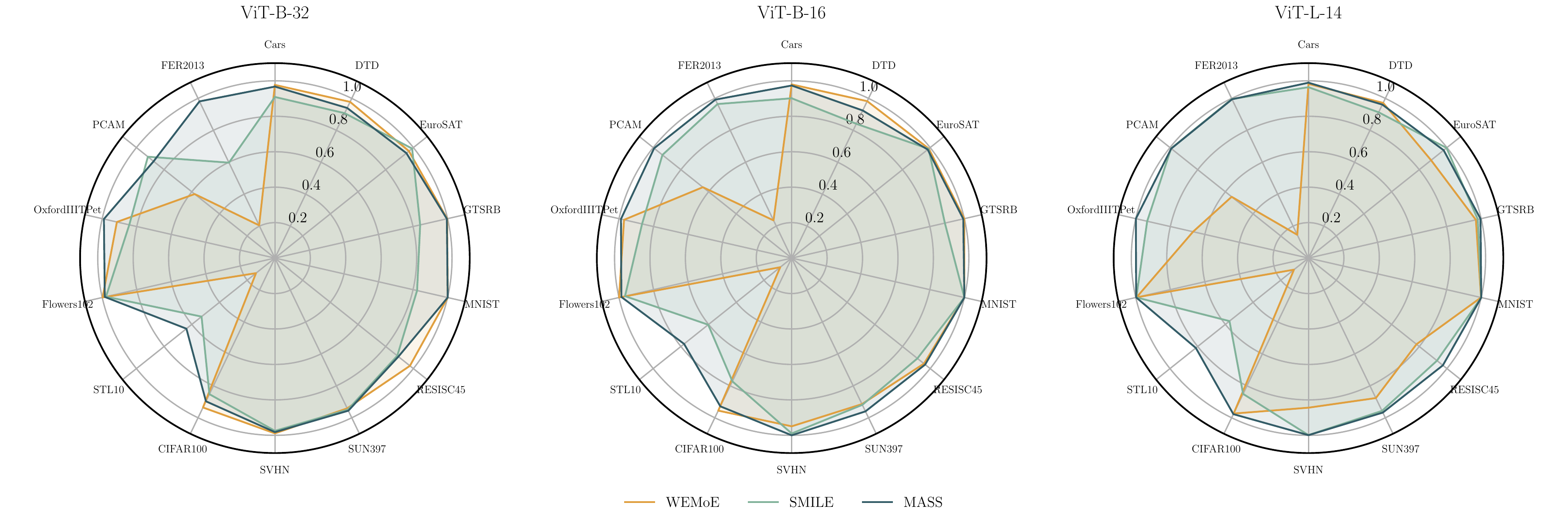}
  \caption[Normalized task accuracies for the 14 tasks benchmark]{Normalized task accuracies over models \model{ViT-B-32}, \model{ViT-B-16} and \model{ViT-L-14} for the 14 tasks benchmark.}
  \label{fig:radar-charts-14}
\end{figure}
\begin{figure}[htbp]
  \begin{subfigure}{0.48\textwidth}
    \centering
    \includegraphics[width=\textwidth]{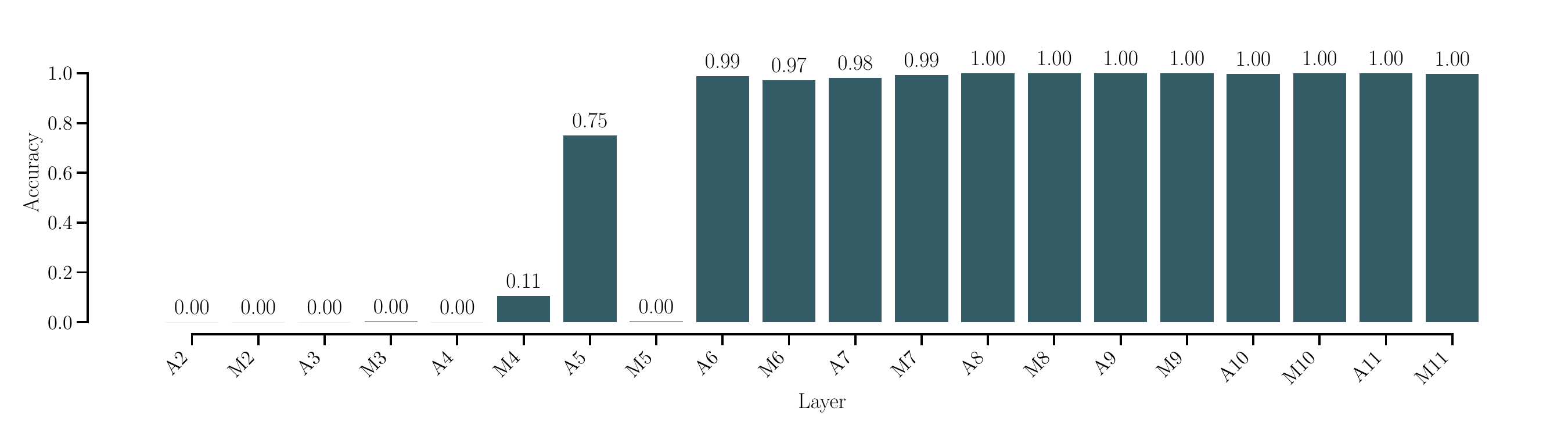}
    \caption{\dataset{CARS}.}
    \label{fig:layer-task-accuracies-cars}
  \end{subfigure}
  \hfill
  \begin{subfigure}{0.48\textwidth}
    \centering
    \includegraphics[width=\textwidth]{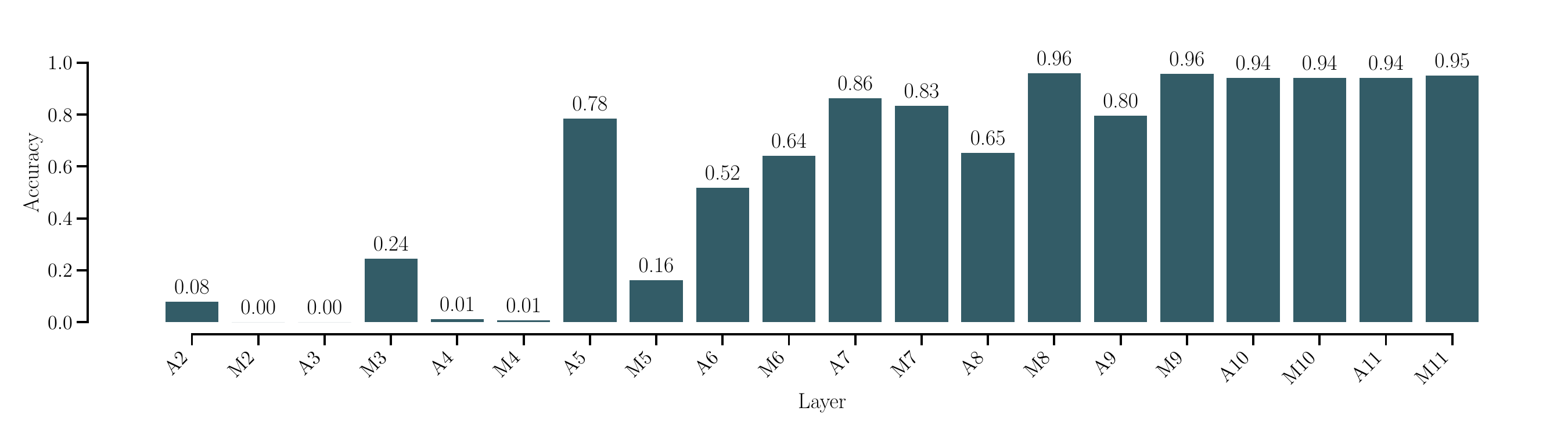}
    \caption{\dataset{DTD}.}
    \label{fig:layer-task-accuracies-dtd}
  \end{subfigure}
  \hfill
  \begin{subfigure}{0.48\textwidth}
    \centering
    \includegraphics[width=\textwidth]{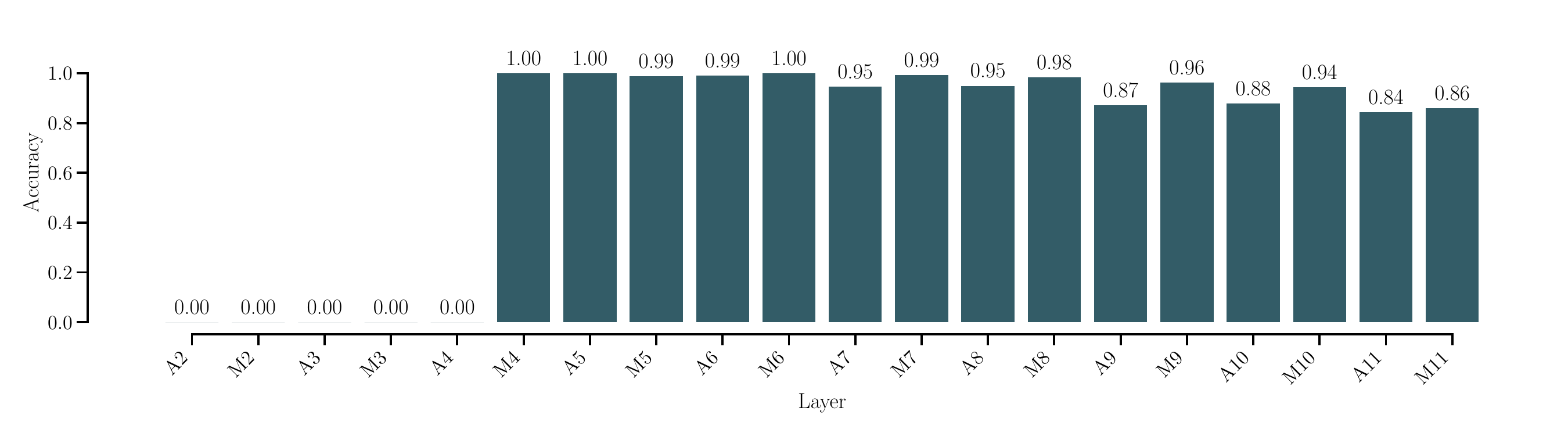}
    \caption{\dataset{EuroSAT}.}
    \label{fig:layer-task-accuracies-eurosat}
  \end{subfigure}
  \hfill
  \begin{subfigure}{0.48\textwidth}
    \centering
    \includegraphics[width=\textwidth]{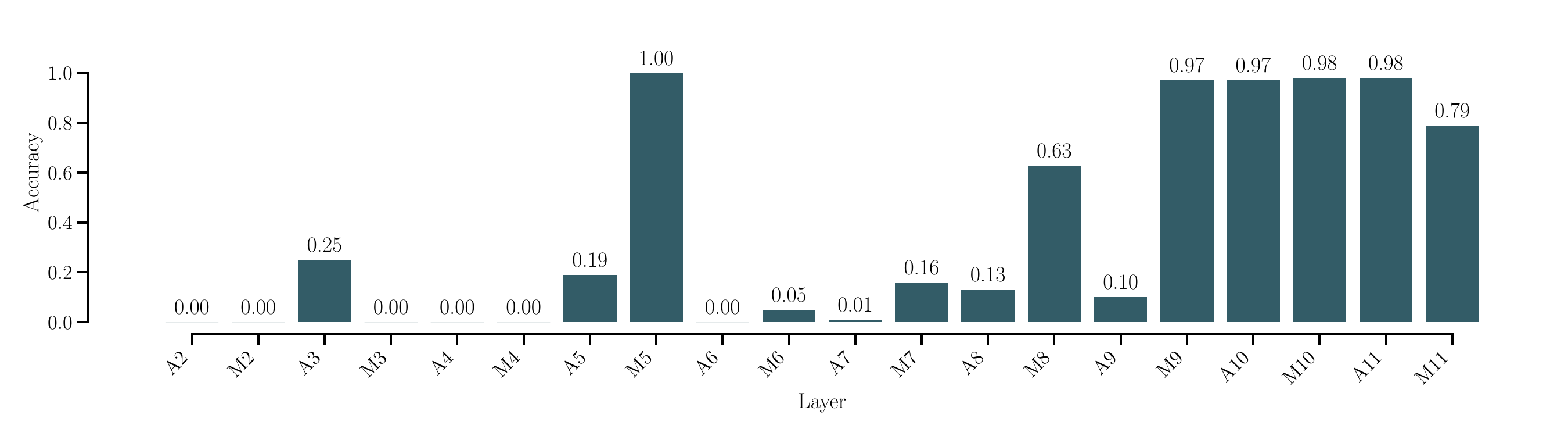}
    \caption{\dataset{GTSRB}.}
    \label{fig:layer-task-accuracies-gtsrb}
  \end{subfigure}
  \caption[Per-layer task accuracies for \model{ViT-B-32} (first 4 datasets)]{Per-layer task accuracies for \model{ViT-B-32} on \dataset{Cars}, \dataset{DTD}, \dataset{EuroSAT}, and \dataset{GTSRB}.}
  \label{fig:layer-task-accuracies-first-4-datasets}
\end{figure}
\begin{figure}
  \begin{subfigure}{0.48\textwidth}
    \centering
    \includegraphics[width=\textwidth]{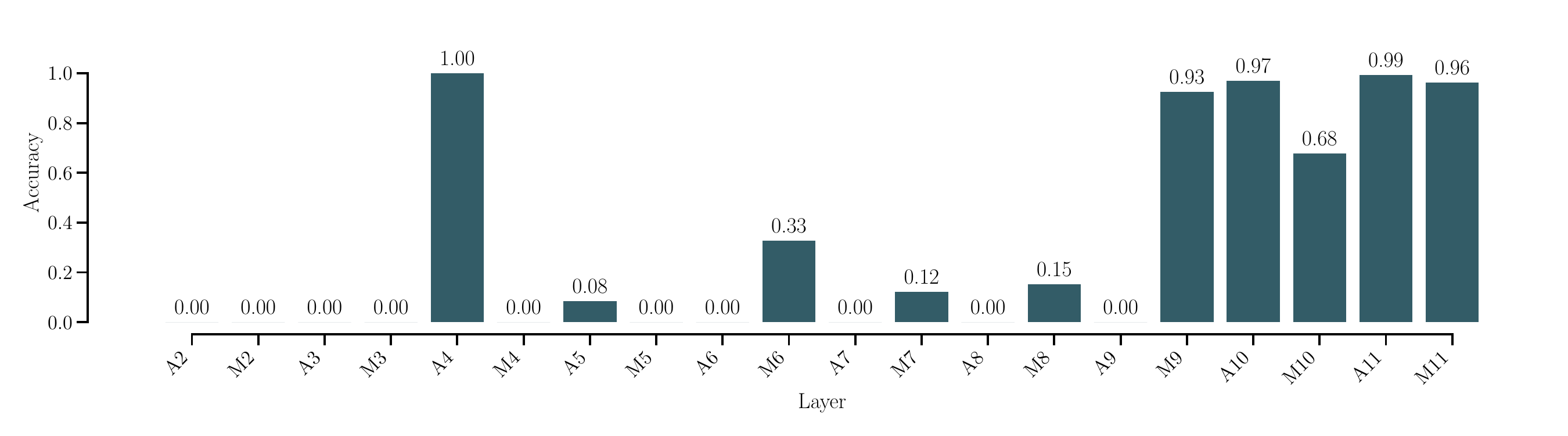}
    \caption{\dataset{MNIST}.}
    \label{fig:layer-task-accuracies-mnist}
  \end{subfigure}
  \hfill
  \begin{subfigure}{0.48\textwidth}
    \centering
    \includegraphics[width=\textwidth]{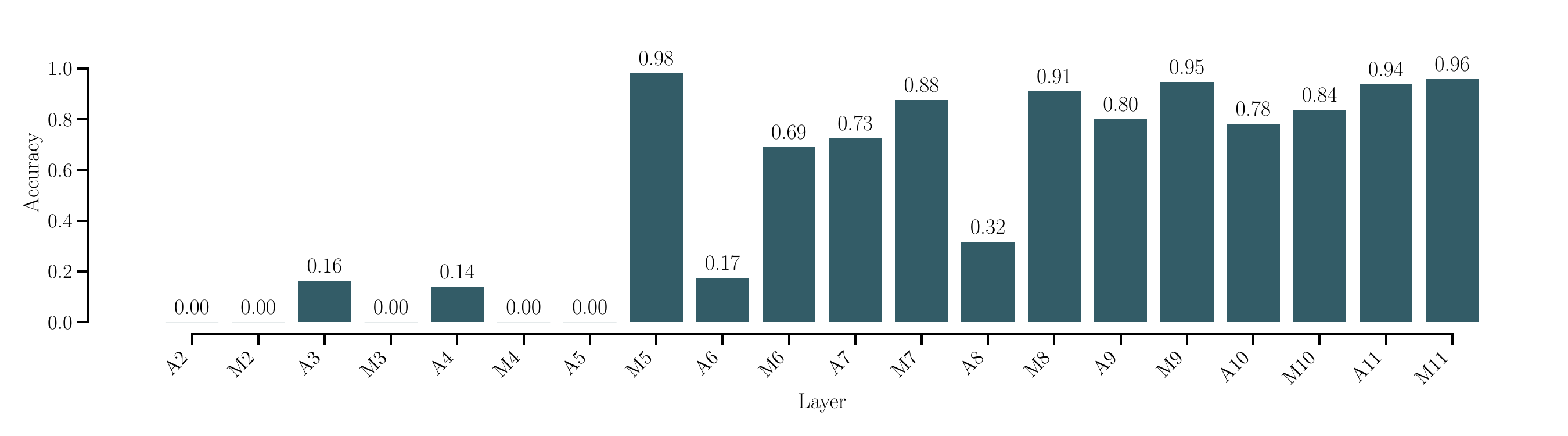}
    \caption{\dataset{RESISC45}.}
    \label{fig:layer-task-accuracies-resisc45}
  \end{subfigure}
  \hfill
  \begin{subfigure}{0.48\textwidth}
    \centering
    \includegraphics[width=\textwidth]{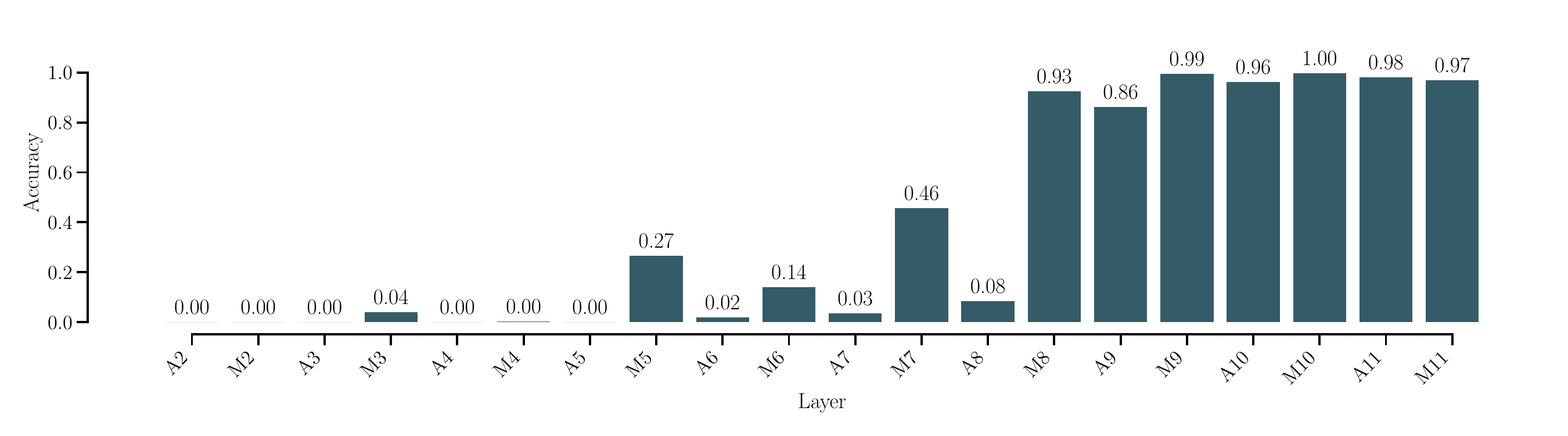}
    \caption{\dataset{SUN397}.}
    \label{fig:layer-task-accuracies-sun397}
  \end{subfigure}
  \hfill
  \begin{subfigure}{0.48\textwidth}
    \centering
    \includegraphics[width=\textwidth]{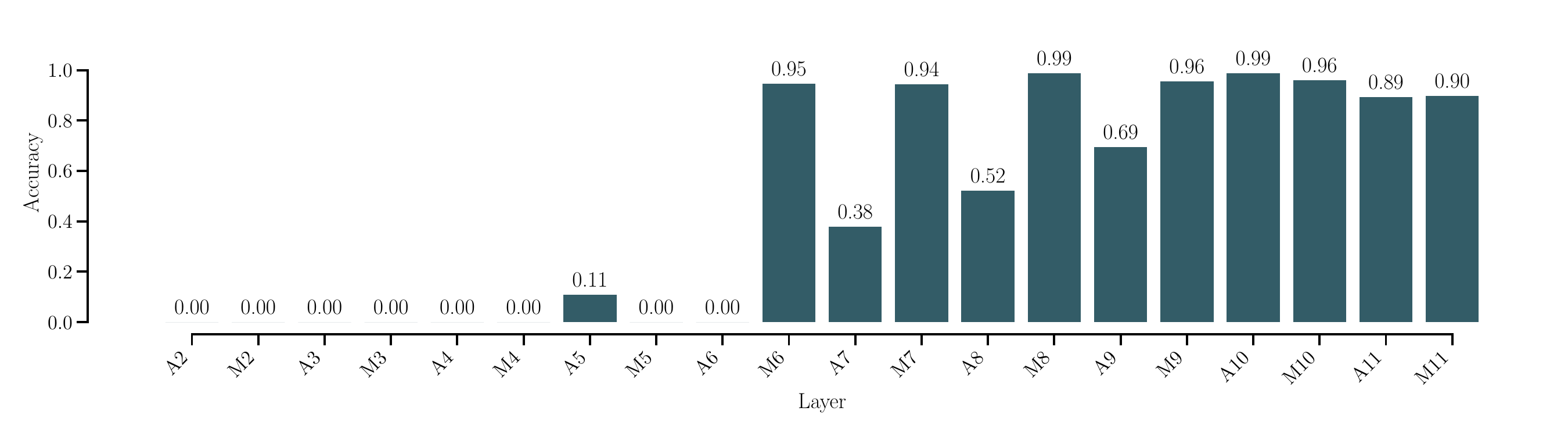}
    \caption{\dataset{SVHN}.}
    \label{fig:layer-task-accuracies-svhn}
  \end{subfigure}
  \caption[Per-layer task accuracies for \model{ViT-B-32} (second 4 datasets)]{Per-layer task accuracies for \model{ViT-B-32} on \dataset{MNIST}, \dataset{RESISC45}, \dataset{SUN397}, and \dataset{SVHN}.}
  \label{fig:layer-task-accuracies-second-4-datasets}
\end{figure}
\paragraph{Layer-wise accuracies for individual datasets.}
Figures~\ref{fig:layer-task-accuracies-first-4-datasets} and \ref{fig:layer-task-accuracies-second-4-datasets} show per-layer accuracies for \model{ViT-B-32} on different subsets of the 8-task benchmark:
\begin{itemize}
\item \Cref{fig:layer-task-accuracies-first-4-datasets} focuses on \dataset{Cars}, \dataset{DTD}, \dataset{EuroSAT}, and \dataset{GTSRB}.
\item \Cref{fig:layer-task-accuracies-second-4-datasets} displays results for \dataset{MNIST}, \dataset{RESISC45}, \dataset{SUN397}, and \dataset{SVHN}.
\end{itemize}
Again, the top-performing layer is not shared across the tasks, confirming what we observed in the main paper.

\paragraph{Layer-wise accuracies for \model{ViT-L-14}.}
\Cref{fig:layer-task-accuracies-l14} reports average per-layer task accuracy for the larger \model{ViT-L-14}, showing that in this case the most predictive layer for routing is $\ell=20$. Since \model{ViT-L-14} has 24 layers (compared to the 12 layers in \model{ViT-B-32} and \model{ViT-B-16}), the most predictive layer is roughly at the same relative depth. 

\paragraph{Visualizations for additional datasets.}
Following the approach in \cref{subsec:exp-decoding-text}, \Cref{fig:decoding-exp-other-datasets} shows examples of decoded task vectors for datasets like \dataset{SVHN}, \dataset{GTSRB}, \dataset{SUN397}, and \dataset{RESISC45}. Here, we see textual prompts such as ``\emph{An image of the number 4}'' or ``\emph{Aerial view of an industrial area}'', which align with each dataset's distinct domain. This reaffirms that our singular vectors capture domain-specific transformations while preserving high-level semantic alignment to the pre-trained model.
\begin{figure}[htbp]
  \centering
  \begin{subfigure}[t]{0.47\textwidth}
    \centering
    \textsc{``An image of the number 4''}\\[3pt]
    \begin{tikzpicture}
      \node[draw=myblue!90,rounded corners=3pt,line width=0.4pt]
      {\includegraphics[width=0.27\linewidth]{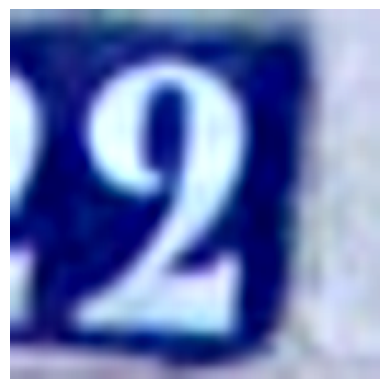}};
    \end{tikzpicture}
    \hspace{2pt}
    \begin{tikzpicture}
      \node[draw=myblue!90,rounded corners=3pt,line width=0.4pt]
      {\includegraphics[width=0.27\linewidth]{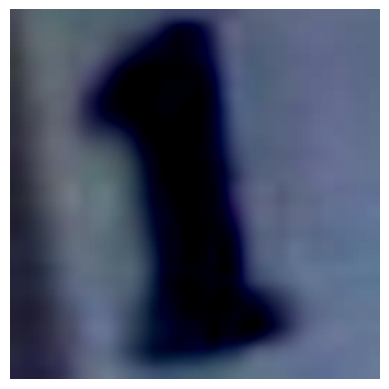}};
    \end{tikzpicture}
    \hspace{2pt}
    \begin{tikzpicture}
      \node[draw=myblue!90,rounded corners=3pt,line width=0.4pt]
      {\includegraphics[width=0.27\linewidth]{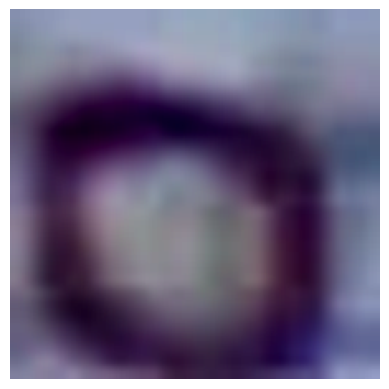}};
    \end{tikzpicture}
    \caption{\dataset{SVHN}}
  \end{subfigure}
  \hfill
  \begin{subfigure}[t]{0.47\textwidth}
    \centering
    \textsc{``A badge''}\\[3pt]
    \begin{tikzpicture}
      \node[draw=myblue!90,rounded corners=3pt,line width=0.4pt]
      {\includegraphics[width=0.27\linewidth]{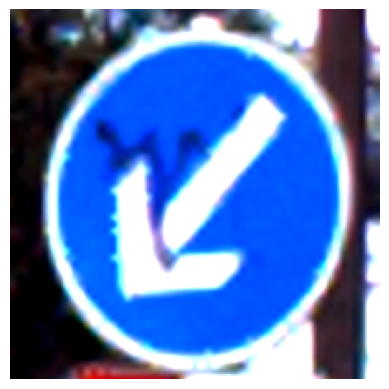}};
    \end{tikzpicture}
    \hspace{2pt}
    \begin{tikzpicture}
      \node[draw=myblue!90,rounded corners=3pt,line width=0.4pt]
      {\includegraphics[width=0.27\linewidth]{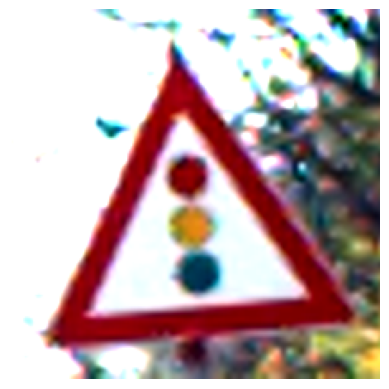}};
    \end{tikzpicture}
    \hspace{2pt}
    \begin{tikzpicture}
      \node[draw=myblue!90,rounded corners=3pt,line width=0.4pt]
      {\includegraphics[width=0.27\linewidth]{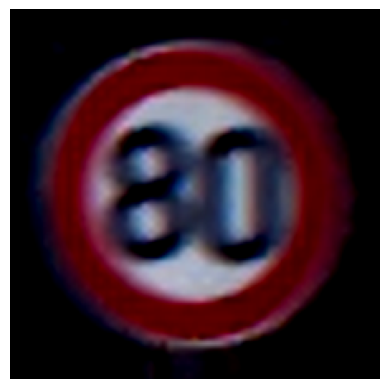}};
    \end{tikzpicture}
    \caption{\dataset{GTSRB}}
  \end{subfigure}

  \vspace{0.5cm}

  \begin{subfigure}[t]{0.47\textwidth}
    \centering
    \textsc{``An image of an interior of a room''}\\[3pt]
    \begin{tikzpicture}
      \node[draw=myblue!90,rounded corners=3pt,line width=0.4pt]
      {\includegraphics[width=0.27\linewidth]{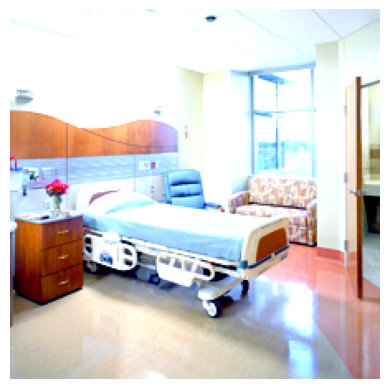}};
    \end{tikzpicture}
    \hspace{2pt}
    \begin{tikzpicture}
      \node[draw=myblue!90,rounded corners=3pt,line width=0.4pt]
      {\includegraphics[width=0.27\linewidth]{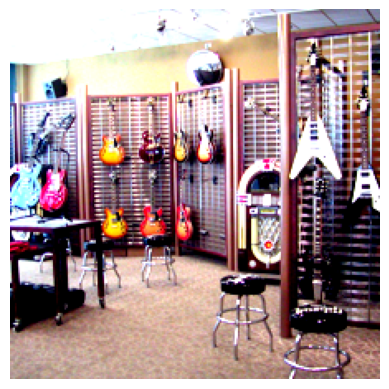}};
    \end{tikzpicture}
    \hspace{2pt}
    \begin{tikzpicture}
      \node[draw=myblue!90,rounded corners=3pt,line width=0.4pt]
      {\includegraphics[width=0.27\linewidth]{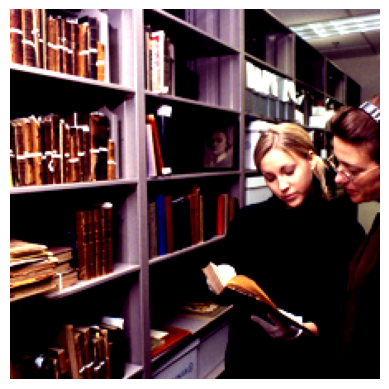}};
    \end{tikzpicture}
    \caption{\dataset{SUN397}}
  \end{subfigure}
  \hfill
  \begin{subfigure}[t]{0.47\textwidth}
    \centering
    \textsc{``Aerial view of an industrial area''}\\[3pt]
    \begin{tikzpicture}
      \node[draw=myblue!90,rounded corners=3pt,line width=0.4pt]
      {\includegraphics[width=0.27\linewidth]{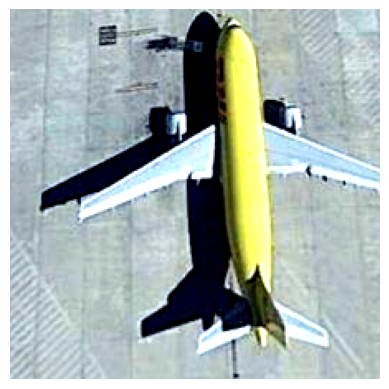}};
    \end{tikzpicture}
    \hspace{2pt}
    \begin{tikzpicture}
      \node[draw=myblue!90,rounded corners=3pt,line width=0.4pt]
      {\includegraphics[width=0.27\linewidth]{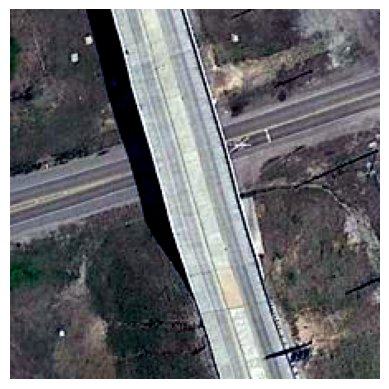}};
    \end{tikzpicture}
    \hspace{2pt}
    \begin{tikzpicture}
      \node[draw=myblue!90,rounded corners=3pt,line width=0.4pt]
      {\includegraphics[width=0.27\linewidth]{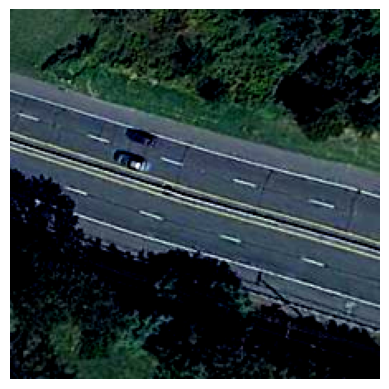}};
    \end{tikzpicture}
    \caption{\dataset{RESISC45}}
  \end{subfigure}

  \caption[Decoded task singular vectors for additional datasets]{Captions obtained by decoding task singular vectors as text for datasets \dataset{SVHN}, \dataset{GTSRB}, \dataset{SUN397}, and \dataset{RESISC45} as described in \cref{subsec:exp-decoding-text}, accompanied by three representative images for each dataset.}
  \label{fig:decoding-exp-other-datasets}
\end{figure}

\begin{table}
    \centering
    \begin{tabular}{c ccc ccc}
        \toprule
        \mass{}                      & \multicolumn{3}{c}{\model{ViT-B-32}} & \multicolumn{3}{c}{\model{ViT-B-16}}                                            \\
        \cmidrule(lr){2-4} \cmidrule(lr){5-7}
        +                                & 8 tasks                              & 14 tasks                             & 20 tasks & 8 tasks & 14 tasks & 20 tasks \\
        \midrule
        \method{nn}                      & 92.7                                 & 89.6                                 & 89.4     & 92.7    & 90.0     & 90.2     \\
        \method{mlp}                     & 96.8                                 & 95.8                                 & 95.8     & 97.1    & 94.8     & 96.5     \\
        \cdashlinelr{1-7}
        \method{proj}$_{\texttt{PRE}}$   & 98.2                                 & 88.6                                 & 79.4     & 98.7    & 92.0     & 81.2     \\ \rowcolor{mygreen!50}
        \method{proj}$_{\texttt{TSV-M}}$ & $96.5$                               & $93.2$                               & $90.9$   & $98.0$  & $96.1$   & $88.7$   \\
        \bottomrule
    \end{tabular}
    \caption{Average normalized accuracy for different routers.}
    \label{tab:remaining-routers}
\end{table}

\section{Propositions and proofs}
\begin{proposition}[Optimality of Orthogonal Projection]
  \label{prop:optimal_projection}
  Let $V \in \mathbb{R}^{d \times k}$ have orthonormal columns spanning a subspace
  $\mathcal{S} \subseteq \mathbb{R}^d$, and let $\mathbf{a} \in \mathbb{R}^d$.
  Then the unique minimizer of $\|\mathbf{a} - \mathbf{w}\|_2^2$ over all
  $\mathbf{w} \in \mathcal{S}$ is
  \[
    \widehat{\mathbf{w}}
    \;=\;
    V\,V^\top\,\mathbf{a}.
  \]
\end{proposition}
\begin{proof}
  Any $\mathbf{w} \in \mathcal{S}$ can be written as $V\,\boldsymbol{\alpha}$ for some
  $\boldsymbol{\alpha} \in \mathbb{R}^k$. The problem
  \[
    \min_{\mathbf{w} \,\in\, \mathcal{S}}
    \|\mathbf{a} - \mathbf{w}\|_2^2
    \quad\Longleftrightarrow\quad
    \min_{\boldsymbol{\alpha} \,\in\, \mathbb{R}^k}
    \|\mathbf{a} - V\,\boldsymbol{\alpha}\|_2^2
  \]
  has a strictly convex objective, so its global minimizer is found by setting the gradient
  to zero. A short calculation shows
  \[
    \boldsymbol{\alpha}
    \;=\;
    V^\top \mathbf{a}
    \quad\Longrightarrow\quad
    \widehat{\mathbf{w}}
    \;=\;
    V\,(V^\top \mathbf{a})
    \;=\;
    V\,V^\top\,\mathbf{a}.
  \]
  Uniqueness follows from the strict convexity, and $\|\mathbf{a} - \widehat{\mathbf{w}}\|_2$
  is necessarily the smallest possible distance in $\mathcal{S}$. Equivalently,
  $\mathbf{a} - \widehat{\mathbf{w}}$ is orthogonal to $\mathcal{S}$, so no further
  reduction in norm is possible.
\end{proof}

  

\begin{proposition}[\S \ref{thm:MAP_l2}]  
Let \(z_\ell \in \mathbb{R}^d\) be a feature vector, and for each task \(i\), decompose it as  
\[
  z_\ell = V_i V_i^{\top} z_\ell + \varepsilon_i,
  \qquad
  \varepsilon_i = \bigl(I - V_i V_i^{\top}\bigr) z_\ell .
\]  
Assume \(\varepsilon_i \sim \mathcal{N}(0, \sigma^2 I)\). Then the maximum a posteriori estimate of the task reduces to  
\[
  \hat{\imath}_{\mathrm{MAP}}
  = \arg\max_i p(\text{task}=i \mid z_\ell)
  = \arg\min_i \|\varepsilon_i\|_2^2 .
\]  
Thus, under these assumptions, selecting the task with the smallest squared Euclidean residual is exactly equivalent to maximizing the posterior.  
\end{proposition} 
\begin{proof}
By assumption,
\(
  p(\varepsilon_i)
  = (2\pi\sigma^2)^{-d/2}
    \exp\!\Bigl(-\|\varepsilon_i\|_2^2/(2\sigma^2)\Bigr),
\)
so \(-\log p(\varepsilon_i)\propto \|\varepsilon_i\|_2^2\).  Since
\(V_iV_i^\top z_\ell\) is a deterministic shift, the likelihood
\(p(z_\ell\mid \text{task}=i)\) depends only on \(\varepsilon_i\).  With a uniform
prior over tasks,
\[
  \hat{\imath}_{\mathrm{MAP}}
  = \arg\max_i p(z_\ell\mid \text{task}=i)
  = \arg\max_i p(\varepsilon_i)
  = \arg\min_i \|\varepsilon_i\|_2^2.
\]
Hence, minimizing the \(\ell_2\) residual is exactly equivalent to maximizing the posterior.
\end{proof}

%% file: 99_appendix/E_mass/tables/parameters.tex

\renewcommand{\arraystretch}{1.4}
\begin{table}
    \centering
    \resizebox{0.9\textwidth}{!}{
        \begin{tabular}{cc ccc ccc ccc}
            %
            %
            \toprule
             & \multirow{2}{*}{Method}     & \multicolumn{3}{c}{\model{ViT-B-32}} & \multicolumn{3}{c}{\model{ViT-B-16}} & \multicolumn{3}{c}{\model{ViT-L-14}}                                                                 \\
            \cmidrule(lr){3-5} \cmidrule(lr){6-8} \cmidrule(lr){9-11}
             &                                      & 8 tasks                              & 14 tasks                             & 20 tasks                             & 8 tasks & 14 tasks & 20 tasks & 8 tasks & 14 tasks & 20 tasks \\

            \cmidrule(r){2-2} \cmidrule(lr){3-3}\cmidrule(lr){4-4} \cmidrule(lr){5-5} \cmidrule(lr){6-6} \cmidrule(lr){7-7} \cmidrule(lr){8-8} \cmidrule(lr){9-9} \cmidrule(lr){10-10} \cmidrule(lr){11-11}
            %
            %
            \multirow{4}{*}{\small \vertical{MoE}}
             & \multicolumn{1}{c}{\method{WeMoE}}   & 5.06                                  & 8.06                                 & 11.05                                 &  5.12      & 8.16       & 11.20       & 5.78     &  9.31     &  12.84     \\
             & \multicolumn{1}{c}{\method{SMILE-1}} & \textbf{1.61}                                 & 2.07                                 & 2.52                                 & \textbf{1.62}    & 2.09     & 2.55     &\textbf{ 1.47}    & \textbf{ 1.82}     & 2.18     \\
             & \multicolumn{1}{c}{\method{SMILE-2}} & 3.05                                 & 4.60                                 & 6.14                                 & 3.09    & 4.67     & 6.24     & 2.59    & 3.78     & 4.96     \\
            \rowcolor{mygreen!50}
             \cellcolor{white} & \textbf{\mass{}}                 & 2.00                                 & \textbf{2.00}                                & \textbf{2.00}                                 & 2.00    &\textbf{ 2.00}     & \textbf{2.00}     & 2.00    &2.00     & \textbf{2.00}     \\
            \bottomrule
        \end{tabular}
    }
    \caption[Relative parameter increase per method]{Relative parameter increase with respect to the base model.}
    \vspace{-0.45cm}
    \label{tab:parameter-count}
\end{table}

%% file: 99_appendix/E_mass/tables/vits.tex
\begin{table}[htbp]
  \centering
    \begin{tabular}{cccccc}
      \toprule
      \textbf{Model}   & \textbf{Layers} & \textbf{Hidden Dimension} & \textbf{Heads} & \textbf{Patch Size} & \textbf{Parameters} \\
      \midrule
      \model{ViT-B-32} & 12              & 768                 & 12             & 32×32               & $\sim$86M       \\
      \model{ViT-B-16} & 12              & 768                 & 12             & 16×16               & $\sim$86M       \\
      \model{ViT-L-14} & 24              & 1024                & 16             & 14×14               & $\sim$307M      \\
      \bottomrule
    \end{tabular}
  \caption[ViT architecture comparison]{Comparison of \model{ViT-B-32}, \model{ViT-B-16}, and \model{ViT-L-14} architectures.}
  \label{tab:vit_comparison}
\end{table}

%% file: 99_appendix/F_merge3/content.tex

\section{Mergenetic}
\input{E_Library/content}

\section{Additional Details}
\input{A_Details/content}

\section{Additional Experiments}
\input{B_Additional_Experiments/content}

\section{Mathematical proofs}\label{app:proofs}
\input{D_Proofs/content}

%% file: 99_appendix/F_merge3/E_Library/content.tex
\label{app:mergenetic}
Each experiment was run using a library developed specifically for this work, which will be released as open-source software, called \textit{Mergenetic} \cite{minut2025mergeneticsimpleevolutionarymodel}. This library allows for defining a merging problem as either a single-objective or multi-objective optimization problem, where one only needs to specify the merging method, a fitness function, and select the hyperparameters for a chosen evolutionary algorithm. 

The implementation relies on \textit{Mergekit} \cite{goddard2025arceesmergekittoolkitmerging} for merging the models, \textit{Pymoo} \cite{Blank_2020} for optimizing the objective function through evolutionary algorithms, and \textit{Lm-Evaluation-Harness} \cite{eval-harness} for implementing some of the fitness functions. In \cref{tab:merge_methods} we outline the supported merging methods, while in \cref{tab:algorithms} we outline the currently available evolutionary algorithms.

We believe this library is a significant contribution as it facilitates evolutionary model merging and aligns well with the chapter's approach, which aims to reduce computational burden. It can be a valuable tool for the community and for users interested in cross-lingual transfer or creating multilingual models for target low-resource languages.

\begin{table}
    \centering
    \caption[Supported merging methods in Mergenetic]{Overview of supported merging methods in Mergenetic.}
    \vspace{10pt}
    \begin{tabular}{lcc}
    \toprule
    \textbf{Method} & \textbf{Multi-Model} & \textbf{Uses Base Model} \\
    \midrule
    Linear (Model Soups) & Yes & No \\
    SLERP & No & Yes \\
    Task Arithmetic & Yes & Yes \\
    TIES & Yes & Yes \\
    DARE (TIES) & Yes & Yes \\
    DARE (Task Arithmetic) & Yes & Yes \\
    \bottomrule
    \end{tabular}
    \label{tab:merge_methods}
\end{table}

\begin{table}
    \centering
    \caption[Supported evolutionary algorithms in Mergenetic]{Overview of supported Pymoo's evolutionary algorithms in Mergenetic.}
    \vspace{10pt}
    \begin{tabular}{lccc}
    \toprule
    \textbf{Algorithm} & \textbf{Class} & \textbf{Objective(s)} & \textbf{Constraints} \\ 
    \midrule
    Genetic Algorithm  & GA & single & x \\ 
    Differential Evolution & DE & single & x \\ 
    Biased Random Key GA & BRKGA & single & x \\ 
    Nelder Mead & NelderMead & single & x \\ 
    Pattern Search & PatternSearch & single & x \\ 
    CMAES & CMAES & single &  \\ 
    Evolutionary Strategy & ES & single &  \\ 
    SRES & SRES & single & x \\ 
    ISRES & ISRES & single & x \\ 
    NSGA-II & NSGA2 & multi & x \\ 
    R-NSGA-II & RNSGA2 & multi & x \\ 
    NSGA-III & NSGA3 & many & x \\ 
    U-NSGA-III & UNSGA3 & many & x \\ 
    R-NSGA-III & RNSGA3 & many & x \\ 
    MOEAD & MOEAD & many &  \\ 
    AGE-MOEA & AGEMOEA & many &  \\ 
    C-TAEA & CTAEA & many & x \\ 
    SMS-EMOA & SMS-EMOA & many & x \\ 
    RVEA & RVEA & many & x \\ 
    \bottomrule
    \end{tabular}
    \label{tab:algorithms}
\end{table}

%% file: 99_appendix/F_merge3/A_Details/content.tex
\label{app:add-details}
This section provides additional implementation and experimental details that were not included in the main paper.

\subsection{IRT Fitting Details} \label{app:fitting_details}
As previously stated, we used the implementation from \citet{tinybenchmarks} and adopted their configuration settings. Specifically, we used $\gamma_m \sim N(\mu_{\gamma}\ones_d, 1/u_{\gamma}I_d)$, $a_{i} \sim N(\mu_{a}\ones_d, 1/u_{a}I_d)$, and $\beta_{i} \sim N(\mu_\beta, 1/u_\beta)$.
Following \citet{tinybenchmarks}, we also applied (hyper)priors to the prior parameters using the software for fitting hierarchical Bayesian models \citep{lalor2023py}: $\mu_{\gamma} \sim N(0, 10)$, $u_{\gamma} \sim \Gamma(1, 1)$, $\mu_{a} \sim N(0, 10)$, $u_{a} \sim \Gamma(1, 1)$, $\mu_\beta \sim N(0, 10)$, and $u_\beta \sim \Gamma(1, 1)$. 
For both the model and example-specific parameters $\gamma_m$, $a_{i}$, and $\beta_{i}$, we take their point estimates as the means of their respective variational distributions. 
The $\gamma$ model dimensionality is set to $15$ following the parameter choice suggested by \citet{tinybenchmarks}.

\begin{table}[ht]
\centering
\caption[Mistral-based models used in experiments]{Mistral-based models with shortened column headers and names. Role can be either E, M or B, referring to endpoint, merge or base model respectively. Spec refers instead to specialization, with \texttt{mth}, \texttt{ger}, \texttt{ita}, \texttt{jpn}, \texttt{dut} and \texttt{gen} referring to Math, German, Italian, Japanese, Dutch and General respectively. We finally have the author and model ID as per the Huggingface.}
\label{tab:short-mistral-models}
\vspace{10pt}
\resizebox{0.7\textwidth}{!}{%
\begin{tabular}{ccll}
\toprule
\textbf{Role} & \textbf{Spec} & \textbf{Author} & \textbf{Model} \\
\midrule
E & \texttt{mth} & \emph{upaya07 }      & \model{Arithmo2-Mistral-7B}     \\
E & \texttt{mth,jpn} & \emph{SakanaAI }      & \model{EvoLLM-JP-v1-7B}     \\
E & \texttt{mth} & \emph{GAIR}          & \model{Abel-7B-002}      \\
E & \texttt{mth} & \emph{meta-math}     & \model{MetaMath-Mistral-7B}     \\
B & \texttt{gen}  &\emph{ mistralai}     & \model{Mistral-7B-v0.1}      \\
E & \texttt{ger} &\emph{ jphme  }          & \model{em\_german\_mistral\_v01}   \\
E & \texttt{ger} & \emph{LeoLM  }       & \model{leo-mistral-hessianai-7b}       \\
E & \texttt{ita} & \emph{DeepMount00}   & \model{Mistral-Ita-7b}       \\
E & \texttt{jpn} & \emph{augmxnt  }     & \model{shisa-gamma-7b-v1}       \\
E & \texttt{dut} & \emph{BramVanroy}    & \model{GEITje-7B-ultra}     \\
E & \texttt{ro} & \emph{OpenLLM-Ro}    & \model{RoMistral-7b-Instruct}     \\
M & \texttt{gen} &\emph{ chlee10 }      & \model{T3Q-Merge-Mistral7B}    \\
E & \texttt{gen} &\emph{ liminerity}    & \model{M7-7b}        \\
E & \texttt{gen} & \emph{yam-peleg}     & \model{Experiment26-7B}     \\
M & \texttt{gen} & \emph{PracticeLLM}   & \model{SOLAR-tail-10.7B-Merge-v1.0}    \\
E & \texttt{gen} &\emph{ upstage}       & \model{SOLAR-10.7B-v1.0}      \\
E & \texttt{gen} & \emph{Yhyu13}        & \model{LMCocktail-10.7B-v1}   \\
M & \texttt{gen} & \emph{FuseAI}        & \model{FuseChat-7B-Slerp}  \\
M & \texttt{gen} & \emph{FuseAI}        & \model{FuseChat-7B-TA}   \\
E & \texttt{gen} & \emph{FuseAI}        & \model{OpenChat-3.5-7B-Mixtral}  \\
E & \texttt{gen} & \emph{FuseAI}        & \model{OpenChat-3.5-7B-Solar}  \\
M & \texttt{gen} & \emph{jan-hq}        & \model{supermario-slerp-v3}\\
E & \texttt{gen} & \emph{jan-hq}        & \model{supermario-slerp-v2}\\
E & \texttt{gen} & \emph{jan-hq}        & \model{supermario-v2}\\
M & \texttt{gen} & \emph{superlazycoder}& \model{NeuralPipe-7B-slerp}   \\
E & \texttt{gen} & \emph{OpenPipe}      & \model{mistral-ft-optimized-1218}     \\
E & \texttt{gen} & \emph{mlabonne}      & \model{NeuralHermes-2.5-Mistral-7B} \\
\bottomrule
\end{tabular}%
}
\end{table}

\subsection{Experimental Details}
\label{app:add-details-evo}

\paragraph{Models.}
One key assumption of model merging is that the endpoint models lie within the same basin \cite{task-vectors}. This means that merging arbitrary models is not feasible; rather, \textbf{all models involved must be fine-tuned versions of the same base model}. To satisfy this requirement, we selected several fine-tuned models from the Hugging Face Hub that originated from the same base model. Specifically, we focused on models fine-tuned starting from \model{Mistral-7B} \cite{mistral}, following common best practices in the community \cite{sakana}.
\Cref{tab:short-mistral-models} lists all the models used in our experiments, along with their corresponding names on the Hugging Face Hub. A total of $27$ models were considered for our experiments.

\paragraph{Number of models.}
The initialization step, shared across all evolutionary algorithms, involves sampling merging configurations (e.g., interpolation coefficients) and applying them to merge the endpoint models. Consequently,~\approach requires the same number of models as any standard merging approach.

In the rest of the section, we provide further details for reproducing the experiments in \cref{sec:cross-lingual} and \cref{sec:multi-lingual} of the main paper.

\subsubsection{Cross-lingual transfer}

In the cross-lingual transfer evolutionary merging, we evolved four merged models with mathematical capabilities in different languages: Japanese, Romanian, German, and Dutch. In each of these experiments, we deployed an ad-hoc genetic algorithm for single-objective optimization. We employed the Simulated Binary Crossover~\cite{sbx} operator to generate offspring solutions by combining parent solutions. To maintain diversity and explore the search space, we applied Polynomial Mutation~\cite{sbx}, which introduces small perturbations to offspring solutions and enhances the algorithm's ability to escape local optima. This combination of SBX and PM effectively balances exploration and exploitation, facilitating efficient convergence toward optimal solutions.

Furthermore, guided by empirical tests, we decided to deploy \method{SLERP} to evolve solutions for the Romanian and Dutch problems, while we used a combination of \method{TIES} and \method{DARE} for the Japanese and the German ones. We deployed four different sizes of the fitness datasets for Japanese, namely 20, 30, 50, and 100, in order to obtain a more detailed analysis of the method for comparison with the work of~\cite{sakana}. On the other hand, we kept the fitness dataset size fixed to 20 for all other aforementioned experiments. The fitness dataset was extracted from the test set of \dataset{GSM8K}, and we used the remaining, non-overlapping samples as test set for evaluating the model. To get the language-specific versions of \dataset{GSM8K}, we used \model{Unbabel/TowerInstruct-7B-v0.2} \cite{alves2024tower} to translate the datasets. In each experiment, the population size was fixed to 25 and the number of iterations to 7.

To check the correctness of the solution, following \citet{sakana}, we used a regex to extract the last numerical value returned in the model's answer and compare it with the ground truth. The solution is also checked to be in the correct language with the language identifier from fastText~\cite{joulin2017bag}.

The mathematical models used in combination with \method{TIES} and \method{DARE} were \model{Abel-7B-002} and \model{Arithmo2-Mistral-7B}, whereas we used \model{MetaMath-Mistral-7B} in combination with \method{SLERP}. Moreover, we employed the following language-specialized models: \model{shisa-gamma-7b-v1}, \model{em\_german\_mistral\_v01}, \model{GEITje-7B-ultra}, and \model{RoMistral-7b-Instruct}. More information about these models can be found in \cref{tab:short-mistral-models}.

Lastly, we evaluated \model{EvoLLM-JP-v1-7B} \cite{sakana} under the same conditions as \approach{} to assess its accuracy, following the prompting structure outlined by \citet{sakana}.

\subsubsection{Multi-lingual transfer}

In this experiment, we tackle the ARC dataset in multiple languages (Italian, Dutch, German, and English)\footnote{We used the dataset on the Hugging Face Hub from openGPT-X/arcx} \cite{thellmann2024crosslingual} using a multi-objective evolutionary merging procedure based this time on NSGA-II \cite{nsga-ii}. We configure the population size to 25 and the number of evolutionary iterations to 7. We deployed a combination of \method{TIES} and \method{DARE} as merging strategy. As in previous settings, both the fitness function and the test metrics operate by extracting the final model-generated choice via a regex, but this time they look for an instance from the set \{A, B, C, D\} rather than a number. On top of this, we employed a dataset composed by 20 datapoints for each language from the relative translation of \dataset{ARC} to compute the fitness, and we extracted the test set as for the previous experiments. Furthermore, unlike the single-objective approach described earlier, here we explicitly optimize multiple objectives simultaneously. This time, the employed models are \model{Mistral-Ita-7b}, \model{GEITje-7B-ultra}, \model{leo-mistral-hessianai-7b}, and the base model \model{Mistral-7B-v0.1}. 

\subsubsection{Ability and Performance Estimator} In these experiments (reported in \cref{sec:exp_validation} and \cref{par:pe}) we used the test set of the standard version of \dataset{GSM8K}, \dataset{HellaSwag}, \dataset{ARC}, \dataset{Winogrande}, and \dataset{TruthfulQA}. Furthermore, we used 6 different models to test the different performance of the ability and performance estimator: \model{SOLAR-tail-10.7B-Merge-v1.0}, \model{FuseChat-7B-Slerp}, \model{NeuralPipe-7B-slerp}, \model{T3Q-Merge-Mistral7B}, \model{FuseChat-7B-TA}, and \model{supermario-slerp-v3}. These models were chosen as already available on the Open LLM Leaderboard.



\begin{table}
\centering
\caption{Notation used in \cref{ch:merge3}.}
\label{tab:notations}
\vspace{10pt}
\begin{tabular}{cl}
    \toprule
    \textbf{Notation} & \textbf{Description} \\ 
    \midrule
    $D$ & Full dataset. \\
    $\bar{D}$ & Reduced subset of the dataset. \\
    $D_{i}$ & Subdataset for task $i$. \\
    $\gamma_m$ & Latent abilities of model $m$. \\
    $\Gamma_m$ & True latent abilities of model $m$. \\
    $\gamma_m^{\{\mathrm{p,gp}\}-\mathrm{IRT}} $ &  Latent abilities of model $m$ via \pirt{} ability estimator.\\
    $\gamma_m^{\{\mathrm{mp,gmp}\}-\mathrm{IRT}} $ &  Latent abilities of model $m$ via \mpirt{} ability estimator.\\
    $a_i, \beta_i$ & IRT parameters related to dataset item $i$. \\
    $\xi$ & Interpolation coefficients for latent abilities. \\
    $\hat{\xi},\hat{\gamma},\hat{a},\hat{\beta}$ & MLE of the aforementioned parameters. \\
    $p_{i,m}$ & IRT model for datapoint $i$ and model $m$. \\
    \multirow{2}{*}{$\hat{p}_{i,m}$} & IRT model for datapoint $i$ and model $m$ \\
                    & parametrized by MLE estimators of $a, \beta,\gamma, \xi$. \\
    $\tilde{m}$ & Merged language model. \\
    $Y_{i,m}$ & Sample-level correctness of model $m$ for example $i$. \\
    $\hat{Z}^{\mpirt{}}$ & Merged performance estimator \mpirt{}. \\
    $\hat{Z}^{\gmpirt{}}$ & Generalized merged performance estimator \gmpirt{}. \\
    $F(m)$ & Fitness value of a model $m$. \\
    $\theta$ & Parameters being optimized in evolutionary search. \\
    $P_{\bar{F}_D}$ & Pareto front defined by function's set $\bar{F}$ and data $D$ \\
    $\theta^{\star}$ & Global optimum on $D$. \\
    $\hat{\theta}$ & Global optimum on $\bar{D}$. \\
    $N$ & Number of samples in the dataset. \\
    \bottomrule
\end{tabular}
\end{table}

%% file: 99_appendix/F_merge3/B_Additional_Experiments/content.tex
\label{app:add-exp}
We report here additional experiments and analyses.

\subsection{Extract step} \label{app:add-exp:extract-step} 
In the extract step outlined in \cref{sec:estimate}, random sampling has been proposed as the main method to subsample the dataset $\bar{D} \subset D$. 
While we explored various dataset subsampling strategies, we ultimately opted for uniform random sampling, as our experiments showed that more complex approaches offered no significant advantage over this simpler method. In this section we report some of the experiments behind this decision and the two alternative methods tried in the extraction step: IRT Clustering (IRT), introduced by \citet{tinybenchmarks}, and a custom Representation Clustering (RC) method.

\subsubsection{IRT clustering} Given a dataset $D$ and the parameter of a fitted IRT model $a$ and $\beta$, one can define a low-dimensional embedding of each datapoint $i \in D$ by $E_i = [a_i \| \beta_i]$. Therefore, IRT-clustering obtains a representative subset by first obtaining a clustering over this embedding space through $K$-Means, and then choosing the points closest to the centroids as representative samples.

\subsubsection{Representation clustering} 

Let \(\{m_j\}_{j=1}^M\) be the set of endpoint models, and let \(D = \{x_i\}_{i=1}^N\) be our full dataset. For each sample \(x_i\), we first encode it into a high-dimensional vector by concatenating model-specific embeddings. Concretely, we compute:
\[
    E_{i,j} = \frac{1}{T_i}\sum_{t=1}^{T_i} E_{i,j,t} \in \mathbb{R}^d,
\]
where \(E_{i,j,t}\) is the embedding of the \(t\)-th token of sample \(x_i\) under model \(m_j\), and \(T_i\) is the number of tokens in \(x_i\). We form the concatenated representation:
\[
    E_i = [E_{i,1}\|E_{i,2}\|\cdots\|E_{i,M}] \in \mathbb{R}^{M\cdot d}.
\]
Since \(E_i\) can be very high-dimensional, we apply Principal Component Analysis (PCA) to project \(E_i\) onto a lower-dimensional space:
\[
    \tilde{E}_i = \text{PCA}_k(E_i) \in \mathbb{R}^k, \quad k \ll M\cdot d.
\]
Next, we apply \(k\)-means clustering to the reduced embeddings \(\{\tilde{E}_i\}_{i=1}^N\):
\[
    \min_{\{\mathbf{c}_k\}_{k=1}^K}\sum_{i=1}^N \min_{1\leq k \leq K} \|\tilde{E}_i - \mathbf{c}_k\|^2,
\]
where \(\mathbf{c}_k\) is the centroid of the \(k\)-th cluster. This partitions the dataset into \(K\) clusters, each capturing a distinct region of the representation space.
From each cluster \(k\), we select the representative sample \(x_{i_k^\star}\) whose embedding \(\tilde{E}_{i_k^\star}\) is closest to the centroid \(\mathbf{c}_k\):
\[
    i_k^\star = \arg\min_{x_i \in C_k} \|\tilde{E}_i - \mathbf{c}_k\|,
\]
where \(C_k\) is the set of samples assigned to cluster \(k\).
To approximate the full-dataset metrics from the selected subset \(\bar{D}=\{x_{i_k^\star}\}_{k=1}^K\), we assign a weight to each representative sample. Since the size of the cluster \(C_k\) indicates how prevalent that region of representation space is, we define $w_{i_k^\star} = \frac{|C_k|}{|D|}$.
These weights ensure that the contribution of each representative sample to the overall metric reflects the true proportion of samples that it represents in the original dataset. By evaluating a new model \(m\) only on \(\bar{D}\) and using \(\{w_{i_k^\star}\}\) to calculate a weighted average, we approximate \(m\)'s performance on the full dataset \(D\) at a fraction of the computational cost.

A schematic overview of the full process is outlined in \cref{alg:repr-clust}. 

\begin{algorithm}
\caption{Representation Clustering Extractor}\label{alg:repr-clust}
\begin{algorithmic}[1]
    \REQUIRE Dataset $D$, Endpoint Models $m_1,...,m_n$, Desired subset size $K$
    \ENSURE Subset of size $K$ with weights $w_i$

    \FOR{$i$ in $D$}
        \STATE $E_i \gets []$ 
        \FOR{$m$ in $\{m_1, \dots, m_n\}$}
        \STATE $E_{im} \gets$ embed $i$ with model $m$
        \STATE  $E_i \gets E_i | E_{im}$ 
        \ENDFOR
    \ENDFOR
    \STATE $\{\mathbf{E}_i\}_{i \in D} \gets \operatorname{PCA}(\{\mathbf{E}_i\}_{i \in D})$
    \STATE Apply k-means clustering to $\{\mathbf{E}_i\}_{i \in D}$, obtaining $K$ centroids $\{\mathbf{c}_k\}_{k=1}^K$
    \STATE For each cluster $k$, select the closest example $i_k^{\star} = \arg\min_{i \in D} \|\mathbf{E}_i - \mathbf{c}_k\|_2$
    \STATE Let $C_k = \{i \in D \ | \ \arg\min_{c \in \{\mathbf{c}_k\}_{k=1}^K} \|\mathbf{E}_i - c\|_2 = \mathbf{c}_k\}$ be the set of examples in cluster $k$
    \STATE Assign weights $w_{i_k^{\star}} = \frac{|C_k|}{|D|}$ for $k=1,...,K$
    \STATE \textbf{return}  $\left\{ i_k^{\star}, w_{i_k^{\star}} \right\}_{k=1}^{K}$
\end{algorithmic}
\end{algorithm}

\subsubsection{Experiments}
To compare the performance of the Sample Extractors, we followed a procedure similar to that described in \cref{par:pe}, computing the absolute estimation error for each extractor. For random sampling, the accuracy estimator was obtained via uniform averaging, whereas for IRT and RC it was obtained via weighted averaging. We evaluated the estimator in two different settings: (1) merging a math model with a language-tuned model (similar to the cross-lingual setting of \cref{sec:cross-lingual}) for several languages (Italian, German, Romanian, Dutch) and testing the extractor on the corresponding translations of \dataset{GSM8K} (see \cref{fig:extractor_language}), and (2) merging several math models and testing the extractor on the English version of \dataset{GSM8K} (see \cref{fig:extractor_merges}).

Focusing on \cref{fig:extractor_language}, we see that performance variability is somewhat higher (larger error bars) due to different language-specific datasets. Even so, Random sampling never falls behind IRT or RC, especially for small sample sizes. By the time the subset size reaches 50 or more examples, all three methods converge to comparable accuracy-error levels, underscoring the robustness of Random sampling. Instead, in \cref{fig:extractor_merges}, the trends are broadly similar for RC and Random sampling, while slightly worse for IRT.  Again, as the dataset sample size grows, overall error drops and the gap among methods narrows. 

To sum up, the Random sampler can sometimes lag slightly behind the more sophisticated IRT and RC. Nevertheless, neither of these methods has a clear advantage over the others. Given its simplicity and negligible overhead, the Random strategy stands out as a highly practical choice for dataset subsampling, especially when the marginal improvements of more complex methods do not clearly justify their added complexity.

\begin{figure}
    \includegraphics[width=\linewidth]{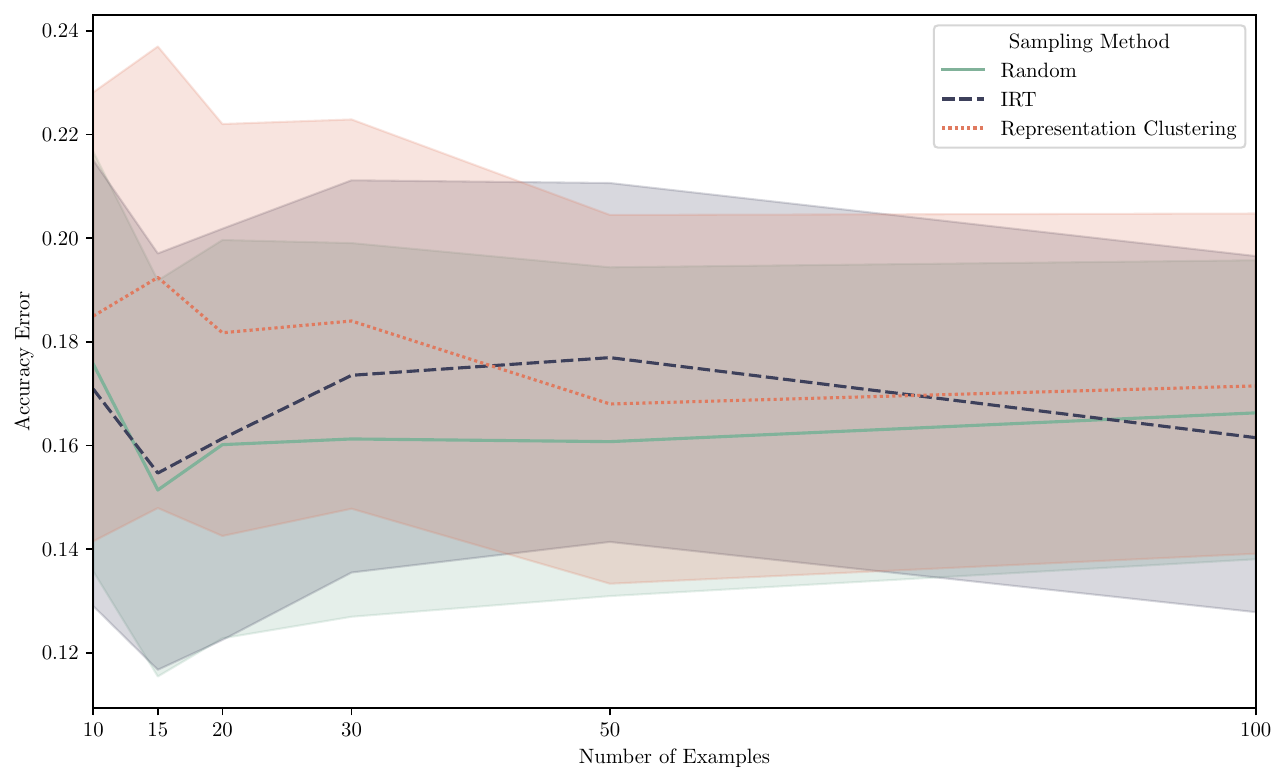}
    \caption[Sample extractor error across languages]{Absolute error of the estimated accuracy of Sample Extractors, averaged across merges of language-specific and English Math fine-tunings of \model{Mistral-7B-v0.1}, evaluated on translations of \dataset{GSM8K} and presented as a function of the number of dataset samples.}
    \label{fig:extractor_language}
\end{figure}

\begin{figure}
    \includegraphics[width=\linewidth]{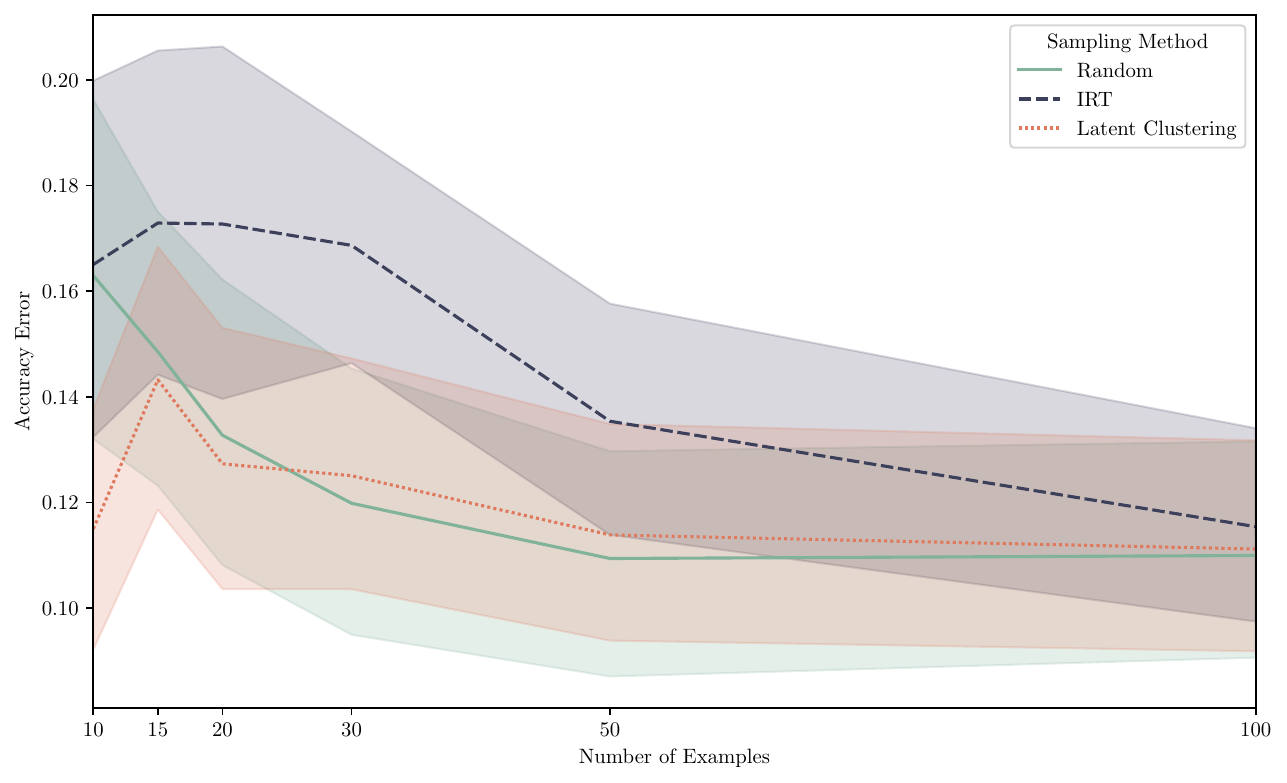}
    \caption[Sample extractor error across merges]{Absolute error of the estimated accuracy of Sample Extractors, averaged across merges of English Math models based on \model{Mistral-7B-v0.1}, evaluated on \dataset{GSM8K} and presented as a function of the dataset sample size.}
    \label{fig:extractor_merges}
\end{figure}

\subsection{Estimation step} \label{app:estimation-step}

\subsubsection{Additional experiment for ability estimator}
We report in \cref{fig:ability-estimator-comparison-all} the Euclidean distance between the estimated and ground-truth ability vectors across different sample sizes. The results are consistent with the case $n=10, 20$ seen in \cref{fig:ability-estimator-comparison}, with our estimated ability vector being significantly closer to the ground-truth one compared to the ability vector estimated by pIRT and gp-IRT. Similarly, we report the corresponding cosine similarity in \cref{fig:ability-estimator-comparison-cosine}, confirming much higher similarity in our case.  
\begin{figure}
    \includegraphics[width=\linewidth]{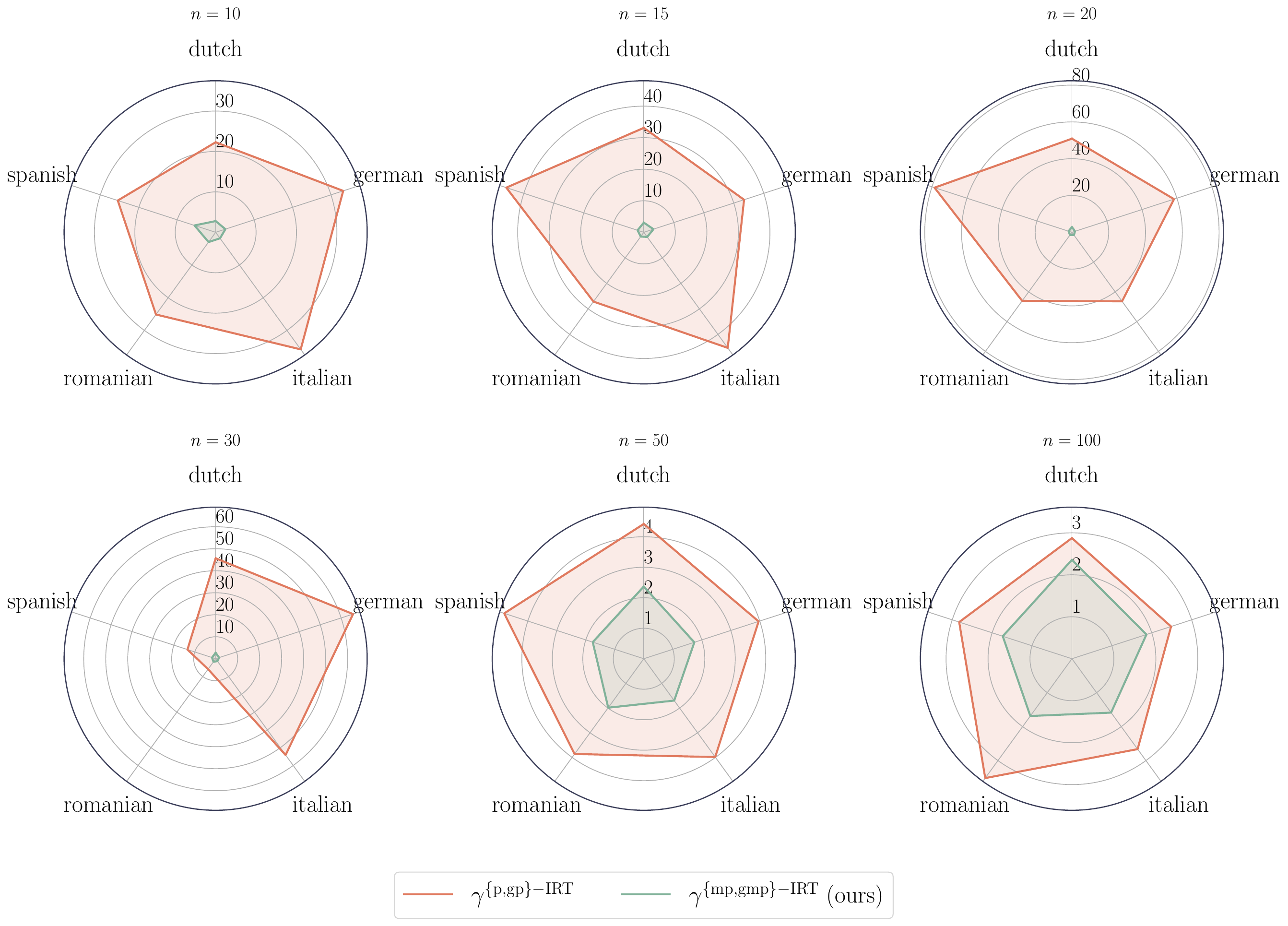}
    \caption[Ability estimator Euclidean distance across languages]{Euclidean distance (lower is better) between estimated and true abilities for different languages.}
    \label{fig:ability-estimator-comparison-all}
\end{figure}
\begin{figure}
    \includegraphics[width=\linewidth]{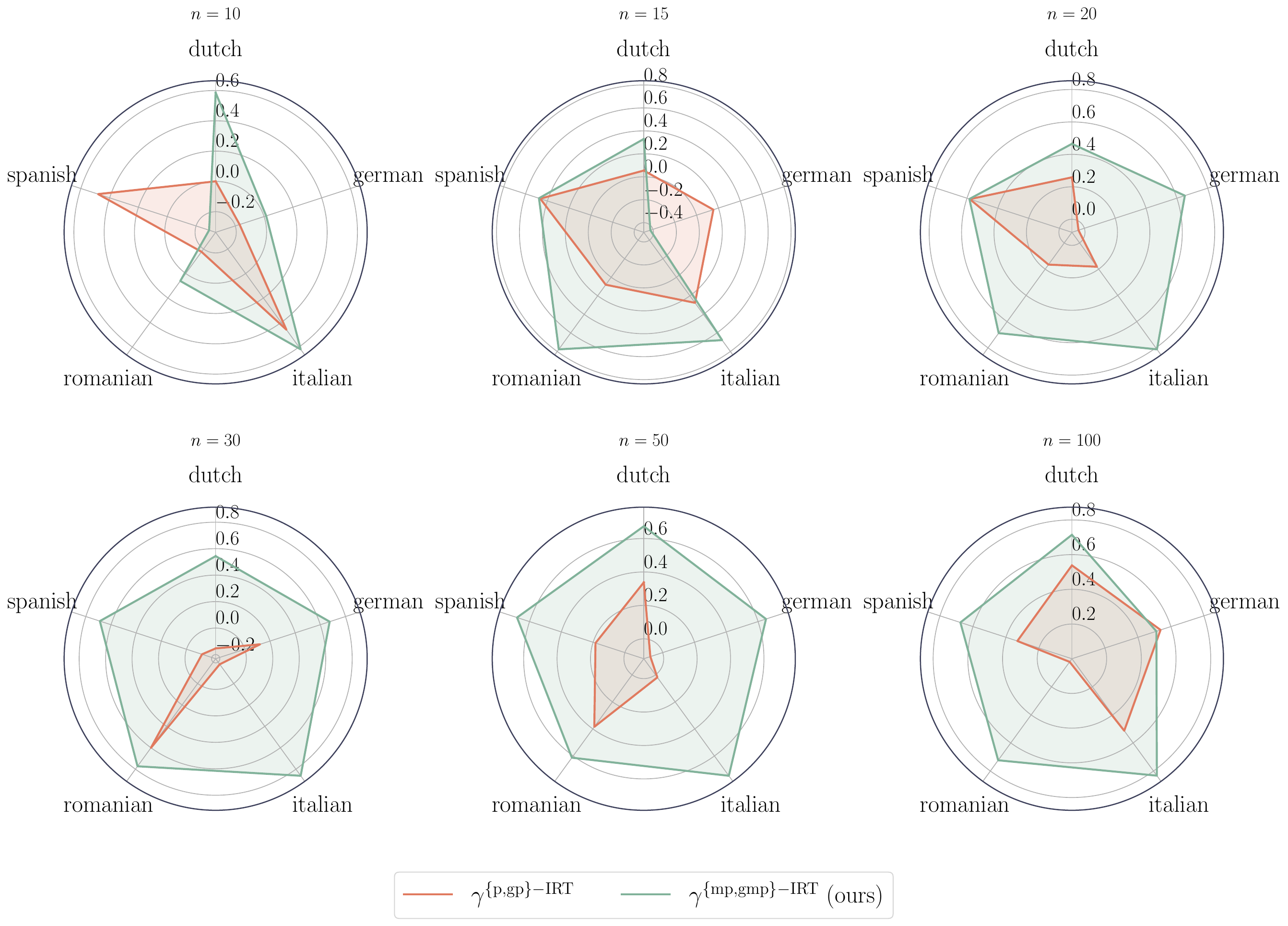}
    \caption[Ability estimator cosine similarity across languages]{Cosine similarity (higher is better) between estimated and true abilities for different languages.}
    \label{fig:ability-estimator-comparison-cosine}
\end{figure}
\begin{figure}
    \includegraphics[width=\linewidth]{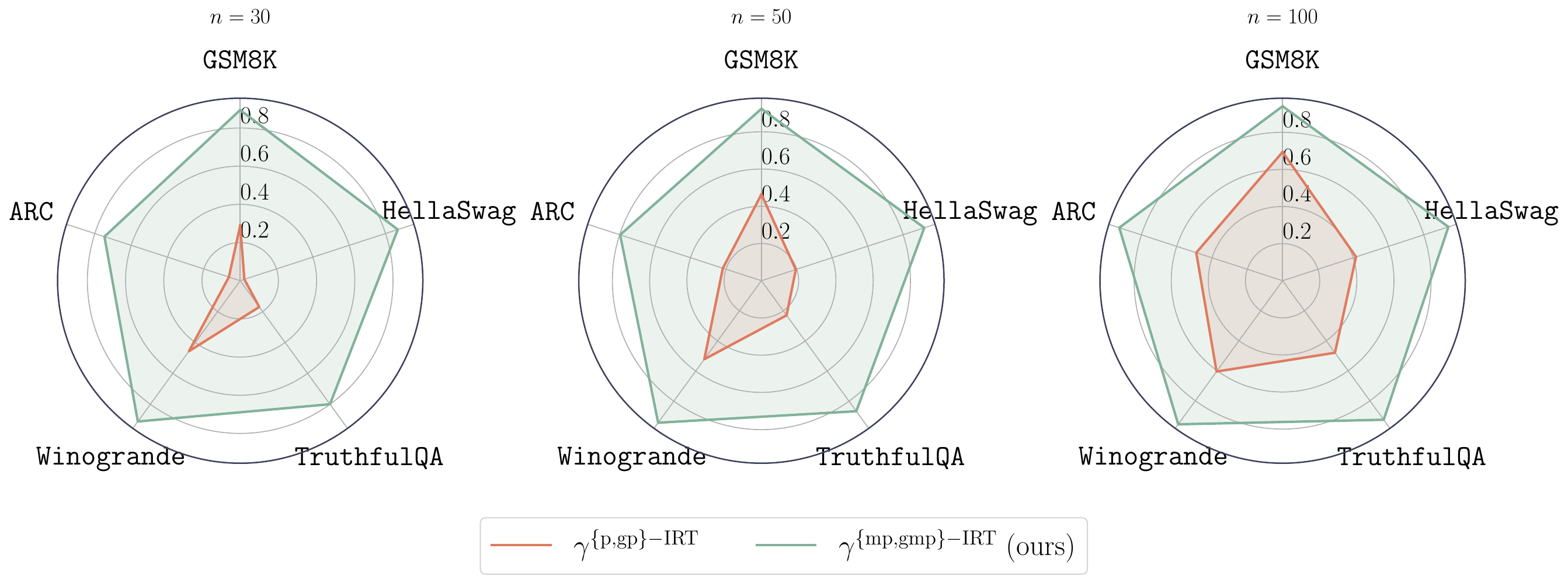}
    \caption[Ability estimator cosine similarity across tasks]{Cosine similarity (higher is better) between estimated and true abilities for different tasks.}
    \label{fig:ability-estimator-comparison-cosine-tasks}
\end{figure}
\subsubsection{Additional experiment for performance estimator} 
\label{app: add perf es}
We report in \cref{fig:estimation_comparison_winogrande_hellaswag} the evaluation of performance estimators across \dataset{Winogrande} and \dataset{Hellaswag}, extending the results in \cref{fig:estimation_comparison}.

\begin{figure}
    \centering
    \includegraphics[width=\linewidth]{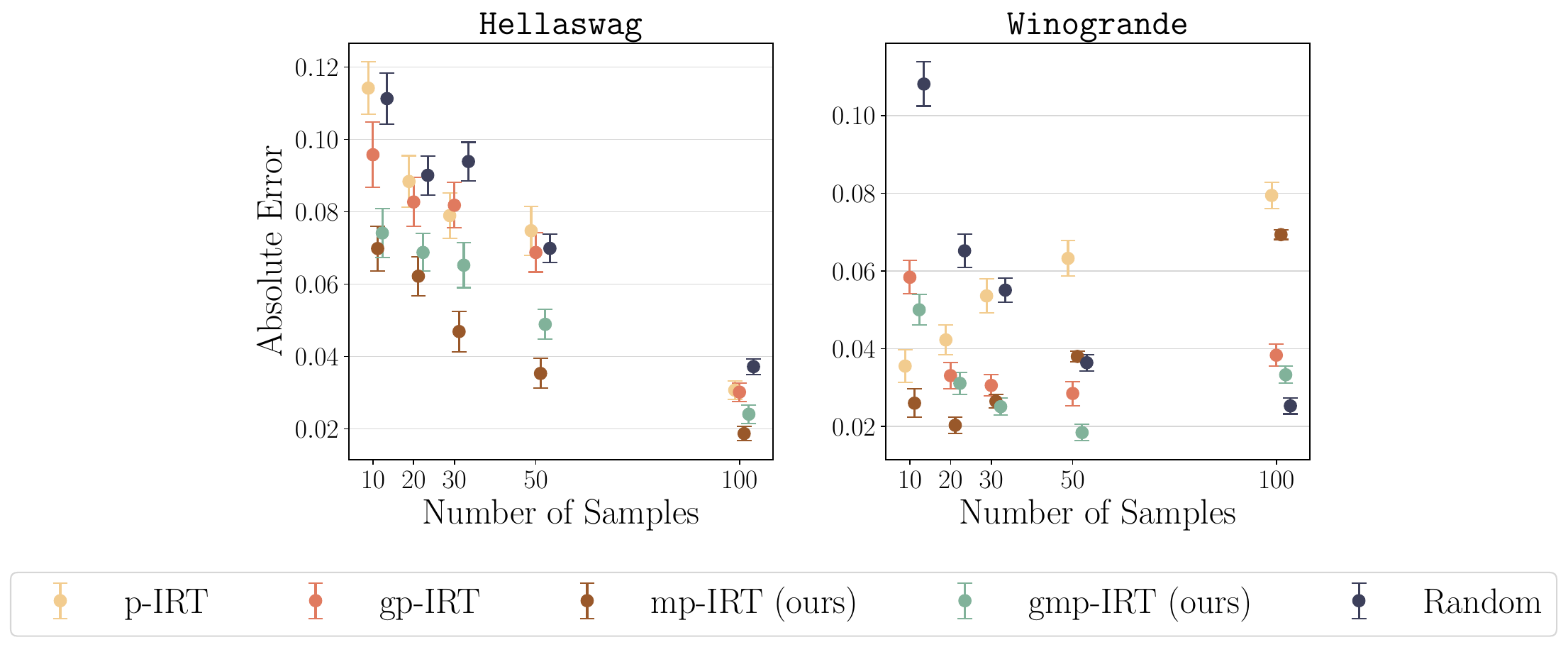}
    \caption[Performance estimator error on \dataset{Winogrande} and \dataset{Hellaswag}]{Absolute error of various estimators as a function of sample size (lower is better). gmp-IRT consistently achieves lower error.}
\label{fig:estimation_comparison_winogrande_hellaswag}
\end{figure}


\subsubsection{Hyperparameter analysis for performance estimator}
\label{app:hyper-search}

We now analyse the optimal choice of the scalar \(c\) required in \gmpirt{} and \gpirt{} (see \cref{gp-irt}). In the experiments reported in the main paper (\cref{par:pe}) and above (\cref{app: add perf es}), we used as a heuristic \(c = \frac{1}{2}\). Despite its empirical effectiveness, this uniform interpolation may not be optimal across all model pairs and data regimes. Therefore, we introduce a grid-search-based strategy to estimate an improved value of \(c\) and empirically validate its potential to reduce estimation error.

\paragraph{Methodology.} 
We propose a two-step approach to selecting the optimal interpolation coefficient \(c\) for use in the \gmpirt{} estimator. The procedure is as follows:

\begin{enumerate}
    \item \textbf{Optimizing \(c\) for Endpoint Models} For a given dataset \( \bar{D} \) \footnote{ Such a dataset is available because it relies solely on the correctness of the endpoint models' answers, rather than on the unknown answers of the merged model.}, we find a \(c \in [0,1]\) by minimizing the absolute estimation error of \gpirt{} on each individual model. This yields a set of optimal coefficients \( \{c_1, \ldots, c_M\} \), one for each endpoint model \(M\). 
    \item \textbf{Averaging \(c\) for the Merged Model.} To obtain a suitable interpolation coefficient for \gmpirt{}, we compute the average of the optimal values across the endpoint models, \( \bar{c} = \frac{1}{M} \sum_{m=1}^{M} c_m \), and use this as the parameter for the merged model's \gmpirt{}.

\end{enumerate}

\paragraph{Results Discussion} We evaluate the absolute estimation error across five benchmarks, comparing the adaptive averaging strategy for \(c\) (denoted by $\star$) to the baseline fixed heuristic \(c = \frac{1}{2}\). Results for $\gmpirt{}^\star$ and $\gpirt{}^\star$, along with the baseline models \gmpirt{}, \gpirt{}, \mpirt{}, and \pirt{}, are reported in \cref{tab:adaptive_c_results}.

\begin{table}[htbp]
    \centering
    \caption[Absolute estimation error with adaptive interpolation]{Absolute estimation error across datasets. The $\star$ symbol indicates the adaptive $c$ strategy described above. Lower is better.}
    \label{tab:adaptive_c_results}
    \vspace{10pt}
    \resizebox{0.7\textwidth}{!}{
    \begin{tabular}{lccccc}
        \toprule
        \textbf{Dataset} & \textbf{\gmpirt{}$^\star$} & \textbf{\gmpirt{}} & \textbf{\gpirt{}$^\star$} & \textbf{\gpirt{}} & \textbf{\mpirt{}} \\
        \midrule
        ARC         & \textbf{0.035} & 0.040 & 0.046 & 0.049 & 0.048 \\
        Winogrande  & \textbf{0.018} & 0.031 & 0.032 & 0.037 & 0.036 \\
        GSM8K       & \textbf{0.057} & \textbf{0.057} & 0.074 & 0.064 & 0.062 \\
        HellaSwag   & \textbf{0.046} & 0.056 & 0.077 & 0.071 & 0.047 \\
        TruthfulQA  & \textbf{0.040} & 0.045 & 0.062 & 0.055 & 0.044 \\
        \bottomrule
    \end{tabular}
    }
\end{table}
Across most datasets, the adaptive coefficient strategy yields consistent improvements in both \gmpirt{} and \gpirt{}, with the largest gains observed on Winogrande and HellaSwag. Notably, \gmpirt{} with a tuned \(c\) performs on par or better than any other method across all benchmarks.

\subsubsection{Spearman correlation} 
Following \citet{spearman}, we report the Spearman correlation between the ground-truth ranking and the ranking induced by each estimator's predictions. We compare the best merging-specific estimator, \gmpirt{}, against the strongest vanilla baseline, \gpirt{}. The results, averaged across dataset sizes and types, are shown in \cref{fig:spearman}.
Notably, \gmpirt{} achieves the highest correlation in each setting, further underscoring the benefits of using estimators specifically designed for the model merging context.

\begin{figure}
    \centering
        \includegraphics[width=.8\linewidth]{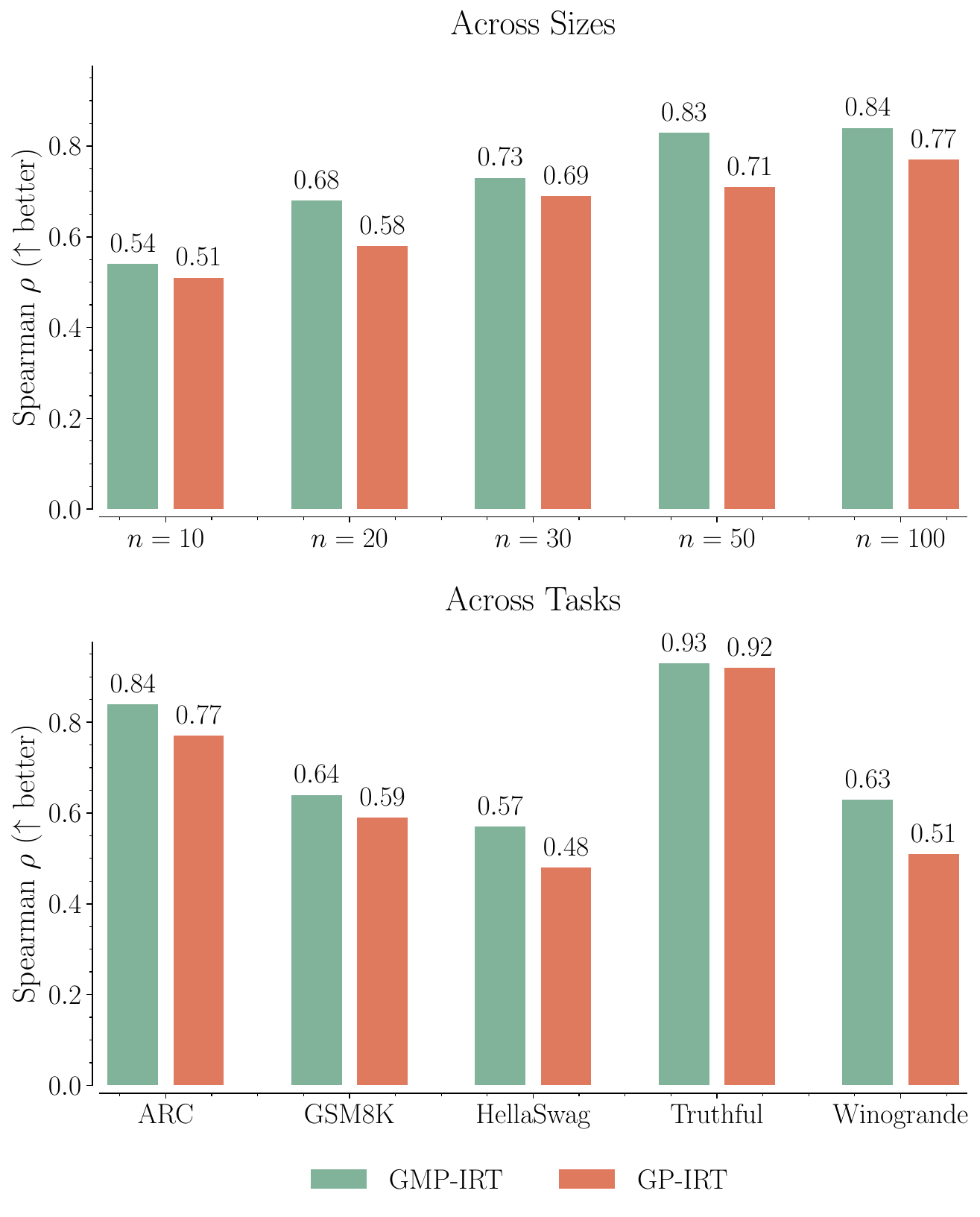}
    \caption[Spearman rank correlations for performance estimators]{Spearman rank correlations  using the \cref{app:hyper-search} setup. Higher is better. \gmpirt{} shows substantially higher correlation than \gpirt{}.}
    \label{fig:spearman}
\end{figure}

\subsection{Evolve step} 
\label{app:evostepadd}

\subsubsection{Comparison with in-context learning}
\label{app:add-exp-icl}

A natural question when evaluating merging-based methods is how their performance compares to inference-time adaptation strategies such as \textit{In-Context Learning} (ICL). In particular, given access to a small validation set, one might ask whether directly providing these examples as input context at evaluation time can match the performance achieved through evolutionary merging.

To investigate this, we evaluate a 20-shot ICL setup in the multilingual transfer setting introduced in \cref{sec:multi-lingual}. Prompts are constructed using 20 validation examples and prepended to the input at inference time. We apply this setup to two merging baselines, \method{TIES-DARE} and \method{Task Arithmetic}, and compare the results against our method using \gmpirt{} with a fitness dataset of size 20.
\begin{table}[htbp]
    \centering
    \caption[Multilingual \dataset{ARC} accuracy: ICL vs.\ evolutionary merging]{Accuracy on multilingual \dataset{ARC} using few-shot ICL (20-shot) versus \approach{} (\gmpirt-20{}).}
    \label{tab:icl_comparison}
    \vspace{10pt}
    \begin{tabular}{lcccc}
        \toprule
        \textbf{Method} & \textbf{DE} & \textbf{IT} & \textbf{NL} & \textbf{EN} \\
        \midrule
        \method{TIES-DARE} Few-shot (20) & 0.227 & 0.226 & 0.227 & 0.226 \\
        \method{Task Arithmetic} Few-shot (20) & 0.427 & 0.406 & 0.491 & 0.566 \\
        \approach{} (\gmpirt-20{}) & \textbf{0.720} & \textbf{0.690} & \textbf{0.690} & \textbf{0.790} \\
        \bottomrule
    \end{tabular}
\end{table}
Across all languages, \approach{} significantly outperforms the ICL-augmented baselines. While ICL can provide moderate improvements over the original models, it increases inference-time memory usage and latency due to the expanded context. In contrast, the merged models produced by our method operate without additional overhead and are immediately deployable as standalone networks.

\subsubsection{Analyzing negative transfer}
\label{app:add-exp-neg-transfer}
While our main analysis focused on cross-lingual transfer of \approach{} (\cref{sec:cross-lingual}), we did not explicitly examine the phenomenon of \emph{negative transfer}; that is, cases where merging degrades performance on specific inputs. In this subsection, we formally define negative transfer in the context of multiple-choice questions (MCQs) and introduce a framework for measuring its prevalence in merged multilingual models. We then present an analysis of negative transfer in \dataset{GSM8K}.

\paragraph{Methodology} We consider a multiple-choice question (MCQ) evaluation setup, where knowledge is operationalized as a model's ability to correctly answer a question. Let \( m_1, m_2, \dots, m_K \) denote the set of \(K\) endpoint models, and let \( \tilde{m} \) represent the merged model resulting from their combination. Following our earlier notation (see \cref{tab:notations}), correctness for a given sample \( i \) is defined as a binary variable:
\begin{itemize}
    \item \( Y_{i,m_j} \in \{0, 1\} \): indicates whether endpoint model \( m_j \) answers sample \( i \) correctly,
    \item \( Y_{i,\tilde{m}} \in \{0, 1\} \): indicates whether the merged model answers sample \( i \) correctly.
\end{itemize}
We define negative transfer on example \( i \) as occurring when at least one of the base models answers correctly, but the merged model fails:
\[
\exists\, j \in \{1, \dots, K\} \text{ such that } Y_{i,m_j} = 1 \quad \text{and} \quad Y_{i,\tilde{m}} = 0.
\]
To track this, we introduce a binary indicator variable \( n_i \) for each input:
\[
n_i =
\begin{cases}
1, & \text{if } \left(\exists\, j : Y_{i,m_j} = 1\right) \text{ and } Y_{i,\tilde{m}} = 0, \\
0, & \text{otherwise}.
\end{cases}
\]

Finally, we compute the \emph{Negative Transfer Rate (NTR)} as the proportion of examples exhibiting negative transfer among those for which at least one base model answered correctly:
\[
\text{NTR} = \frac{\sum_{i=1}^{N} n_i}{\sum_{i=1}^{N} \mathbf{1}\left\{\exists\, j \in \{1, \dots, K\} : Y_{i,m_j} = 1\right\}}.
\]
This metric provides a task-level perspective on the potential degradation introduced by merging and complements the aggregate performance measures reported in the main paper.

\begin{figure}
    \centering
    \includegraphics[width=.8\linewidth]{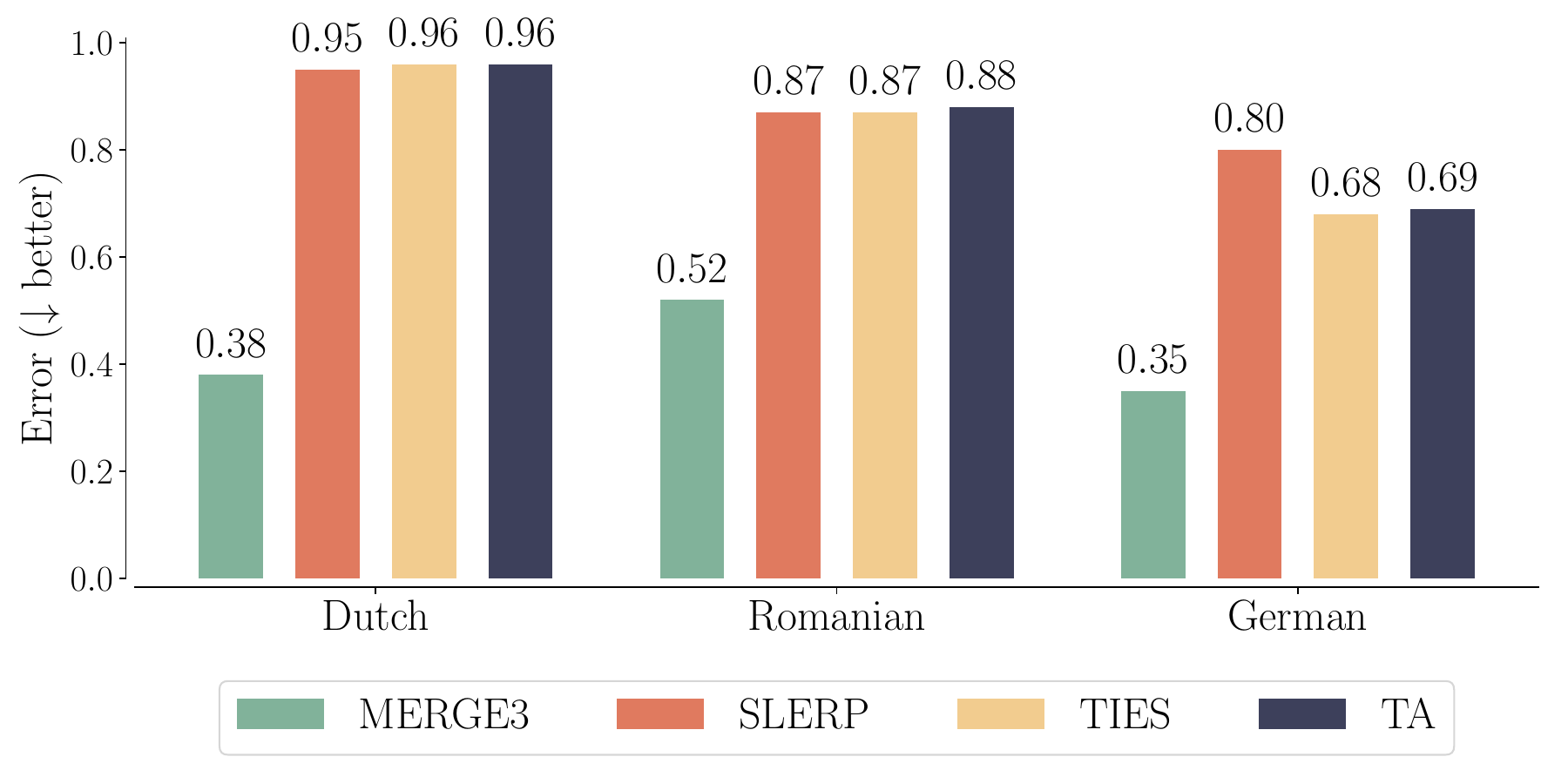}
    \caption[Negative Transfer Rate across languages]{Negative Transfer Rate  across languages. Lower is better. \method{MERGE$^3$} shows substantially less negative transfer than standard baselines.}
    \label{fig:ntr_comparison}
\end{figure}

\paragraph{Results Discussion} We compute the NTR using the same experimental setting described in \cref{sec:cross-lingual}. As shown in \cref{fig:ntr_comparison}, \mergethree consistently yields substantially lower negative transfer than \method{SLERP}, \method{TIES}, and \method{Task Arithmetic} across all languages. This indicates that \mergethree not only improves average accuracy but also preserves correct knowledge from its component models, thereby maintaining per-example competence during merging.

\subsubsection{Time \& FLOPs requirements for evolutionary merging}
\label{app:add-exp-4090}
\paragraph{Hardware Setting.} To compare the efficiency of different model evaluation strategies, we measured the time required to evolve merged LLM models using a single NVIDIA 4090 with 24GB of VRAM, and report the Throughput $R$ in \cref{tab:throughput_comparison}. 
We also benchmark evaluation and merging times across three GPU models (3090, 4090, V100) to illustrate practical runtimes for \mergethree on both modern and older hardware. We report the results in \cref{tab:multi_gpu_timing}

\begin{table}[htbp]
    \centering
    \caption[Comparison of Evolve methods by cost and accuracy]{Comparison of Evolve methods by number of trials, estimated total time on a single NVIDIA 4090, sample size used for Fitness computation, and final accuracy on \dataset{GSM8K}. The number of trials is the result of $\text{population size} \times \text{iterations}$, parameters of the genetic algorithms of each method, and represents the total number of merged models evaluated during the entire Evolve run.}
    \label{tab:evolution_comparison_app}
    \vspace{10pt}
    \resizebox{0.75\textwidth}{!}{
    \begin{tabular}{lcccc}
        \toprule
        \textbf{Method} & \textbf{$N_{\text{models}}$} & \textbf{Estimated total time} & \textbf{Sample size} & \textbf{Accuracy} \\
        \midrule
        EvoLLM-JP-7B & 1000 & 62 days & 1000 & 0.49 \\
        \method{MERGE$^3_{100}$} & 175 & 21h & 100 & 0.42 \\
        \method{MERGE$^3_{50}$} & 175 & 12h 20m & 50 & 0.38 \\
        \method{MERGE$^3_{30}$} & 175 & 10h 30m & 30 & 0.38 \\
        \method{MERGE$^3_{20}$} & 175 & 10h 15m & 20 & 0.34 \\
        \bottomrule
    \end{tabular}
    }
\end{table}

\paragraph{Results Discussion.} Over a 12-hour period, we were able to evaluate 8 models on 1000 samples of \dataset{GSM8K} with a single NVIDIA 4090, allowing us to estimate that evaluating 1000 models would take approximately 62 days under similar conditions. In contrast, \approach enabled the evaluation of a larger number of merged models in significantly less time by using a reduced dataset. These results suggest that researchers and practitioners could leverage consumer-grade GPUs for efficient LLM merging and evaluation, making rapid experimentation of model merging methods more accessible. We report in \cref{tab:evolution_comparison_app} the estimated total time of the Evolve runs, which we calculated using the following formula:
\[
T(N_{\text{models}}) \;=\; \frac{N_{\text{models}}}{R_{\text{Dataset Size}}}
\]

Lastly, table~\ref{tab:multi_gpu_timing} shows that \mergethree maintains practical runtimes across a range of GPUs. While the 4090 offers the fastest evaluation, older hardware like the V100 still supports feasible experimentation, highlighting the framework's accessibility and the generalizability of results across different consumer GPUs.
\begin{table}[htbp]
    \centering
    \caption[Evaluation throughput by sample size on \dataset{GSM8K}]{Throughput (\( R \)) in models per hour for different sample sizes per fitness evaluation on \dataset{GSM8K}. These estimates are based on 12-hour Evolve runs on a single NVIDIA 4090 with 24GB of VRAM.}
    \label{tab:throughput_comparison}
    \vspace{10pt}
    \begin{tabular}{lccccc}
        \toprule
        \textbf{Sample size} & 1000 & 100 & 50 & 30 & 20 \\
        \midrule
        \textbf{Throughput (Models/Hour)} & 0.67 & 8.33 & 14.17 & 16.67 & 17.08 \\
        \bottomrule
    \end{tabular}
\end{table}

\begin{table}[htbp]
    \centering
    \caption[Evaluation and merge time across GPU models]{Evaluation and merge time across different GPU models using Mistral-7B on 10 examples (4-bit, SLERP).}
    \label{tab:multi_gpu_timing}
    \vspace{10pt}
    \begin{tabular}{lcc}
        \toprule
        \textbf{GPU Model} & \textbf{Eval Time (s)} & \textbf{Merge Time (s)} \\
        \midrule
        NVIDIA 3090 24GB & 65 & 135 \\
        NVIDIA 4090 24GB & 45 & 160 \\
        NVIDIA V100 32GB & 80 & 220 \\
        \bottomrule
    \end{tabular}
\end{table}

\paragraph{FLOPs Calculation.} We provide a Jupyter Notebook that describes the FLOPs calculations for our experiments in the supplementary material, based on the \textit{calc-flops} library\footnote{ \url{https://github.com/MrYxJ/calculate-flops.pytorch}.}. This script has been used to estimate the FLOPs for the experiment in \Cref{fig:speed_accuracy}.

%% file: 99_appendix/F_merge3/D_Proofs/content.tex
We outline in \cref{tab:notations} a scheme of the notation used throughout this chapter.

\subsection{Proof of \Cref{thm:eps-opt-preserve}}
\label{proof:thm-eps-opt-preserve}

\begin{proof}
Let \(m := F(\theta^*;D)\) and \(\hat{m} := F(\hat{\theta};\bar{D})\).  
We must show that \(\lvert\,m - \hat{m}\rvert \le \epsilon\).

\begin{enumerate}
\item 
By \(\epsilon\)-stability, for \emph{all} \(\theta\in\Theta\):
\[
  \bigl\lvert F(\theta;D) \;-\; F(\theta;\bar{D})\bigr\rvert
  \;\le\;
  \epsilon.
\]
In particular, for \(\theta = \theta^*\),
\[
  \bigl\lvert F(\theta^*;D) \;-\; F(\theta^*;\bar{D})\bigr\rvert
  \;\le\;
  \epsilon.
\]
Hence
\[
  F(\theta^*;\bar{D}) 
  \;\ge\;
  F(\theta^*;D)\;-\;\epsilon
\]
and
\[
  F(\theta^*;\bar{D}) 
  \;\le\;
  F(\theta^*;D)\;+\;\epsilon.
\]

\item 
Since \(\hat{\theta}\) is the minimizer of \(F(\cdot;\bar{D})\), we have
\[
  F(\hat{\theta};\bar{D}) 
  \;\le\; 
  F(\theta^*;\bar{D}).
\]
Because \(\theta^*\) is the minimizer of \(F(\cdot;D)\), 
\[
  F(\hat{\theta};D) 
  \;\ge\; 
  F(\theta^*;D).
\]

\item 
To bound \(\hat{m}-m\), we can add and subtract $F(\theta^*;\bar{D})$ to have
\begin{align*}
  \hat{m}-m
  &\;=\;
  \Bigl(F(\hat{\theta};\bar{D}) \;-\; F(\theta^*;\bar{D})\Bigr) \\
  &\;+\;
  \Bigl(F(\theta^*;\bar{D}) \;-\; F(\theta^*;D)\Bigr).
\end{align*}
The first term is \(\le 0\) (since \(\hat{\theta}\) is a minimizer on \(\bar{D}\)), 
and the second term is \(\le \epsilon\).  
Hence
\[
  \hat{m}-m
  \;\le\;
  0 + \epsilon
  \;=\;
  \epsilon.
\]

\item 
Analogously, to bound \(m-\hat{m}\), we can rewrite
\begin{align*}
  m-\hat{m}
  &\;=\;
  \Bigl(F(\theta^*;D) \;-\; F(\hat{\theta};D)\Bigr) \\
  &\;+\;
  \Bigl(F(\hat{\theta};D) \;-\; F(\hat{\theta};\bar{D})\Bigr).
\end{align*}
The first term is \(\le 0\) (since \(\theta^*\) is a minimizer on \(D\)), 
and the second term is \(\le \epsilon\).  
Thus,
\[
  m-\hat{m}
  \;\le\;
  0 + \epsilon
  \;=\;
  \epsilon.
\]

\item 
Combining these inequalities:
\[
  -\epsilon 
  \;\le\; 
  \hat{m}-m 
  \;\le\; 
  \epsilon
  \quad\Longrightarrow\quad
  \lvert
    m - \hat{m}
  \rvert
  \;\le\;
  \epsilon.
\]
Hence 
\(\bigl\lvert F(\theta^*;D) - F(\hat{\theta};\bar{D})\bigr\rvert\le\epsilon\), 
completing the proof.
\end{enumerate}
\end{proof}

\subsection{Proof of \Cref{thm:expected-eps-stability}}
\label{proof:thm-expected-eps-stability}
\begin{proof}
By hypothesis, for every \(\theta\in\Theta\),
\[
  \mathbb{E}_{\bar{D}}\!\Bigl[
    \bigl\lvert F(\theta;D) \;-\; F(\theta;\bar{D})\bigr\rvert
  \Bigr]
  \;\le\;\epsilon.
\]
Using Jensen's inequality for the absolute value,
\begin{align*}
  &\bigl\lvert
    \mathbb{E}_{\bar{D}}\!\bigl[
      F(\theta;D) \;-\; F(\theta;\bar{D})
    \bigr]
  \bigr\rvert \\
  &\;\le\;
  \mathbb{E}_{\bar{D}}\!\Bigl[
    \bigl\lvert
      F(\theta;D) \;-\; F(\theta;\bar{D})
    \bigr\rvert
  \Bigr] 
  \;\le\;\epsilon.
\end{align*}
Hence,
\[
  -\epsilon 
  \;\le\;
  \mathbb{E}_{\bar{D}}\!\bigl[
    F(\theta;\bar{D}) - F(\theta;D)
  \bigr]
  \;\le\;
  \epsilon
\]
for each fixed $\theta$. It thus follows that
\[
  \mathbb{E}_{\bar{D}}\!\bigl[F(\theta;\bar{D})\bigr]
  \;\le\;
  F(\theta;D) + \epsilon
\]
and
\[
  \mathbb{E}_{\bar{D}}\!\bigl[F(\theta;\bar{D})\bigr]
  \;\ge\;
  F(\theta;D) - \epsilon.
\]
Consequently,
\[
  \min_{\theta\in\Theta}\,
  \mathbb{E}_{\bar{D}}\!\bigl[F(\theta;\bar{D})\bigr]
  \;\le\;
  \min_{\theta\in\Theta}
  \bigl[
    F(\theta;D) + \epsilon
  \bigr]
  \;=\;
  m^* + \epsilon,
\]
where \(m^* := \min_{\theta\in\Theta} F(\theta;D)\).  
Meanwhile, by a min-versus-expectation (Jensen-type) inequality,
\[
  \mathbb{E}_{\bar{D}}\!\bigl[\min_{\theta\in\Theta}\,
  F(\theta;\bar{D})\bigr]
  \;\ge\;
  \min_{\theta\in\Theta}
  \mathbb{E}_{\bar{D}}\!\bigl[
    F(\theta;\bar{D})
  \bigr].
\]
Hence,
\begin{align*}
  \mathbb{E}_{\bar{D}}\!\bigl[\widehat{m}(\bar{D})\bigr]
  &\;=\;
  \mathbb{E}_{\bar{D}}\!\Bigl[\min_{\theta\in\Theta}
    F(\theta;\bar{D})
  \Bigr]
  \\ 
  &\;\ge\; 
  \min_{\theta\in\Theta}
  \mathbb{E}_{\bar{D}}\!\bigl[
    F(\theta;\bar{D})
  \bigr]
  \;\ge\;
  m^* - \epsilon.
\end{align*}
Combining these two bounds results in
\[
  m^* - \epsilon
  \;\le\;
  \mathbb{E}_{\bar{D}}\!\bigl[\widehat{m}(\bar{D})\bigr]
  \;\le\;
  m^* + \epsilon
\]
and, therefore,
\[
\Bigl|
    m^*
    \;-\;
    \mathbb{E}_{\bar{D}}\!\bigl[\widehat{m}(\bar{D})\bigr]
  \Bigr|
  \;\le\;\epsilon.
\]
\end{proof}

\subsection{Proof of \Cref{prop:estimator-unbiased}} \label{proof:prop-estimator-unbiased}
\begin{proof}
We must show that 
\[
\bigl\lvert
  \mathbb{E}\!\bigl[\hat{Z}_{jl} \,\bigm|\,
    Y_{i_0l},\dots,Y_{i_kl}
  \bigr]
  \;-\;
  \mathbb{E}\!\bigl[Z_{jl} \,\bigm|\,
    Y_{i_0l},\dots,Y_{i_kl}
  \bigr]
\bigr\rvert
\;\to\;0
\]
in probability as
$|\hat{I}|\to\infty$.
Under the assumptions of the proposition (including linear inheritance of abilities, 
\(\hat{\xi}\to\xi\) in probability, and bounded \(\|a_i\|\)),
we may bound this difference as follows:
\begin{align*}
&\bigl\lvert
  \mathbb{E}\!\bigl[\hat{Z}_{jl} \, \bigm|\,
    Y_{i_0l}, \ldots, Y_{i_kl}
  \bigr]
  \;-\;
  \mathbb{E}\!\bigl[Z_{jl} \,\bigm|\,
    Y_{i_0l}, \ldots, Y_{i_kl}
  \bigr]
\bigr\rvert\\
&\;\le\;
\frac{1 - \hat{\xi}}{|I_j \setminus \hat{I}_j|}
\sum_{i \in I_j \setminus \hat{I}_j}
\Bigl|
  \sigma\!\bigl((\hat{\xi}_1\theta_{l_1}
    +\hat{\xi}_2\theta_{l_2})^\top a_i
    \;-\;\beta_i\bigr)\\
  &\;-\;
  \sigma\!\bigl(\theta_{l_m}^\top a_i 
    \;-\;\beta_i\bigr)
\Bigr|.
\end{align*}
Since \(\sigma\) is \(1/4\)-Lipschitz on \(\mathbb{R}\), we have
\begin{align*}
    &\;\le\;
    \frac{1}{|I_j|}
    \sum_{i \in \hat{I}_j}
    \Bigl|
      \bigl(
        (\hat{\xi}_1\theta_{l_1}
         +\hat{\xi}_2\theta_{l_2})
        \;-\;
        \theta_{l_m}
      \bigr)^\top a_i
    \Bigr| \\
    &\;\le\;
    \frac{1}{|I_j|}
    \sum_{i \in \hat{I}_j}
    \|a_i\|_2
    \,
    \|(\hat{\xi}_1\theta_{l_1}
        +\hat{\xi}_2\theta_{l_2})
      \;-\;\theta_{l_m}\|_2.
\end{align*}
Since \(\sup_{i\in I_j}\|a_i\|_2 \le c\), it follows that
\[
\;\le\;
c\,
\bigl\|
  (\hat{\xi}_1-\xi_1)\,\theta_{l_1}
  \;+\;
  (\hat{\xi}_2-\xi_2)\,\theta_{l_2}
\bigr\|_2
\;\to\;0
\]
in probability as $|\hat{I}|\to\infty$.
(The last step uses \(\hat{\xi}\to\xi\) in probability, with 
\(\theta_{l_1}, \theta_{l_2}\) fixed in \(\mathbb{R}^d\).)  
Hence \(\hat{Z}_{jl}\) converges in probability to \(Z_{jl}\), completing the proof.
\end{proof}

\subsection{Proof of \Cref{thm:eps-opt-preserve} via \Cref{prop:estimator-unbiased}}
\label{proof:eps-opt-preserve-via-irt}
\begin{proof}
By \Cref{prop:estimator-unbiased}, \(\hat{Z}^{\mathrm{mp\text{-}IRT}}\) 
becomes arbitrarily close (in probability) to \(Z\) as 
\(\lvert \bar{D}\rvert\to\infty\).  Under standard regularity conditions, this 
implies 
\[
  \bigl|
    Z(\theta;D)
    \;-\;
    \hat{Z}^{\mathrm{mp\text{-}IRT}}(\theta;\bar{D})
  \bigr|
  \;\le\;\epsilon
\]
in expectation, for all sufficiently large $\lvert \bar{D}\rvert$, hence \(\hat{Z}^{\mathrm{mp\text{-}IRT}}\) is \(\epsilon\)-stable in expectation.  
Applying \Cref{thm:expected-eps-stability} completes the argument.
\end{proof}